%% file: main_arxiv.tex
\documentclass[10pt]{article}
\usepackage[margin=1in]{geometry}
\usepackage[round,compress]{natbib}
\usepackage{parskip}

\input{preamble.tex}
\input{notation.tex}

\input{math_commands.tex}

\title{Transformers as Decision Makers: Provable In-Context Reinforcement Learning via Supervised Pretraining}

\date{\today}
\author{
  Licong Lin\thanks{UC Berkeley. Email: \texttt{liconglin@berkeley.edu}}\hspace{.35em}
  \and
  Yu Bai\thanks{Salesforce AI Research. Email: \texttt{yu.bai@salesforce.com}}\hspace{.35em}\footnotemark[4]
  \and
  Song Mei\thanks{UC Berkeley. Email: \texttt{songmei@berkeley.edu}}\hspace{.35em}\thanks{Equal contribution.}
}

\def\shownotes{0}  %set 1 to show author notes
\ifnum\shownotes=1
\newcommand{\authnote}[2]{{\scriptsize $\ll$\textsf{#1 notes: #2}$\gg$}}
\else
\newcommand{\authnote}[2]{}
\fi
\newcommand{\yub}[1]{{\color{red}\authnote{Yu}{#1}}}

\newcommand{\lc}[1]{{\color{blue}\authnote{Licong}{#1}}}

\makeatletter
\def\blfootnote{\gdef\@thefnmark{}\@footnotetext}
\makeatother

\begin{document}

\maketitle

\input{Sections_arxiv/abstract_arxiv.tex}
\input{Sections_arxiv/intro_arxiv.tex}

\input{Sections_arxiv/related-work_arxiv.tex}

\input{Sections_arxiv/prelim_arxiv.tex}

\input{Sections_arxiv/transformer-arch_arxiv.tex}
\input{Sections_arxiv/supervised-pretraining_arxiv.tex}

\input{Sections_arxiv/ICRL_arxiv.tex}
\input{Sections_arxiv/experiments_arxiv.tex}

\input{Sections_arxiv/conclusion_arxiv.tex}

\input{Sections_arxiv/acknowledgement_arxiv}
% \clearpage
\bibliography{references}
\bibliographystyle{plainnat}

\clearpage
\appendix 
\tableofcontents
\clearpage

\input{Sections_arxiv/limitation}

\input{Sections_arxiv/app-experimental-details_arxiv.tex}

\input{Sections_arxiv/app-tech-prelim_arxiv.tex}

\input{Sections_arxiv/app-supervised-pretraining_arxiv.tex}
\input{Sections_arxiv/app-linucb_arxiv.tex}
\input{Sections_arxiv/app-TS_arxiv.tex}

\input{Sections_arxiv/app-mdp_arxiv.tex}

\end{document}

%% file: preamble.tex
\usepackage{parskip}
\usepackage{subcaption}
\usepackage{amsmath}
\usepackage{amsfonts,amscd,amssymb,bm,bbm,mathrsfs}
\usepackage{nicefrac}       % compact symbols for 1/2, etc.
\usepackage{amsthm}
\usepackage{mathtools}

\newcommand\labelAndRemember[2]
  {\expandafter\gdef\csname labeled:#1\endcsname{#2}%
   \label{#1}#2}
\newcommand\recallLabel[1]
   {\csname labeled:#1\endcsname\tag{\ref{#1}}}
\newcommand\labelr[2]
  {\expandafter\gdef\csname labeled:#1\endcsname{#2}%
   \label{#1}#2}
\newcommand\recall[1]
   {\csname labeled:#1\endcsname}

\usepackage[algo2e,linesnumbered,vlined,ruled]{algorithm2e}
\usepackage{algorithm}
\usepackage[noend]{algorithmic}

\usepackage{url}
\usepackage[breaklinks=true]{hyperref}
\usepackage{breakcites}
\hypersetup{
  colorlinks,
  citecolor=blue!70!black,
  linkcolor=blue!70!black,
  urlcolor=blue!70!black
}

\usepackage[capitalise]{cleveref}
\usepackage[titletoc,title]{appendix}
\usepackage{cancel}
\usepackage{enumerate}
\usepackage{enumitem}
\usepackage{comment}
\crefformat{equation}{#2(#1)#3}
\Crefformat{equation}{#2(#1)#3}
\crefformat{figure}{Figure #2#1#3}
\Crefformat{figure}{Figure #2#1#3}

\usepackage{booktabs}       % professional-quality tables
\usepackage{multirow}
\usepackage{multicol}
\usepackage{makecell}
\usepackage{array}
\usepackage{arydshln}
\newcolumntype{H}{>{\setbox0=\hbox\bgroup}c<{\egroup}@{}}
\newcolumntype{Z}{>{\setbox0=\hbox\bgroup}c<{\egroup}@{\hspace*{-\tabcolsep}}}

\usepackage[normalem]{ulem}
\usepackage{exscale}
\usepackage{tikz}
\usepackage{float}
\usepackage{graphicx,graphics}
\graphicspath{ {./fig/} }
\allowdisplaybreaks

\newtheorem{theorem}{Theorem}
\newtheorem*{theorem*}{Theorem}
\newtheorem{lemma}[theorem]{Lemma}

\newtheorem*{remark*}{Remark}
\newtheorem*{lemma*}{Lemma}

\newtheorem{proposition}[theorem]{Proposition}
\newtheorem{definition}[theorem]{Definition}

\newtheorem{assumption}{Assumption}

\newenvironment{proof-sketch}{\noindent{\bf Proof Sketch}
  \hspace*{1em}}{\qed\bigskip\\}
\newenvironment{proof-idea}{\noindent{\bf Proof Idea}
  \hspace*{1em}}{\qed\bigskip\\}
\newenvironment{proof-of}[1][{}]{\noindent{\bf Proof of \cref{#1}}
  \hspace*{1em}}{\qed\bigskip\\}
\newenvironment{proof-of-lemma}[1][{}]{\noindent{\bf Proof of Lemma {#1}}
  \hspace*{1em}}{\qed\bigskip\\}
\newenvironment{proof-of-proposition}[1][{}]{\noindent{\bf
    Proof of Proposition {#1}}
  \hspace*{1em}}{\qed\bigskip\\}
\newenvironment{proof-of-theorem}[1][{}]{\noindent{\bf Proof of Theorem {#1}}
  \hspace*{1em}}{\qed\bigskip\\}
\newenvironment{inner-proof}{\noindent{\bf Proof}\hspace{1em}}{
  $\bigtriangledown$\medskip\\}
\newenvironment{proof-attempt}{\noindent{\bf Proof Attempt}
  \hspace*{1em}}{\qed\bigskip\\}
\renewcommand{\hat}{\widehat}
\renewcommand{\bar}{\overline}
\renewcommand{\epsilon}{\varepsilon}

\usepackage{pifont}
\newcommand{\id}{\bI}

\newcommand{\eps}{\varepsilon}

\newcounter{cnt}
\foreach \num in {1,2,...,26}{%
  \stepcounter{cnt}%
  \expandafter\xdef \csname c\Alph{cnt}\endcsname {\noexpand\mathcal{\Alph{cnt}}}%
  \expandafter\xdef \csname b\Alph{cnt}\endcsname {\noexpand\mathbb{\Alph{cnt}}}%
}

\DeclareMathOperator*{\argmin}{arg\,min}
\DeclareMathOperator*{\argmax}{arg\,max}

\DeclarePairedDelimiterX{\ddiv}[2]{(}{)}{%
  #1\;\delimsize\|\;#2%
}
\newcommand{\KL}{\operatorname{KL}\ddiv}
\newcommand{\norm}[1]{\left\|{#1}\right\|} % A norm with 1 argument
\newcommand{\lone}[1]{\norm{#1}_1} % l1 norm
\newcommand{\ltwo}[1]{\norm{#1}_2} % l2 norm
\newcommand{\ltwop}[2]{\norm{#1}_{2,{#2}}} % l2p norm

\newcommand{\lops}[1]{\|{#1}\|_{\mathrm{op}}}
\newcommand{\linf}[1]{\norm{#1}_\infty} % l-infinity norm
\newcommand{\normbig}[1]{\big\|{#1}\big\|} % A norm with 1 argument and big
\newcommand{\ltwopbig}[2]{\normbig{#1}_{2,{#2}}} % l2p norm

\renewcommand{\cO}{\mathcal{O}}

\newcommand{\<}{\left\langle}
\renewcommand{\>}{\right\rangle}

\renewcommand{\bQ}{\mathbf{Q}}

\newcommand{\relu}{\mathrm{ReLU}}

\newenvironment{talign*}
 {\csname align*\endcsname}
 {\endalign}
 
\newcommand{\lth}{{(\ell)}}

\newcommand{\Attn}{{\rm Attn}}

\newcommand{\TF}{{\rm TF}}

\newcommand{\barsig}{\sigma}

\newcommand{\bzero}{{\mathbf 0}}
\newcommand{\bone}{{\mathbf 1}}

\newcommand{\clip}{\mathsf{clip}}

\renewcommand{\read}{{\sf read}}

\newcommand{\ridge}{{\rm ridge}}

\newcommand{\nrmp}[1]{{\left|\!\left|\!\left|{#1}\right|\!\right|\!\right|}}

\newcommand{\MLP}{\mathrm{MLP}}

\newcommand{\bAtt}{\btheta_{\tt attn}}
\newcommand{\bMAtt}{\btheta_{\tt mattn}}
\newcommand{\bthetamlp}{\btheta_{\tt mlp}}
\newcommand{\bmlp}{\bthetamlp}

\newcommand{\mlp}{{\tt mlp}}
\newcommand{\attn}{{\tt attn}}

\mathchardef\mhyphen="2D

\newcommand{\paren}[1]{{\left( #1 \right)}}
\newcommand{\brac}[1]{{\left[ #1 \right]}}
\newcommand{\set}[1]{{\left\{ #1 \right\}}}
\newcommand{\sets}[1]{{\{ #1 \}}}

\newcommand{\defeq}{\mathrel{\mathop:}=}

\newcommand{\mat}[1]{\ensuremath{\mathbf{#1}}}

\newcommand{\E}{\mathbb{E}}

\renewcommand{\P}{\mathbb{P}}
\newcommand{\Q}{\mathbb{Q}}

\newcommand{\R}{\mathbb{R}}

\newcommand{\B}{\mat{B}}

%% file: notation.tex
\newcommand{\pre}{{\mathrm{pre}}}
\newcommand{\posv}{{\mathbf{pos}}}

\newcommand{\post}{{\mathrm{post}}}
\newcommand{\GD}{{\mathrm{GD}}}
\newcommand{\AGD}{{\mathrm{AGD}}}
\newcommand{\parta}{{a}}
\newcommand{\partb}{{b}}
\newcommand{\partc}{{c}}
\newcommand{\partd}{{d}}

\newcommand{\lran}{{l_1}} %low range
\newcommand{\uran}{{l_2}}  %up range
\newcommand{\appr}{{\mathrm{approx}}}  %up range
\newcommand{\inst}{{\mathsf{M}}} 
 
\newcommand{\trajsp}{{\mathcal{T}}}

\newcommand{\HelDs}{{\mathrm{D}^2_{\mathrm{H}}}} 
\newcommand{\VarD}{{\mathrm{D}_{\mathrm{TV}}}} 
\newcommand{\HelD}{{\mathrm{D}_{\mathrm{H}}}}
\newcommand{\Covnum}{{n_{\mathrm{cov}}}} 
\newcommand{\LinUCB}{{\mathrm{LinUCB}}} 
\newcommand{\UCB}{{\mathrm{UCB}}} 
\newcommand{\unif}{{\mathrm{unif}}}

\newcommand{\sLinUCB}{{\mathrm{sLinUCB}}} 

\newcommand{\temp}{{\tau}} 
\newcommand{\cwid}{{\alpha}}  %width of the confidence set
\newcommand{\TS}{{\mathrm{TS}}}

\newcommand{\Tpsmean}{{{\mathbf\mu}}}
\newcommand{\Tpscov}{{{\mathbf\Sigma}}}
\newcommand{\Tpssam}{{\tilde\bw}}
\newcommand{\Tpspar}{{\lambda}}  % prior variance in TS
\newcommand{\Tpsparn}{{r}} %noise variance in TS

\newcommand{\trunprob}{{\eta_1}}
\newcommand{\hpevent}{{\mathcal E}}
\newcommand{\prodeig}{{\mu}}
\newcommand{\Errmat}{{\mathbf {E}}}
\newcommand{\padecond}{{\tilde \kappa}}

\newcommand{\intvec}{{\mathbf {q}}}
\newcommand{\intmat}{{\mathbf {M}}}

\newcommand{\Trunreg}{{D}}
\newcommand{\Trunregpa}{{\eta_2}}  %fot the assumption

\newcommand{\Trunregp}{{\eta}}
\newcommand{\bigvec}{{\mathbf{v}}}

\newcommand{\tcO}{{\tilde{\mathcal O}}}

\newcommand{\proj}{{\rm proj}}

\newcommand{\widebar}[1]{\overline{#1}}

\newcommand{\transmodel}{{\mathbb T}}
\newcommand{\rewmodel}{{\mathbb R}}
\newcommand{\state}{{s}}
\newcommand{\action}{{a}}
\newcommand{\eaction}{{\widebar{a}}}%expert action
\newcommand{\reward}{{r}}
\newcommand{\estreward}{{{r}}}

\newcommand{\totlen}{{T}} %%total length

\newcommand{\cat}{{\tt cat}}
\newcommand{\extractmap}{{\tt A}}
\newcommand{\embedmap}{{\tt h}}
\newcommand{\softmax}{{\rm softmax}}
\newcommand{\ith}{{i}}
\newcommand{\sAlg}{{\mathsf{Alg}}}
\newcommand{\osAlg}{\overline{\mathsf{Alg}}}

\newcommand{\dset}{{D}}
\newcommand{\adset}{{\widebar{D}}}  %all dataset
\newcommand{\Numobs}{{n}}   %number of samples, not completed

\newcommand{\Par}{{\theta}}
\newcommand{\Parspace}{{\Theta}}
\newcommand{\appeps}{\eps_{{\mathsf{approx}}}}
\newcommand{\EstPar}{{\widehat{\theta}}}
\newcommand{\esttfpar}{{\widehat{\btheta}}}
\newcommand{\TruePar}{{\theta^*}}
\newcommand{\plc}{{\pi}}
\newcommand{\plcset}{{\Pi}}
\newcommand{\optplc}{{\pi^*}}

\newcommand{\prior}{{\Lambda}}
\newcommand{\aset}{{\sA}}
\newcommand{\TrueLBPar}{{\bw^*}}

\newcommand{\Noise}{{\eps}}
\newcommand{\Numepi}{{K}}
\newcommand{\horizon}{{H}}
\newcommand{\statesp}{{\mathcal{S}}}
\newcommand{\actionsp}{{\mathcal{A}}}
\newcommand{\rewardsp}{{\mathcal{R}}}
\renewcommand{\horizon}{{H}}
\newcommand{\Noisedist}{{\mathcal{E}}}

\newcommand{\transit}{{P}}
\newcommand{\bonus}{{b}}
\newcommand{\estbonus}{{b}}

\newcommand{\esttransit}{\widehat{P}}
\newcommand{\tresttransit}{\widehat{P}}

\newcommand{\rewardfun}{{R}}
\newcommand{\init}{{\mu_1}}
\newcommand{\valuefun}{{V}}
\newcommand{\expert}{{\mathrm{expert}}}
\newcommand{\shortexp}{{E}}
\newcommand{\geneps}{{\epsilon}_{\sf real}}

\newcommand{\tfpar}{{\btheta}}
\newcommand{\tfparspace}{{\Theta}}
\newcommand{\layer}{{L}}
\newcommand{\hidden}{{D'}}

\newcommand{\head}{{M}}
\newcommand{\normb}{{B}}
\newcommand{\clipval}{{\mathsf{R}}}

\newcommand{\tfmat}{{\bH}}  %matrices sent into TF
\newcommand{\embd}{{D}}  %embedding dim
\newcommand{\actnum}{{K}}  %number of actions
\newcommand{\totreward}{{\mathfrak{R}}}  %total reward

\newcommand{\distratio}{{\mathcal{R}}}
\newcommand{\tfthres}{{\mathsf{B}}}   %the threshold value in TF constructions.

\newcommand{\roll}{{\mathrm{roll}}}

\newcommand{\Numst}{{S}}
\newcommand{\Numact}{{A}}
\newcommand{\Numvi}{{N}}
\newcommand{\Qfun}{{Q}}
\newcommand{\estQfun}{{\widehat{\Qfun}}}
\newcommand{\trestQfun}{{\widehat{\Qfun}}} %true estimated fun

\newcommand{\Vfun}{{\valuefun}}
\newcommand{\estVfun}{{\widehat{\Vfun}}}
\newcommand{\trestVfun}{{\widehat{\Vfun}}}

\newcommand{\oddeven}{{v}}

\newcommand{\sUCBVI}{{\mathrm{sUCBVI}}}

\newcommand{\ssm}{{\mathsf{m}}}

\newcommand{\s}{{\mathsf{sm}}}

\newcommand{\epstemp}{{\eps_{\mathrm{sfmax}}}}

\newcommand{\MD}{{\mathrm{MD}}}

\newcommand{\conO}{{\mathrm{O}}}

\newcommand{\sst}{{\mathsf{t}}}
\newcommand{\emp}{{\mathsf{emp}}}
\newcommand{\ssl}{{{\ell}}}
\newcommand{\neuron}{{\mathsf{M_0}}}
\newcommand{\weightn}{{{\mathsf{C_0}}}}

%% file: math_commands.tex
\def\sA{{\mathbb{A}}}

\def\sC{{\mathbb{C}}}
\def\sD{{\mathbb{D}}}
\def\sP{{\mathbb{P}}}
\def\sQ{{\mathbb{Q}}}

\def\sS{{\mathbb{S}}}

\def\bA{{\mathbf A}}

\def\bH{{\mathbf H}}
\def\bI{{\mathbf I}}
\def\bK{{\mathbf K}}

\def\bQ{{\mathbf Q}}

\def\bV{{\mathbf V}}
\def\bW{{\mathbf W}}

\def\btheta{{\boldsymbol \theta}}

\def\ba{{\mathbf a}}

\def\be{{\mathbf e}}

\def\bh{{\mathbf h}}

\def\bp{{\mathbf p}}

\def\bu{{\mathbf u}}
\def\bv{{\mathbf v}}
\def\bw{{\mathbf w}}
\def\bx{{\mathbf x}}

\def\bz{{\mathbf z}}

\def\si{{\mathsf{i}}}
\def\sj{{\mathsf{j}}}
\def\sk{{\mathsf{k}}}
\def\sh{{\mathsf{h}}}

%% file: Sections_arxiv/abstract_arxiv.tex
\begin{abstract}
% In this project, we study using transformers to implement online bandit algorithms, with potential generalization to offline bandit and RL settings. 

Large transformer models pretrained on offline reinforcement learning datasets have demonstrated remarkable in-context reinforcement learning (ICRL) capabilities, where they can make good decisions when prompted with interaction trajectories from unseen environments. However, when and how transformers can be trained to perform ICRL have not been theoretically well-understood. In particular, it is unclear which reinforcement-learning algorithms transformers can perform in context, and how distribution mismatch in offline training data affects the learned algorithms. 
This paper provides a theoretical framework that analyzes supervised pretraining for ICRL. This includes two recently proposed training methods --- algorithm distillation and decision-pretrained transformers. First, assuming model realizability, we prove the supervised-pretrained transformer will imitate the conditional expectation of the expert algorithm given the observed trajectory. The generalization error will scale with model capacity and a distribution divergence factor between the expert and offline algorithms. Second, we show transformers with ReLU attention can efficiently approximate near-optimal online reinforcement learning algorithms like LinUCB and Thompson sampling for stochastic linear bandits, and UCB-VI for tabular Markov decision processes. This provides the first quantitative analysis of the ICRL capabilities of transformers pretrained from offline trajectories.

% This paper theoretically investigates the supervised-pretraining approach that trains a transformer to predict an expert action given a query state and interaction trajectories. 

% In this paper, we provide a theoretical study of using transformers for in-context reinforcement learning (ICRL) with supervised pretraining. We show transformers can be used to implement online bandit algorithms like LinUCB, Thompson sampling, and UCB-VI. Our unified framework demonstrates supervised pretraining enables finding transformers that achieve competitive regret guarantees compared to expert algorithms.

% Specifically, we formally characterize regret bounds and transformer constructions when pretraining on different sources of expert actions, including algorithm distillation, optimal actions, and approximate optimal actions. Our analysis subsumes prior empirical works on algorithm distillation and decision transformer pretraining under a common theoretical lens.

% Overall, our results deliver fundamental insights on when and how supervised pretraining allows transformers to learn ICRL algorithms. We provide the first statistical and computational guarantees for using transformers as parametric decision makers that can efficiently adapt online in reinforcement learning contexts.
\end{abstract}

%% file: Sections_arxiv/intro_arxiv.tex
\section{Introduction}

% \begin{itemize}[leftmargin=1.5em]
%     \item We theoretically show that transformers can do reinforcement learning in-context.~\yub{both statistical results and TF constructions are new. Previous paper (DPT) only analyzes minimizers under idealized assumptions.}
%     \item Namely, we prove that there exist TFs that can approximately implement LinUCB, Thompson sampling under the linear bandit setting (and maybe UCBVI under the MDP setting).
%     \item We use supervised pretraining (imitation learning) to find the TF. This idea appears in previous empirical works~\cite{laskin2022context,lee2023supervised}. \yub{Our statistical framework provides a unified framework that subsumes both AD and DPT.}
%     \yub{Can we say something about drawback of reward objective (instead of imitation)?}
%     \item We provide regret guarantees for the algorithm generated by the TF obtained vis supervised pretraining.
% \end{itemize}

The transformer architecture~\citep{vaswani2017attention} for sequence modeling has become a key weapon for modern artificial intelligence, achieving success in language~\citep{devlin2018bert,brown2020language,openai2023gpt} and vision~\citep{dosovitskiy2020image}. Motivated by these advances, the research community has actively explored how to best harness transformers for reinforcement learning (RL)~\citep{chen2021decision, janner2021offline, lee2022multi, reed2022generalist, laskin2022context, lee2023supervised,yang2023foundation}. While promising empirical results have been demonstrated, the theoretical understanding of transformers for RL remains limited. %\sm{deleted last sentence} Even fundamental questions such as \emph{approximation capabilities} (what algorithms in RL can be approximated by transformers) and \emph{sample efficiency} (how many samples are needed for training) remain largely open. 

%% Yu's original paragraph
% The transformer architecture~\citep{vaswani2017attention} for sequence modeling has become a key weapon for modern artificial intelligence. Motivated by their success in modalities such as language~\citep{devlin2018bert,brown2020language,openai2023gpt} and vision~\citep{dosovitskiy2020image}, how to best unleash their power for Reinforcement Learning (RL) has become an active research area, where various approaches have been proposed~\citep{chen2021decision, janner2021offline, lee2022multi, reed2022generalist, laskin2022context, lee2023supervised,yang2023foundation}. Despite promising empirical results, theoretical understandings on how to use transformers to do RL are still lacking---Even fundamental questions such as \emph{approximation} (what sequence-to-sequence functions in RL can be approximated by transformers) and \emph{statistical efficiency} (how many samples are needed to train such a transformer) remain largely open.

This paper provides theoretical insights into in-context reinforcement learning (ICRL)---an emerging approach that utilizes sequence-to-sequence models like transformers to perform reinforcement learning in newly encountered environments. In ICRL, the model takes as input the current state and past interaction history with the environment (the \emph{context}), and outputs an action. The key hypothesis in ICRL is that pretrained transformers can act as \emph{RL algorithms}, progressively improving their policy based on past observations. Approaches such as Algorithm Distillation \citep{laskin2022context} and Decision-Pretrained Transformers \citep{lee2023supervised} have demonstrated early successes, finding that \emph{supervised pretraining} can produce good ICRL performance. However, many concrete theoretical questions remain open about the ICRL capabilities of transformers, including but not limited to (1) what RL algorithms can transformers implement in-context; (2) what performance guarantees (e.g. regret bounds) can such transformers achieve when used iteratively as an online RL algorithm; and (3) when can supervised pretraining find such a good transformer. Specifically, this paper investigates the following open question:
\begin{center}
\emph{How can supervised pretraining on Transformers learn in-context reinforcement learning?}
\end{center}

In this paper, we initiate a theoretical study of the ICRL capability of transformers under supervised pretraining to address the open questions outlined above. We show that (1) Transformers can implement prevalent RL algorithms, including LinUCB and Thompson sampling for stochastic linear bandits, and UCB-VI for tabular Markov decision processes; (2) The algorithms learned by transformers achieve near-optimal regret bounds in their respective settings; (3) Supervised pretraining find such algorithms as long as the sample size scales with the covering number of transformer class and distribution ratio between expert and offline algorithms. %\lc{offline or context algorithm?}

\paragraph{Summary of contributions and paper outline}
\begin{itemize}[leftmargin=1.5em]
\item We propose a general framework for supervised pretraining approaches to meta-reinforcement learning (Section~\ref{sec:framework}). This framework encompasses existing methods like Algorithm Distillation \citep{laskin2022context}, where the expert and context algorithms are identical, as well as Decision-Pretrained Transformers \citep{lee2023supervised}, where the expert generates optimal actions for the MDP. It also includes approximate DPT variants where the expert estimates optimal actions from full interaction trajectories. 
\item We prove that the supervised-pretrained transformer will imitate the conditional expectation of the expert algorithm given the observed trajectory (Section~\ref{sec:supervised-pretraining}). The generalization error scales with both model capacity and a distribution ratio measuring divergence between the expert algorithm and the algorithm that generated offline trajectories. 

\item We demonstrate that transformers can effectively approximate several near-optimal reinforcement learning algorithms by taking observed trajectories as context inputs (Section~\ref{sec:ICRL}). Specifically, we show transformers can approximate LinUCB (Section~\ref{sec:LinUCB-statement}) and Thompson sampling algorithms (Section~\ref{sec:TS-statement}) for stochastic linear bandit problems, and UCB-VI (Section~\ref{sec:Tabular-MDP-statement}) for tabular Markov decision processes. Combined with the generalization error bound from supervised pretraining and regret bounds of these RL algorithms, this provides regret bounds for supervised-pretrained transformers. 
\item Preliminary experiments validate that transformers can perform ICRL in our setup (Section~\ref{sec:experiments}).
\item Technically, we prove efficient approximation of LinUCB by showing transformers can implement accelerated gradient descent for solving ridge regression (\cref{sec:pf_thm:approx_smooth_linucb}), enabling fewer attention layers than the vanilla gradient descent approach in \cite{bai2023transformers}. To enable efficient Thompson sampling implementation, we prove transformers can compute matrix square roots through the Pade decomposition (\cref{sec:pf_thm:approx_thompson_linear-formal}). These approximation results are interesting in their own right. 
\end{itemize}
% \yub{We have two main contributions, transformers and statistical framework. In main sections, we do statistical framework first. In contribution list here, shall we reverse the order and state TF first? (TF can do LinUCB)}
% \sm{Experiments? }~\yub{added}

% \paragraph{Related work} Our work is intimately related to the lines of work on meta-reinforcement learning, in-context learning, transformers for decision-making, and the approximation theory of transformers. Due to limited space, we discuss these related works in~\cref {sec:related-work}.

% \begin{itemize}
%     \item Framework (supervised pretraining; context algorithm, expert algorithm), instantiation to bandit and MDP. 
    
%     \item Algorithm (MLE) and statistical results. Variations (source of the expert algorithm). Talk about application to both bandit and MDP.
%     \yub{Maybe connections to imitation learning.}
    
%     \item Transformer constructions (bandit, MDP).

%     \yub{We have two main contributions, transformers and statistical framework. In main sections, we do statistical framework first. In contribution list here, shall we reverse the order and state TF first? (TF can do LinUCB)}
%     \item Experiments.

%     \sm{Approximation: AGD, matrix square root. }
% \end{itemize}

%% file: Sections_arxiv/related-work_arxiv.tex
\subsection{Related work}\label{sec:related-work}

\paragraph{Meta-learning and meta-reinforcement learning} In-context reinforcement learning can be cast into the framework of meta-learning and meta-reinforcement learning \citep{schmidhuber1987evolutionary, schmidhuber1992learning, bengio1990learning, naik1992meta, ishii2002control, schaul2010metalearning, thrun2012learning}. More recently, a line of work focuses on meta-learn certain shared structures such as the dynamics of the shared tasks \citep{fu2016one, nagabandi2018learning}, a task context identifier \citep{rakelly2019efficient, humplik2019meta, zintgraf2019varibad}, exploration strategies \citep{gupta2018meta}, or the initialization of the network policy \citep{finn2017model, hochreiter2001learning, nichol2018first, rothfuss2018promp}. Theories for this last approach of model-agnostic meta-learning have been explored by \cite{wang2020global}. 

Our work focuses on a more agnostic approach to learning the learning algorithm itself \citep{wang2016learning, duan2016rl, dorfman2021offline, mitchell2021offline, li2020focal, pong2022offline, laskin2022context, lee2023supervised}. Among these works, \cite{wang2016learning, duan2016rl} focus on the online meta-RL setting with the training objective to be the total reward. Furthermore, \cite{dorfman2021offline, mitchell2021offline, li2020focal, pong2022offline} focus on offline meta-RL, but their training objectives differ from the cross entropy loss used here, requiring explicit handling of distribution shift. The supervised pretraining approach we consider is most similar to the algorithm distillation methods of \cite{laskin2022context} and the decision-pretrained transformers of \cite{lee2023supervised}. We provide quantitative sample complexity guarantees and transformer constructions absent from previous work.

\paragraph{In-context learning}

The in-context learning (ICL) capability of pretrained transformers has gained significant attention since being demonstrated on GPT-3 \cite{brown2020language}. Recent work investigates why and how pretrained transformers perform ICL \citep{garg2022can, li2023transformers, von2023transformers, akyurek2022learning, xie2021explanation, bai2023transformers, 
zhang2023trained, ahn2023transformers, raventos2023pretraining}. In particular, \cite{xie2021explanation} propose a Bayesian framework explaining how ICL works. \cite{garg2022can} show transformers can be trained from scratch to perform ICL of simple function classes.  \cite{von2023transformers, akyurek2022learning, bai2023transformers} demonstrate transformers can implement in-context learning algorithms via in-context gradient descent, with \cite{bai2023transformers} showing transformers can perform in-context algorithm selection. \cite{zhang2023trained} studied training dynamics of a single attention layer for in-context learning of linear functions. Our work focuses on the related but distinct capability of in-context decision-making for pretrained transformers. 

\paragraph{Transformers for decision making} Besides the ICRL approach, recent work has proposed goal-conditioned supervised learning (GCSL) for using transformers to make decisions \citep{chen2021decision, janner2021offline, lee2022multi, reed2022generalist, brohan2022rt, shafiullah2022behavior, yang2023foundation}. In particular, Decision Transformer (DT) \citep{chen2021decision, janner2021offline} uses transformers to autoregressively model action sequences from offline data, conditioned on the achieved return. During inference, one queries the model with a desired high return. Limitations and modifications of GCSL have been studied in \cite{yang2022dichotomy, paster2022you, vstrupl2022upside, brandfonbrener2022does}. A key distinction between GCSL and ICRL is that GCSL treats the transformer as a policy, whereas ICRL treats it as an algorithm for improving the policy based on observed trajectories.

\paragraph{Expressivity of transformers} The transformer architecture, introduced by \cite{vaswani2017attention}, has revolutionized natural language processing and is used in most recently developed large language models like BERT and GPT \citep{devlin2018bert, brown2020language}. The expressivity of transformers has been extensively studied~\citep{yun2019transformers, perez2019turing, hron2020infinite,yao2021self, bhattamishra2020computational, zhang2022unveiling, liu2022transformers, wei2022statistically, fu2023can, bai2023transformers, akyurek2022learning, von2023transformers}. Deep neural networks such as ResNets and transformers have been shown to efficiently approximate various algorithms, including automata \citep{liu2022transformers}, Turing machines \citep{wei2022statistically}, variational inference \citep{mei2023deep}, and gradient descent \citep{bai2023transformers, akyurek2022learning, von2023transformers}. Our work provides efficient transformer constructions that implement accelerated gradient descent and matrix square root algorithms, complementing existing expressivity results.

\paragraph{Statistical theories of imitation learning} Our generalization error analysis adapts classical analysis of maximum-likelihood estimator \citep{geer2000empirical}. The error compounding analysis for imitation learning appeared in early works \citep{ross2011reduction, ross2010efficient}. More recent theoretical analyses of imitation learning also appear in \cite{rajaraman2020toward, rajaraman2021provably, rashidinejad2021bridging}.

% \sm{Add} \lc{add "throughout the proof, we assume the small probability $\delta_0,\delta<1/2$" to somewhere in the paper.}

% \begin{itemize}
% \item Meta-learning: \cite{schmidhuber1987evolutionary, schmidhuber1992learning, bengio1990learning, naik1992meta, ishii2002control, schaul2010metalearning, thrun2012learning}. Reward maximization Meta-RL: \cite{wang2016learning}, \cite{duan2016rl}. Hyperparameter Meta-RL: \cite{hochreiter2001learning}, \cite{finn2017model}, \cite{nichol2018first}. Offline meta-RL \cite{dorfman2021offline}, \cite{mitchell2021offline}, \cite{li2020focal}, \cite{pong2022offline}. AD: \cite{laskin2022context}, DPT: \cite{lee2023supervised}. 

% \item ICL: \cite{brown2020language}, \cite{garg2022can}, \cite{von2023transformers}, \cite{li2023transformers}, \cite{akyurek2022learning}, \cite{xie2021explanation}. \cite{bai2023transformers}, \cite{zhang2023trained}, \cite{ahn2023transformers}, \cite{xie2021explanation}. 

% \item Transformers for DM: \cite{chen2021decision}, \cite{janner2021offline}, \cite{lee2022multi}, \cite{reed2022generalist}, \cite{brohan2022rt}, \cite{shafiullah2022behavior}. \cite{yang2023foundation}, \cite{yang2022dichotomy}, \cite{brandfonbrener2022does}. 
% \item Imitation learning: \cite{rajaraman2020toward, rashidinejad2021bridging}. Meta-RL theory \cite{wang2020global}. \sm{More imitation learning literature? }
% \end{itemize}

%% file: Sections_arxiv/prelim_arxiv.tex
\section{Framework for In-Context Reinforcement Learning}\label{sec:framework}

% \yub{should this section be actually our contribution instead of prelim? (since we subsume AD and DPT) If so, we can remove ``prelim'' in title, and attract more attention to this section.}

% \paragraph{Decision making environment} 

Let $\cM$ be the space of decision-making environments, where each environment $\inst \in \cM$ shares the same number of rounds $\totlen$ and state-action-reward spaces $\{ \statesp_t,  \actionsp_t, \rewardsp_t \}_{t \in [\totlen]}$. Each $\inst = \{\transmodel_\inst^{t-1}, \rewmodel_\inst^t \}_{t \in [\totlen]}$ has its own transition model $\transmodel_\inst^t: \statesp_{t} \times \actionsp_{t} \to \Delta(\statesp_{t+1})$ (with $\statesp_0$, $\actionsp_0 = \{ \emptyset \}$ so $\transmodel_\inst^0(\cdot) \in \Delta(\statesp_1)$ gives the initial state distribution) and reward functions $\rewmodel_\inst^{t}: \statesp_{t} \times \actionsp_{t} \to \Delta(\rewardsp_t)$. We equip $\cM$ with a distribution $\prior \in \Delta(\cM)$, the environment prior. While this setting is general, we later give concrete examples taking $\cM$ as $\totlen$ rounds of bandits or $K$ episodes of $H$-step MDPs with $\totlen = K H$. 

\paragraph{Distributions of offline trajectories} We denote a partial interaction trajectory, consisting of observed state-action-reward tuples, by $\dset_t=\{(\state_1,\action_1,\reward_1),\ldots,(\state_t,\action_t,\reward_t)\} \in \trajsp_t = \prod_{s \le t} (\statesp_s \times \actionsp_s \times \rewardsp_s)$ and write $\dset = \dset_{\totlen}$ for short. An algorithm $\sAlg$ maps a partial trajectory $\dset_{t-1} \in \trajsp_{t-1}$ and state $\state_t \in \statesp_t$ to a distribution over the actions $\sAlg(\cdot | \dset_{t-1}, \state_t) \in \Delta(\actionsp_t)$. Given an environment $\inst$ and algorithm $\sAlg$, the distribution over a full trajectory $\dset_\totlen$ is fully specified: 
\begin{align*}
\textstyle \P_{\inst}^{\sAlg}(\dset_\totlen) =
\prod_{t=1}^{\totlen}\transmodel_{\inst}^{t-1}(\state_{t}|\state_{t-1},\action_{t-1}) \sAlg(\action_t|\dset_{t-1},\state_t)\rewmodel_{\inst}^t(\reward_t|\state_t,\action_t).
\end{align*}
% \lc{introduce $\transmodel^0_\inst$ ($\state_1\sim\init$).} \sm{I added a parathesis in the paragraph above}
In supervised pretraining, we use a \textit{context algorithm} $\sAlg_0$ (which we also refer to as the offline algorithm) to collect the offline trajectories $\dset_\totlen$. For each trajectory $\dset_\totlen$, we also assume access to expert actions $\eaction = ( \eaction_t \in \actionsp_t )_{t \in \totlen} \sim \sAlg_{\shortexp}(\cdot | \dset_\totlen, \inst)$, sampled from an expert algorithm $\sAlg_{\shortexp}: \trajsp_\totlen \times \inst \to \prod_{t \in [\totlen]} \Delta(\actionsp_t)$. This expert could omnisciently observe the full trajectory $\dset_\totlen$ and environment $\inst$ to recommend actions. Let $\adset_\totlen = \dset_\totlen \cup \{ \eaction \}$ be the augmented trajectory. Then we have
\begin{align*}
\textstyle \P^{\sAlg_0,\sAlg_{\shortexp}}_{\inst}(\adset_\totlen)=\P^{\sAlg_0}_{\inst}(\dset_\totlen)\prod_{t=1}^\totlen \sAlg_{\shortexp}^t (\eaction_t|\dset_{\totlen},\inst).
\end{align*}
We denote $\P^{\sAlg_0,\sAlg_\shortexp}_{\prior}$ as the joint distribution of $(\inst,\adset_\totlen)$ where $\inst \sim \prior$ and $\adset_\totlen \sim \P^{\sAlg_0,\sAlg_\shortexp}_{\inst}$, and $\P^{\sAlg_0}_{\prior}$ as the joint distribution of $(\inst,\dset_\totlen)$ where $\inst \sim \prior$ and $\dset_\totlen \sim \P^{\sAlg_0}_{\inst}$. 

% When the actions $\eaction = \{ \eaction_t \in \actionsp_t \}_{t \in \totlen} \sim \prod_{t = 1}^T \sAlg_\Par^t( \cdot | \trajsp_{t-1}, \state_t)$ are sampled from algorithms of transformers $\sAlg_\Par$, we write $\P^{\sAlg_0,\sAlg_\Par}_{\inst}(\adset_\totlen)=\P^{\sAlg_0}_{\inst}(\dset_\totlen)\prod_{t=1}^\totlen \sAlg_\Par^t (\eaction_t|\trajsp_{t-1}, \state_t )$. We further denote $\P_{\prior}^{\sAlg_0,\sAlg_\Par}$ to be the joint distribution of $(\inst,\adset_\totlen)$ when $\inst \sim \prior$ and $\adset_\totlen \sim \P^{\sAlg_0,\sAlg_\Par}_{\inst}$, and define $\P^{\sAlg_0}_{\prior}$ and $\P^{\sAlg_0,\sAlg_\shortexp}_{\prior}$ similarly. 

\paragraph{Three special cases of expert algorithms} We consider three special cases of the expert algorithm $\sAlg_{\shortexp}$, corresponding to three supervised pretraining setups:
\begin{itemize}[leftmargin=1.5em]
\item[(a)] {\it Algorithm distillation \citep{laskin2022context}. } The algorithm depends only on the partial trajectory $\dset_{t-1}$ and current state $\state_t$: $\sAlg_{\shortexp}^t(\cdot|\dset_{\totlen},\inst) = \sAlg_{\shortexp}^t(\cdot|\dset_{t-1},\state_t)$. For example, $\sAlg_{\shortexp}$ could be a bandit algorithm like the Uniform Confidence Bound (UCB). 
\item[(b)] {\it Decision pretrained transformer (DPT) \citep{lee2023supervised}. } The algorithm depends on the environment $\inst$ and the current state $s_t$: $\sAlg_{\shortexp}^t(\cdot|\dset_\totlen, \inst) = \sAlg_{\shortexp}^t(\cdot|s_t, \inst)$. For example,  $\sAlg_{\shortexp}$ could output the optimal action $\action^*_t$ in state $\state_t$ for environment $\inst$. 
\item[(c)]{\it Approximate DPT. } The algorithm depends on the full trajectory $\dset_{\totlen}$ but not the environment $\inst$: $\sAlg_{\shortexp}^t(\cdot|\dset_\totlen, \inst) =\sAlg_{\shortexp}^t(\cdot|\dset_\totlen)$. For example, $\sAlg_{\shortexp}$ could estimate the optimal action $\widehat \action^*_t$ from the entire trajectory $\dset_\totlen$. 
\end{itemize}  
For any expert algorithm $\sAlg_{\shortexp}$, we define its reduced algorithm where the $t$-th step is $$\osAlg_{\shortexp}(\cdot|\dset_{t-1},\state_t) := \E_\prior^{\sAlg_0}[\sAlg_{\shortexp}^t(\cdot|\dset_\totlen,\inst)|\dset_{t-1},\state_t].$$ The expectation on the right is over $\P_{\prior}^{\sAlg_0} ( \dset_\totlen, \inst |\dset_{t-1},\state_t) =\prior(\inst) \cdot \P_\inst^{\sAlg_0}(\dset_\totlen) / \P_\inst^{\sAlg_0}(\dset_{t-1},\state_t).$ Note that the reduced expert algorithm $\osAlg_{\shortexp}$ generally depends on the context algorithm $\sAlg_0$. However, for cases (a) and (b), $\osAlg_{\shortexp}$ is independent of the context algorithm $\sAlg_0$. Furthermore, in case (a), we have $\osAlg_{\shortexp}^t = \sAlg_{\shortexp}^t$.

%% file: Sections_arxiv/transformer-arch_arxiv.tex
\paragraph{Transformer architecture} We consider a sequence of $N$ input vectors $\set{\bh_i}_{i=1}^N\subset \R^D$, compactly written as an input matrix $\bH=[\bh_1,\dots,\bh_N]\in \R^{D\times N}$, where each $\bh_i$ is a column of $\bH$ (also a \emph{token}). Throughout this paper, we define $\sigma(t)\defeq \relu(t)=\max\sets{t,0}$ as the standard relu activation function. %\lc{add one sentence to explain why using relu instead of softmax is fine?}

% Decoder TFs are the same as encoder TFs, except that the attention layers are replaced by masked attention layers with a specific decoder-based (causal) attention mask.

\begin{definition}[Masked attention layer]
\label{def:masked-attention}
A masked attention layer with $M$ heads is denoted as $\Attn_{\btheta}(\cdot)$ with parameters $\btheta=\sets{ (\bV_m,\bQ_m,\bK_m)}_{m\in[M]}\subset \R^{D\times D}$. On any input sequence $\bH\in\R^{D\times N}$, we have $\bar{\bH} = \Attn_{\btheta}(\bH) = [\bar{\bh}_1, \ldots, \bar{\bh}_N] \in \R^{D \times N}$, where
% \begin{talign}
% \label{eqn:masked-attention}
%     \bar{\bH} = \Attn_{\btheta}(\bH)\defeq \bH + \sum_{m=1}^M (\bV_m \bH) \times \Big( (\MSK_{1:N',1:N'}) \circ \sursf\paren{ (\bQ_m\bH)^\top (\bK_m\bH) } \Big) \in \R^{D\times N'},
% \end{talign}
% where $\circ$ denotes the entry-wise (Hadamard) product of two matrices, and $\MSK \in \R^{N \times N}$ is the mask matrix given by 
% \[
% \MSK = \begin{bmatrix}
% 1 & 1/2 & 1/3 & \cdots & 1/N \\
% 0 & 1/2 & 1/3 & \cdots & 1/N\\
% 0 & 0 & 1/3 & \cdots & 1/N\\
% \cdots & \cdots & \cdots & \cdots & \cdots \\
% 0 & 0 & 0 & \cdots & 1/N
% \end{bmatrix}. 
% \]
% In vector form, we have
\begin{align*}
\textstyle    \bar{\bh}_i = \brac{\Attn_{\btheta}(\bH)}_i = \bh_i + \sum_{m=1}^M \frac{1}{i}\sum_{j=1}^i \barsig\paren{ \<\bQ_m\bh_i, \bK_m\bh_j\> }\cdot \bV_m\bh_j \in \R^D.
\end{align*}
\end{definition}

We remark that the use of ReLU attention layers is for technical reasons. In practice, both ReLU attention and softmax attention layers should perform well. Indeed, several studies have shown that ReLU transformers achieve comparable performance to softmax transformers  across a variety of tasks \citep{wortsman2023replacing, shen2023study, bai2023transformers}.

\begin{definition}[MLP layer]
\label{def:mlp}
An MLP layer with hidden dimension $D'$ is denoted as $\MLP_{\btheta}(\cdot)$ with parameters $\btheta=(\bW_1,\bW_2)\in\R^{D'\times D}\times\R^{D\times D'}$. On any input sequence $\bH\in\R^{D\times N}$, we have $\bar{\bH} = \MLP_{\btheta}(\bH) = [\bar{\bh}_1, \ldots, \bar{\bh}_N] \in \R^{D \times N}$, where
% $\MLP_{\btheta}(\bH)$ gives output
% \begin{talign*}
%     \bar{\bH} = \MLP_{\btheta}(\bH) \defeq \bH + \bW_2\barsig(\bW_1\bH),
% \end{talign*}
% where $\barsig: \R \to \R$ is the ReLU function. In vector form, we have 
\[
\bar{\bh}_i=\bh_i+\bW_2 \cdot \sigma(\bW_1\bh_i) \in \R^D.
\]
\end{definition}
We next define $L$-layer decoder-based transformers. Each layer consists of a masked attention layer (see Definition \ref{def:masked-attention}) followed by an MLP layer (see Definition \ref{def:mlp}) and a clip operation.

% \begin{definition}[Clipped Transformer]
% \label{def:decoder-tf-clip}
% An $L$-layer decoder-based $\clipval$-clipped transformer, denoted as $\TF_\btheta^{\clipval}(\cdot)$, is a composition of $L$ self-attention layers, each followed by an MLP layer and a clip operation: $\bH^{(L)}=\TF_{\btheta}(\bH^{(0)})$, where $\bH^{(0)} = \clip_{\clipval}(\bH) \in\R^{D\times N}$ is the input sequence, and for $\ell\in\set{1,\dots,L}$, 
% \begin{talign*}
% \bH^{(\ell)} = \clip_{\clipval}\Big( \MLP_{\bthetamlp^{(\ell)}}\paren{ \Attn_{\bMAtt^{(\ell)}}\paren{\bH^{(\ell-1)}} } \Big),~~~~~ \clip_{\clipval}(\bh) = [\proj_{\| \bh \|_2 \le \clipval}(\bh_i)]_i. 
% \end{talign*}
% Above, the parameter $\btheta=(\bMAtt^{(1:L)},\bthetamlp^{(1:L)})$ consists of  $\bMAtt^{(\ell)}=\sets{ (\bV^{(\ell)}_m,\bQ^{(\ell)}_m,\bK^{(\ell)}_m)}_{m\in[M]} \subset \R^{D\times D}$ and  $\bthetamlp^{(\ell)}=(\bW^{(\ell)}_1,\bW^{(\ell)}_2)\in\R^{D' \times D}\times \R^{D\times D'}$.

\begin{definition}[Decoder-based Transformer]
\label{def:decoder-tf}
An $L$-layer decoder-based transformer, denoted as $\TF_\btheta^{\clipval}(\cdot)$, is a composition of $L$ masked attention layers, each followed by an MLP layer and a clip operation: $\TF_{\btheta}^{\clipval}(\bH) = \bH^{(L)} \in \R^{D \times N}$, where $\bH^{(L)}$ is defined iteratively by taking $\bH^{(0)} = \clip_{\clipval}(\bH) \in\R^{D\times N}$, and for $\ell\in [L]$, 
\begin{talign*}
\bH^{(\ell)} =\clip_{\clipval}\Big( \MLP_{\bthetamlp^{(\ell)}}\paren{ \Attn_{\bMAtt^{(\ell)}}\paren{\bH^{(\ell-1)}} } \Big) \in \R^{D \times N},~~~~~ \clip_{\clipval}(\bH) = [\proj_{\| \bh \|_2 \le \clipval}(\bh_i)]_i. 
\end{talign*}
Above, the parameter $\btheta=(\bMAtt^{(1:L)},\bthetamlp^{(1:L)})$ consists of  $\bMAtt^{(\ell)}=\sets{ (\bV^{(\ell)}_m,\bQ^{(\ell)}_m,\bK^{(\ell)}_m)}_{m\in[M]} \subset \R^{D\times D}$ and  $\bthetamlp^{(\ell)}=(\bW^{(\ell)}_1,\bW^{(\ell)}_2)\in\R^{D' \times D}\times \R^{D\times D'}$. We define the parameter class of transformers as $\Theta_{D, L, M, \hidden, B} \defeq \{ \btheta=(\bAtt^{(1:L)}, \bmlp^{(1:L)}): \nrmp{\btheta}\le B \}$, where the norm of a transformer $\TF_\btheta^{\clipval}$ is denoted as 
\begin{align}
\label{eqn:tf-norm}
    \nrmp{\btheta}\defeq \max_{\ell\in[L]} \Big\{  
    \max_{m\in[M]} \set{\lops{\bQ_m^\lth}, \lops{\bK_m^\lth} } + \sum_{m=1}^M \lops{\bV_m^\lth} +
    \lops{\bW_1^\lth} + \lops{\bW_2^\lth}
    \Big\}.
\end{align}
\end{definition}
We introduced clipped operations in transformers for technical reasons. For brevity, we will write $\TF_\btheta = \TF_\btheta^{\clipval}$ when there is no ambiguity. We will set the clipping value $\clipval$ to be sufficiently large so that the clip operator does not take effect in any of our approximation results.

\paragraph{Algorithm induced by Transformers} We equip the transformer with an embedding mapping $\embedmap: \cup_{t \in [\totlen]} \statesp_t \cup \cup_{t \in [\totlen]} (\actionsp_t \times \rewardsp_t) \to \R^D$.  This assigns any state $\state_t \in \statesp_t$ a $D$-dimensional embedding vector $\embedmap(\state_t) \in \R^D$, and any action-reward pair $(\action_t, \reward_t) \in \actionsp_t \times \rewardsp_t$ a $D$-dimensional embedding $\embedmap(\action_t, \reward_t) \in \R^D$. The embedding function $\embedmap$ should encode the time step $t$ of the state, action, and reward. With abuse of notation, we denote $\embedmap(\dset_{t-1}, \state_t) = [\embedmap(\state_1), \embedmap(\action_1, \reward_1), \ldots, \embedmap(\action_{t-1}, \reward_{t-1}), \embedmap(\state_t)]$. We define a concatenation operator $\cat: \R^{D \times *} \to \R^{D \times *}$ that concatenates its inputs $\cat(\bh_1, \ldots, \bh_N) = [\bh_1, \ldots, \bh_N]$ in most examples, but it could also insert special tokens at certain positions (in MDPs we add an additional token at the end of each episode). For a partial trajectory and current state $(\dset_{t-1}, \state_t)$, we input $\bH = \cat(\embedmap(\state_1), \embedmap(\action_1, \reward_1), \ldots, \embedmap(\action_{t-1}, \reward_{t-1}), \embedmap(\state_t)) \in \R^{D \times *}$ into the transformer. This produces output $\bar{\bH} = \TF_{\btheta}^{\clipval}(\bH) = [\bar{\bh}_1, \bar{\bh}_2 \ldots, \bar{\bh}_{-2},\bar{\bh}_{-1}]$ with the same shape as $\bH$. To extract a distribution over the action space $\actionsp_t$ with $| \actionsp_t | = \Numact$ actions, we assume a fixed linear extraction mapping $\extractmap \in \R^{\Numact \times D}$. The induced algorithm is then defined as: $\sAlg_\btheta(\cdot | \dset_{t-1}, \state_t) = \softmax(\extractmap \cdot \bar{\bh}_{-1})$. The overall algorithm induced by the transformer is: 
\begin{equation}\label{eqn:transformer-algorithm}
\sAlg_\btheta(\cdot | \dset_{t-1}, \state_t) = \softmax ( {\extractmap} \cdot {\TF_\btheta^{\clipval}} ( {\cat}( {\embedmap} ( \dset_{t-1}, \state_t)))_{-1}). 
\end{equation}
We will always choose a proper concatenation operator $\cat$ in examples, so that in the pretraining phase, all the algorithm outputs $\{ \sAlg_\btheta(\cdot | \dset_{t-1}, \state_t) \}_{t \le \totlen}$ along the trajectory can be computed in a single forward propagation. 

%% file: Sections_arxiv/supervised-pretraining_arxiv.tex
\section{Statistical analysis of supervised pretraining}\label{sec:supervised-pretraining}
%That is, for each environment $\inst^\ith \sim_{iid} \prior$, an offline trajectory $\dset^\ith_\totlen$ is collected from the interaction of $\inst^\ith$ with $\sAlg_0$, augmented by expert actions $\eaction^\ith$ generated from $\sAlg_{\shortexp}$. Given a class of algorithms $\{ \sAlg_\Par \in \cup_{t \in [\totlen]} \{ (\trajsp_{t-1} \times \cS_t) \to \Delta(\actionsp_t) \}, ~\Par\in\Parspace\}$, supervised pretraining amounts to maximize the log-likelihood 

In supervised pretraining, we are given $\Numobs$ i.i.d offline trajectories $\{\dset^\ith_\totlen =  (\state^\ith_1,\action^\ith_1, \reward^\ith_1, \ldots, \state^\ith_\totlen, \allowbreak \action^\ith_\totlen, \allowbreak\reward^\ith_\totlen) \}_{i=1}^\Numobs \sim_{iid} \P_\prior^{\sAlg_0}$ from the interaction of $\inst^\ith \sim_{iid} \prior$ with an offline algorithm $\sAlg_0$. Given an expert algorithm $\sAlg_{\shortexp}$, we augment each trajectory $\dset_{\totlen}^i$ by $\{ \eaction_t^i \sim_{iid} \sAlg_{\shortexp}( \cdot |\dset_{t-1}^i, \state_t^i)\}_{t \in [\totlen]}$. Supervised pretraining maximizes the log-likelihood over the algorithm class $\{ \sAlg_\Par\}_{\Par\in\Parspace}$
\begin{align}
\EstPar=\argmax_{\Par\in\Parspace}  \frac{1}{\Numobs}\sum_{i=1}^\Numobs\sum_{t=1}^\totlen\log \sAlg_\Par(\eaction^\ith_{t}|\dset_{t-1}^\ith,\state^\ith_t). \label{eq:general_mle}
\end{align}
This section discusses the statistical properties of the algorithm learned via supervised pretraining. 

% When $\sAlg_\Par$ is specified by a transformer, this objective function coincides with the cross-entropy loss used to train transformers in supervised learning tasks. 
% where $\sAlg_\Par(\eaction^\ith_{t}|\dset_{t-1}^\ith,\state_t)$ denotes the probability of algorithm $\sAlg_\Par$ selects the action $\eaction^\ith_t$ given the historical trajectory $\dset_{t-1}^\ith$ and current state $\state^\ith_t$. 

\subsection{Main result}

% \yub{to be organized.}

% In this work, we will also apply the following standard concentration inequality (see e.g. Lemma A.4 in~\cite{foster2021statistical}).
% \begin{lemma}\label{lm:exp_concen}
%     For any sequence of random variables $(X_t)_{t\leq T}$ adapted to a filtration $\{\cF_{t}\}_{t=1}^T$, we have with probability at least $1-\delta$ that
%     \begin{align*}
%         \sum_{s=1}^t X_s\leq \sum_{s=1}^t\log\E[\exp(X_s)\mid\cF_{s-1}]+\log(1/\delta),~~~\text{for all } t\in[T].
%     \end{align*}
% \end{lemma}

Our main result demonstrates that the algorithm maximizing the supervised pretraining loss will imitate $\osAlg_{\shortexp}(\cdot|\dset_{t-1},\state_t) = \E_{\inst\sim \prior,  \dset_{\totlen} \sim \sAlg_0}[\sAlg_{\shortexp}^t(\cdot|\dset_\totlen,\inst)|\dset_{t-1},\state_t]$, the conditional expectation of the expert algorithm $\sAlg_{\shortexp}$ given the observed trajectory. The imitation error bound will scale with the covering number of the algorithm class and a  distribution ratio factor, defined as follows.

\begin{definition}[Covering number]\label{def:cover_number_general} For a class of algorithms $\{\sAlg_\Par,\Par\in\Parspace\}$, 
we say $\Parspace_0 \subseteq\Parspace$ is an  $\rho$-cover of $\Parspace$, if $\Parspace_0$ is a finite set such that for any $\Par\in\Parspace$, there exists $\Par_0\in\Parspace_0$ such that 
\[
\|\log \sAlg_{\Par_0}(\cdot|\dset_{t-1},\state_t)-\log \sAlg_{\Par}(\cdot|\dset_{t-1},\state_t)\|_{\infty}\leq\rho,~~~ \text{for all } \dset_{t-1},\state_t, t\in[\totlen].
\]
The covering number $\cN_{\Parspace}(\rho)$ is the minimal cardinality of $\Parspace_0$ such that $\Parspace_0$ is a $\rho$-cover of $\Parspace$.
\end{definition}

%\yub{Define covering number as sum of $t\in[T]$, then $T$ disappears in MLE rate?}

% \begin{lemma}[General guarantee for supervised pretraining]\label{lm:general_imit}
% Suppose Assumption~\ref{asp:realizability} holds. Then  the solution to~Eq.~\eqref{eq:general_mle} achieves
% \begin{align*}
% \E_{\dset_\totlen\sim \P^{\sAlg_0}_\prior}\brac{ \sum_{t=1}^\totlen \HelDs\paren{ \sAlg_{{\EstPar}}(\cdot|\dset_{t-1},\state_t ), \sAlg_{\shortexp}(\cdot|\dset_{t-1},\state_t )} } \le c\frac{\totlen \log \brac{ \cN_{\Parspace}(1/(\Numobs\totlen)^2) \totlen/\delta } }{n} + \totlen\geneps.
% \end{align*}
% with probability at least $1-\delta$ for some universal constant $c>0$.
% \end{lemma}
% See the proof in Section~\ref{sec:pf_lm:general_imit}. 

\begin{definition}[Distribution ratio]\label{def:dist_ratio}
\label{def:distribution-ratio}
We define the distribution ratio of two algorithms $\sAlg_1,\sAlg_2$ by
\begin{align*}\distratio_{\sAlg_1,\sAlg_2}
:=
\E_{\inst\sim\prior,\dset_\totlen\sim\P_\inst^{\sAlg_1}}
\Big[\prod_{s=1}^{\totlen}\frac{\sAlg_1(\action_s|\dset_{s-1},\state_s)}{\sAlg_2(\action_s|\dset_{s-1},\state_s)}\Big] = 1 + \chi^2\Big( \P_\prior^{\sAlg_1};\P_\prior^{\sAlg_2} \Big).
\end{align*}
\end{definition}

% \yub{This is chi-squared distance between the full trajectory under $\sAlg_1$ and $\sAlg_2$.}

Our main result requires the realizability assumption of algorithm class $\{ \sAlg_\Par\}_{\Par \in \Parspace}$ with respect to the conditional expectation of the expert algorithm. 

\begin{assumption}[Approximate realizability]
\label{asp:realizability}
% \lc{I think this includes both imitation and Bayes learning (PSfull). (Note that PS $=$ PSfull in bandit; later need another lemma  PS learning in MDPs).}
There exists $\TruePar\in\Parspace$ and $\geneps > 0$ such that for all $t\in[\totlen]$, 
\begin{align}
\label{eqn:plc_approx_general}
\log\E_{\inst \sim \prior, \adset_\totlen \sim \P_{\inst}^{\sAlg_0,\sAlg_\shortexp}}\Big[\frac{\osAlg_{\shortexp}(\eaction_t|\dset_{t-1},\state_t )}{\sAlg_\TruePar(\eaction_t|\dset_{t-1},\state_t )}\Big] \le \geneps. 
\end{align}
% where $\adset_\totlen$ follows the distribution  $\P_{\inst}^{\sAlg_0,\sAlg_\shortexp}$ with $\inst\sim\prior$. 
\end{assumption}
% It can be verified that a sufficient condition for Assumption~\ref{asp:realizability} is 
% \begin{align}\label{asp:realizability_suff}
% \log\frac{\osAlg_{\shortexp}(\eaction_t|\dset_{t-1},\state_t )}{\sAlg_\TruePar(\eaction_t|\dset_{t-1},\state_t )}\leq\geneps
% \end{align} almost surely over $\adset_\totlen\sim\P_{\inst}^{\sAlg_0,\sAlg_\shortexp},\inst\sim\prior$.  

% Let $\dset_\totlen=(\state_1,\action_1,\reward_1,\ldots,\state_\totlen,\action_\totlen,\reward_\totlen)$ denote a trajectory obtained by rolling out an algorithm $\sAlg$ in a problem instance $\inst$.  
% We define the (instance-dependent) expected cumulative reward\lc{maybe a different name?}
% \begin{align*}
% \textstyle \totreward_{\inst,\sAlg}(\totlen)
% :=\E_{\action\sim\sAlg}[\sum_{t=1}^\totlen \reward_t],
% \end{align*} where  the expectation is over the  states $\state_{t}\sim\P^s_{\inst,t}(\cdot|\state_{t-1},\action_{t-1})$, actions  $\action_t\sim\sAlg(\cdot|\dset_{t-1},\state_t)$ and rewards $\reward_t\sim\P^r_{\inst,t}(\cdot|\state_t,\action_t)$ for $t\in[\totlen]$.  

We aim to bound the performance gap between $\sAlg_{\EstPar}$ and $\sAlg_\shortexp$ in terms of expected cumulative rewards, where the expected cumulative reward is defined as 
\begin{align*}
\textstyle \totreward_{\prior,\sAlg}(\totlen)
:= \E_{\inst\sim\prior}\big[\totreward_{\inst,\sAlg}(\totlen) \big],~~~~~~~~~ \totreward_{\inst,\sAlg}(\totlen) = \E_{\dset_{\totlen} \sim\P^{\sAlg}_\inst}[\sum_{t=1}^\totlen \reward_t].
\end{align*}
An intermediate step of the result is controlling the expected Hellinger distance between two algorithms, where for distributions $p, q$, we have $\HelDs(p, q) = \int (\,\sqrt{p(x)} - \sqrt{q(x)} \,)^2 d x$. 

\begin{theorem}[Performance gap between expected cumulative rewards]\label{thm:diff_reward} Let Assumption~\ref{asp:realizability} hold and let $\EstPar$ be a solution to Eq.~\eqref{eq:general_mle}. Take $\distratio = \distratio_{\osAlg_\shortexp,\sAlg_0}$ as defined in Definition~\ref{def:dist_ratio}, and $\cN_{\Parspace} = \cN_{\Parspace}((\Numobs\totlen)^{-2})$  as defined in Definition~\ref{def:cover_number_general}. Then for some universal constant $c>0$, with probability at least $1-\delta$, we have 
%\lc{I still prefer the old version in statistical-proof.tex.}~
%\yub{$\geneps$ in separate sqrt?}
\begin{align}\label{eqn:Hellinger-bound-main-theorem}
&~ \E_{\dset_\totlen\sim \P^{\sAlg_\shortexp}_\prior}\Big[ \sum_{t=1}^\totlen \HelD \paren{  \sAlg_{{\EstPar}}(\cdot|\dset_{t-1},\state_t ),\osAlg_{\shortexp}(\cdot|\dset_{t-1},\state_t )} \Big] 
\le c {\totlen} \sqrt{\distratio}
\bigg(\sqrt{\frac{\log \brac{ \cN_{\Parspace} \cdot 
 \totlen/\delta } }{n}} +  \sqrt{\geneps}\bigg).
\end{align} 
Further assume that $|\reward_t| \leq 1$ almost surely. Then with probability at least $1-\delta$, the difference of the expected cumulative rewards between $\sAlg_\EstPar$ and $\osAlg_\shortexp$ satisfies
\begin{align}\label{eqn:reward-bound-main-theorem}
\Big|\totreward_{\prior,\sAlg_\EstPar}(\totlen)-\totreward_{\prior,\osAlg_\shortexp}(\totlen)\Big|
&\leq 
c \totlen^2 \sqrt{\distratio} \bigg(\sqrt{\frac{\log \brac{ \cN_{\Parspace} \cdot 
 \totlen/\delta } }{n}} +  \sqrt{\geneps}\bigg).
\end{align}
\end{theorem}

The proof of Theorem~\ref{thm:diff_reward} is contained in Section~\ref{sec:pf_thm:diff_reward}.

We remark that when the expectation on the left-hand-side of (\ref{eqn:Hellinger-bound-main-theorem}) is with respect to the measure $\P_\prior^{\sAlg_0}$, standard MLE analysis will provide a bound without the distribution ratio factor $\distratio = \distratio_{\osAlg_\shortexp,\sAlg_0}$ in the right-hand side. The distribution ratio factor arises from the distribution shift between trajectories generated by the expert algorithm $\sAlg_\shortexp$ versus the context algorithm $\sAlg_0$.  In addition, it should be noted that the result in Theorem~\ref{thm:diff_reward} holds generally provided Assumption~\ref{asp:realizability} is satisfied, which does not require that the algorithm class is induced by transformers.

% \sm{Can we provide an example in which reward maximization supervision is not the correct objective? }

% \yub{update statement to hellinger + total reward.} \yub{Strengthen the result to subsume the fast rate case when $C\approx 1$.} \sm{Discuss the distribution ratio. } \sm{Discuss the reward maximization pretraining objective. }

% \yub{Corollary: 1. AD; 2. DPT (comment on approximate DPT, or give another corollary box).}

\subsection{Implications in special cases}

\paragraph{Algorithm Distillation} When we set $\sAlg_\shortexp = \sAlg_0$, the supervised pretraining approach corresponds to the Algorithm Distillation method introduced in \cite{laskin2022context}. In this case, it suffices to set $\eaction^\ith = \action^\ith$ for every pretraining trajectory, eliminating the need to sample additional expert actions. The conditional expectation of the expert algorithm is given by $\osAlg_\shortexp = \sAlg_0$, and the distribution ratio $\distratio_{\sAlg_\shortexp,\sAlg_0}=1$. Under these conditions, Theorem~\ref{thm:diff_reward} ensures that $\sAlg_\EstPar$ imitates $\sAlg_0$ with a reward difference bounded by
\begin{align*}
\Big|\totreward_{\prior,\sAlg_\EstPar}(\totlen)-\totreward_{\prior,\sAlg_0}(\totlen)\Big|
&\leq c \totlen^2 \Big( \sqrt{\frac{\log \brac{ \cN_{\Parspace} \cdot \totlen/\delta } }{n} } + \sqrt{\geneps} \Big). 
\end{align*}
If the context algorithm $\sAlg_0$ does not perform well, we cannot expect the learned algorithm $\sAlg_\EstPar$ to have good performance, regardless of the number of offline trajectories. 
% For good performance of the learned algorithm $\sAlg_\EstPar$, it is crucial to have pretraining data generated by a high-quality context algorithm $\sAlg_0$. 

% The algorithm distillation approach studied in~\cite{laskin2022context} corresponds to the case when $\sAlg_\shortexp=\sAlg_0$. Under this setting, we have the distribution ratio $\distratio_{\sAlg_\shortexp,\sAlg_0}=1$ and Lemma~\ref{thm:diff_reward} provides the guarantee \begin{align*}
% |\totreward_{\prior,\sAlg_\EstPar}(\totlen)-\totreward_{\prior,\sAlg_\shortexp}(\totlen)|
% &\leq 
% c \genrewardb\cdot\totlen^2\sqrt{\frac{\log \brac{ \cN_{\Parspace}((\Numobs\totlen)^{-2}) \totlen/\delta } }{n} + \geneps}
% \end{align*}
% with probability at least $1-\delta$.  

\paragraph{Decision Pretrained Transformer} When we set $\sAlg_\shortexp^t = \sAlg_\shortexp^t(\state_t,\inst)=\action^*_t$ to be the optimal action at time $t$, the supervised pretraining approach corresponds to Decision-Pretrained Transformers (DPT) proposed in \cite{lee2023supervised}. In this case, the conditional expectation of the expert algorithm $\osAlg_\shortexp(\cdot|\dset_{t-1},\state_t)=\E[\sAlg_{\shortexp}(\cdot|\state_t,\inst)|\dset_{t-1},\state_t]=\sAlg_{\TS}(\cdot|\dset_{t-1},\state_t)$ is the Thompson sampling algorithm \citep[Theorem 1]{lee2023supervised}, which samples from the posterior distribution of the optimal action $\action^*_t$ given by $\P(a^*_t(\inst) |\dset_{t-1},\state_t)\propto \prior(\inst)\cdot\P_\inst^{\sAlg_0}(\dset_{t-1},\state_t)$. This implies that learning from optimal actions effectively learns to imitate Thompson sampling. Furthermore, the context algorithm is not required to perform well for the learned algorithm to be consistent with Thompson sampling. However, a high-quality context algorithm $\sAlg_0$ may help reduce the distribution ratio $\distratio$, thereby learning Thompson sampling with fewer samples.

% The decision-pretrained transformer in~\cite{lee2023supervised} corresponds to the case  where the expert $\sAlg_\shortexp=\sAlg_\shortexp(\state_t,\inst)=\action^*_t$ gives the optimal action at time $t$. Therefore, the expert policy $\sAlg_\shortexp(\cdot|\dset_{t-1},\state_t)=\E[\sAlg_{\shortexp}(\cdot|\state_t,\inst)|\dset_{t-1},\state_t]=\sAlg_{\TS}(\cdot|\dset_{t-1},\state_t)$ equals the policy of Thompson sampling (i.e., the distribution of the optimal action $\action^*_t$  under the posterior $\P(\inst|\dset_{t-1},\state_t)\propto \prior(\inst)\cdot\P_\inst^{\sAlg_0}(\dset_{t-1},\state_t)$). This implies that learning from the optimal actions is effectively learning to imitate the Thompson sampling algorithm.

\paragraph{Approximate DPT} In practical scenarios, the learner may not have access to the optimal action $\action^*_t$ of the environment $\inst$ during pretraining. Instead, they might rely on an estimated optimal action $\widehat\action_t^* \sim \sAlg_{\shortexp}^t(\cdot | \dset_\totlen)$, derived from the entire trajectory $\dset_\totlen$. We can offer a guarantee analogous to Theorem~\ref{thm:diff_reward}, provided the distribution of the estimated action closely aligns with its posterior distribution: 
\begin{align}\E_{\dset_\totlen\sim\P_{\prior}^{\sAlg_0}}\KL{\sAlg_{\shortexp}^t(\cdot | \dset_\totlen)}{\P_{\TS,t}(\cdot|\dset_\totlen)}\leq\appeps,~~~ \forall t \in [\totlen]. 
\label{eq:app_opt_cond}
\end{align}
Here, $\P_{\TS,t}(\cdot|\dset_\totlen)$ represents the posterior distribution of the optimal action $\action^*_t=\action^*_t(\inst)$ at time $t$, given the observation $\dset_\totlen$, where $(\inst, \dset_\totlen) \sim \P_\prior^{\sAlg_0}$. 

% Note that in practice the learner may not know the optimal action $\action^*_t$ during pretraining, and may use an estimated optimal action $\widehat\action_t^*=\widehat\action^*_t(\dset_\totlen)$ as a surrogate. It can be shown that a similar guarantee as in Lemma~\ref{thm:diff_reward} can be established if the approximation error \begin{align}\E_{\dset_\totlen\sim\P_{\prior}^{\sAlg_0}}\KL{\widehat\action^*_t}{\P_{\TS,t}(\cdot|\dset_\totlen)}\leq\appeps
% \label{eq:app_opt_cond}
% \end{align} for all $t\in[\totlen]$ and some small $\appeps>0$, where $\P_{\TS,t}$ is the distribution of the optimal action $\action^*_t$ under the posterior $\inst|\dset_\totlen$. Note that  the error term $\appeps$  characterizes  the distance between the approximated optimal action $\widehat\action^*_t$ and the best guess we have about $\action^*_t$ given $\dset_\totlen$. \lc{Should put the result to appendix later}

\begin{proposition}\label{prop:app_opt_diff_reward} Let Assumption~\ref{asp:realizability} hold and let $\EstPar$ be the solution to Eq.~\eqref{eq:general_mle}. Take $\distratio = \distratio_{\sAlg_\TS,\sAlg_0}$ as defined in Definition~\ref{def:dist_ratio}, and $\cN_{\Parspace} = \cN_{\Parspace}((\Numobs\totlen)^{-2})$  as defined in Definition~\ref{def:cover_number_general}. Assume that for each trajectory, an estimated optimal action is provided $\widehat\action_t^* \sim \sAlg_{\shortexp}^t(\cdot | \dset_\totlen)$ at each time $t\in[\totlen]$ satisfying Eq.~\eqref{eq:app_opt_cond}. 
% Then for some universal constant $c>0$, with probability at least $1-\delta$, we have
% \begin{align*}
% &~  \E_{\dset_\totlen\sim \P^{\sAlg_\TS}_\prior} \Big[ \sum_{t=1}^\totlen \HelD \paren{ \sAlg_{{\EstPar}}(\cdot|\dset_{t-1},\state_t ), \sAlg_{\TS}(\cdot|\dset_{t-1},\state_t )} \Big] \\
% \le&~ c{\totlen} \sqrt{\distratio}
% \sqrt{\frac{\log \cN_{\Parspace} +\log(\totlen/\delta)}{n}+\geneps+\appeps}. 
% \end{align*}
% with probability at least $1-\delta$ for some universal constant $c>0$. 
Assume that the rewards $|\reward_t|\leq 1$  almost surely. Then for some universal constant $c>0$, with probability at least $1-\delta$, the difference of the expected cumulative rewards between $\sAlg_\EstPar$ and $\sAlg_\TS$ satisfies 
\begin{align*}
|\totreward_{\prior,\sAlg_\EstPar}(\totlen)-\totreward_{\prior,\sAlg_\TS}(\totlen)|
&\leq 
c \sqrt{\distratio}\cdot\totlen^2 \Big( \sqrt{\frac{\log \brac{ \cN_{\Parspace} \cdot \totlen/\delta } }{n} } + \sqrt{\geneps} + \sqrt{\appeps} \Big). 
\end{align*}
\end{proposition}
The proof of Proposition~\ref{prop:app_opt_diff_reward} is contained in Appendix~\ref{app:proof-prop-diff-reward-app-opt}.

% Next, we choose the class of algorithms $\{\sAlg_\Par,\Par\in\Parspace\}$ to be algorithms induced by transformers. We show that such classes of algorithms has the potential to imitate state-up-the-art algorithms in  bandits and Markov Decision Processes (MDPs) settings. As a consequence, the algorithms we derived from transformers also have benign regret guarantee

%% file: Sections_arxiv/ICRL_arxiv.tex
\section{Approximation by transformers}\label{sec:ICRL}

In this section, we demonstrate the capability of transformers to implement prevalent reinforcement learning algorithms that produce near-optimal regret bounds. Specifically, we illustrate the implementation of LinUCB for stochastic linear bandits in Section~\ref{sec:LinUCB-statement}, Thompson sampling for stochastic linear bandits in Section~\ref{sec:TS-statement}, and UCB-VI for tabular Markov decision process in Section~\ref{sec:Tabular-MDP-statement}. %\sm{Add explanations of $\conO$, $\cO$, $\tcO$}

\subsection{LinUCB for linear bandits}\label{sec:LinUCB-statement}

A stochastic linear bandit environment is defined by $\inst=(\TrueLBPar,\Noisedist,\aset_1,\ldots,\aset_\totlen)$. For each time step $t\in[\totlen]$, the learner chooses an action $\action_t\in\R^{d}$ from a set of actions $\sA_t=\{\ba_{t,1},\ldots,\ba_{t,\Numact}\}$, which consists of $\Numact$ actions and may vary over time. Upon this action selection, the learner receives a reward $\reward_t=\<\action_t,\TrueLBPar\>+\Noise_t$. Here,$\{ \Noise_t \} \sim_{ iid} \Noisedist$ are zero-mean noise variables, and $\TrueLBPar\in\R^{d}$ represents an unknown parameter vector. Stochastic linear bandit can be cast into our general framework by setting $s_t = \aset_t$ and adopting a deterministic transition where $s_t$ transits to $s_{t+1}$ deterministically regardless of the chosen action.

% Consider a class of linear bandit problems, in which at time $t\in[\totlen]$ the learner selects an action $\action_t\in\R^{d}$ from an action set $\sA_t=\{\ba_{t,1},\ldots,\ba_{t,\Numact}\}$ that consists of $\Numact$ actions, and observe the reward $\reward_t=\<\action_t,\TrueLBPar\>+\Noise_t$, where $\Noise_t\sim\Noisedist$ are i.i.d. zero mean noise and $\TrueLBPar\in\R^{d}$ is the unknown parameter vector. We allow the action sets $\aset_t$ to be vary across $t\in[\totlen]$. Note that the linear bandit instance is specified by $\inst=(\TrueLBPar,\Noisedist,\aset_1,\ldots,\aset_\totlen)$. To cast the linear bandit problem into our general framework, we assume $\inst$ follows the prior $\prior$ and let the state $\state_t=\aset_t$ for all $t\in[\totlen]$. Therefore, given the instance $\inst$, there exists only one possible state at each time $t$. 

% We now demonstrate that transformers are able to approximate  LinUCB~\cite{chu2011contextual}. 
We assume the context algorithm $\sAlg_0$ is the soft LinUCB \citep{chu2011contextual}. Specifically, for each time step $t\in[\totlen]$, the learner estimates the parameter $\TrueLBPar$ using linear ridge regression $\bw^t_{\ridge,\lambda}:=\argmin_{\bw\in\R^d} \sum_{j=1}^{t-1}(\reward_j-\<\ba_j,\bw\>)^2+ \lambda \|\bw\|_2^2$. Subsequently, the learner calculates the upper confidence bounds for the reward of each action as $v^*_{tk}:=\langle \ba_{t,k},\bw^t_{\ridge,\lambda}\rangle +\cwid \cdot (\ba_{t,k}^\top (\lambda\id_d+\sum_{j=1}^{t-1}\action_j\action_j^\top)^{-1}  \ba_{t,k})^{1/2}$. Finally, the learner selects an action $\action_t$ according to probability $\{ p^*_{t,j} \}_{j \in [\Numact]} = \softmax(\{v^*_{tj}/\temp \}_{j \in [\Numact]})$ for some sufficiently small $\temp>0$. Note that the soft LinUCB $\sAlg_{\sLinUCB(\temp)}$ recovers LinUCB as $\temp\to 0$. 
%\lc{(shall we put this into appendix?) }

% \begin{enumerate}
%     \item For some $\lambda>0$, computes  $\bw^t_{\ridge,\lambda}:=\argmin_{\bw\in\R^d}\frac{1}{2t}\sum_{j=1}^{t-1}(r_j-\<\ba_j,\bw\>)^2+\frac{\lambda}{2t}\|\bw\|_2^2$.
%     \item For each action $k\in[\Numact]$, computes $v^*_{tk}:=\<\ba_{t,k},\bw^t_{\ridge,\lambda}\>+\cwid\sqrt{\ba_{t,k}^\top \bA_{i}^{-1}  \ba_{t,k}}$ for some $\cwid$, where $\bA_t=\lambda\id_d+\sum_{j=1}^{t-1}\ba_j\ba_j^\top$.
%     \item Selects the action $\ba_{t,j}$ with probability $p^*_{t,j}=\frac{\exp(v^*_{tj}/\temp)}{\sum_{j=1}^\Numact\exp(v^*_{tk}/\temp)}$ for $j\in[\Numact]$ for some sufficiently small $\temp>0$.
% \end{enumerate} 

We further assume the existence of constants $\sigma,b_a,B_a,B_w>0$ such that the following conditions hold:   $|\Noise_t|\leq\sigma$, $b_a\leq\ltwo{\ba_{t,k}}\leq B_a$, and $\ltwo{\bw^*}\le B_w$ for all $t\in[\totlen],k\in[\Numact]$. Given these, the confidence parameter is defined as: $\cwid=\sqrt{\lambda}B_w+\sigma\sqrt{2\log (2B_aB_w \totlen )+d\log((d\lambda+\totlen B_a^2)/(d\lambda))} = \tcO(\sqrt{d})$. The following result shows that the soft LinUCB algorithm can be efficiently approximated by transformers, for which the proof is contained in Appendix~\ref{sec:pf_thm:approx_smooth_linucb}. 

\begin{theorem}[Approximating the soft LinUCB]\label{thm:approx_smooth_linucb}
Consider the embedding mapping $\embedmap$, extraction mapping $\extractmap$, and concatenation operator $\cat$ as in \ref{sec:tf_embed_bandit}. For any small $\eps,\temp>0$, there exists a transformer $\TF_\btheta^{\clipval}(\cdot)$ with $\log \clipval = \tcO(1)$, 
\begin{equation}
\begin{aligned}
&~D \le \conO(d\Numact),~ \layer=\tilde \cO(\sqrt{\totlen}),~ M \le 4\Numact,~D' = \tcO(d+A\sqrt{Td/(\temp\eps)}), ~\nrmp{\tfpar} = \tcO(\Numact+\totlen\sqrt{d}/(\temp\eps^{1/4})), \label{eq:linucb_tf_param} 
\end{aligned}
\end{equation}
such that taking $\sAlg_{\tfpar}$ as defined in Eq.~(\ref{eqn:transformer-algorithm}), we have
\[
\Big|\log \sAlg_{\sLinUCB(\tau)}(\ba_{t,k}|\dset_{t-1},\state_t) - \log \sAlg_{\tfpar}(\ba_{t,k}|\dset_{t-1},\state_t) \Big|\leq \eps, ~~~~\forall t\in[\totlen],k\in[\Numact].
\]
Here $\conO(\cdot)$ hides some absolute constant, and  $\tilde \cO(\cdot)$ additionally hides polynomial terms in $(\sigma, b_a^{-1}, B_a, B_w, \lambda^{\pm1})$, and poly-logarithmic terms in $(\totlen, \Numact, d, 1/\eps,1/\temp)$. 
\end{theorem}

% $\cO(\cdot)$ hides absolute constant and 
%\lc{not polynomial, but rational polynomials. Maybe do not mention this?} \lc{Also, this $D = \cO(d \Numact)$ is confusing, should be $D=O(d\Numact)$.  We have three things: universal constant, problem-dependent constant, logarithmic dependence.}

A key component in proving Theorem~\ref{thm:approx_smooth_linucb} is demonstrating that the transformer can approximate the accelerated gradient descent algorithm for solving linear ridge regression (Lemma~\ref{lm:approx_ridge}), a result of independent interest. Leveraging Theorem~\ref{thm:approx_smooth_linucb}, we can derive the following regret bound for the algorithm obtained via Algorithm Distillation, with the proof provided in Appendix~\ref{sec:pf_thm:smooth_linucb}.

\begin{theorem}[Regret of LinUCB and ICRL]\label{thm:smooth_linucb}
% Take  $0<\temp\leq\log(2\Numact B_a(B_w+2\alpha/\sqrt{\lambda}))/\sqrt{4\totlen}=\tilde O(1/\sqrt{\totlen})$, 
% the expected regret of the  smooth LinUCB satisfies 
%  \begin{align*}
%      \sum_{t=1}^\totlen\max_{k}\<\ba_{t,k},\bw^*\>-\totreward_{\inst,\sAlg_\sLinUCB}(\totlen)\leq\sqrt{8 d ({B_a}/{\sqrt{\lambda}}+1)^2T\alpha^2\log((d\lambda+TB_a^2)/(d\lambda))}+1+\sqrt{T}\leq Cd\sqrt{T}\log(T)
%  \end{align*}
% for some problem-dependent constant $C>0$ and all problem instance $\inst$. 
Let $\Theta = \Theta_{D, L, M, \hidden, B}$ be the class of transformers satisfying Eq.~\eqref{eq:linucb_tf_param} with $\eps=1/\totlen^3$ and $\temp = 1/ \log(4\totlen\Numact B_a(B_w+2\alpha/\sqrt{\lambda}))/\sqrt{4\totlen}=\tcO(\totlen^{-1/2})$, and choose the clip value $\log\clipval = \tcO(1)$. Let both the context algorithm $\sAlg_0$ and the expert algorithm $\sAlg_\shortexp$ coincide with the soft LinUCB algorithm $\sAlg_{\sLinUCB(\tau)}$ with parameter $\tau$ during supervised pretraining. Then with probability at least $1-\delta$, the learned algorithm $\sAlg_{\esttfpar}$, a solution to Eq.~\eqref{eq:general_mle}, entails the regret bound
\begin{align*}
\E_{\inst\sim\prior}\Big[\sum_{t=1}^\totlen\max_{k}\<\ba_{t,k},\bw^*\>-\totreward_{\inst,\sAlg_\esttfpar}(\totlen)\Big]&\leq   \cO\bigg( d\sqrt{\totlen}\log(\totlen)+ \totlen^2\sqrt{\frac{\log ( \cN_{\Parspace} \cdot \totlen/\delta )}{n} } \bigg),
% &\leq Cd\sqrt{T}\log(T)+\tilde O\Big(\totlen^2 \sqrt{\frac{\log(\Numobs/\delta)}{\Numobs}}\Big)
\end{align*}
where $\log \cN_{\Parspace} \le \tcO(\layer^2\embd(\head\embd+\hidden) \log\Numobs) \leq \tcO(\totlen^{3.5} d^2 \Numact^3\log\Numobs)$. Here $\cO$ hides polynomial terms in $(\sigma, b_a^{-1}, B_a, \\B_w, \lambda^{\pm1})$, and $\tcO$ additionally hides poly-logarithmic terms in $(\totlen, \Numact, d, 1/\eps,1/\temp)$. 
\end{theorem}

% \lc{do we want to hide $\log n$?}

\subsection{Thompson sampling for linear bandit}\label{sec:TS-statement}

We continue to examine the stochastic linear bandit framework of Section~\ref{sec:LinUCB-statement}, now assuming a Gaussian prior $\bw^\star\sim \cN(0,\Tpspar\id_d)$ and Gaussian noises $\{ \eps_t \}_{t \ge 0} \sim_{iid} \cN(0,\Tpsparn)$. Additionally, we assume existence of $(b_a, B_a)$ such that $b_a\leq\ltwo{\ba_{t,k}}\leq B_a$. In this model, Thompson sampling also utilizes linear ridge regression. Subsequently, we establish that transformers trained under the DPT methodology can learn Thompson sampling algorithms. We state the informal theorem in Theorem~\ref{thm:approx_thompson_linear} below, where its formal statement and proof are contained in Appendix~\ref{example:ts-app}. 

% \begin{theorem}[Approximating the Thompson sampling, Formal statement]\label{thm:approx_thompson_linear-formal}
% For any $0<\delta_0<1/2$, consider the same embedding mapping $\embedmap$ and extraction mapping $\extractmap$ as for soft LinUCB in \ref{sec:tf_embed_bandit},
% %\sm{Link}\lc{same as before},
% and consider the standard concatenation operator $\cat$. Under Assumption~\ref{ass:thompson_mlp_approx_linear},~\ref{ass:thompson_mlp_diff_action_linear}, for $\eps<(\Trunregp\wedge1)/4$, there exists a  transformer $\TF_\btheta(\cdot)$ with \sm{$\TF_\btheta^{\clipval}(\cdot)$ with $\log \clipval = \tcO(1)$?}
% \begin{align}
% &D=\tcO(T^{1/4}Ad),~L= \tcO(\sqrt{T}),~\sup_{l\in[L]}M^{\lth}=\tcO(AT^{1/4}),~\hidden=\tcO(A(T^{1/4}d+\neuron))~,\notag\\
% &~~~\nrmp{\btheta}\leq \tcO(T+AT^{1/4}+\sqrt{ \neuron A}+\weightn),\label{eq:ts_tf_param-formal}
% \end{align} 
% such that 
% $\log\frac{\sAlg_{\TS}(\action_t|\dset_{t-1},\state_t)}{\sAlg_{\tfpar}(\action_t|\dset_{t-1},\state_t)}\leq \eps$ for all $t\in[T],k\in[A]$ with probability at least $1-\delta_0$.  Here  
% $\neuron,\weightn$ are the values defined in   Assumption~\ref{ass:thompson_mlp_approx_linear} with $\trunprob=\eps/(4A),\Trunregpa=\Trunregp,\delta=\delta_0$, and $\tcO(\cdot)$ hides rational polynomial terms in $(\lambda,\Tpsparn,b_a,B_a)$ and logarithmic terms in $(\totlen,A,d,1/\delta_0,1/\eps)$.
% \end{theorem}

\begin{theorem}[Approximating the Thompson sampling, Informal]\label{thm:approx_thompson_linear}
Consider the embedding mapping $\embedmap$, extraction mapping $\extractmap$, and concatenation operator $\cat$ as in \ref{sec:tf_embed_bandit}. 
% Consider the same embedding mapping $\embedmap$ and extraction mapping $\extractmap$ as for soft LinUCB in \ref{sec:tf_embed_bandit}, and consider the standard concatenation operator $\cat$. 
Under Assumption~\ref{ass:thompson_mlp_approx_linear},~\ref{ass:thompson_mlp_diff_action_linear}, for sufficiently small $\eps$, there exists a transformer $\TF_\btheta^{\clipval}(\cdot)$ with $\log \clipval = \tcO(1)$, 
\begin{equation}\label{eq:ts_tf_param-main}
\begin{aligned}
&~D = \tcO(AT^{1/4}d),~~~~~ L=\tilde \cO(\sqrt{T}),~~~M =\tilde \cO(A T^{1/4}), \\
&~\nrmp{\btheta} = \tcO(T+A T^{1/4}+\sqrt{A}),~~~~~ \hidden = \tcO(A T^{1/4}d),
\end{aligned}
\end{equation}
such that taking $\sAlg_{\tfpar}$ as defined in Eq.~(\ref{eqn:transformer-algorithm}), with probability at least $1-\delta_0$ over $(\inst, \dset_{\totlen}) \sim \P_{\prior}^{\sAlg}$ for any $\sAlg$, we have
\[
 \log \sAlg_{\TS}(\ba_{t,k}|\dset_{t-1},\state_t) - \log \sAlg_{\tfpar}(\ba_{t,k}|\dset_{t-1},\state_t) \leq \eps,~~~~\forall t\in[T],k\in[A]. 
\]
Here, $\tcO(\cdot)$ hides polynomial terms in $(\neuron,\weightn, \Tpspar^{\pm1}, \Tpsparn^{\pm1}, b_a^{-1}, B_a)$, and poly-logarithmic terms in $(\totlen, \Numact, d, 1/\eps,\\ 1/\delta_0)$, where $(\neuron, \weightn)$ are parameters in Assumption~\ref{ass:thompson_mlp_approx_linear} and \ref{ass:thompson_mlp_diff_action_linear}. 
%Here $M,C$ are the values defined in   Assumption~\ref{ass:thompson_mlp_approx_linear} with $\trunprob=c\eps/K$ for some universal constant $c>0$, and $\tilde O(\cdot)$ hides logarithmic dependency on $\totlen,K,d,1/\delta_0,1/\eps$.\lc{the dependence on $M,C$ is written in terms of the dependence on other variables.}
\end{theorem}

Central to proving Theorem~\ref{thm:approx_thompson_linear} is establishing that the transformer can approximate matrix square roots via Pade decomposition (\cref{sec:pf_thm:approx_thompson_linear-formal}), a result of independent interest. Theorem~\ref{thm:approx_thompson_linear} thereby implies the subsequent regret bound for transformers trained under DPT. 

\begin{theorem}
[Regret of Thompson sampling and ICRL]\label{thm:ts_linear_regret}
% Let the distributional assumptions stated in Section~\ref{example:ts} holds. %Thompson sampling achieves the Bayes regret 
% \begin{align*}  \E_{\inst\sim\prior}\Big[\sum_{t=1}^\totlen\max_{k}\<\ba_{t,k},\bw^*\>-\totreward_{\inst,\sAlg_\TS}(\totlen)\Big]\leq Cd\sqrt{\totlen}\log(\totlen d)
% \end{align*}
% for some problem-dependent constant $C>0$. 
% Suppose in addition Assumption~\ref{ass:thompson_mlp_approx_linear},~\ref{ass:thompson_mlp_diff_action_linear} are in force. 
Follow the assumptions of Theorem~\ref{thm:approx_thompson_linear}. Let $\Theta = \Theta_{D, L, M, \hidden, B}$ be the class of transformers satisfying Eq.~\eqref{eq:ts_tf_param-main} with $\eps=1/(\distratio\totlen^3)$,   $\delta_0=\delta/(2n)$, and choose the clip value $\log \clipval = \tcO(1)$. Assume the trajectories are collected by some context algorithm $\sAlg_0$, and we choose the expert algorithm $\sAlg_\shortexp(\state_t,\inst)=\action^*_t=\argmax_{\ba\in\sA_t}\<\ba,\bw^*\>$ to be the optimal action of the bandit instance $\inst$ for each trajectory. Then with probability at least $1-\delta$, the learned algorithm $\sAlg_{\esttfpar}$, a solution to Eq.~\eqref{eq:general_mle}, entails regret bound
\begin{align*}
\E_{\inst\sim\prior}\Big[\sum_{t=1}^\totlen\max_{k}\<\ba_{t,k},\bw^*\>-\totreward_{\inst,\sAlg_\esttfpar}(\totlen)\Big]&\leq \cO \bigg( d\sqrt{T}\log(Td)+ \sqrt{\distratio} \cdot\totlen^2\sqrt{\frac{\log ( \cN_{\Parspace} \totlen/\delta ) }{n} } \bigg), 
\end{align*}
where $\distratio = \distratio_{\sAlg_\TS,\sAlg_0}$, and $\log \cN_{\Parspace} \le \tcO(\layer^2\embd(\head\embd+\hidden)\log \Numobs )  \leq\tcO(\totlen^{5/4}\Numact^2 d(\neuron+\Numact\sqrt{\totlen}d)\log\Numobs)$. Here $\cO$ hides polynomial terms in $(\Tpspar^{\pm1}, \Tpsparn^{\pm1}, b_a^{-1}, B_a)$, and $\tcO$ additionally hides poly-logarithmic terms in   $(\neuron,\weightn,\\~ \totlen, \Numact, d, 1/\eps, 1/\delta_0)$. 
\end{theorem}

% \lc{the $\delta_0=\delta/(2n)$ dependence is annoying. This means $\log(1/\delta)$ contains polylogarithmic terms in $n$... may need to redefine $\tcO(\cdot)$ to include $\log n$ terms?}\lc{Also, do we want to remove the term $\log(Td)$?}

\subsection{UCB-VI for Tabular MDPs}\label{sec:Tabular-MDP-statement}

A finite-horizon tabular MDP is specified by $\inst=(\statesp,\actionsp, \horizon, \{\transit_h\}_{h\in[\horizon]},\{\rewardfun_h\}_{h\in[\horizon]},\init)$, with $\horizon$ being the time horizon, $\statesp$ the state space of size $\Numst$, $\actionsp$ the action space of size $\Numact$, and $\init\in\Delta(\statesp)$ defining the initial state distribution. At each time step $h\in[\horizon]$, $\transit_h: \statesp\times\actionsp \to \Delta(\statesp)$ denotes the state transition dynamics and $\rewardfun_h:\statesp \times \actionsp \to [0,1]$ gives the reward function. A policy $\plc:=\{\plc_h:(\statesp \times\actionsp \times \R)^{h-1}\times\statesp \to\Delta(\actionsp)\}_{h \in [\horizon]}$ maps history and state to a distribution over actions. The value of policy $\pi$ interacting with environment $\inst$ is defined as the expected cumulative reward $\valuefun_\inst(\plc)=\E_{\inst,\plc}[\sum_{h=1}^\horizon \rewardfun_h (\state_h,\action_h)]$. A policy $\optplc$ is said to be optimal if $\optplc=\argmax_{\pi\in\Delta(\plcset)}\valuefun_\inst(\pi)$. 
%We use $\plcset$ to denote all possible deterministic policies,  and let $\Delta(\plcset)$ be either the set of all  policies or the set of all  distributions over the  deterministic policies. $\optplc$ is said to be an optimal policy if $\optplc=\argmax_{\pi\in\Delta(\plcset)}\valuefun_\inst(\pi)$. 

We let the context algorithm $\sAlg_0$ interact with an MDP instance $\inst$ to generate $\Numepi$ episodes, each consisting of $\horizon$ horizon sequences $ (\state_{k,h},\action_{k,h},\reward_{k,h})_{k \in [\Numepi], h \in [\horizon]}$. These can be reindexed into a single trajectory $\dset_{\totlen} = \{ (\state_t,\action_t,\reward_t) \}_{t \in [\totlen]}$ with $t=H(k-1)+h$ and $\totlen=\Numepi\horizon$. The Bayes regret of any algorithm $\sAlg$ gives $\E_{\inst\sim\prior}[\Numepi\Vfun_\inst(\plc^*)-\totreward_{\inst,\sAlg}(\totlen)]$. 
%We may view the whole interaction trajectory of length $\totlen$ as $\Numepi$ episodes from the finite-horizon MDP. For any $t\in[\totlen]$, write $t=H(k-1)+h$ for some and denote $\state_t,\action_t,\reward_t$ by $\state_{k,h},\action_{k,h},\reward_{k,h}$, respectively.

% \paragraph{The UCB-VI algorithm} 

Near minimax-optimal regret for tabular MDPs can be attained through the UCB-VI algorithm \citep{azar2017minimax}. We demonstrate that transformers are capable of approximating the soft UCB-VI algorithm $\sAlg_{\sUCBVI(\tau)}$, a slight modification of UCB-VI formalized in Appendix~\ref{sec:tf_embed_mdp}.

\begin{theorem}[Approximating the soft UCB-VI]\label{thm:approx_ucbvi}
% Let $R=2\max\{(B_a+\alpha/\sqrt{\lambda})\}$.
Consider the embedding mapping $\embedmap$, extraction mapping $\extractmap$, and concatenation operator $\cat$ as in Appendix~\ref{sec:tf_embed_mdp}. There exists a transformer $\TF_\btheta^{\clipval}(\cdot)$ with $\log \clipval = \tcO(1)$, 
\begin{equation}\label{eq:ucbvi_tf_param-main}
\begin{aligned}
&~D =\conO(\horizon\Numst^2\Numact),~~~L= 2\horizon+8,~~~M= \conO(\horizon\Numst^2\Numact),\\
&~\hidden= \conO(\Numepi^2\horizon\Numst^2\Numact),~~~~\nrmp{\btheta}\leq \tcO(\Numepi^2\horizon\Numst^2\Numact+\Numepi^3+1/\temp),
\end{aligned} 
\end{equation}
such that 
$\sAlg_{\sUCBVI(\tau)}(\action|\dset_{t-1},\state_t) = \sAlg_{\tfpar}(\action|\dset_{t-1},\state_t)$ for all $t\in[T],\action\in\actionsp$. Here $\conO(\cdot)$ hides universal constants and $\tcO(\cdot)$  hides poly-logarithmic terms in $(\horizon,\Numepi,\Numst,\Numact,1/\temp)$. 
\end{theorem}

Leveraging Theorem~\ref{thm:approx_ucbvi}, we can derive the following regret bound for the algorithm obtained via Algorithm Distillation. 

\begin{theorem}[Regret of UCB-VI and ICRL]\label{thm:ucbvi_icrl-main}
Let $\Theta = \Theta_{D, L, M, \hidden, B}$ be the class of transformers satisfying Eq.~\eqref{eq:ucbvi_tf_param-main} with $\temp = 1/\Numepi$, and choose the clip value $\log \clipval = \tcO(1)$. Let both the context algorithm $\sAlg_0$ and the expert algorithm $\sAlg_\shortexp$ coincide with the soft UCB-VI algorithm $\sAlg_{\sUCBVI(\tau)}$ during supervised pretraining.  Then with probability at least $1-\delta$, the learned algorithm $\sAlg_{\esttfpar}$, a solution to Eq.~\eqref{eq:general_mle}, entails regret bound
\begin{align*}
\E_{\inst\sim\prior}[\Numepi\Vfun_\inst(\plc^*)-\totreward_{\inst,\sAlg_\esttfpar}(\totlen)]\leq \tcO \bigg(\horizon^2\sqrt{\Numst\Numact\Numepi}+\horizon^3\Numst^2\Numact+\totlen^2\sqrt{\frac{\log ( \cN_{\Parspace} \totlen/\delta ) }{n} } \bigg),
\end{align*}
where $\log \cN_{\Parspace} \le \tcO(\layer^2\embd(\head\embd+\hidden) \log\Numobs) = \tcO(\horizon^4\Numst^4\Numact^3(\Numepi^2+\horizon\Numst^2\Numact)\log\Numobs)$, and $\tcO(\cdot)$ hides poly-logarithmic terms in $(\horizon,\Numepi,\Numst,\Numact)$. 
\end{theorem}

%% file: Sections_arxiv/experiments_arxiv.tex
\section{Experiments}
\label{sec:experiments}

% \sm{LinUCB imitation ReLU}

% \sm{TS K-arm, Bernoulli}

In this section, we perform preliminary simulations to demonstrate the ICRL capabilities of transformers and validate our theoretical findings. We remark that while similar experiments have been conducted in existing works~\citep{laskin2022context,lee2023supervised}, our setting differs in several aspects such as imitating the entire interaction trajectory in our pretrain loss~\eqref{eq:general_mle} as opposed to on the last (query) state only as in~\citet{lee2023supervised}. The code is available at~\href{https://github.com/licong-lin/in-context-rl}{https://github.com/licong-lin/in-context-rl}.

% as we train the transformer to imitate the whole interaction trajectory.\lc{not precise?}

We compare pretrained transformers against empirical average, LinUCB (or UCB), and Thompson sampling. We use a GPT-2 model~\cite{garg2022can,lee2023supervised} with $L = 8$ layers, $M=4$ heads, and embedding dimension $D=32$. We utilize ReLU attention layers, aligning with our theoretical construction. We pretrain the transformer with two setups: (1) Both context algorithm $\sAlg_0$ and expert algorithm $\sAlg_{\shortexp}$ use LinUCB (the Algorithm Distillation approach); (2) Context algorithms $\sAlg_0$ mixes uniform policy and Thompson sampling, while expert $\sAlg_{\shortexp} = \action_t^*$ provides optimal actions (DPT). See Appendix \ref{sec:exp_details} for further experimental details. 

% In this section, we investigate the performance of ICRL empirically on  stochastic linear bandit problems, and compare it with other online bandit algorithms including empirical average,  LinUCB (or UCB)  and Thompson sampling. We refer the reader to Appendix~\ref{sec:exp_details} for more details. In the experiments,  we use an $8$-layer transformer GPT-2 model with ReLU attention layers and pretrain the model with two different choices of context and expert algorithms. Namely, we consider (1). the context algorithm $\sAlg_{0}$ and the expert algorithm $\sAlg_{\shortexp}$  both to be LinUCB (i.e., Algorithm distillation). (2). The context algorithm $\sAlg_{0}$ is a mixture of the uniform policy and Thompson sampling, while  the expert  gives the optimal action $\sAlg_{\shortexp}=\action_t^*$  (i.e., the DPT approach). 

In the first setup, we consider stochastic linear bandits with $d=5$ and $A=10$. At each $t \in [200]$, the agent chooses an action $\action_t$ and receives reward $\reward_t=\<\action_t,\bw^*\>+\eps_t$ where $\eps_t\sim\cN(0,1.5^2)$. The parameter $\bw^*$ is from ${\rm Unif}([0,1]^d)$. The action set $\sA_t=\sA$ is fixed over time with actions i.i.d. from ${\rm Unif}([-1, 1]^d)$. We generate 100K trajectories using $\sAlg_0=\sAlg_{\shortexp}=\LinUCB$ and train transformer $\TF_\EstPar(\cdot)$ via Eq.~\eqref{eq:general_mle}. Figure~\ref{fig:regret_1} (left) shows regrets of the transformer (TF), empirical average (Emp), LinUCB, and Thompson sampling (TS). The transformer outperforms Thompson sampling and empirical average, and is comparable to LinUCB, agreeing with Theorem~\ref{thm:smooth_linucb}. The small regret gap between TF and LinUCB may stem from the limited capacity of the GPT2 model.

% For the first case, we consider a stochastic linear bandit with $d=5,A=10$.  We choose the time step $\totlen=200$. At each time $t\in[\totlen]$, the agent deploys an action $\action_t$ and the environment generates the reward $\reward_t=\<\action_t,\bw^*\>+\eps_t$, where the noise $\eps_t\sim\cN(0,\sigma^2)$ and $\sigma=1.5$.  We assume the true parameter $\bw^*$ is generated from the uniform distribution on $[0,1]^d$. Moreover, we assume the action set $\sA_t=\sA$ is fixed across time and the actions are i.i.d. samples from the uniform distribution on $[-1,1]^d$.  We generate $100000$ interaction trajectories with $\sAlg_0=\sAlg_{\shortexp}=\LinUCB$ and find the transformer $\TF_\EstPar(\cdot)$ via solving Eq.~\eqref{eq:general_mle}. We see in Figure~\ref{fig:regret_1} (left) that the transformer outperforms Thompson sampling and empirical average and is close to  LinUCB  in terms of regret. This aligns with our theoretical findings in Theorem~\ref{thm:smooth_linucb}. Moreover, we conjecture that the minor gap between TF and LinUCB is likely due to the limited expressivity of our small transformer model.

In the second setup, we consider multi-armed Bernoulli bandits with $d = 5$. The parameter $\bw^*$ is from ${\rm Unif}([0,1]^d)$. The fixed action set $\sA_t=\sA$ contains one-hot vectors $\{\be_i\}_{i=1}^d$ (multi-armed bandits). At each $t \in [200]$, the agent selects $\action_t$ receives reward $r_t \sim {\rm Bern}(\<\action_t,\bw^*\>)$. Let $\sAlg_{\mathrm{unif}}$ be the uniform policy. We use $\sAlg_{\mathrm{unif}}$ and $\sAlg_\TS$ as context algorithms to generate $50$K trajectories each. The expert is fixed as $\sAlg_\shortexp=\action^*$. We train transformer $\TF_\EstPar(\cdot)$ via Eq.~\eqref{eq:general_mle}. Figure~\ref{fig:regret_1} (right) shows regrets for the pretrained transformer (TF), empirical average (Emp), UCB, and Thompson sampling (TS).  The transformer aligns with Thompson sampling, validating Theorem~\ref{thm:ts_linear_regret}. However, TS underperforms UCB for Bernoulli bandits, as shown. 

% For the second case, we study a multi-armed Bernoulli bandit problem with $d=5$ and time step $\totlen=200$. We assume the true parameter $\bw^*$ is also generated following the uniform distribution on $[0,1]^d$ and the action set $\sA_t=\sA$  consists of the one-hot vectors $\{\be_i\}_{i=1}^d$.  At each time $t\in[\totlen]$, the agent selects an action $\action_t$ and the environment generates a reward $r_t$ from $\mathrm{Bern}(\<\action_t,\bw^*\>)$. Let $\sAlg_{\mathrm{unif}}$ denote the policy that  uniformly chooses  an action $\action_t$ at each time step. For both $\sAlg_\TS$ and $\sAlg_{\mathrm{unif}}$, we utilize them as the context algorithms and $\sAlg_\shortexp=\action_t^*$ as the expert to produce $50000$ i.i.d. interaction trajectories. These datasets are then merged to form one with $\sAlg_0=[\sAlg_{\TS}+\sAlg_{\mathrm{unif}}]/2$ and $\sAlg_\shortexp=\action_t^*$. We find $\TF_{\EstPar}$ via solving Eq.~\eqref{eq:general_mle} and the result is displayed in Figure~\ref{fig:regret_1}~(right). We observe that the regret of transformer is aligned with that of Thompson sampling, validating our finding in Theorem~\ref{thm:ts_linear_regret}. Moreover, it should be noted that TS is sub-optimal in Figure~\ref{fig:regret_1}~(right) since we only consider a small number of time steps.

\begin{figure}[t]
\centering  % Center the figure
\includegraphics[width=0.35\linewidth]{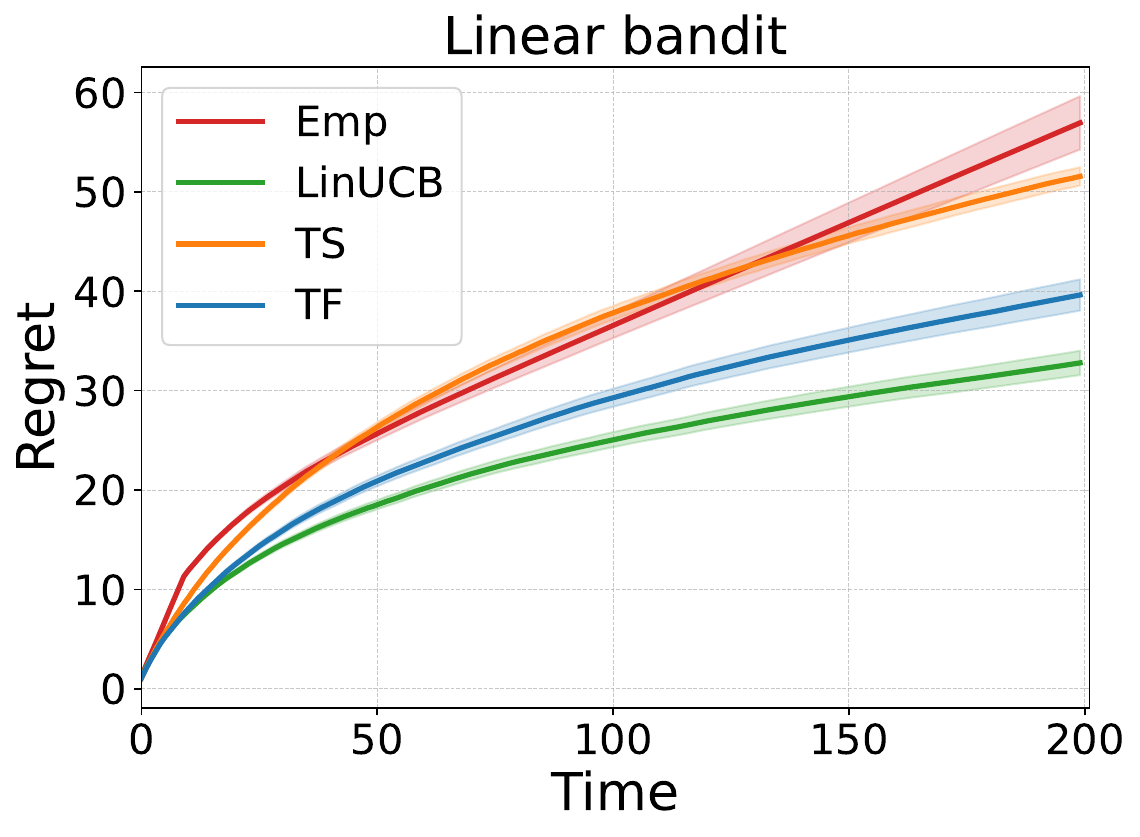}
\hspace{2em}
\includegraphics[width=0.35\linewidth]{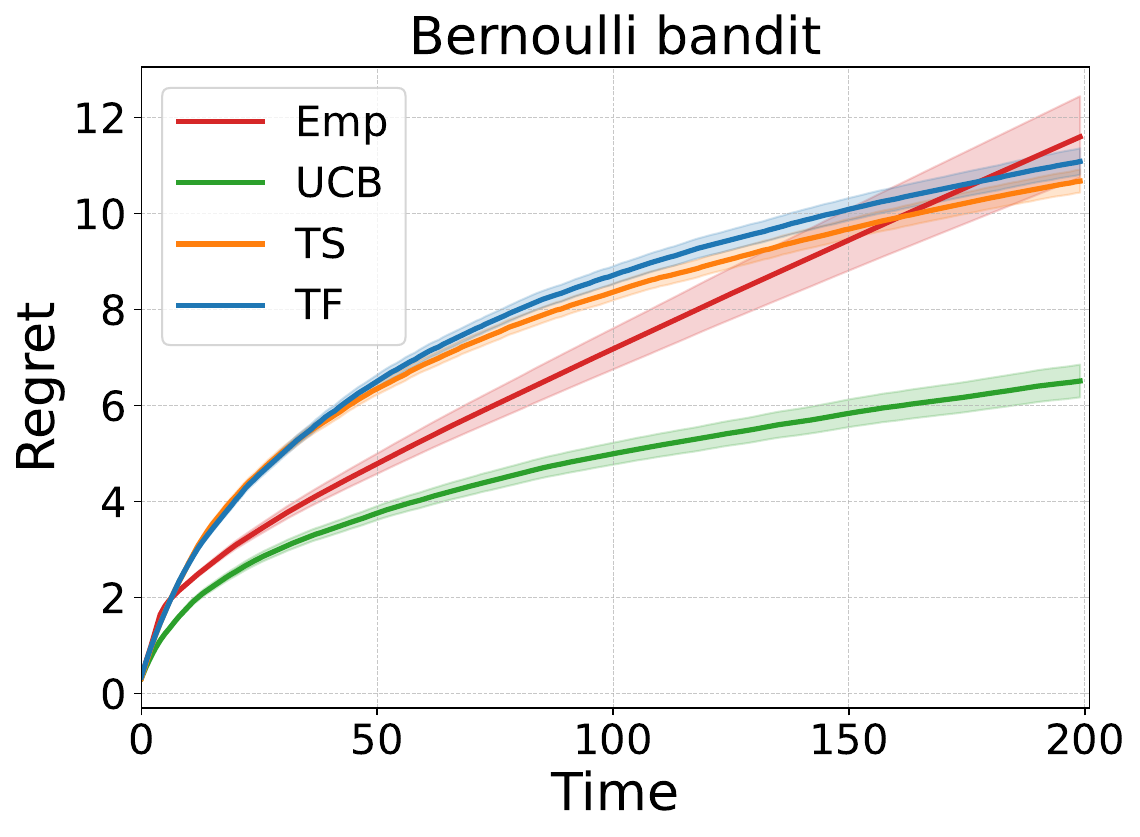}
\vspace{-1em}
\caption{Regrets of transformer (TF), empirical average (Emp), Thompson sampling (TS) and LinUCB or UCB (LinUCB reduces to UCB for Bernoulli bandits). Left: linear bandit with $d=5$, $A=10$, $\sigma=1.5$, $\sAlg_0=\sAlg_\shortexp=\LinUCB$. Right: Bernoulli bandit with $d=5$, $\sAlg_0=(\sAlg_{\mathrm{unif}}+\sAlg_{\TS})/2$ and $\sAlg_\shortexp=\action^*$. The simulation is repeated 500 times. Shading displays the standard deviation of the regret estimates. } 
\label{fig:regret_1} 
\end{figure}

%% file: Sections_arxiv/conclusion_arxiv.tex
\section{Conclusions}

This paper theoretically investigates the ICRL capability of supervised-pretrained transformers. We demonstrate how transformers can efficiently implement prevalent RL algorithms including LinUCB, Thompson sampling, and UCB-VI, achieving near-optimal regrets in respective settings. We also provide sample complexity guarantees for the supervised pretraining approach to learning these algorithms. The generalization error scales with the covering number of the transformer class as well as the distribution ratio between the expert and offline algorithms. Simulations validate our theoretical findings.  
Finally, we discuss the limitations of our results and provide additional discussions in Appendix~\ref{sec:limitation}.

%% file: Sections_arxiv/acknowledgement_arxiv.tex
\section*{Acknowledgement}

The authors would like to thank Peter L. Bartlett for the valuable discussions, and Jonathan Lee for the valuable discussions regarding Decision-Pretrained Transformers as well as providing an early version of its implementation. This work is supported by NSF CCF-2315725, DMS-2210827, NSF Career award DMS-2339904, and an Amazon Research Award.

%% file: Sections_arxiv/limitation.tex
\section{Limitation and discussion}
\label{sec:limitation}
In this section, we discuss some limitations of our work and some potential future directions. 

\paragraph{Distribution ratio} In Theorem~\ref{thm:diff_reward}, our regret bound of the learned algorithms $\sAlg_{\esttfpar}$ depends on the distribution ratio $\distratio_{\sAlg_{\shortexp},\sAlg_0}$. While in cases like algorithm distillation~\citep{laskin2022context} the distribution ratio equals one since the offline algorithm matches the expert algorithm, in the worst case, the ratio can exponentially depend on $\totlen$ or even become arbitrarily large. To control the distribution ratio in practice, one approach is to augment the offline trajectory dataset with a portion of trajectories generated by an expert algorithm or no-regret algorithms resembling the expert algorithm. On the other hand, further research could investigate structural assumptions of decision-making problems that avoid pessimistic dependence on the distribution ratio in regret bounds. 

\paragraph{Guarantee of pretrained transformer} Our statistical result (Theorem~\ref{thm:diff_reward}) only guarantees that the pretrained transformer learns an ``algorithm'' matching the expert algorithm under the pre-training distribution, even though our approximation results (Theorem~\ref{thm:approx_smooth_linucb},~\ref{thm:approx_thompson_linear},~\ref{thm:approx_ucbvi}) show the existence of a transformer approximating the expert algorithm over the entire input space. In our early experiments, we noticed the learned transformers do not generalize well on out-of-distribution instances, such as with shifted reward distributions or increased number of runs $\totlen$. Similar phenomena occur in other in-context learning problems (e.g.~\cite{garg2022can}). Understanding the actual algorithm implemented by the pretrained transformer through theoretical and empirical analysis is an interesting question for future work.
%Hence, it is important to understand in the future, through both empirical and theoretical analysis, if pretrained transformers truly imitate expert algorithms or simply implement alternative algorithms that mimic expert behaviors on the specific problem distribution $\prior$. 

\paragraph{Alternative pretraining methods} Our theoretical results study pretraining the transformer by maximizing the log-likelihood of i.i.d. offline trajectories as in Eq.~(\ref{eq:general_mle}). This aligns with standard supervised pretraining of large language models. However, alternative pretraining methods may also be effective. For instance, one could replace the log-probability in Eq.~(\ref{eq:general_mle}) with an $\ell_2$ loss for continuous action spaces, consider other objectives like cumulative reward \citep{duan2016rl}, or explore goal-conditioned reinforcement learning \citep{chen2021decision} for in-context RL. While our work focuses on log-likelihood pretraining, theoretical investigation of alternative methods is an interesting direction for future work. 

\paragraph{Possibility of surpassing the expert algorithm by online training} Our work considers offline pretraining by imitating the expert algorithm (i.e., $\osAlg_\shortexp$), which can only learn a transformer matching the expert's performance at best. However, through online training, where the transformer interacts with the environment, the learned transformer may surpass existing experts by training to improve itself rather than imitating a specific algorithm. Investigating whether online training enables surpassing expert algorithms is an interesting direction for future work. 

% \paragraph{Internal behaviors of transformers} Our work demonstrates the ability of transformers to implement complicated RL algorithms efficiently, and establishes that via supervised pretraining an algorithm performs as well as the constructed RL algorithm on the problem distribution $\prior$ can be learned. However, the internal workings of transformers are still not well-understood. It is important to understand in the future, through both empirical and theoretical analysis, if transformers truly replicate expert algorithms or simply implement alternative algorithms that mimic expert behaviors on the specific problem distribution $\prior$. \red{ (do we want this part or not?) }

\paragraph{Implications for practice} While the focus of our work is primarily theoretical, our results lead to several practical implications for in-context reinforcement learning. One key implication is the importance of training labels (i.e., expert actions $\eaction$). When the expert algorithm depends solely on past observations, we can learn ${\sAlg}_{\shortexp}$ (see Theorem~\ref{thm:smooth_linucb}). In contrast, when ${\sAlg}_{\shortexp}$ is the optimal action $a^\star$ (involving knowledge of the underlying MDP), we can learn the posterior average of this algorithm given past observations. This corresponds to the Thompson sampling algorithm, as in Decision-Pretrained Transformers (see Theorem~\ref{thm:ts_linear_regret}). 

Furthermore, as discussed previously, the distribution ratio between the offline and expert algorithms may impact the generalization of the learned algorithm. Both our theory (see Theorem~\ref{thm:approx_smooth_linucb}) and simulations (see Figure~\ref{fig:compare_ratio}) show that a small distribution ratio between the offline algorithm $\sAlg_0$ and the expert algorithm $\osAlg_{\shortexp}$ is essential, otherwise the online performance of the learned algorithm may substantially degrade. This suggests that incorporating trajectories generated purely from the expert (``on-policy ICRL'') into the offline dataset is advantageous, when feasible.

%% file: Sections_arxiv/app-experimental-details_arxiv.tex
\section{Experimental details}\label{sec:exp_details}

% \lc{Should be put into appendix later.}

This section provides implementation details of our experiments and some additional simulations.
% Our code is available at \url{https://anonymous.4open.science/r/in-context-rl}.

\subsection{Implementation details}
\paragraph{Model and embedding}

Our experiments use a GPT-2 model \citep{radford2019language} with ReLU activation layers. The model has $L=8$ attention layers, $M=4$ attention heads, and embedding dimension $D=32$. Following standard implementations in~\cite{vaswani2017attention}, we add Layer Normalization~\citep{ba2016layer} after each attention and MLP layer to facilitate optimization. We consider the embedding and extraction mappings as described in Appendix~\ref{sec:tf_embed_bandit}, and train transformer $\TF_{\EstPar}(\cdot)$ via  maximizing Eq.~\eqref{eq:general_mle}.
\paragraph{Online algorithms}
We compare the regret of the algorithm induced by the transformer with empirical average, Thompson sampling, and LinUCB (or UCB for Bernoulli bandits).
\begin{itemize}
\item[\textbf{(Emp)}]\textbf{Empirical average}. 
For time $t\leq A$, the agent selects each action once. For time $t>A$, the agent computes the average of the historical rewards for each action and selects the action with the maximal
averaged historical rewards.
\item[\textbf{(TS)}]\textbf{Thompson sampling}. For linear bandits with Gaussian noises, we consider Thompson sampling introduced in Appendix~\ref{app:ts_algorithm_formula} with $\Tpsparn=\sigma=1.5$ and $\lambda=1$ (note that in this case TS does not correspond to posterior sampling as we assume $\bw^*$ follows the uniform distribution on $[0,1]^d$). For Bernoulli bandits,  we consider the standard TS sampling procedure (see, for example, Algorithm 3.2 in~\cite{russo2018tutorial}). 
\item[\textbf{(LinUCB)}]\textbf{Linear UCB and UCB}. For linear bandits, we use LinUCB (Appendix \ref{sec:soft-LinUCB}) with $\lambda=1$ and $\alpha=2$. For multi-armed Bernoulli bandits, LinUCB reduces to UCB, which selects $\action_t=\argmax_{\action\in\actionsp}\{\hat \mu_{t,\action}+\sqrt{1/\Numvi_t(\action)}\}$, where $\mu_{t,\action}$ is the average reward for action $\action$ up to time $t$, and $\Numvi_t(\action)$ is the number of times action $\action$ was selected up to time $t$.
\end{itemize} 
\subsection{Additional experiments and plots}
We provide additional experiments and plots in this section. In all experiments, we choose the number of samples $\Numobs=100$K.

Additional plots of suboptimality $\<\action_t^*-\action_t,\bw^*\>$ over time are shown in Figure~\ref{fig:subopt_1} for the two experiments in Section~\ref{sec:experiments}. In both cases, the transformer is able to imitate the expected expert policy $\osAlg_{\shortexp}$, as its suboptimality closely matches   $\osAlg_{\shortexp}$ (LinUCB and TS for the left and right panel, respectively).  While the empirical average (Emp) has lower suboptimality early on, its gap does not converge to zero. In contrast, both LinUCB and Thompson sampling are near-optimal up to $\tcO(1)$ factors in terms of their (long-term) regret.\lc{Is this correct?} 

\begin{figure}[ht]
\centering  % Center the figure
\includegraphics[width=0.46\linewidth]{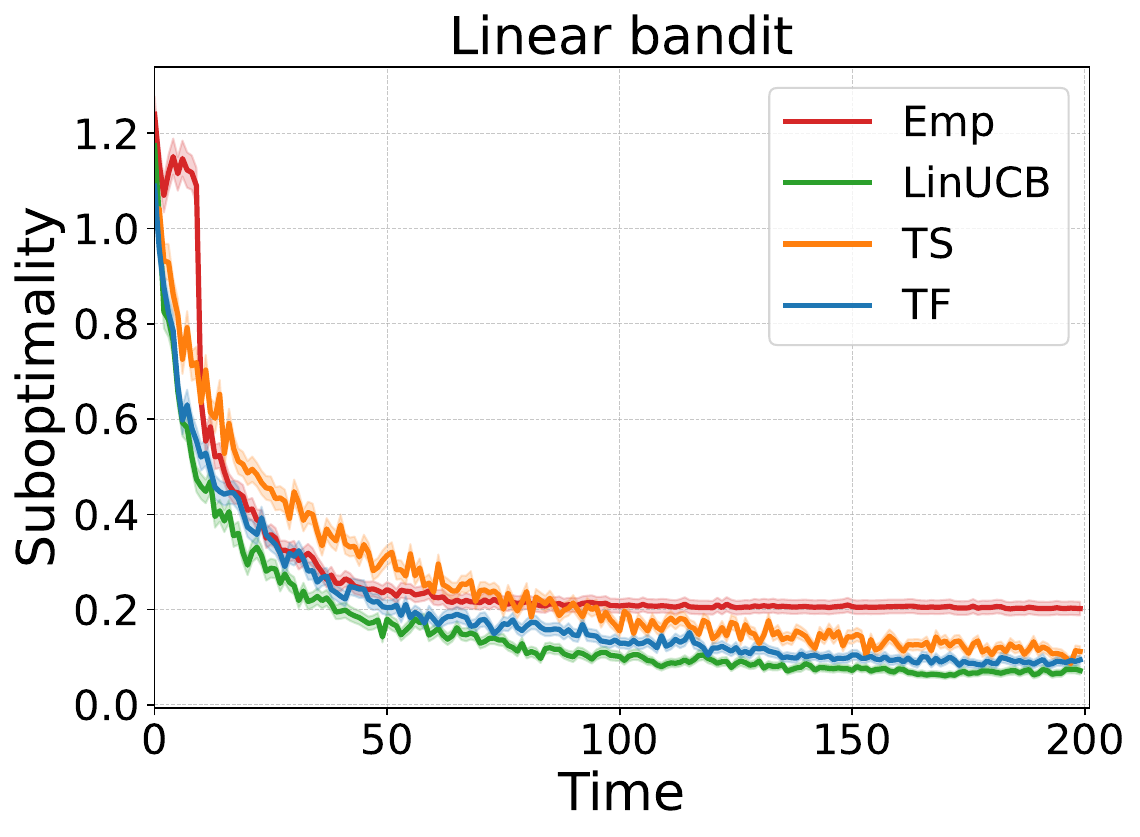}
\includegraphics[width=0.45\linewidth]{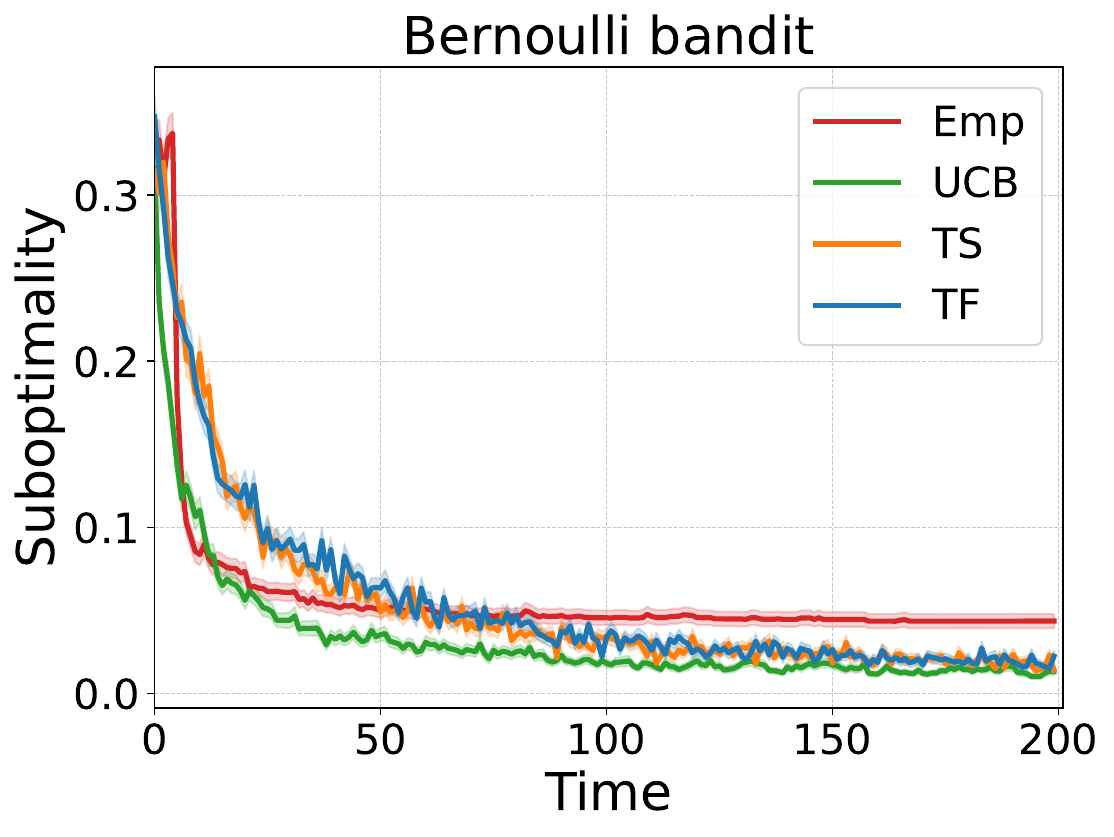}
\caption{Suboptimalities of transformer (TF), empirical average (Emp), Thompson sampling (TS), and LinUCB (or UCB). Left: linear bandit with $d=5$, $A=10$, $\sigma=1.5$, $\sAlg_0=\sAlg_\shortexp=\LinUCB$. Right: Bernoulli bandit with $d=5$, $\sAlg_0=(\sAlg_{\mathrm{unif}}+\sAlg_{\TS})/2$, and $\sAlg_\shortexp=\action_t^*$. The simulation is repeated 500 times. Shading displays the standard deviation of the sub-optimality estimates. %\sm{Check the shades}
} 
\label{fig:subopt_1} 
\end{figure}

Additional simulations were run with $\sAlg_0=\sAlg_{\shortexp}=\UCB$ for Bernoulli bandits, which has fewer actions ($\Numact=5$) than linear bandits ($\Numact=10$). Figure~\ref{fig:linucb_bernoulli} shows the regret and suboptimality of UCB and the transformer overlap perfectly, with both algorithms exhibiting optimal behavior. This suggests the minor gaps between LinUCB and transformer in the left panel of Figure~\ref{fig:regret_1} and \ref{fig:subopt_1} are likely due to limited model capacity. %, compared with the complexity of the bandit problem.

\begin{figure}[ht]
\centering  % Center the figure
\includegraphics[width=0.47\linewidth]{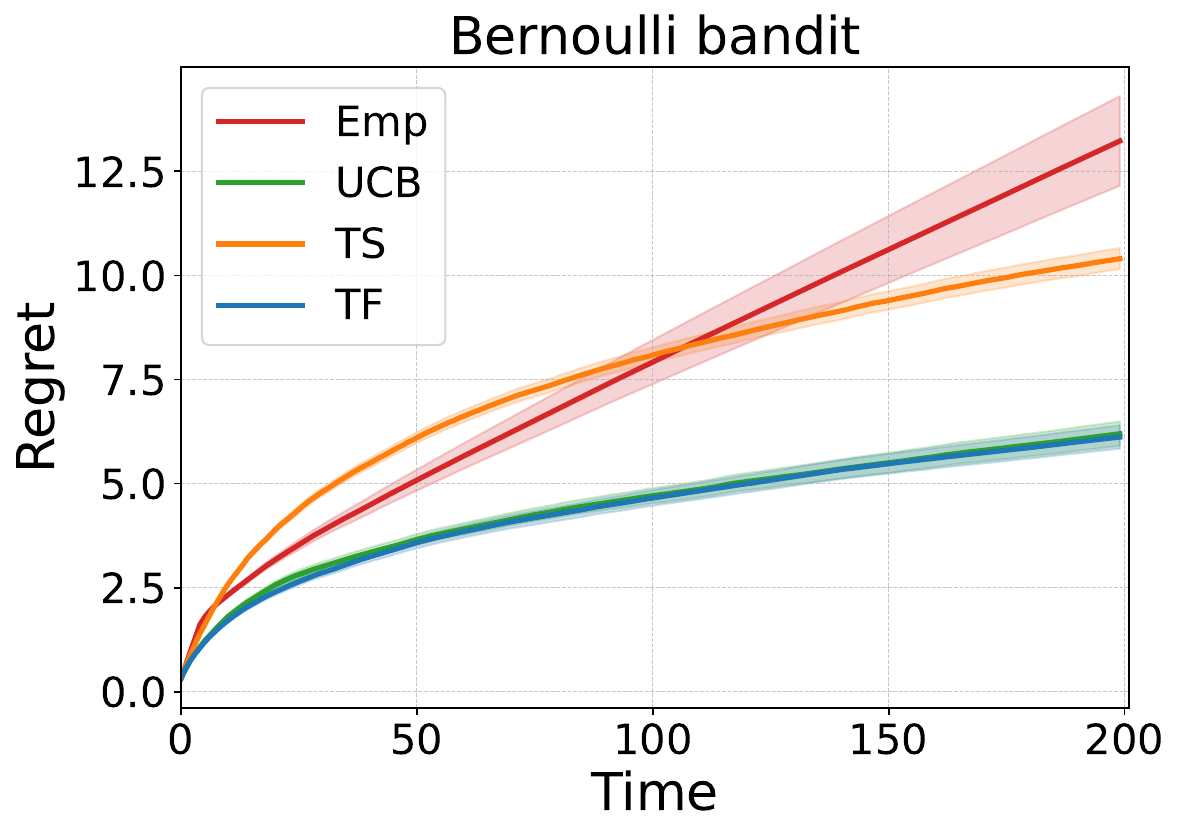}
\includegraphics[width=0.45\linewidth]{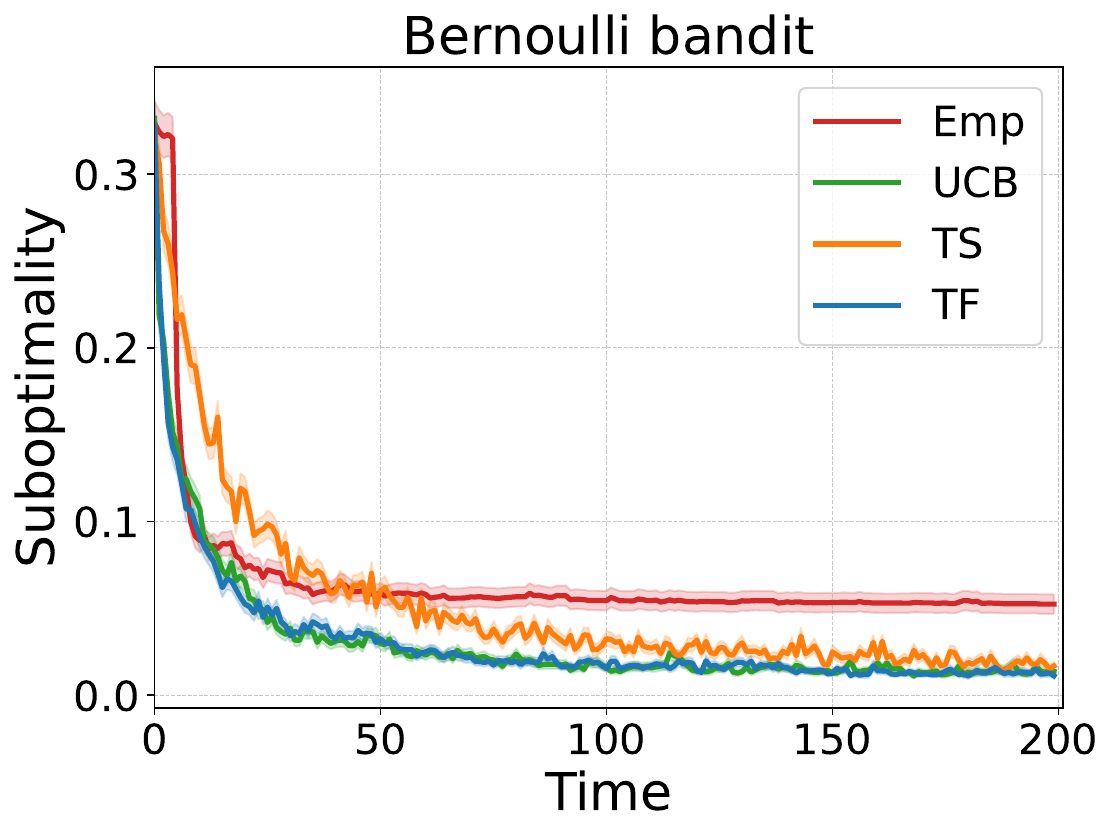}
\caption{Regrets and suboptimalities of transformer (TF), empirical average (Emp), Thompson sampling (TS), and  UCB. Settings: Bernoulli bandit with $d=5$, and $\sAlg_0=\sAlg_\shortexp=\LinUCB$. 
The simulation is repeated 500 times. Shading displays the standard deviation of the estimates. } 
\label{fig:linucb_bernoulli} 
\end{figure}

\subsection{The effect of distribution ratio}

We evaluate the effect of the distribution ratio $\distratio=\distratio_{\osAlg_{\shortexp},\sAlg_0}$ (Definition \ref{def:dist_ratio}) on transformer performance. We consider the Bernoulli bandit setting from Section~\ref{sec:experiments} with expert $\sAlg_\shortexp=\action^*$ giving optimal actions. The context algorithm is $$\sAlg_0=\alpha\sAlg_{\TS}+(1-\alpha)\sAlg_{\unif},$$
mixing uniform policy $\sAlg_{\unif}$ and Thompson sampling $\sAlg_{\TS}$, for $\alpha\in\{0,0.1,0.5,1\}$. The case $\alpha=0$ corresponds to the context algorithm being the i.i.d. uniform policy, and  $\alpha=1$ corresponds to the context algorithm being Thompson sampling. Note that the distribution ratio $\distratio$ may scale as $\cO((1/\alpha)\wedge \Numact^{\cO(\totlen)})$ in the worst case. 

Figure~\ref{fig:compare_ratio} evaluates the learned transformers against Thompson sampling for varying context algorithms.  The left plot shows cumulative regret for all algorithms. The right plot shows the regret difference between transformers and Thompson sampling. The results indicate that an increased distribution ratio impairs transformer regret, as expected. Moreover, it is observed that the transformer, even with the uniform policy (i.e., $\alpha=0$), is capable of imitating Thompson sampling in the early stages $(\mathrm{Time}\leq 30)$, exceeding theoretical predictions. This suggests the transformer can learn Thompson sampling even when the context algorithm differs significantly from the expert algorithm.

\begin{figure}[ht]
\centering  % Center the figure
\includegraphics[width=0.45\linewidth]{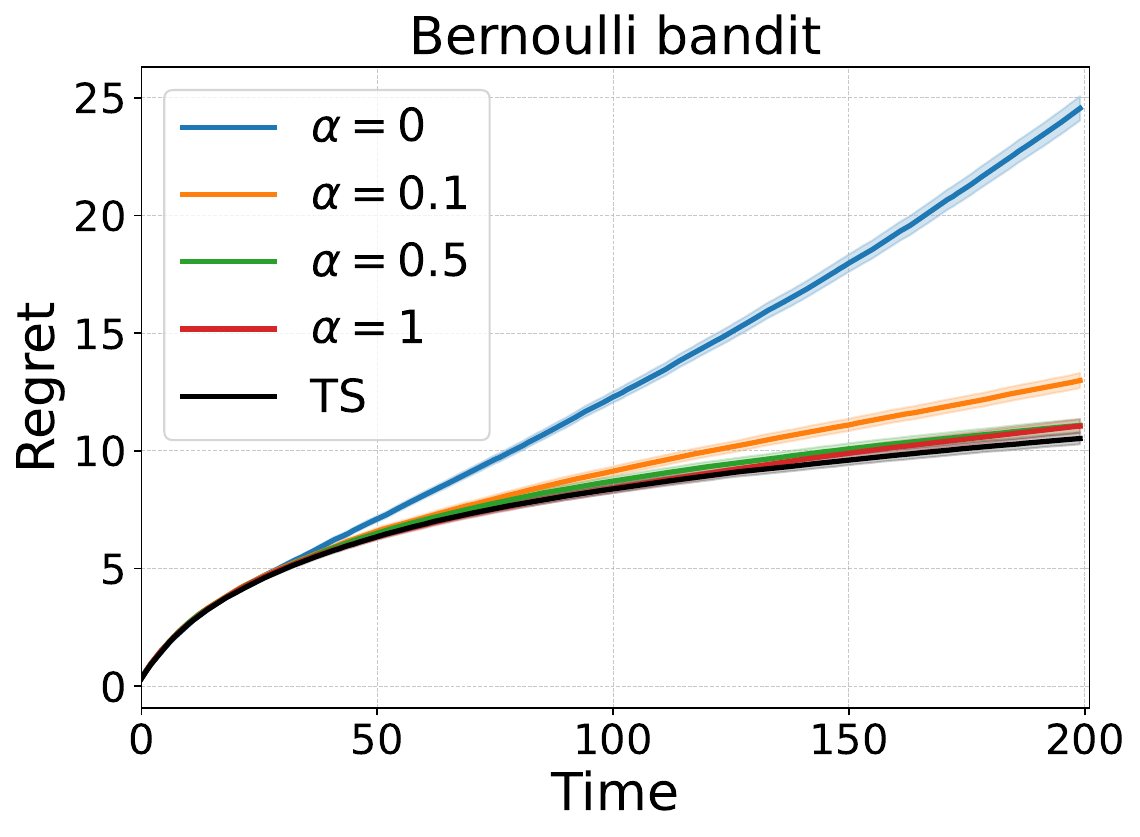}
\includegraphics[width=0.45\linewidth]{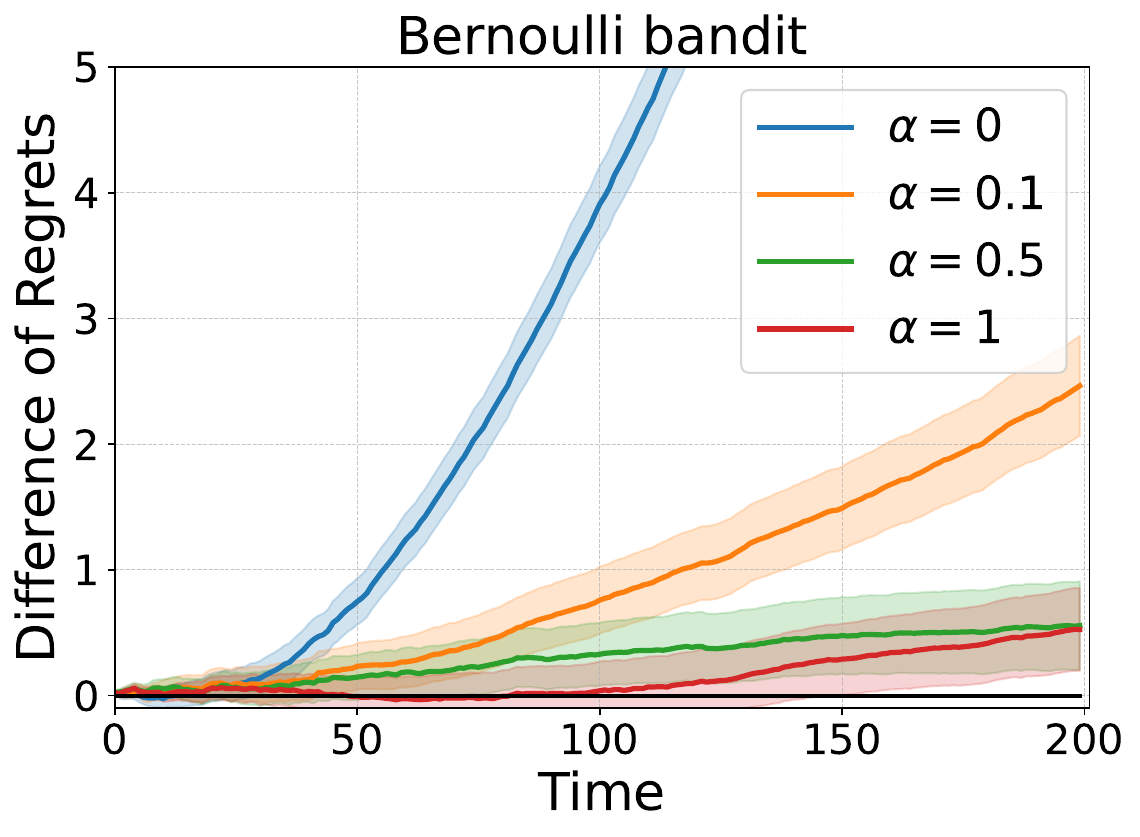}
\caption{Regrets and difference of regrets between transformers and Thompson sampling, for different context algorithms. Settings: Bernoulli bandit with $d=5$, $\sAlg_\shortexp=\action_t^*$ and $\sAlg_0=\alpha\sAlg_\TS+(1-\alpha)\sAlg_\unif$ with $\alpha \in \{0,0.1,0.5,1\}$. The simulation is repeated 500 times. Shading displays the standard deviation of the estimates. } 
\label{fig:compare_ratio} 
\end{figure}

%% file: Sections_arxiv/app-tech-prelim_arxiv.tex
\section{Technical preliminaries}

In this work, we will apply the following standard concentration inequality (see e.g. Lemma A.4 in~\cite{foster2021statistical}).
\begin{lemma}\label{lm:exp_concen}
    For any sequence of random variables $(X_t)_{t\leq T}$ adapted to a filtration $\{\cF_{t}\}_{t=1}^T$, we have with probability at least $1-\delta$ that
    \begin{align*}
        \sum_{s=1}^t X_s\leq \sum_{s=1}^t\log\E[\exp(X_s)\mid\cF_{s-1}]+\log(1/\delta),~~~\text{for all } t\in[T].
    \end{align*}
\end{lemma}

\begin{lemma}\label{lm:cover_num_corr}
   Adopt the notations in Definition~\ref{def:cover_number_general}. Then for any $\Par\in\Parspace$, there exists $\Par_0\in\Parspace_0$ such that $\| \sAlg_{\Par_0}(\cdot|\dset_{t-1},\state_t)-\sAlg_{\Par}(\cdot|\dset_{t-1},\state_t)\|_{1}\leq 2\rho$. 
\end{lemma}

\begin{proof}[Proof of Lemma~\ref{lm:cover_num_corr}]
For any $\Par\in\Parspace$, let $\Par_0\in\Parspace_0$ be such that $\|\log \sAlg_{\Par_0}(\cdot|\dset_{t-1},\state_t)-\log \sAlg_{\Par}(\cdot|\dset_{t-1},\state_t)\|_{\infty}\leq\rho$ for all $\dset_{t-1},\state_t$ and $t\in[\totlen]$. Then
\begin{align*}
   &~\| \sAlg_{\Par_0}(\cdot|\dset_{t-1},\state_t)-\sAlg_{\Par}(\cdot|\dset_{t-1},\state_t)\|_{1}\\
   =&~
   \sum_{\action\in\actionsp_t}|\sAlg_{\Par_0}(\action|\dset_{t-1},\state_t)-\sAlg_{\Par}(\action|\dset_{t-1},\state_t)|\\
   \leq &~
   \sum_{\action\in\actionsp_t}e^{\max\{\log\sAlg_{\Par_0}(\cdot|\dset_{t-1},\state_t),\log\sAlg_{\Par}(\cdot|\dset_{t-1},\state_t)\}}\\
   &~\cdot|\log\sAlg_{\Par_0}(\cdot|\dset_{t-1},\state_t)-\log\sAlg_{\Par}(\cdot|\dset_{t-1},\state_t)|\\
   \leq&~\rho \sum_{\action\in\actionsp_t}e^{\max\{\log\sAlg_{\Par_0}(\cdot|\dset_{t-1},\state_t),\log\sAlg_{\Par}(\cdot|\dset_{t-1},\state_t)\}}\\
   \leq&~ \rho \sum_{\action\in\actionsp_t} (\sAlg_{\Par_0}(\cdot|\dset_{t-1},\state_t)+\sAlg_{\Par}(\cdot|\dset_{t-1},\state_t))\leq 2\rho,
\end{align*}
where the second line uses a Taylor expansion of $e^x$, the fourth line uses the assumption on $\Par_0$, the last line uses $e^{\max\{x,y\}}\leq e^x+e^y$ and the fact that $\sAlg_{\Par_0}(\cdot|\dset_{t-1},\state_t),\sAlg_{\Par}(\cdot|\dset_{t-1},\state_t)$ are probability functions.
\end{proof}

% Suppose in addition the token vectors $\bh^{(\ell)}$ are clipped such that the values in all entries are between $[-\clipval,\clipval]$ after each layer of transformer. 
We have the following upper bound on the covering number of the transformer class $\{\TF^\clipval_{\tfpar}:\tfpar\in\tfparspace_{D, \layer,\head,\hidden,\normb}\}$.  %~\yub{moved (feel free to move back)}\lc{need to be moved to other places.}
\begin{lemma}\label{lm:cover_num_bound}
For the space of transformers $\{\TF^\clipval_{\tfpar}:\tfpar\in\bar{\tfparspace}_{D, \layer,\head,\hidden,\normb}\}$ with 
% \sm{We didn't define $M^{(\ell)}$}
\begin{align*}
\bar{\tfparspace}_{D, \layer,\head,\hidden,\normb}:= \Big\{\tfpar=(\tfpar^{[\layer]}_\attn,\tfpar^{[\layer]}_\mlp):\max_{\ell\in[\layer]}\head^\lth\leq \head, \max_{\ell\in[\layer]}\hidden^\lth\leq \hidden, \nrmp{\tfpar}\leq\normb \Big\}, 
\end{align*}where $\head^\lth,\hidden^\lth$ denote the number  of heads and hidden neurons in the $\ell$-th layer respectively, the covering number of the set of induced algorithms $\{\sAlg_\tfpar,\tfpar\in\bar{\tfparspace}_{D, \layer,\head,\hidden,\normb}\}$ (c.f. Eq.~\ref{eqn:transformer-algorithm}) satisfies
\begin{align*}
    \log \cN_{\bar{\tfparspace}_{D, \layer,\head,\hidden,\normb}}(\rho)
    &\leq c\layer^2\embd(\head\embd+\hidden)\log\Big(2+\frac{\max\{\normb,\layer,\clipval\}}{\rho}\Big)
\end{align*} for some universal constant $c>0$. 

% Moreover, if the number of heads $M$ and the hidden dimension $D'$ are bounded in each layer by some $M_\ell,D'_\ell >0$  for $\ell\in[L]$, we have a refined bound 
% \begin{align*}
%   \log \cN_{\tfparspace_{\layer,\head,\hidden,\normb}}(\rho)   &\leq cLD\Big[\sum_{\ell\in[L]} (M_\ell D+D'_\ell)\Big]\log\Big(2+\frac{\max\{\normb,\layer,\clipval\}}{\rho}\Big).  
% \end{align*}
% \lc{This proof works only for clipped TFs. Later in TF constructions need to show the token matrix has bounded norm with high probability in each layer. }

\end{lemma}
\textbf{Remark of Lemma~\ref{lm:cover_num_bound}.} Note that the transformer classes ${\tfparspace}_{D, \layer,\head,\hidden,\normb},\bar{\tfparspace}_{D, \layer,\head,\hidden,\normb}$ have the same expressivity as one can augment any $\TF_\tfpar\in \bar{\tfparspace}_{D, \layer,\head,\hidden,\normb}$ such that the resulting $\TF_{\tfpar,\mathrm{aug}}\in{\tfparspace}_{D, \layer,\head,\hidden,\normb}$ by adding heads or hidden neurons with fixed zero weights. Therefore, the same bound in Lemma~\ref{lm:cover_num_bound} follows for ${\tfparspace}_{D, \layer,\head,\hidden,\normb}$, and  throughout the paper we do not distinguish ${\tfparspace}_{D, \layer,\head,\hidden,\normb}$ and $\bar{\tfparspace}_{D, \layer,\head,\hidden,\normb}$ and use them interchangeably. We also use $\head^\lth,\hidden^\lth$ to  represent the number  of heads and hidden neurons in the $\ell$-th layer of transformers, respectively. 

\begin{proof}[Proof of Lemma~\ref{lm:cover_num_bound}]
%\label{sec:pf_lm:cover_num_bound}
We start with introducing Proposition J.1 in~\cite{bai2023transformers}. 
\begin{proposition}[Proposition J.1 in~\cite{bai2023transformers}]\label{prop:bai_j1}
The function $\TF^\clipval$ is $(\layer\normb^{\layer}_{H}\normb_{\Theta})$-Lipschitz w.r.t. $\tfpar\in\tfparspace_{D, \layer,\head,\hidden,\normb}$ for any fixed input $\tfmat$. Namely, for any $\tfpar_1,\tfpar_2\in\tfparspace_{D, \layer,\head,\hidden,\normb}$, we have
\begin{align*}
    \ltwopbig{\TF^\clipval_{\tfpar_1}(\tfmat)-\TF^\clipval_{\tfpar_2}(\tfmat)}{\infty}\leq \layer\normb^{\layer}_{H}\normb_{\Theta}\nrmp{\tfpar_1-\tfpar_2},
\end{align*}
where $\ltwop{\bA}{\infty}:=\sup_{t\in[T]}\ltwo{\bA_{\cdot t}}$  for any matrix $\bA\in\R^{K\times T}$, and $\normb_\Theta:=\normb\clipval(1+\normb\clipval^2+\normb^3\clipval^2),\normb_H:=(1+\normb^2)(1+\normb^2\clipval^3)$. 
\end{proposition}
As in the Proof of Theorem 20 in~\cite{bai2023transformers}, we can verify using Example 5.8 in~\cite{wainwright2019high} that the $\delta$-covering number 
\begin{align}\log N(\delta;\normb_{\nrmp{\cdot}}(r),\nrmp{\cdot})\leq \layer(3\head\embd^2+2\embd\hidden)\log(1+2r/\delta),\label{eq:cover_norm_ball}
\end{align}where $\normb_{\nrmp{\cdot}}(r)$ denotes any ball of radius $r$ under the norm $\nrmp{\cdot}$.  Moreover, we have the following continuity result on the log-softmax function
\begin{lemma}[Continuity of log-softmax]\label{lm:log_softmax}
    For any $\bu,\bv\in\R^d$, we have 
    \begin{align*}
    \linf{\log \Big(\frac{e^\bu}{\lone{e^\bu}}\Big)-\log \Big(\frac{e^\bv}{\lone{e^\bv}}\Big)}\leq 2\linf{\bu-\bv}
    \end{align*}
\end{lemma}
We defer the proof of Lemma~\ref{lm:log_softmax} to the end of this section.

Note that  $\sAlg_{\tfpar}(\cdot|\dset_{t-1},\state_t)$ corresponds to $\actnum$ entries in one column of $\tfmat^{(\layer)}$ applied  through the softmax function. Therefore, combining Proposition~\ref{prop:bai_j1},~Lemma~\ref{lm:log_softmax}~and~Eq.~\eqref{eq:cover_norm_ball}, we conclude that for any $r>0$, there exists a subset $\tfparspace_0\in\tfparspace_{D, \layer,\head,\hidden,\normb}$ with size $\layer(3\head\embd^2+2\embd\hidden)\log(1+2r/\delta)$ such that for any $\tfpar\in\tfparspace_{D, \layer,\head,\hidden,\normb}$, there exists $\tfpar_0\in\tfparspace_0$ with
\begin{align*}
    \linf{\log \sAlg_{\tfpar}(\cdot|\dset_{t-1},\state_t)-\log \sAlg_{\tfpar_0}(\cdot|\dset_{t-1},\state_t)}
     \leq 2\layer\normb_{H}^\layer\normb_\Theta \delta
\end{align*} for all $\dset_\totlen$. Substituting $r=\normb$ and letting $ \delta=\rho/( 2\layer\normb_{H}^\layer\normb_\Theta)$ yields the upper bound on $\cN_{\tfparspace_{D, \layer,\head,\hidden,\normb}}(\rho)$ in Lemma~\ref{lm:cover_num_bound}.

\begin{proof}[Proof of Lemma~\ref{lm:log_softmax}]
Define $\bw:=\bu-\bv$. Then
  \begin{align*}
    &\quad\linf{\log \Big(\frac{e^\bu}{\lone{e^\bu}}\Big)-\log \Big(\frac{e^\bv}{\lone{e^\bv}}\Big)}\\
    &\leq
    \linf{\bu-\bv}+|{\log \lone{e^\bu}-\log \lone{e^\bv}}|\\
    &= 
    \linf{\bu-\bv}+\int_{0}^1 \<\frac{e^{\bv+t\bw}}{\lone{e^{\bv+t\bw}}},\bw\> dt\\
     &\leq 
    \linf{\bu-\bv}+\int_{0}^1 \lone{\frac{e^{\bv+t\bw}}{\lone{e^{\bv+t\bw}}}}\cdot\linf{\bw} dt\\
    &=2\linf{\bu-\bv},
    \end{align*} where the third line uses the Newton-Leibniz formula.
\end{proof}

\end{proof}

% \subsection{Convergence of GD and AGD}

We present the following standard results on the convergence of GD and AGD. We refer the reader to~\cite{nesterov2003introductory} for the proof of these results. 

\begin{proposition}[Convergence guarantee of GD and AGD]\label{prop:conv_gd_agd}
Suppose $L(\bw)$ is a $\alpha$-strongly convex and $\beta$-smooth function on $\R^d$. Denote the condition number $\kappa:=\beta/\alpha$ and $\bw^*:=\argmin_{\bw}L(\bw)$.
\begin{enumerate}
\item[(a).]
The gradient descent iterates $\bw^{\sst+1}_{\GD}:=\bw^{\sst}_{\GD}-\eta\nabla L(\bw^{\sst}_{\GD})$ with stepsize $\eta=1/\beta$ and initial point $\bw^{0}_{\GD}=\bzero_d$ satisfies
\begin{align*}
    \|\bw^{\sst}_{\GD}-\bw^*\|_2^2
    &\leq\exp(-\frac{\sst}{\kappa}) \|\bw^{0}_{\GD}-\bw^*\|_2^2,
    \\
    L(\bw^{\sst}_{\GD})-L(\bw^*)
    &\leq \frac{\beta}{2} \exp(-\frac{\sst}{\kappa})\|\bw^{0}_{\GD}-\bw^*\|_2^2.
\end{align*}
    \item [(b).] 
    The accelerated gradient descent (AGD,~\cite{nesterov2003introductory}) iterates $\bw^{\sst+1}_{\AGD}:=\bv^{\sst}_{\GD}-\frac{1}{\beta} L(\bv^{\sst}_{\AGD}),~~ \bv^{\sst+1}_{\AGD}:=\bw^{\sst+1}_{\AGD}+\frac{\sqrt{\kappa}-1}{\sqrt{\kappa}+1}(\bw^{\sst+1}_{\AGD}-\bw^{\sst}_{\AGD})$ with $\bw^{0}_{\AGD}=\bv^{0}_{\AGD}=\bzero_d$ satisfies
    \begin{align*}
    \|\bw^{\sst}_{\AGD}-\bw^*\|_2^2
    &\leq(1+\kappa)\exp(-\frac{\sst}{\sqrt{\kappa}}) \|\bw^{0}_{\AGD}-\bw^*\|_2^2,
    \\
    L(\bw^{\sst}_{\AGD})-L(\bw^*)
    &\leq \frac{\alpha+\beta}{2} \exp(-\frac{\sst}{\sqrt\kappa})\|\bw^{0}_{\AGD}-\bw^*\|_2^2.
\end{align*}
% Moreover, if $L(\bw)$ is quadratic in $\bw$, then $ \|\bw^{t}_{\AGD}-\bw^*\|_2
%     \leq\|\bw^{0}_{\AGD}-\bw^*\|_2$ for all $t\geq1$ (see e.g.,~\cite{assran2020convergence,attia2021algorithmic} for stability results of the AGD trajectory).
\end{enumerate}
    
\end{proposition}

%% file: Sections_arxiv/app-supervised-pretraining_arxiv.tex
\section{Proofs in Section \ref{sec:supervised-pretraining}}

In this section, $c>0$ denotes universal constants that may differ across equations.

\subsection{Proof of Theorem~\ref{thm:diff_reward}}\label{sec:pf_thm:diff_reward}
\paragraph{Proof of Eq.~(\ref{eqn:Hellinger-bound-main-theorem})} Note that we have
\begin{align*}
&~
\sum_{t=1}^\totlen\E_{\inst\sim\prior,\action_{1:t-1}\sim\osAlg_\shortexp,\state_t}\sqrt{\HelDs(\osAlg_{\shortexp}(\cdot|\dset_{t-1},\state_t),\sAlg_{\EstPar}(\cdot|\dset_{t-1},\state_t))}\\
=&~
\sum_{t=1}^\totlen\E_{\inst\sim\prior,\action_{1:t-1}\sim\sAlg_0,\state_t} \Big[
\Big(\prod_{s=1}^{t-1}\frac{\osAlg_\shortexp(\action_s|\dset_{s-1},\state_s)}{\sAlg_0(\action_s|\dset_{s-1},\state_s)}\Big)
\cdot
\sqrt{\HelDs(\osAlg_{\shortexp}(\cdot|\dset_{t-1},\state_t),\sAlg_{\EstPar}(\cdot|\dset_{t-1},\state_t))}\Big] \\
\leq&~
\sum_{t=1}^\totlen
\sqrt{\E_{\prior,\sAlg_0}
\Big(\prod_{s=1}^{t-1}\frac{\osAlg_\shortexp(\action_s|\dset_{s-1},\state_s)}{\sAlg_0(\action_s|\dset_{s-1},\state_s)}\Big)^2\cdot\E_{\prior,\sAlg_0}\HelDs(\osAlg_{\shortexp}(\cdot|\dset_{t-1},\state_t),\sAlg_{\EstPar}(\cdot|\dset_{t-1},\state_t))}\\
\leq &~
\sqrt{\E_{\prior,\sAlg_0}
\Big(\prod_{s=1}^{\totlen}\frac{\osAlg_\shortexp(\action_s|\dset_{s-1},\state_s)}{\sAlg_0(\action_s|\dset_{s-1},\state_s)}\Big)^2}
\cdot
\sum_{t=1}^\totlen\sqrt{\E_{\prior,\sAlg_0}\HelDs(\osAlg_{\shortexp}(\cdot|\dset_{t-1},\state_t),\sAlg_{\EstPar}(\cdot|\dset_{t-1},\state_t))},\end{align*}
where the second  line follows from a change of distribution argument, the third line follows from Cauchy-Schwartz inequality,  and the fourth line uses the fact that
\begin{align*}
&~\E_{x,y\sim\P_1\cdot\P_2}\Big(\frac{\Q_1(x)\Q_2(y|x)}{\P_1(x)\P_2(y|x)}\Big)^2
    =
    \int\frac{\Q_1(x)^2\Q_2^2(y|x)}{\P_1(x)\P_2(y|x)}d\mu(x,y)\\
    =&~
    \int\frac{\Q_1(x)^2}{\P_1(x)}\Big(\int\frac{\Q_2^2(y|x)}{\P_2(y|x)}d\mu(y|x)\Big)d\mu(x)
    \geq  
    \int\frac{\Q_1(x)^2}{\P_1(x)}d\mu(x)= \E_{x\sim\P_1}\Big(\frac{\Q_1(x)}{\P_1(x)}\Big)^2,
\end{align*}
for any probability densities $\{ \Q_i,\P_i \}_{i=1,2}$ with respect to some base measure $\mu$. 

Continuing the calculation of the above lines of bounds, we have
\begin{align*}
&\qquad
\sum_{t=1}^\totlen\E_{\inst\sim\prior,\action_{1:t-1}\sim\osAlg_\shortexp,\state_t}\sqrt{\HelDs(\osAlg_{\shortexp}(\cdot|\dset_{t-1},\state_t),\sAlg_{\EstPar}(\cdot|\dset_{t-1},\state_t))}\\
\leq&~
\sqrt{\totlen}\sqrt{\E_{\inst\sim\prior,\action_{1:\totlen-1}\sim\sAlg_0,\state_t}
\Big(\prod_{s=1}^{\totlen}\frac{\osAlg_\shortexp(\action_s|\dset_{s-1},\state_s)}{\sAlg_0(\action_s|\dset_{s-1},\state_s)}\Big)^2}
\\
&\qquad\qquad\cdot\sqrt{\sum_{t=1}^\totlen{\E_{\inst\sim\prior,\action_{1:t-1}\sim\sAlg_0,\state_t}\HelDs(\osAlg_{\shortexp}(\cdot|\dset_{t-1},\state_t),\sAlg_{\EstPar}(\cdot|\dset_{t-1},\state_t))}}\\
=&~
\sqrt{\totlen}\sqrt{\E_{\inst\sim\prior,\action_{1:\totlen-1}\sim\osAlg_\shortexp,\state_t}
\Big[\prod_{s=1}^{\totlen}\frac{\osAlg_\shortexp(\action_s|\dset_{s-1},\state_s)}{\sAlg_0(\action_s|\dset_{s-1},\state_s)}\Big]}
\\
&\qquad\qquad\cdot\sqrt{\sum_{t=1}^\totlen{\E_{\inst\sim\prior,\action_{1:t-1}\sim\sAlg_0,\state_t}\HelDs(\osAlg_{\shortexp}(\cdot|\dset_{t-1},\state_t),\sAlg_{\EstPar}(\cdot|\dset_{t-1},\state_t))}}
\\
\leq&~
c{\totlen}\sqrt{\distratio_{\osAlg_\shortexp,\sAlg_0}}
\sqrt{\frac{\log \cN_{\Parspace}(1/(\Numobs\totlen)^2)+\log(\totlen/\delta)}{n}+\geneps}\\
\leq&~ c
{\totlen}\sqrt{\distratio_{\osAlg_\shortexp,\sAlg_0}}
\Big(\sqrt{\frac{\log [\cN_{\Parspace}(1/(\Numobs\totlen)^2)\totlen/\delta]}{n}}+\sqrt{\geneps}\Big),
\end{align*} 
where the first inequality follows from the Cauchy-Schwartz inequality, the first equality is due to a change of distribution argument, the second inequality uses Lemma~\ref{lm:general_imit}. This completes the proof of Eq.~(\ref{eqn:Hellinger-bound-main-theorem}).

\paragraph{Proof of Eq.~(\ref{eqn:reward-bound-main-theorem})}

For any bounded function $f$ such that $|f(\dset_\totlen)|\leq F$ for some $F>0$,  we have
\begin{align*}
  &~ \Big| \E_{\inst\sim\prior,\action\sim\osAlg_\shortexp}[f(\dset_\totlen)]-
   \E_{\inst\sim\prior,\action\sim\sAlg_\EstPar}[f(\dset_\totlen)]\Big|
   \\
   =&~
\Big| \sum_{t=1}^\totlen\E_{\inst\sim\prior,\action_{1:t}\sim\osAlg_\shortexp,\action_{t+1:\totlen}\sim\sAlg_\EstPar}[f(\dset_\totlen)]-
\E_{\inst\sim\prior,\action_{1:t-1}\sim\osAlg_{\shortexp},\action_{t:\totlen}\sim\sAlg_\EstPar}[f(\dset_\totlen)]\Big|
   \\
\leq&~
2F\sum_{t=1}^\totlen\E_{\inst\sim\prior,\action_{1:t-1}\sim\osAlg_{\shortexp},\state_t}\VarD(\osAlg_{\shortexp}(\cdot|\dset_{t-1},\state_t),\sAlg_{\EstPar}(\cdot|\dset_{t-1},\state_t)),
\end{align*}
where the first equality uses the performance difference lemma, the last line follows from the variational representation of the total variation distance $$\VarD(\sP,\sQ)=\sup_{\linf{f}=1}\E_\sP[f(X)]/2-\E_\sQ[f(X)]/2,$$  and \begin{align}
\VarD(\sP_1(x)\sP_2(y\mid x)\sP_3(z\mid y),\sP_1(x)\sP_{4}(y\mid x)\sP_3(z\mid y))=\E_{x\sim\sP_1}\VarD(\sP_2(y\mid x),\sP_{4}(y\mid x))\label{eq:kl_telescope}
\end{align} 
for probability densities $\{ \P_i\}_{i=1,2,3,4}$ with respect to some base measure $\mu$. Since $\VarD(\P,\Q)\leq\sqrt{\HelDs(\P,\Q)}$ for any distributions $\P,\Q$,  it follows from Eq.~(\ref{eqn:Hellinger-bound-main-theorem}) that
\begin{align*}
 &~ \Big| \E_{\inst\sim\prior,\action\sim\osAlg_{\shortexp}}[f(\dset_\totlen)]-
   \E_{\inst\sim\prior,\action\sim\sAlg_\EstPar}[f(\dset_\totlen)]\Big|\\
%    \leq &~
% cF\sqrt{\distratio_{\osAlg_{\shortexp},\sAlg_0}}\cdot\totlen\sqrt{\frac{\log \brac{ \cN_{\Parspace}(1/(\Numobs\totlen)^2) \totlen/\delta } }{n} + \geneps}\\
  \leq &~cF\sqrt{\distratio_{\osAlg_{\shortexp},\sAlg_0}}\cdot\totlen\Big(\sqrt{\frac{\log \brac{ \cN_{\Parspace}(1/(\Numobs\totlen)^2) \totlen/\delta } }{n}} + \sqrt{\geneps}\Big)
\end{align*}
with probability at least $1-\delta$ for some universal constant $c>0$. Letting $f(\dset_\totlen)=\sum_{t=1}^\totlen\reward_t$  and noting that  $|f(\dset_\totlen)|\leq\totlen$ concludes the proof of Theorem~\ref{thm:diff_reward}.

\subsection{Proof of Proposition~\ref{prop:app_opt_diff_reward}}\label{app:proof-prop-diff-reward-app-opt}

By the jointly convexity of $\KL{\P}{\Q}$ with respect to $(\P,\Q)$ and the fact that $\HelDs(\P,\Q)\leq\KL{\P}{\Q}$, we have 
\begin{align*}
  &~\E_{\dset_{t-1},\state_t\sim\P_\prior^{\sAlg_0}}\HelDs(\osAlg_{\shortexp}(\cdot|\dset_{t-1},\state_t),\P_{\TS}(\cdot|\dset_{t-1},\state_t))\\
  \leq &~
  \E_{\dset_{t-1},\state_t\sim\P_\prior^{\sAlg_0}}\KL{\osAlg_{\shortexp}(\cdot|\dset_{t-1},\state_t)}{\P_{\TS}(\cdot|\dset_{t-1},\state_t)}\\
   \leq &~
 \E_{\dset_\totlen\sim\P_{\prior}^{\sAlg_0}}\KL{\widehat\action^*_t}{\P_{\TS,t}(\cdot|\dset_\totlen)}\leq\appeps.
\end{align*} Therefore, applying Lemma~\ref{lm:general_imit} gives
 \begin{align*}
        &~\E_{\inst\sim \prior, \dset_\totlen\sim \P^{\sAlg_0}_\inst}\brac{ \sum_{t=1}^\totlen \HelDs\paren{ \sAlg_{{\EstPar}}(\cdot|\dset_{t-1},\state_t ), \sAlg_{\TS}(\cdot|\dset_{t-1},\state_t )} }\\
        \le&~
        2\E_{\inst\sim \prior, \dset_\totlen\sim \P^{\sAlg_0}_\inst}\brac{ \sum_{t=1}^\totlen \HelDs\paren{ \sAlg_{{\EstPar}}(\cdot|\dset_{t-1},\state_t ), \osAlg_{\shortexp}(\cdot|\dset_{t-1},\state_t )}+\HelDs\paren{ \sAlg_{{\shortexp}}(\cdot|\dset_{t-1},\state_t ), \sAlg_{\TS}(\cdot|\dset_{t-1},\state_t )} }\\
       \le&~
       c\Big(\frac{\totlen \log \brac{ \cN_{\Parspace}(1/(\Numobs\totlen)^2) \totlen/\delta } }{n} + \totlen(\geneps+\appeps)\Big)
    \end{align*} 
    with probability at least $1-\delta$. Proposition~\ref{prop:app_opt_diff_reward} follows from similar arguments as in the proof of Theorem~\ref{thm:diff_reward} with $\geneps$ replaced by $\geneps+\appeps$. 

\subsection{An auxiliary lemma}

\begin{lemma}[General guarantee for supervised pretraining]\label{lm:general_imit}
Suppose Assumption~\ref{asp:realizability} holds. Then  the solution to~Eq.~\eqref{eq:general_mle} achieves
\begin{align*}
\E_{\dset_\totlen\sim \P^{\sAlg_0}_\prior}\brac{ \sum_{t=1}^\totlen \HelDs\paren{ \sAlg_{{\EstPar}}(\cdot|\dset_{t-1},\state_t ), \osAlg_{\shortexp}(\cdot|\dset_{t-1},\state_t )} } \le c\frac{\totlen \log \brac{ \cN_{\Parspace}(1/(\Numobs\totlen)^2) \totlen/\delta } }{n} + \totlen\geneps.
\end{align*}
with probability at least $1-\delta$ for some universal constant $c>0$.
\end{lemma}
% See the proof in Section~\ref{sec:pf_lm:general_imit}. 

\begin{proof}[Proof of Lemma~\ref{lm:general_imit}]~

Define \begin{align*}\cL_{nt}(\btheta):=\sum_{i=1}^n\log\sAlg_\Par(\eaction^\ith_{t}|\dset_{t-1}^\ith,\state^\ith_t),~~\text{ and }~~\cL_{nt}(\expert):=\sum_{i=1}^n\log\osAlg_{\shortexp}(\eaction^\ith_{t}|\dset_{t-1}^\ith,\state^\ith_t),\end{align*}  and let 
$\cL_n(\btheta)=\sum_{t=1}^\totlen \cL_{nt}(\btheta)$, $\cL_n(\expert)=\sum_{t=1}^\totlen \cL_{nt}(\expert)$. We claim that   with probability at least $1-\delta$ 
\begin{align}
&~ \sum_{t=1}^\totlen\E_{\dset_\totlen}\Big[\HelDs(\sAlg_\Par(\cdot|\dset_{t-1},\state_t),\osAlg_{\shortexp}(\cdot|\dset_{t-1},\state_t))\Big]\notag\\
\leq &~
\frac{\cL_{n}(\expert)-\cL_{n}(\Par)}{n}+2\frac{\totlen\log \cN_{\Parspace}(1/(\Numobs\totlen)^2)}{n}+2\frac{\totlen\log(\totlen/\delta)}{n}+\frac{4}{n}\label{eq:pf_hellinger_control_general}
\end{align}
for all $\Par\in\Parspace,i\in[T]$, where $\dset_\totlen$ follows  distribution $\P_\inst^{\sAlg_0}(\cdot)$, $\inst\sim\prior$. For now, we assume this claim holds. Moreover, it follows from Lemma~\ref{lm:exp_concen} and the fact $\cL_\Numobs(\EstPar)\geq\cL_\Numobs(\TruePar)$ that
\begin{align}
    \frac{\cL_\Numobs(\expert)-\cL_\Numobs(\EstPar)}{\Numobs}
    &\leq
      \frac{\cL_\Numobs(\expert)-\cL_\Numobs(\TruePar)}{\Numobs}=\sum_{t=1}^\totlen
      \frac{\cL_{\Numobs t}(\expert)-\cL_{\Numobs t}(\TruePar)}{\Numobs}\notag\\
      &\leq
      \frac{T\log(T/\delta)}{\Numobs}+\sum_{t=1}^\totlen\log\E_{\adset_\totlen}\Big[{\frac{\osAlg_{\shortexp}(\eaction_t|\dset_{t-1},\state_t)}{\sAlg_\TruePar(\eaction_t|\dset_{t-1},\state_t)}}\Big]\notag\\
&\leq
 \frac{T\log(T/\delta)}{\Numobs}+\totlen\geneps
\label{eq:pf_hellinger_control_general2}
\end{align}
with probability at least $1-\delta$.

Choosing $\Par=\EstPar$ in Eq.~\eqref{eq:pf_hellinger_control_general} and combining it with Eq.~\eqref{eq:pf_hellinger_control_general2} and a union bound, we obtain
\begin{align*}
&~\sum_{t=1}^\totlen\E_{\dset_\totlen}\Big[\HelDs(\sAlg_\EstPar(\cdot|\dset_{t-1},\state_t),\osAlg_{\shortexp}(\cdot|\dset_{t-1},\state_t))\Big]\\
\leq &~
\totlen\geneps
+
2\Big(\frac{\totlen\log \cN_{\Parspace}(1/(\Numobs\totlen)^2)+2\totlen\log(2\totlen/\delta)+2}{n}\Big)\\
\leq&~ \totlen\geneps+c\totlen\Big(\frac{\log \cN_{\Parspace}(1/(\Numobs\totlen)^2)+\log(\totlen/\delta)}{n}\Big)
\end{align*}
with probability at least $1-\delta$ for some universal constant $c>0$.  This completes the proof. 

\paragraph{Proof of Eq.~\eqref{eq:pf_hellinger_control_general}}
Let $\Parspace_{0}$ be a $1/(\Numobs\totlen)^2$-covering set of $\Parspace$ with covering number $\Covnum=|\Parspace_{i}|$. 
For $k\in[\Covnum],t\in[T],i\in[\Numobs]$, define  $$
\ell^i_{kt}
=\log 
\frac{\osAlg_{\shortexp}(\eaction^\ith_t|\dset_{t-1}^\ith,\state^\ith_t)}{\sAlg_{\Par_k}(\eaction^\ith_t|\dset_{t-1}^\ith,\state^\ith_t)}
,$$
where $(\dset_\totlen^\ith,\eaction^\ith)$ are the trajectory and expert actions collected in the $i$-th instance. Using Lemma~\ref{lm:exp_concen} with $X_s=-\ell^s_{kt}$ and a union bound over $(k,t)$, conditioned on the trajectories $(\dset^1_\totlen,\ldots,\dset^{n}_\totlen)$, we have
\begin{align*}
    \frac{1}{2}\sum_{i=1}^\Numobs \ell_{kt}^i+\log(\Covnum \totlen/\delta)\geq
    \sum_{i=1}^n-\log\E\Big[\exp\Big(-\frac{\ell_{kt}^{i}}{2}\Big)\Big]
\end{align*}
for all $k\in[\Covnum],t\in[\totlen]$
with probability at least $1-\delta$. Note that
\begin{align*}
   \E\Big[\exp\Big(-\frac{\ell_{kt}^{i}}{2}\Big)\Big|\dset_{t-1}^\ith,\state_t^\ith\Big]
   =&~
\E_{\sD}\Bigg[\sqrt{\frac{\sAlg_{\Par_k}(\eaction^\ith_t|\dset_{t-1}^\ith,\state^\ith_t)}{\osAlg_{\shortexp}(\eaction^\ith_t|\dset_{t-1}^\ith,\state^\ith_t)}}\Bigg|\dset_{t-1}^\ith,\state_t^\ith\Bigg]\\
=&~
\sum_{\action\in\actionsp_t}\sqrt{\sAlg_{\Par_k}(\action|\dset_{t-1}^\ith,\state^\ith_t) \osAlg_{\shortexp}(\action|\dset_{t-1}^\ith,\state^\ith_t)},
\end{align*}
where the last inequality uses the assumption that the actions $\eaction^\ith$ are generated using the expert $\osAlg_{\shortexp}(\cdot|\dset_{t-1}^\ith,\state^\ith_t)$. Therefore, for any $\btheta\in\Theta$ covered by $\btheta_k$, we have
\begin{align*}
&~
    -\log\E\Big[\exp\Big(-\frac{\ell_{kt}^{i}}{2}\Big)\Big]\\
    \geq&~ 1-
 \E_{\dset^\ith}\Big[\sum_{\action\in\actionsp_t}\sqrt{\sAlg_{\Par_k}(\action|\dset_{t-1}^\ith,\state^\ith_t) \osAlg_{\shortexp}(\action|\dset_{t-1}^\ith,\state^\ith_t)}\Big]
    \\
    =&~
    1-
\E_{\dset^\ith}\Big[\sum_{\action\in\actionsp_t}\sqrt{\sAlg_{\Par}(\action|\dset_{t-1}^\ith,\state^\ith_t) \osAlg_{\shortexp}(\action|\dset_{t-1}^\ith,\state^\ith_t)}\Big]\\
&\qquad~- \E_{\dset^\ith}\Big[\sum_{\action\in\actionsp_t}\sqrt{\osAlg_{\shortexp}(\action|\dset_{t-1}^\ith,\state^\ith_t)}\Big(\sqrt{\sAlg_{\Par_k}(\action|\dset_{t-1}^\ith,\state^\ith_t)}-\sqrt{\sAlg_{\Par}(\action|\dset_{t-1}^\ith,\state^\ith_t)}\Big)\Big]
\\
\geq &~
    \frac{1}{2} \E_{\dset^\ith}\Big[\HelDs(\osAlg_{\shortexp}(\cdot|\dset_{t-1}^\ith,\state^\ith_t),\sAlg_{\Par}(\cdot|\dset_{t-1}^\ith,\state^\ith_t))\Big] 
    \\
    &\qquad~-\E_{\dset^\ith}\Big[\sum_{\action\in\actionsp}\Big(\sqrt{\sAlg_{\Par}(\cdot|\dset_{t-1}^\ith,\state^\ith_t)}-\sqrt{\sAlg_{\Par_k}(\cdot|\dset_{t-1}^\ith,\state^\ith_t)}\Big)^2\Big]^{1/2}
\\
\geq &~
 \frac{1}{2} \E_{\dset^\ith}\Big[\HelDs(\osAlg_{\shortexp}(\cdot|\dset_{t-1}^\ith,\state^\ith_t),\sAlg_{\Par}(\cdot|\dset_{t-1}^\ith,\state^\ith_t))\Big] 
    -\|\sAlg_{\Par}(\cdot|\dset_{t-1}^\ith,\state^\ith_t)-\sAlg_{\Par_k}(\cdot|\dset_{t-1}^\ith,\state^\ith_t)\|_1^{1/2}
\\
\geq&~
 \frac{1}{2} \E_{\dset^\ith}\Big[\HelDs(\osAlg_{\shortexp}(\cdot|\dset_{t-1}^\ith,\state^\ith_t),\sAlg_{\Par}(\cdot|\dset_{t-1}^\ith,\state^\ith_t))\Big] 
  -\frac{\sqrt{2}}{\Numobs\totlen}
\end{align*}
for all $i\in[n],t\in[\totlen]$, 
where the first inequality uses $-\log x\geq 1-x$, the second inequality follows from Cauchy-Schwartz inequality, the third inequality uses $(\sqrt{x}-\sqrt{y})^2\leq |x-y|$ for $x,y\geq0$, the last inequality uses the fact that $\Par$ is covered by $\Par_k$ and Lemma~\ref{lm:cover_num_corr}. Since any $\Par\in\Parspace$ is covered by $\Par_k$ for some $k\in[\Covnum]$, and for this $k$ summing over $t\in[T]$ gives 
$$\sum_{i=1}^\Numobs\sum_{t=1}^\totlen\ell_{kt}^i=\cL_\Numobs(\expert)-\cL_\Numobs(\Par_k) \leq \cL_\Numobs(\expert)-\cL_\Numobs(\Par)+\frac{1}{\Numobs\totlen}\leq \cL_\Numobs(\expert)-\cL_\Numobs(\Par)+1.$$ 
Therefore, with probability at least $1-\delta$, we have
\begin{align*}
&~\frac{1}{2}\Big(\cL_\Numobs(\expert)-\cL_\Numobs(\Par)+1 \Big)+\totlen\log(\Covnum \totlen/\delta)+\sqrt{2}\\
\geq&~ \frac{\Numobs}{2}\sum_{t=1}^\totlen\E_{\dset_\totlen}\Big[\HelDs(\sAlg_\Par(\cdot|\dset_{t-1},\state_t),\sAlg_\expert(\cdot|\dset_{t-1},\state_t))\Big]
\end{align*}
for all $\Par\in\Parspace$, where $\dset_\totlen$  follows $\P_\prior^{\sAlg_0}$. Multiplying both sides by $2/\Numobs$ and letting $\Covnum=\cN_{\Parspace}(1/(\Numobs\totlen)^2)$ yields Eq.~\eqref{eq:pf_hellinger_control_general}.
\end{proof}

%% file: Sections_arxiv/app-linucb_arxiv.tex
\section{Soft LinUCB for linear stochastic bandit}\label{app:linUCB}
Throughout this section, we use $c>0$ to denote universal constants whose values may vary from line to line.
Moreover, for notational simplicity, we use $\conO(\cdot)$ to hide universal constants, $\cO(\cdot)$ to hide polynomial terms in the problem parameters $(\sigma,b_a^{-1},B_a,B_w,\lambda^{\pm1})$, and $\tcO(\cdot)$ to hide both poly-logarithmic terms in $(T,A,d,1/\eps,1/\temp)$ and  polynomial terms in  $(\sigma,b_a^{-1},B_a,B_w,\lambda^{\pm1})$. We also use the bold font $\ba_t\in\R^d$ to denote the selected action vector $\action_t$ at time $t\in[\totlen]$.

This section is organized as follows. Section~\ref{sec:tf_embed_bandit} discusses the embedding and extraction formats of transformers for the stochastic linear bandit environment. Section~\ref{sec:soft-LinUCB} describes the LinUCB and the soft LinUCB algorithms. Section~\ref{app:approx-ridge-estimator} introduces and proves a lemma on approximating the linear ridge regression estimator, which is important for proving Theorem~\ref{thm:approx_smooth_linucb}. We prove Theorem~\ref{thm:approx_smooth_linucb} in Section~\ref{sec:pf_thm:approx_smooth_linucb} and prove Theorem~\ref{thm:smooth_linucb} in Section~\ref{sec:pf_thm:smooth_linucb}.

\subsection{Embedding and extraction mappings}\label{sec:tf_embed_bandit}

% \sm{Explicitly write the operator $\embedmap$, $\extractmap$, $\cat$} 
Consider the embedding in which for each $t\in[\totlen]$, we have two tokens $\bh_{2t-1},\bh_{2t}\in\R^D$ such that
\[
\begin{aligned}
\bh_{2t-1}=
\left[
\begin{array}{c}
     \bzero_{d+1} \\
     \hdashline
     \sA_t\\  
     \hdashline
     \bzero_A\\  
     \hdashline
      \bzero\\ \posv_{2t-1}
\end{array}
\right]
% \begin{bmatrix}
%      \bzero_{d+1} \\
%      \hdashline
%      \sA_t\\  
%      \hdashline
%      \bzero_A\\  
%      \hdashline
%       \bzero\\ \posv_{2t-1}
% \end{bmatrix}
=:
\begin{bmatrix}
     \bh_{2t-1}^{\parta} \\  \bh_{2t-1}^{\partb}\\  \bh_{2t-1}^{\partc}\\   \bh_{2t-1}^{\partd}\\
\end{bmatrix},~~
\bh_{2t}=
\left[
\begin{array}{cc}
     \ba_{t} \\
      r_t\\  
      \hdashline
       \bzero_{Ad}\\ 
       \hdashline 
       \bzero_{A}\\ 
       \hdashline  
       \bzero\\ \posv_{2t}
\end{array}
\right]=:
\begin{bmatrix}
    \bh_{2t}^{\parta} \\  \bh_{2t}^{\partb}\\   \bh_{2t}^{\partc}\\   \bh_{2t}^{\partd}
\end{bmatrix},
\end{aligned}
\]
where $\bh_{2t-1}^{\partb}=\sA_t=\begin{bmatrix}
    \ba_{t,1}^\top &\ldots & \ba_{t,A}^\top
\end{bmatrix}^\top$ denotes the action set at time $t$, $\bh_{2t}^{\parta}=\begin{bmatrix}
    \ba_t^\top &r_t
\end{bmatrix}^\top$ denotes the action and the observed reward at time $t$, $\bh^\partc_{2t-1}$ is used to store the (unnormalized) policy at time  $t$, $\bzero$ in $\bh^\partd$ denotes an additional zero vector with  dimension $\conO(dA)$, and $\posv_i:=(i,i^2,1)^\top$ for $i\in[2\totlen]$ is the positional embedding.    Note that the token dimension $D= O(dA)$. In addition, we define the token matrix $\bH_t:=\begin{bmatrix}
    \bh_1,\ldots,\bh_{2t}
\end{bmatrix}\in\R^{D\times 2t}$ for all $t\in[\totlen]$.

\paragraph{Offline pretraining} 
During pretraining, the transformer $\TF_\tfpar$ takes in   $\bH_\totlen^\pre:=\bH_\totlen$ as the input token matrix and generates $\bH_\totlen^\post:=\TF_\tfpar(\bH_\totlen^\pre)$ as the output. For each step $t\in[\totlen]$, we define the  induced policy  $\sAlg_\tfpar(\cdot|\dset_{t-1},\state_t):=\frac{\exp(\bh^{\post,\partc}_{2t-1})}{\|\exp(\bh^{\post,\partc}_{2t-1})\|_1}\in\Delta^A$, whose $i$-th entry is the probability of selecting action $\ba_{t,i}$ given $(\dset_{t-1},\state_t)$. We then find the transformer $\esttfpar\in\tfparspace$ by solving Eq.~\eqref{eq:general_mle}. 
Due to the decoder structure of transformer $\TF_\tfpar$, the $2t-1$-th token only has access to the first $2t-1$ tokens. Therefore the induced policy is  determined by the historical data $(\dset_{t-1},\state_t)$ and does not depend on future observations. 
\paragraph{Rollout}
At each time $t\in[\totlen]$, given the action set $\sA_t$ (i.e., current state $\state_t$) and the previous data $\dset_{t-1}$, we first construct the token matrix $\bH^{\pre}_{\roll,t}=[\bH_{t-1},\bh_{2t-1}]\in\R^{D\times (2t-1)}$.   The transformer then takes $\bH^{\pre}_{\roll,t}$ as the input  and generates $\bH^{\post}_{\roll,t}=[\bH^{\post}_{t-1},\bh^{\post}_{2t-1}]=\TF_\tfpar(\bH^{\pre}_{\roll,t})$. Next,  the agent selects an action $\ba_t\in\sA_t$ according to the induced  policy $\sAlg_\tfpar(\cdot|\dset_{t-1},\state_t):=\frac{\exp(\bh^{\post,\partc}_{2t-1})}{\|\exp(\bh^{\post,\partc}_{2t-1})\|_1}\in\Delta^A$ and observes the reward $r_t$.

\paragraph{Embedding and extraction mappings}
To integrate the above construction into the  framework described in Section~\ref{sec:framework},  we have the embedding vectors $\embedmap(\state_t):=\bh_{2t-1},\embedmap(\action_t,\reward_t):=\bh_{2t}$,  the concatenation operator $\cat(\bh_1, \ldots, \bh_N): = [\bh_1, \ldots, \bh_N]$, the input token matrix  $$\bH=\bH^\pre_{\roll,t}: = \cat(\embedmap(\state_1), \embedmap(\action_1, \reward_1), \ldots, \embedmap(\action_{t-1}, \reward_{t-1}), \embedmap(\state_t)) \in \R^{D \times (2t-1)},$$ the output token matrix $\bar{\bH}=\bH^\post_{\roll,t}$, and the linear extraction map $\extractmap$  satisfies $\extractmap\cdot\bar{\bh}_{-1}=\extractmap\cdot\bar{\bh}^\post_{2t-1}=\bh^{\post,\partc}_{2t-1}$.

\subsection{LinUCB and soft LinUCB}\label{sec:soft-LinUCB}
Let $T$ be the total time and $\lambda,\cwid>0$ be some prespecified values. At each time $t\in[T]$, LinUCB consists of the following steps: 
\begin{enumerate}
    \item Computes the ridge estimator $\bw^t_{\ridge,\lambda}=\argmin_{\bw\in\R^d}\frac{1}{2t}\sum_{j=1}^{t-1}(r_j-\<\ba_j,\bw\>)^2+\frac{\lambda}{2t}\|\bw\|_2^2$.
    \item For each action $k\in[A]$, computes $v^*_{tk}:=\<\ba_{t,k},\bw^t_{\ridge,\lambda}\>+\cwid\sqrt{\ba_{t,k}^\top \bA_{t}^{-1}  \ba_{t,k}}$, where $\bA_t=\lambda\id_d+\sum_{j=1}^{t-1}\ba_j\ba_j^\top$.
    \item Selects the action $\ba_{t,j}$ with $j:=\argmax_{k\in[A]}v^*_{tk}$.
\end{enumerate}
Unless stated otherwise, in step 2 above we choose $\alpha=\alpha(\delta)$ with $\delta=1/(2B_aB_wT)$ and $$\cwid(\delta):=\sqrt{\lambda}B_w+\sigma\sqrt{2\log(1/\delta)+d\log((d\lambda+TB_a^2)/(d\lambda))}=\cO(\sqrt{d\log \totlen})=\tcO(\sqrt{d}).$$

In this work, to facilitate the analysis of supervised pretraining, we consider soft LinUCB (denoted by $\sLinUCB(\temp)$), which replaces step 3 in LinUCB with
\begin{enumerate}
    \item [3'] Selects the action $\ba_{t,j}$ with probability $\frac{\exp(v^*_{tj}/\temp)}{\lone{\exp(v^*_{tj}/\temp)}}$ for $j\in[A]$. 
\end{enumerate} Note that soft LinUCB recovers the standard LinUCB as $\temp\to0$.

% By Theorem 19.2, Corollary 19.3 in Lattimore and Szepesvari~\cite{lattimore2020bandit}, we have the regret $$\bR(T)=\sum_{t=1}^T \max_{j\in[A]}\<\ba_{t,j},\bw^*\>-\<\ba_t,\bw^*\>$$ satisfies $\bR(T)\leq\sqrt{8 dT\alpha\log((d\lambda+TB_a^2)/(d\lambda))}$ with probability over $1-\delta$. Choosing $\delta=1/T$ gives $$\bR(T)\leq Cd\sqrt{T}\log(T)$$ for some  problem-dependent constant $C>0$.

% We will show that for any $\eps>0$ there exists a transformer that can perform an approximate LinUCB algorithm such that $\bR(T)\leq\sqrt{8 dT\alpha^2\log((d\lambda+TB_a^2)/(d\lambda))}+\eps T$.

\subsection{Approximation of the ridge estimator}\label{app:approx-ridge-estimator}
In this section, we present a lemma on how transformers can  approximately implement the ridge regression estimator in-context. 

% \begin{theorem}[Approximated LinUCB]\label{thm:approx_linucb}
% % Let $R=2\max\{(B_a+\alpha/\sqrt{\lambda})\}$.
% For any small $\eps,\delta>0$, there exists a  transformer $\DTF_\btheta(\cdot)$ with 
% $$L=\tcO(\sqrt{T}),~~\max_{\ell\in[L]}M^{(l)}\leq4A,~~~ \nrmp{\btheta}\leq \tcO(A+T^{3/2}(1+\alpha/\eps))=\tcO(A+T^{3/2}\sqrt{d+\log(1/\delta)}/\eps)  $$ such that 
% $|v^*_{tk}-\max_{j\in[k]}v_{tj}^*|\leq \eps$ for all $t\in[T]$ and  $k$ such that $p_{tk}>0$. 
% Here $\tcO(\cdot)$ hides constants and logarithmic dependence on $T,A,\log(1/\delta),1/\eps$.

% \end{theorem}

% See the proof in Section~\ref{sec:pf_thm:approx_linucb}.

% \begin{corollary}[Regret analysis of approximated LinUCB ]\label{cor:approx_linucb}
%     For any small $\eps,\delta>0$, there exists a  transformer $\DTF_\btheta(\cdot)$ with 
% $$L=\tcO(\sqrt{T}),~~\max_{\ell\in[L]}M^{(l)}\leq4A,~~~ \nrmp{\btheta}\leq \tcO(A+T^{3/2}(1+\alpha/\eps))=\tcO(A+T^{3/2}\sqrt{d+\log(1/\delta)}/\eps)  $$ that approximately implements LinUCB such that the regret
% $$\bR(T)\leq\sqrt{8 dT\alpha^2\log((d\lambda+TB_a^2)/(d\lambda))}+\eps T$$
% with probability over $1-\delta$. Moreover, choosing $\eps=1/\sqrt{T},\delta=1/{T}$ gives a transformer the implements approximated  LinUCB with the regret
% $$\bR(T)\leq Cd\sqrt{T}\log(T)$$ for some problem-dependent constant $C>0$, with probability over $1-1/T$.
% \end{corollary}

% See the proof in Section~\ref{sec:pf_cor:approx_linucb}.

Throughout the proof, for $t\in[2\totlen]$,  we let $\bh_{t}^{(L)}$ denote the $i$-th token in the output token matrix obtained after passing through an $L$-layer transformer. We also define  $\read_{\bw_\ridge}: \R^{D}\mapsto\R^{d}$ be the operator that gives the values of  $d$ coordinates  in the token vector that are used to store the estimation of  the ridge estimate.

\begin{lemma}[Approximation of the ridge estimator]\label{lm:approx_ridge}
For any small $\eps>0$, there exists an attention-only (i.e., no MLP layers) transformer $\TF_\btheta(\cdot)$ with 
$$L=\Big\lceil\frac{4T(B_a^2+\lambda)}{\lambda}\log({TB_a(B_aB_w+\sigma)}/({\lambda}\eps))\Big\rceil=\tcO(T),~~~\max_{\ell\in[L]}M^{(l)}\leq3,~~~ \nrmp{\btheta}\leq  \sqrt{2}+\frac{\lambda+2}{B_a^2+\lambda}=\cO(1)$$ such that $\|\read_{\bw_{\ridge}}(\bh_{2t-1}^{(L)})-\bw^t_{\ridge,\lambda}\|_2\leq\eps$ for all $t\in[T]$. 

Moreover, there exists a  transformer $\TF_\btheta(\cdot)$ with  \begin{align*}&L=\Big\lceil2\sqrt{2T}\sqrt{\frac{B_a^2+\lambda}{\lambda}}\log\Big(\frac{(2T(B_a^2+\lambda)+\lambda)TB_a(B_aB_w+\sigma)}{\lambda^2\eps}\Big)\Big\rceil=\tcO(\sqrt{T}),~~~\max_{\ell\in[L]}M^{(l)}\leq4,~~~ \\
&~~~~~~\max_{\ell\in[L]}\hidden^{\lth}\leq 4d,~~~\nrmp{\btheta}\leq  10+\frac{\lambda+2}{B_a^2+\lambda}=\cO(1) \end{align*}
 such that $\|\read_{\bw_{\ridge}}(h_{2t-1}^{(L)})-\bw^t_{\ridge,\lambda}\|_2\leq\eps$ for all $t\in[T]$. 
\end{lemma}
% See the proof is Section~\ref{sec:pf_lm:approx_ridge}.

Results similar to Lemma~\ref{lm:approx_ridge} have been shown in~\cite{bai2023transformers} under a different scenario.   However, we remark that  the second part of Lemma~\ref{lm:approx_ridge} has a weaker requirement on the number of layers as we prove that transformers can implement accelerated gradient descent (AGD,~\cite{nesterov2003introductory}) in-context.

\begin{proof}[Proof of Lemma~\ref{lm:approx_ridge}]
%\label{sec:pf_lm:approx_ridge}
Note that $\lambda\id_d\preceq \bA_t\preceq (TB_a^2+\lambda)\id_d$. Therefore the optimization problem 
$$\bw^t_{\ridge,\lambda}=\argmin_{\bw\in\R^d}L(\bw):=\argmin_{\bw\in\R^d}\frac{1}{2(2t-1)}\sum_{j=1}^{t-1}(r_j-\<\ba_j,\bw\>)^2+\frac{\lambda}{2(2t-1)}\|\bw\|_2^2$$
is $\lambda/(2t-1)$-strongly convex  and $(B_a^2+\lambda)$-smooth and the condition number $\kappa\leq 2T(B_a^2+\lambda)/\lambda$. Moreover, by the definition of $\bw_{\ridge,\lambda}^t$ we have 
\begin{align*}
    \|\bw^t_{\ridge,\lambda}\|_2=\|(\lambda\id_d+\sum_{j=1}^{t-1}\ba_j\ba_j^\top)^{-1}(\sum_{j=1}^{t-1}\ba_jr_j)\|_2&\leq 
    \|(\lambda\id_d+\sum_{j=1}^{t-1}\ba_j\ba_j^\top)^{-1}\|_2\cdot\|\sum_{j=1}^{t-1}\ba_jr_j\|_2\\&\leq
    \frac{TB_a(B_aB_w+\sigma)}{\lambda}
\end{align*} 
for all $t\in[T]$. 
\paragraph{Proof of part 1}
By Proposition~\ref{prop:conv_gd_agd}, we see that $L=\lceil4T(B_a^2+\lambda)\log({TB_a(B_aB_w+\sigma)}/({\lambda}\eps))/\lambda\rceil$ steps of gradient descent with stepsize $\eta=1/(B_a^2+\lambda)$ starting from $\bw_{\GD}^0=\bzero_d$ finds $\bw^L_\GD$ such that $\|\bw^L_{\GD}-\bw^t_{\ridge,\lambda}\|_2\leq\eps$.

Now we prove that one  attention-only layer can implement one step of gradient descent
\begin{align*}
    \bw^{\ssl+1}_{\GD}:=\bw^{\ssl}_{\GD}- \frac{\eta}{2t-1}\sum_{j=1}^{t-1}(\<\ba_j,\bw^\ssl_{\GD}\>-r_j)\ba_j- \frac{\eta\lambda}{2t-1}\bw^\ssl_{\GD}.
\end{align*}
We encode the algorithm using the last token (i.e., the $2t-1$-th token). 
Denote the first $d$ entries of $\bh_{2t-1}^{\partd}$ by $\hat\bw$ and  define $\read_{\bw_{\ridge}}(\bh_{2t-1})=\hat\bw$. Starting from $\hat\bw^{0}=\bzero_d$, for each layer $\ell\in[L]$, we let the number of heads $M^{(\ell)}=3$ and  choose $\bQ_{1,2,3}^{(\ell)},\bK_{1,2,3}^{(\ell)},\bV_{1,2,3}^{(\ell)}$ such that for even tokens $\bh_{2j}$ with $j\leq t-1$ and odd tokens $\bh_{2j-1}$ with $j\leq t$
\begin{align*}
    &\bQ_1^{(\ell)}\bh^{(\ell-1)}_{2t-1}=\begin{bmatrix}
        \hat\bw^{\ell-1}\\ 1
    \end{bmatrix},~~ \bK_1^{(\ell)}\bh^{(\ell-1)}_{2j}=\begin{bmatrix}
        \ba_j\\ -r_j 
    \end{bmatrix},~~ \bV_1^{(\ell)}\bh^{(\ell-1)}_{2j}=-\eta\begin{bmatrix}
        \bzero\\ \ba_j \\ \bzero
    \end{bmatrix},~~
    \bK_1^{(\ell)}\bh^{(\ell-1)}_{2j-1}=\bzero, ~~\bV_1^{(\ell)}\bh^{(\ell-1)}_{2j-1}=\bzero
    \\
    &
    \bQ_2^{(\ell)}=-\bQ_1^{(\ell)},~~ \bK_2^{(\ell)}=\bK_1^{(\ell)},~~  \bV_2^{(\ell)}=-\bV_1^{(\ell)},\\
     &
     \bQ_3^{(\ell)}\bh^{(\ell-1)}_{2t-1}=\begin{bmatrix}
         1\\-(2t-1)\\ 1
    \end{bmatrix},~~ \bK_3^{(\ell)}\bh^{(\ell-1)}_{2j}=\begin{bmatrix}
        1\\ 1 \\2j
    \end{bmatrix},~~ \bK_3^{(\ell)}\bh^{(\ell-1)}_{2j-1}=\begin{bmatrix}
        1\\ 1 \\2j-1
    \end{bmatrix},~~ \bV_3^{(\ell)}\bh^{(\ell-1)}_{2t-1}=-\eta\lambda\begin{bmatrix}
        \bzero\\ \hat\bw^{\ell-1}\\ \bzero
    \end{bmatrix}.
\end{align*}
Summing up the three heads and noting that $t=\sigma(t)-\sigma(-t)$, we see that the $\hat\bw$ part of $\bh_{2t-1}$ (i.e., $\read_{\bw_\ridge}(\bh_{2t-1})$) has the update
\begin{align*}
    \hat\bw^{l}
    &=\hat\bw^{l-1}-\frac{\eta}{2t-1}\sum_{j=1}^t[\sigma(\<\ba_j,\hat\bw^{l-1}\>-r_j)-\sigma(r_j-\<\ba_j,\hat\bw^{l-1}\>)]\ba_j\\&
    \qquad
    -\frac{\eta\lambda}{2t-1}\Big[\sum_{j=1}^{t-1}(\sigma(1+2j-2t)\bV_3^{(\ell)}\bh^{(\ell-1)}_{2j-1}+\sigma(1+2j-2t+1)\bV_3^{(\ell)}\bh^{(\ell-1)}_{2j})+
   \bV_3^{(\ell)}\bh^{(\ell-1)}_{2t-1}\Big]\\
     &=\hat\bw^{l-1}-\frac{\eta}{2t-1}\sum_{j=1}^t[\<\ba_j,\hat\bw^{l-1}\>-r_j]\ba_j-\frac{\eta\lambda}{2t-1}\bV_3^{(\ell)}\bh^{(\ell-1)}_{2t-1}\\
     &=\hat\bw^{l-1}-\frac{\eta}{2t-1}\sum_{j=1}^{t-1}[\<\ba_j,\hat\bw^{l-1}\>-r_j]\ba_j-\frac{\eta\lambda}{2t-1}\hat\bw^{l-1},
\end{align*}
which is one step of gradient descent with stepsize $\eta$.  Moreover, it is easy to see that one can choose the marices such that $\max_{m\in[3]}\lops{\bQ^\lth_m}=\max_{m\in[3]}\lops{\bK^\lth_m}=\sqrt{2}$ and $\lops{\bV^\lth_1}=\lops{\bV^\lth_2}=\eta,\lops{\bV^\lth_3}= \lambda\eta$. Therefore the norm of the transformer $\nrmp{\btheta}\leq \sqrt{2}+(\lambda+2)/(B_a^2+\lambda)$.

\paragraph{Proof of part 2}
Similarly,  Proposition~\ref{prop:conv_gd_agd} shows $L=\lceil2\sqrt{2T(B_a^2+\lambda)/\lambda}\log((1+\kappa){TB_a(B_aB_w+\sigma)}/({\lambda}\eps))\rceil$ steps of accelerated gradient descent gives  $\|\bw^L_{\AGD}-\bw^t_{\ridge,\lambda}\|_2\leq\eps$. 

Again, we encode the algorithm using the last token (i.e., the $2t-1$-th token). 
Denote the first $d,d+1\sim 2d, 2d+1\sim3d $ entries of $\bh_{2t-1}^{\partd}$ by $\hat\bw_a,\hat\bw_b,\hat\bv$ respectively. Starting from $\hat\bw_a^{0}=\hat\bw_b^{0}=\hat\bv^{0}=\bzero_d$,  AGD updates the parameters as follows:
\begin{subequations}
\begin{align}
    \hat\bw^\ell_a&=\hat\bw_a^{\ell-1}+(\hat\bv^{\ell-1}-\hat\bw_a^{\ell-1})-\eta \nabla L(\hat\bv^{\ell-1}),\label{eq:agd_step1}\\
    \hat\bv^{\ell}&=\hat\bv^{\ell-1}+[\hat\bw^\ell_a+\frac{\sqrt{\kappa}-1}{\sqrt{\kappa}+1}(\hat\bw^\ell_a-\hat\bw^{\ell-1}_b)-\hat\bv^{\ell-1}],\label{eq:agd_step2}\\
    \hat\bw^\ell_b&= \hat\bw^{\ell-1}_b+( \hat\bw^\ell_a-\hat\bw^{\ell-1}_b).\label{eq:agd_step3}
\end{align}
\end{subequations}
We show that one attention layer and one MLP layer can implement one step of AGD as above. Namely, Eq.~\eqref{eq:agd_step1} can be obtained using the same attention layer we constructed for gradient descent with $\hat\bv$ replacing $\hat\bw$, and an extra head with
\begin{align*}
     \bQ_4^{(\ell)}\bh^{(\ell-1)}_{2t-1}=\begin{bmatrix}
         2t-1\\-(2t-1)^2\\ 1
    \end{bmatrix},~~ \bK_4^{(\ell)}\bh^{(\ell-1)}_{i}=\begin{bmatrix}
        1\\ 1 \\i^2
    \end{bmatrix},~~ \bV_4^{(\ell)}\bh^{(\ell-1)}_{2t-1}=\begin{bmatrix}
        \bzero\\ \hat\bv^{\ell-1}-\hat\bw_a^{l-1}\\ \bzero
    \end{bmatrix}
\end{align*} for $i\leq 2t-1$ that gives $\hat\bv^{\ell-1}-\hat\bw_a^{\ell-1}$.
Denote the output tokens of the attention layer by $\tilde \bh$. Eq.~\eqref{eq:agd_step2},~\eqref{eq:agd_step3} can be implemented using one layer of MLP. Concretely, we choose $\bW_{1}^\lth,\bW_{2}^\lth$ such that
\begin{align*}
     &\bW_1^{(\ell)}\tilde\bh^{(\ell-1)}_{2t-1}=\begin{bmatrix}
        \bw^{\ell}_a+\frac{\sqrt{\kappa}-1}{\sqrt{\kappa}+1}(\hat\bw^{\ell}_a-\hat\bw^{\ell-1}_b)-\hat\bv^{l-1}
        \\
        -\bw^{\ell}_a-\frac{\sqrt{\kappa}-1}{\sqrt{\kappa}+1}(\hat\bw^{\ell-1}_a-\hat\bw^{\ell-1}_b)+\hat\bv^{l-1}
        \\
        \bw^{\ell}_a- \bw^{\ell-1}_b
        \\
        -\bw^{\ell}_a+\bw^{\ell-1}_b
    \end{bmatrix},~~ \bW_2^{(\ell)}\sigma(\bW_1^{(\ell)}\tilde\bh^{(\ell-1)}_{2t-1})=\begin{bmatrix}
      \bzero\\  \hat\bw^{\ell}_b \\ \hat\bv^\ell\\\bzero
    \end{bmatrix}.
\end{align*} Since $t=\sigma(t)-\sigma(-t)$ for $t\in\R$,  it is readily verified that one can choose the linear maps such that $\lops{\bW_1^\lth}\leq 4\sqrt{2},\lops{\bW_2^\lth}=\sqrt{2}$. Combining this with the attention layer for Eq.~\eqref{eq:agd_step1} and noting that $\lops{\bV_4^\lth}=\sqrt{2}$, we verify that the  transformer we constructed has norm $\nrmp{\btheta}\leq 10+(\lambda+2)/(B_a^2+\lambda)$. This completes the proof of Lemma~\ref{lm:approx_ridge}.
\end{proof}

\subsection{Proof of Theorem~\ref{thm:approx_smooth_linucb}}\label{sec:pf_thm:approx_smooth_linucb}
We construct a transformer that implements the following steps at each time $t\in[T]$ starting with $\bh^{x}_{2t-1}=\bh^{\pre,x}_{2t-1}$ for $x\in\{\parta,\partb,\partc,\partd\}$ 
\begin{align}
    \bh_{2t-1}=
    \begin{bmatrix}
    \bh_{2t-1}^{\pre,\parta} \\  \bh_{2t-1}^{\pre,\partb}\\  \bh_{2t-1}^{\pre,\partc}\\   \bh_{2t-1}^{\pre,\partd}
\end{bmatrix}
\xrightarrow{\text{step 1}}
   \begin{bmatrix}
    \bh_{2t-1}^{\pre,\{\parta,\partb,\partc\}} \\
        \hat\bw_{\ridge} \\ \star\\ \bzero \\\posv
\end{bmatrix}
\xrightarrow{\text{step 2}}
\begin{bmatrix}
    \bh_{2t-1}^{\pre,\{\parta,\partb,\partc\}} \\
        \hat\bw_{\ridge} \\ \star\\ \widehat{\bA_t^{-1}\ba_{t,1}}\\\vdots\\\widehat{\bA_t^{-1}\ba_{t,A}}
        \\ \bzero \\\posv
\end{bmatrix}
\xrightarrow{\text{step 3}}
\begin{bmatrix}
    \bh_{2t-1}^{\pre,\{\parta,\partb,\partc\}} \\
        \hat\bw_{\ridge} \\ \star\\ \widehat{\bA_t^{-1}\ba_{t,1}}\\\vdots\\\widehat{\bA_t^{-1}\ba_{t,A}}\\ {\hat v_{t1}}/{\temp}\\\vdots\\ {\hat v_{tA}}/{\temp}
        \\ \bzero \\\posv
\end{bmatrix}
=:
\begin{bmatrix}
    \bh_{2t-1}^{\post,\parta} \\  \bh_{2t-1}^{\post,\partb}\\  \bh_{2t-1}^{\post,\partc}\\   \bh_{2t-1}^{\post,\partd}
\end{bmatrix},\label{eq:slinucb_pipeline}
\end{align}
where $\posv:=[t,i^2,1]^\top$; $\hat\bw_{\ridge}$ is an approximation to the ridge estimator $\bw^t_{\ridge,\lambda}$; $\widehat{\bA_{t}^{-1}\ba_{t,k}}$ are approximations to ${\bA_{t}^{-1}\ba_{t,k}}$;  $\hat v_{tk}$ are approximations to $v_{tk}:=\<\hat\bw_{\ridge},\ba_{t,k}\>+\alpha\sqrt{\<\ba_{t,k}, \widehat{\bA_t^{-1} \ba_{t,k}}\>}$, which are also approximations to $$
v^*_{tk}:=\<\bw^t_{\ridge,\lambda},\ba_{t,k}\>+\alpha\sqrt{\<\ba_{t,k}, {\bA_t^{-1} \ba_{t,k}}\>}
$$ for $k\in[A]$. After passing through the transformer, we obtain the policy  
$$
\sAlg_{\tfpar}(\cdot|\dset_{t-1},\state_t):=\frac{\exp(\bh^{\post,\partc}_{2t-1})}{\|\exp(\bh^{\post,\partc}_{2t-1})\|_1}\in\Delta^A.$$ 
We claim the following results which we will prove later.
\begin{enumerate}[label=Step \arabic*,ref= \arabic*]
    \item \label{slinucb_step1} For any $\eps>0$, 
    there exists a transformer $\TF_\btheta(\cdot)$ with 
\begin{align*}
&L=\Big\lceil2\sqrt{2T}\sqrt{\frac{B_a^2+\lambda}{\lambda}}\log\Big(\frac{(2T(B_a^2+\lambda)+\lambda)TB_a(B_aB_w+\sigma)}{\lambda^2\eps}\Big)\Big\rceil=\tcO(\sqrt{T}),\\
&~~~\qquad\max_{\ell\in[L]}M^{(l)}\leq4,~~~\max_{\ell\in[L]}\hidden^{(l)}\leq4d,~~~ \nrmp{\btheta}\leq  10+\frac{\lambda+2}{B_a^2+\lambda}=\cO(1) \end{align*}
 that implements step 1 in~\eqref{eq:slinucb_pipeline} with  $\|\hat\bw_{\ridge}-\bw^t_{\ridge,\lambda}\|_2\leq\eps$.
   \item\label{slinucb_step2}
   For any $\eps>0$, there exists a transformer $\TF_\btheta(\cdot)$ with 
\begin{align*}
&L=\Big\lceil2\sqrt{2T}\sqrt{\frac{B_a^2+\lambda}{\lambda}}\log\Big(\frac{(2T(B_a^2+\lambda)+\lambda)B_a}{\lambda^2\eps}\Big)\Big\rceil=\tcO(\sqrt{T}),~~\max_{\ell\in[L]}M^{(l)}\leq4A,\\
&\qquad~~
\max_{\ell\in[L]}\hidden^{(l)}\leq4dA,~~~ \nrmp{\btheta}\leq  10+A(\frac{\lambda+3}{B_a^2+\lambda}+\sqrt{2})=\cO(A) \end{align*}
 that implements step 2 in~\eqref{eq:slinucb_pipeline} with  $\|\widehat{\bA_{t}^{-1}\ba_{t,k}}-\bA_{t}^{-1}\ba_{t,k}\|_2\leq\eps$ for $k\in[A]$.
  \item\label{slinucb_step3} Suppose that the approximation error in Step~\ref{slinucb_step2} satisfies $\eps_2\leq b_a^2/[2(B_a^2+\lambda)TB_a]$.
  For any $\eps>0$, 
  there exists a one-layer transformer $\TF_\btheta(\cdot)$ with 
    $$
    L=2,~\max_{\ell\in[L]}M^{(l)}\leq 4A,~ \max_{\ell\in[L]}\hidden^\lth\leq \cO(A\sqrt{T\alpha/(\temp\eps)}),~\nrmp{\btheta}\leq  \cO(A+T({\alpha/(\temp\eps)})^{1/4}+\alpha/\temp) $$
 that implements step 3 in~\eqref{eq:slinucb_pipeline} with  $|\hat v_{tk}/\temp-v_{tk}/\temp|\leq\eps$ for $k\in[A]$.
\end{enumerate}
Denote the errors $\eps$ appear in each step by $\eps_1,\eps_2,\eps_3$, respectively. Define for all $k\in[A]$ that $$
v^{*}_{tk}:=\<\bw^t_{\ridge,\lambda},\ba_{t,k}\>+\alpha\sqrt{\<\ba_{t,k}, {\bA_t^{-1} \ba_{t,k}}\>},
$$ which are the actual values used to compare across different actions in LinUCB. Then  for all $k\in[A]$, we have the approximation error 
\begin{align*}
    \Big|\frac{v_{tk}^*}{\temp}-\frac{\hat v_{tk}}{\temp}\Big|
    &\leq
     \Big|\frac{v_{tk}^*}{\temp}- \frac{v_{tk}}{\temp}\Big|+ \Big|\frac{v_{tk}}{\temp}-\frac{\hat v_{tk}}{\temp}\Big|\\
     &\leq
   \frac1\temp |\<\bw^t_{\ridge,\lambda}-\hat \bw_{\ridge},\ba_{t,k}\>|+\frac1\temp\Big|\alpha\sqrt{\<\ba_{t,k}, {\bA_t^{-1} \ba_{t,k}}\>}-\alpha\sqrt{\<\ba_{t,k}, \widehat{\bA_t^{-1} \ba_{t,k}}\>}\Big|+\eps_3
    \\&\leq
    \frac{B_a\eps_1}\temp+\frac{\alpha B_a\eps_2}{2\temp\min\Big\{\sqrt{\<\ba_{t,k}, {\bA_t^{-1} \ba_{t,k}}\>},\sqrt{\<\ba_{t,k}, \widehat{\bA_t^{-1} \ba_{t,k}}\>}
    \Big\}}+\eps_3\\
    &\leq \frac{B_a\eps_1}\temp+\frac{\sqrt{T(B_a^2+\lambda)}  \alpha B_a\eps_2}{b_a\temp}+\eps_3,
\end{align*}
where the last line uses Eq.~\eqref{eq:lb_qudratic}. 
For a targeted approximation error $\eps$, choosing 
 $\eps_1=\eps\temp/(12B_a),\eps_2=\min\{b_a\temp\eps/(12\sqrt{T(B_a^2+\lambda)}  \alpha B_a),b_a^2/[2(B_a^2+\lambda)TB_a]\}$ and $\eps_3=\eps/12$, we obtain $|v_{tk}^*/\temp-\hat v_{tk}/\temp|\leq \eps/2$ for all $k\in[A]$. 

 From the proof of each step, we can verify that the token dimension $D$ can be chosen to be of order $\conO(dA)$. Moreover, due to the convergence guarantee for each iteration of AGD in Proposition~\ref{prop:conv_gd_agd}, it can be verified that there exists some sufficiently large value $\clipval>0$ with $\log \clipval=\tcO(1)$ such that  we have $\| \bh_i^{\lth} \|_{2} \leq\clipval$
 for all layer $\ell\in[L]$ and all token $i\in[2\totlen]$ in our TF construction. Therefore, $\TF^\clipval_\btheta$ and $\TF^\infty_\btheta$ generate the same output for all the token matrices we consider, and w.l.o.g. we  may assume in the proof of each step  that the  transformers we consider are those without truncation (i.e., $\TF_\btheta=\TF_\btheta^\infty$).

Finally, combining Step~\ref{slinucb_step1}---\ref{slinucb_step3} with $\alpha=\tcO(\sqrt{d})$ and applying Lemma~\ref{lm:log_softmax} completes the proof of Theorem~\ref{thm:approx_smooth_linucb}.

\paragraph{Proof of Step~\ref{slinucb_step1}} We use the first $d$ entries of $\bh_{2t-1}^{\partd}$ to represent $\hat\bw_{\ridge}$ and the $d+1\sim 3d$ entries (denoted by $\star$) to record intermediate results for computing $\hat\bw_{\ridge}$. Step~\ref{slinucb_step1} follows immediately from the second part of Lemma~\ref{lm:approx_ridge}. 

\paragraph{Proof of Step~\ref{slinucb_step2}}
Note that $$\bA_{t}^{-1}\ba_{t,k}=\argmin_{\bx\in\R^d}\frac{1}{2(2t-1)}\bx^\top\bA_t\bx-\frac{1}{(2t-1)}\<\bx,\ba_{t,k}\>=:\argmin_{\bx\in\R^d}L_k(\bx)$$ is the global minimizer of a $\lambda/(2t-1)$-strongly convex  and $(B_a^2+\lambda)$-smooth  quadratic function with the condition number $\kappa\leq 2T(B_a^2+\lambda)/\lambda$. Moreover, we have $$\|\bA_{t}^{-1}\ba_{t,k}\|_2\leq \lops{\bA_{t}^{-1}}\|\ba_{t,k}\|_2\leq B_a/\lambda.$$ It follows from  Proposition~\ref{prop:conv_gd_agd} that $L=\lceil2\sqrt{2 T(B_a^2+\lambda)/\lambda}\log((1+\kappa)B_a/(\lambda\eps))\rceil$ steps of accelerated gradient descent finds $\widehat{\bA_{t}^{-1}\ba_{t,k}}$   with $\|\widehat{\bA_{t}^{-1}\ba_{t,k}}-\bA_{t}^{-1}\ba_{t,k}\|_2\leq\eps$. 
% Moreover, the gradient $\nabla L_k(\bx)=(\bA_t\bx-\ba_{t,k})/i$. 

Similar to the proof of Lemma~\ref{lm:approx_ridge}, we can construct a transformer such that each (self-attention+MLP) layer implements one step of the accelerated gradient descent (AGD) for all $k\in[A]$. Denote the $(k+2)d+1\sim (k+3)d, (A+1+2k)d+1\sim (A+2+2k)d, (A+2+2k)d+1\sim (A+3+2k)d$ entries of  $\bh_{2t-1}^{\partd}$ by $\hat\bw_{a,tk},\hat\bw_{b,k},\hat\bv_{k}$ for $k\in[A]$. Note that in the input vector $\bh_{2t-1}^{\pre,\partd}$ we have $\hat\bw^0_{a,tk},\hat\bw^0_{b,k},\hat\bv^0_{k}=\bzero_d$.

For each layer $\ell\in[L]$ and $k\in[A]$, we choose 
$\bQ_{k1,k2,k3,k4}^{(\ell)},\bK_{k1,k2,k3,k4}^{(\ell)},\bV_{k1,k2,k3,k4}^{(\ell)}$ such that for even tokens $\bh_{2j}$ with $j\leq t-1$ and odd tokens $\bh_{2j-1}$ with $j\leq t$
\begin{align*}
    &\bQ_{k1}^{(\ell)}\bh^{(\ell-1)}_{2t-1}=\begin{bmatrix}
        \hat\bv_k^{\ell-1}\\\bzero
    \end{bmatrix},~~ \bK_{k1}^{(\ell)}\bh^{(\ell-1)}_{2j}=\begin{bmatrix}
        \ba_{j}\\\bzero
\end{bmatrix},~~\bK_{k1}^{(\ell)}\bh^{(\ell-1)}_{2j-1}=\bzero,~~ \bV_{k1}^{(\ell)}\bh^{(\ell-1)}_{2j}=-\eta\begin{bmatrix}
        \bzero\\ \ba_j \\ \bzero
    \end{bmatrix},~~\bV_{k1}^{(\ell)}\bh^{(\ell-1)}_{2j-1}=\bzero\\
    &
    \bQ_{k2}^{(\ell)}=-\bQ_{k1}^{(\ell)},~~ \bK_{k2}^{(\ell)}=\bK_{k1}^{(\ell)},~~  \bV_{k2}^{(\ell)}=-\bV_{k1}^{(\ell)},\\
     &
     \bQ_{k3}^{(\ell)}\bh^{(\ell-1)}_{2t-1}=\begin{bmatrix}
         1\\1-2t\\ 1\\\bzero
    \end{bmatrix},~~ \bK_{k3}^{(\ell)}\bh^{(\ell-1)}_{2j}=\begin{bmatrix}
        1\\ 1 \\2j\\\bzero
    \end{bmatrix},~~ 
    \bK_{k3}^{(\ell)}\bh^{(\ell-1)}_{2j-1}=\begin{bmatrix}
        1\\ 1 \\2j-1\\\bzero
    \end{bmatrix},~~ \bV_{k3}^{(\ell)}\bh^{(\ell-1)}_{2t-1}=\eta\begin{bmatrix}
        \bzero\\ \ba_{j,k}-\lambda\hat\bv_k^{\ell-1}\\ \bzero
    \end{bmatrix},\\
     &
     \bQ_{k4}^{(\ell)}\bh^{(\ell-1)}_{2t-1}=\begin{bmatrix}
         2t-1\\-(2t-1)^2\\ 1\\\bzero
    \end{bmatrix},~~ \bK_{k4}^{(\ell)}\bh^{(\ell)}_{2j}=\begin{bmatrix}
        1\\ 1 \\(2j)^2\\\bzero
    \end{bmatrix},~~ 
\bK_{k4}^{(\ell)}\bh^{(\ell)}_{2j-1}=\begin{bmatrix}
        1\\ 1 \\(2j-1)^2\\\bzero
    \end{bmatrix},~~ \bV_{k4}^{(\ell)}\bh^{(\ell-1)}_{2t-1}=\begin{bmatrix}
        \bzero\\ \hat\bv^{\ell-1}_{k}-\hat\bw_{a,k}^{\ell-1}\\ \bzero
    \end{bmatrix},
\end{align*}
where $\eta=1/(B_a^2+\lambda)$ and the values $\bV_{kt}^{(\ell)}\bh^{(\ell-1)}_{2j},\bV_{kt}^{(\ell)}\bh^{(\ell-1)}_{2j-1},~t=1,2,3,4$ are supported on the entries corresponding to $\hat\bw_{a,k}$.
Summing up the $M=4A$ heads and noting that $t=\sigma(t)-\sigma(-t)$, we see that the $\hat\bw_{a,k}$ part of $\bh_t$ has the update
\begin{align*}
    \hat\bw_{a,k}^{\ell}
    &=
    \hat\bw_{a,k}^{\ell-1}-\frac{\eta}{2t-1}\sum_{j=1}^{t-1}[\sigma(\<\ba_j,\hat\bv_k^{\ell-1}\>)-\sigma(-\<\ba_j,\hat\bv_k^{\ell-1}\>)]\ba_j-\frac{\eta\lambda}{2t-1}\bV_{k3}^{(\ell)}\bh^{(\ell-1)}_{2t-1}+\bV_{k4}^{(\ell)}\bh^{(\ell-1)}_{2t-1}\\
     &=
     \hat\bw_{a,k}^{\ell-1}-\frac{\eta}{2t-1}\sum_{j=1}^{t-1}\<\ba_j,\hat\bv_k^{\ell-1}\>\ba_j-\frac{\eta\lambda}{2t-1}\bV_{k3}^{(\ell)}\bh^{(\ell-1)}_{2t-1}+\bV_{k4}^{(\ell)}\bh^{(\ell-1)}_{2t-1}\\
     &=
     \hat\bv_k^{\ell-1}-\frac{\eta}{2t-1}\sum_{j=1}^{t-1}\<\ba_j,\hat\bv_k^{\ell-1}\>\ba_j-\frac{\eta\lambda}{2t-1}\hat\bv_k^{\ell-1}+\frac{\eta}{2t-1}\ba_{t,k}\\
     &=\hat\bv_k^{\ell-1}-\eta\nabla L(\bv_k^{\ell-1}),
\end{align*}
which is one step of gradient descent with step size $\eta$ (c.f.   Eq.~\ref{eq:agd_step1}).  Moreover, it can be verified that one can choose the matrices such that $\max_{k\in[A],m\in[4]}\lops{\bQ_{km}^\lth}=\max_{k\in[A],m\in[4]}\lops{\bK^\lth_{km}}\leq\sqrt{2}$ and $\max_{k\in[A]}\lops{\bV^\lth_{k1}}=\max_{k\in[A]}\lops{\bV^\lth_{k2}}=\eta,~\max_{k\in[A]}\lops{\bV^\lth_{k3}}\leq (\lambda+1)\eta,~\max_{k\in[A]}\lops{\bV^\lth_{k4}}\leq \sqrt{2}$. Therefore, the norm of the attention layer $$\nrmp{\btheta}\leq \sqrt{2}(A+1)+A(\lambda+3)/(B_a^2+\lambda).$$ Following the construction as in the proof of Lemma~\ref{lm:approx_ridge}, we can choose $\bW_1^\lth,\bW_2^\lth$ that implement Eq.~\eqref{eq:agd_step2},~\eqref{eq:agd_step2} for all $k\in[A]$ simultaneously and we also have $\lops{\bW_1^\lth}\leq4\sqrt{2},\bW_2^\lth=\sqrt{2}$ with $\hidden'^{\lth}=4dA$. It follows from  combining the bounds for the weight matrices that 
$$
\nrmp{\btheta}\leq \sqrt{2}(A+1)+A(\frac{\lambda+3}{B_a^2+\lambda}+\sqrt{2})+\sqrt{2}+4\sqrt{2}\leq 10+A(\frac{\lambda+3}{B_a^2+\lambda}+\sqrt{2})=\cO(A).$$

\paragraph{Proof of Step~\ref{slinucb_step3}}
Denote the $i$-th token  of the output of step 2 (i.e., the input of step 3) by $\bh_i^{(0)}$. 
We use the $(3A+3)d+1\sim (3A+3)d+A $ entries of $\bh_{2t-1}^{\partd}$ to record $\hat v_{t1}/\temp,\ldots,\hat v_{tA}/\temp$ and  use the $(3A+3)d+A+1\sim (3A+3)d+2A $ entries to store additional information (denoted by $ v_{a,t1},\ldots, v_{a,tA}$) for computing $ \hat v_{t1}/\temp,\ldots, \hat  v_{tA}/\temp$. Concretely, for all $k\in[A]$, we choose 
$\bQ_{k1,k2,k3,k4}^{(\ell)},\bK_{k1,k2,k3,k4}^{(\ell)},\bV_{k1,k2,k3,k4}^{(\ell)}$ such that   for even tokens $\bh_{2j}$ with $j\leq t-1$ and odd tokens $\bh_{2j-1}$ with $j\leq t$
\begin{align*}
    &\bQ^{(1)}_{k1}\bh^{(0)}_{2t-1}=\begin{bmatrix}
        \hat\bw_{\ridge} \\2t-1\\ 1\\\bzero
    \end{bmatrix},~~ \bK^{(1)}_{k1}\bh^{(0)}_{2j-1}=\begin{bmatrix}
        \ba_{j,k}\\  -\tfthres \\ \tfthres (2j-1)\\\bzero
    \end{bmatrix},~~ 
    \bK^{(1)}_{k1}\bh^{(0)}_{2j}=\begin{bmatrix}
        \bzero_d\\  -\tfthres \\ 2\tfthres j\\\bzero
    \end{bmatrix},\\
    &~~~\qquad\bV^{(1)}_{k1}\bh^{(0)}_{2j-1}=\begin{bmatrix}
        \bzero\\ 2j-1 \\ \bzero
\end{bmatrix},~~\bV^{(1)}_{k1}\bh^{(0)}_{2j}=\begin{bmatrix}
        \bzero\\ 2j \\ \bzero
    \end{bmatrix},\\
    &
    \bQ^{(1)}_{k2}\bh^{(0)}_{2t-1}=\begin{bmatrix}
        -\hat\bw_{\ridge} \\2t-1\\1\\\bzero
    \end{bmatrix},~~  \bK^{(1)}_{k2}=\bK^{(1)}_{k1},~~  \ \bV^{(1)}_{k2}=-\bV^{(1)}_{k1},\\
    &
    \bQ^{(1)}_{k3}\bh^{(0)}_{2t-1}=\begin{bmatrix}
        \widehat{\bA_t^{-1}\ba_{t,k}} \\2t-1\\ 1\\\bzero
    \end{bmatrix},~~ \bK^{(1)}_{k3}\bh^{(0)}_{2j-1}=\begin{bmatrix}
        \ba_{j,k}\\  -\tfthres \\ \tfthres (2j-1)\\\bzero
    \end{bmatrix},~~ 
    \bK^{(1)}_{k3}\bh^{(0)}_{2j}=\begin{bmatrix}
        \bzero_d\\  -\tfthres \\ 2\tfthres j\\\bzero
    \end{bmatrix},\\
    &
    \qquad~~~ \bV^{(1)}_{k3}\bh^{(0)}_{2j-1}=\begin{bmatrix}
        \bzero\\ 2j-1 \\ \bzero
    \end{bmatrix},~~ \bV^{(1)}_{k3}\bh^{(0)}_{2j}=\begin{bmatrix}
        \bzero\\ 2j \\ \bzero
    \end{bmatrix},\\
    &
    \bQ^{(1)}_{k4}\bh^{(0)}_{2t-1}=\begin{bmatrix}
        - \widehat{\bA_t^{-1}\ba_{t,k}} \\2t-1\\1\\\bzero
    \end{bmatrix},~~  \bK^{(1)}_{k4}=\bK^{(1)}_{k3},~~  \ \bV^{(1)}_{k4}=-\bV^{(1)}_{k3},
\end{align*}
where $\tfthres:=TB_a^2(B_aB_w+\sigma)/\lambda+2B_a^2/\lambda$; $\bV^{(1)}_{k1}\bh^{(0)}_{c},\bV^{(1)}_{k2}\bh^{(0)}_{c} (c=2j-1,2j)$ are supported on the $[(3A+3)d+k]$-th entry of $\bh_c^{\partd}$; $\bV^{(1)}_{k3}\bh^{(0)}_{c},\bV^{(1)}_{k4}\bh^{(0)}_{c} (c=2j-1,2j)$ are supported on the $[(3A+3)d+A+k]$-th entry of $\bh_c^{\partd}$. 

Since $\<\hat\bw_{\ridge},\ba_{j,k}\>\leq \|\hat\bw_{\ridge}\|_2\|\ba_{j,k}\|_2\leq\tfthres$, it follows that $$\<\bQ^{(1)}_{k1}\bh^{(0)}_{2t-1},\bK^{(1)}_{k1}\bh^{(0)}_{2j-1}\>=\<\hat\bw_{\ridge},\ba_{j,k}\>+(2j-1-(2t-1))\tfthres\leq 0$$ for $j<i$. Likewise $\<\bQ^{(1)}_{k1}\bh^{(0)}_{2t-1},\bK^{(1)}_{k1}\bh^{(0)}_{2j}\>\leq0$ for $j<i$. 
Since we assume the error $\eps_2\leq b_a^2/[2(B_a^2T+\lambda)B_a]$
in Step~\ref{slinucb_step2}, $b_a\leq\|\ba_{t,k}\|_2\leq B_a$ and $\lambda\id_d\preceq\bA_t\preceq (B_a^2T+\lambda)\id_d$, it follows that 
\begin{align}
    \<\ba_{t,k},\widehat{\bA_t^{-1}\ba_{t,k}}\> 
    &\geq 
    \<\ba_{t,k},{\bA_t^{-1}\ba_{t,k}}\>-\|\ba_{t,k}\|_2\|\bA_t^{-1}\ba_{t,k}-\widehat{\bA_t^{-1}\ba_{t,k}}\|_2\notag\\
    &\geq \frac{b_a^2}{2(B_a^2T+\lambda)}\geq \frac{b_a^2}{2T(B_a^2+\lambda)}=:\frac{1}{T}\cdot{\lran},\label{eq:lb_qudratic}
    \\
     \<\ba_{t,k},\widehat{\bA_t^{-1}\ba_{t,k}}\>
     &\leq 
     \<\ba_{t,k},{\bA_t^{-1}\ba_{t,k}}\>+\|\ba_{t,k}\|_2\|\bA_t^{-1}\ba_{t,k}-\widehat{\bA_t^{-1}\ba_{t,k}}\|_2\leq \frac{2B_a^2}{\lambda}=:\uran.\label{eq:ub_qudratic}
\end{align}
Therefore,  $\<\bQ^{(1)}_{k3}\bh^{(0)}_{2t-1},\bK^{(1)}_{k3}\bh^{(0)}_{2j-1}\>=\<\widehat{\bA_t^{-1}\ba_{t,k}},\ba_{j,k}\>+(2j-1-(2t-1))\tfthres\geq 0$ iff $j=i$. Likewise $\<\bQ^{(1)}_{k3}\bh^{(0)}_{2t-1},\bK^{(1)}_{k3}\bh^{(0)}_{2j}\>\leq 0$  for $j<i$. Similar results hold for the $k2,k4$-th heads.  By some basic algebra and noting that $t=\sigma(t)-\sigma(-t)$ for $t\in\R$, we see that the attention layer updates the position for $\hat v_{tk}/\temp,\hat v_{a,tk}$ with the values $\<\hat\bw_{\ridge},\ba_{t,k}\>, \<\ba_{t,k},\widehat{\bA_t^{-1}\ba_{t,k}}\>$ for all $k\in[A]$, respectively.  Moreover, it can be verified that one can choose the matrices such that $$\max_{k\in[A],m\in[4]}\lops{\bQ_{km}^{(1)}}=\max_{k\in[A],m\in[4]}\lops{\bV_{km}^{(1)}}=1,~\max_{k\in[A],m\in[4]}\lops{\bK^{(1)}_{km}}\leq\tfthres.$$

Now, to compute the value of $\hat v_{tk}/\temp$ in step 3 in~\eqref{eq:slinucb_pipeline}, what remains  is to  approximately compute $\cwid\sqrt{\hat v_{a,tk}}$, add the result to the position for $\hat v_{tk}/\temp$, and multiplied it by $1/\temp$.

Since $\hat v_{a,tk}=\<\ba_{t,k},\widehat{\bA_t^{-1}\ba_{t,k}}\>\in[\lran/T,\uran]$, to approximately compute $\sqrt{\<\ba_{t,k},\widehat{\bA_t^{-1}\ba_{t,k}}\>}$, it suffices to approximate $f(x)=\sqrt{x},x\in[\lran/T,\uran]$. For any  level of approximation error $\eps_\appr>0$, let $(x_1,x_2,\ldots,x_N)\in[\lran/T,\uran]$ satisfy
\begin{align*}
x_1=\lran/T,~x_N=\uran,   ~~~~~
    0\leq\sqrt{x_{j+1}}-\sqrt{x_{j}}\leq\eps_\appr,~~\text{ for }j\in[N-1].
\end{align*}
Define the function
\begin{align*}
    \tilde f(x):=\sqrt{x_1}+\sum_{j=1}^{N-1}\sigma(x-x_{j})\frac{1}{\sqrt{x_{j+1}}+\sqrt{x_j}}.
\end{align*} Note that $\tilde f(x)$ is a piecewise linear function on $[\lran/T,\uran]$ with $\tilde f(x_i)=\sqrt{x_i}$ for $i\in[N]$. By some basic algebra,  it can the shown that for $\eps_\appr<\lran/T$, the difference between $f(x)$ and $\tilde f(x)$
\begin{align*}
   \max_{c\in[x_j,x_{j+1}]} |\tilde f(c)-f(c)|=\max_{t\in[0,1]}\Big|\sqrt{x_j}+\frac{t}{\sqrt{x_j+1}+\sqrt{x_j}}-\sqrt{x_j+t(x_{j+1}-x_j)}\Big|\leq\eps_\appr
\end{align*} when $\sqrt{x_{j+1}}-\sqrt{x_{j}}<c\sqrt{\eps_\appr\lran/T}$ for some universal constant $c>0$ and all $j\in[N-1]$. 

Therefore, 
 there exists a function $\tilde f(x)$ with $N=\cO(\sqrt{T/\eps_\appr})$ that satisfies $$\max_{[\lran/T,\uran]}|\tilde f(x)-f(x)|\leq\eps_\appr.$$ 
 As a consequence, we verify that one can implement $\tilde f(\hat v_{a,tk})$ for all $k\in[A]$ simultaneously  by constructing a two-layer MLP with $$\lops{\bW_1^{(1)}}\leq \conO(\sqrt{N}),~~\lops{\bW_2^{(1)}}\leq \cO(\sqrt{TN}),~~\hidden\leq AN.$$

 % Moreover, since $\<\ba_{t,k},\widehat{\bA_t^{-1}\ba_{t,k}}\>\in[\lran/T,\uran]$, it follows that a symmetrically extended version of $f(x)=\sqrt{x},x\in[\lran/T,\uran]$ onto $[-\uran,\uran]$ that is smooth on $[-\lran/T,\lran/T]$, and that $\sup|f(x)|\leq\uran$, $\sup|\nabla f(x)|\leq T/(2\lran),~\sup|\nabla^2 f(x)|\leq T^{3/2}/(4\lran^{3/2})$, is $(\eps_{\appr},\uran,M,C_1)$-approximable by sum of relus  with  $C_1=c\cdot \sqrt{\uran}[(\uran/\lran)^{3/2}+1]T^{3/2}$ and $M=c \cdot C_1^2\log(1+C_1/\eps_{\appr})/\eps^2_{\appr}$  for some universal constant $c>0$.   \lc{need to import results from~\cite{bai2023transformers}}
 
 Choose $\eps_{\appr}=\temp\eps/\alpha$. Substituting the expressions for $N,\eps_\appr$ into the upper bounds on the norms, we obtain
 $$\lops{\bW_1^{(1)}}\leq \cO(({\alpha T/(\temp\eps)})^{1/4}),~~\lops{\bW_2^{(1)}}\leq \cO(T^{3/4}({\alpha/(\temp\eps)})^{1/4}),~~\hidden\leq \cO(A({\alpha T/(\temp\eps)})^{1/2}).$$

 Lastly, we can construct another two-layer MLP with weights $\bW^{(2)}_1,\bW^{(2)}_2$ such that it implements the summation and multiplication updates 
 \begin{align*}
 \hat\bv\leftarrow\hat\bv+\bW^{(2)}_2\sigma(\bW^{(2)}_1\bh_{2t-1}^{(1)})\approx\hat\bv+\Big(\frac1\temp-1\Big)\hat\bv+\frac\alpha\temp\Big[\sqrt{ \hat v_{a,t1}},\ldots,\sqrt{ \hat v_{a,tA}}\Big]^\top=\frac{\hat v_{tk}}{\temp}
 \end{align*}
with $\|\hat v_{tk}/\temp-v_{tk}/\temp\|\leq \eps$ for all $k\in[A]$. We verify that the weight matrices can be chosen with
  $$
 \lops{\bW^{(2)}_1}\leq \conO(1),~~ \lops{\bW^{(2)}_2}\leq \conO(\alpha/\temp)
 $$ and $\hidden\leq \conO(A)$.

 Therefore the norm of the transformer that implements step 3 satisfies $$
 \nrmp{\btheta}\leq \conO(\tfthres+1+4A+T^{3/4}({\alpha/(\temp\eps)})^{1/4}+\cwid/\temp)=
 \cO(A+T({\alpha/(\temp\eps)})^{1/4}+\alpha/\temp).$$  This conclude the proof of Step~\ref{slinucb_step3}.

\subsection{Proof of Theorem~\ref{thm:smooth_linucb}}\label{sec:pf_thm:smooth_linucb}

By Theorem~\ref{thm:diff_reward} and Theorem~\ref{thm:approx_smooth_linucb} with ${\geneps=\eps=1/\totlen^3}$, it suffices to show soft LinUCB with parameter $\temp$ has the regret guarantee \begin{align*}
\E_{\inst\sim\prior}\Big[\sum_{t=1}^\totlen\max_{k}\<\ba_{t,k},\bw^*\>-\totreward_{\inst,\sAlg_{\sLinUCB(\temp)}}(\totlen)\Big]&\leq\cO(d\sqrt{T}\log(T)).
\end{align*}
This follows directly from a  regret analysis similar to  that for LinUCB (see e.g.~\cite{chu2011contextual} or  Theorem 19.2 in~\cite{lattimore2020bandit}). Concretely, note that $v_{tk}^*=\<\bw^t_{\ridge,\lambda},\ba_{t,k}\>+\alpha\sqrt{\<\ba_{t,k}, {\bA_t^{-1} \ba_{t,k}}\>}$ is the solution to the optimization problem 
\begin{align*}&\text{maximize }~~~\<\bw,\ba_{t,k}\>
\\
&\text{subject to }~~~\bw\in \sC_t:=\{\bw|(\bw-\bw^t_{\ridge,\lambda})^\top\bA_t(\bw-\bw^t_{\ridge,\lambda})\leq\alpha^2\},
\end{align*}
where we recall $\alpha=\alpha(\delta_0)$ with $\delta_0=1/(2B_aB_wT)$ and
\begin{align}
\cwid=\cwid(\delta_0):=\sqrt{\lambda}B_w+\sigma\sqrt{2\log(1/\delta_0)+d\log((d\lambda+TB_a^2)/(d\lambda))}.\label{eq:recall_alpha_formula}
\end{align}
Moreover, standard analysis as in the proof of Theorem 19.2 in~\cite{lattimore2020bandit} shows with probability over $1-1/(2B_aB_wT)$ we have $\bw^*\in\sC_t$ for all $t\in[T]$. Denote this event by $\cE_0$. Moreover, let $p_{t,k}$ denote the probability of soft LinUCB selecting the action $\ba_{t,k}$ at time $t$ for all $k\in[A]$.  For any $\eps>0$, let $\sS_t(\eps):=\{k\in[A]:v^*_{tk}-\max_{j\in[A]}v^*_{tj}\leq \eps\}$. 

Therefore, on the event $\cE_0$ at time $t$ we have
\begin{align*}
    \max_{j}v^*_{tj}-\sum_{k=1}^A p_{t,k}v^*_{tk}&= \sum_{k\in\sS_t(\eps_0)} p_{t,k}(\max_{j}v^*_{tj}-v^*_{tk})+\sum_{k\notin\sS_t(\eps_0)} p_{t,k}(\max_{j}v^*_{tj}-v^*_{tk})\\
    &\leq \eps_0+\sum_{k\notin\sS_t(\eps_0)} \exp\Big(-\frac{\eps_0}{\temp}\Big)(\max_{j}v^*_{tj}-v^*_{tk})\\
    &\leq \eps_0+2A\exp\Big(-\frac{\eps_0}{\temp}\Big)B_a(B_w+2\alpha/\sqrt{\lambda}),
\end{align*}
where the second line uses $$
p_{t,k}\leq\exp\Big(-\frac{\eps_0}{\temp}\Big)\cdot\max{p_{t,k}}\leq \exp\Big(-\frac{\eps_0}{\temp}\Big),
$$ and the  last line follows from that $|v^*_{tj}|\leq B_a(B_w+2\alpha/\sqrt{\lambda})$
 on the event $\cE_0$. Choosing $\eps_0=\eps_1/2:=1/\sqrt{4T}$ and noting that $\temp=\eps_0/\log(4TA B_a(B_w+2\alpha/\sqrt{\lambda}))=\tcO(1/\sqrt{T})$, we obtain 
 \begin{align*}
    \max_{j}v^*_{tj}-\sum_{k=1}^A p_{t,k}v^*_{tk}&\leq\eps_1.
\end{align*}
Now, on the event $\cE_0$, we have
\begin{align*}
    \max_{j\in[A]}\<\bw^*,\ba_{t,j}\>\leq  \max_{j\in[A]}v^*_{tj}\leq \sum_{k=1}^A p_{t,k}v^*_{tk}+\eps_1 =\sum_{k=1}^A p_{t,k}\<\tilde{\bw}_k,\ba_{t,k}\>+\eps_1
\end{align*}
for some $\tilde{\bw}_k\in\sC_t,k\in[A]$.
Therefore,  on $\cE_0$ for  each $t\in[T]$
\begin{align*}
    &~~\quad\max_{j\in[A]}\<\bw^*,\ba_{t,j}\>-\sum_{k=1}^Ap_{t,k}\<\bw^*,\ba_{t,k}\>
    \leq \eps_1+ \sum_{k=1}^Ap_{t,k}\<\tilde{\bw}_k-\bw^*,\ba_{t,k}\>
    \\&\leq  
    \eps_1+\sum_{k=1}^Ap_{t,k}\|\tilde{\bw}_k-\bw^*\|_{\bA_t}\cdot\|\ba_{t,k}\|_{\bA^{-1}_t}
    \leq \eps_1+ 2\alpha\E_{k\sim \bp_t}\|\ba_{t,k}\|_{\bA^{-1}_t}.
\end{align*}
 Moreover, note that $ \max_{j\in[A]}\<\bw^*,\ba_{t,j}\>-\<\bw^*,\ba_{t,k}\>\leq 2B_wB_a$ and $\|\ba_{t,k}\|_{\bA^{-1}_t}\leq B_a/\sqrt{\lambda}$.
Summing over $t\in[T]$ and using the tower property of martingales, we obtain
\begin{align*}
&\qquad\E_{\inst\sim\prior}\Big[\sum_{t=1}^\totlen\max_{k}\<\ba_{t,k},\bw^*\>-\totreward_{\inst,\sAlg_{\sLinUCB(\temp)}}(\totlen)\Big]\\
&=
\E\Big[\max_{j\in[A]}\<\bw^*,\ba_{t,j}\>-\sum_{k=1}^Ap_{t,k}\<\bw^*,\ba_{t,k}\>
 \Big]   
\\
&\leq \E[2\sum_{t=1}^T\alpha\E_{k\sim \bp_t}\|\ba_{t,k}\|_{\bA^{-1}_t}+\eps_1 T +2B_wB_a T\cdot\mathbf{1}_{\{\cE^c_0\}}]\\
   &\leq
\E[2\sum_{t=1}^T\alpha\Big(\frac{B_a}{\sqrt{\lambda}}\wedge\|\ba_{t,k}\|_{\bA^{-1}_t}\Big)+\eps_1 T +2B_wB_a T\cdot\mathbf{1}_{\{\cE^c_0\}}]\\
   &\leq 
2\E\Big[\alpha\sqrt{T}(\frac{B_a}{\sqrt{\lambda}}+1)\sqrt{\sum_{t=1}^T(1\wedge\|\ba_{t,k}\|^2_{\bA^{-1}_t})}+\eps_1 T\Big]+2B_wB_a T\P(\cE^c_0)\\
&\leq 
\sqrt{8 d({B_a}/{\sqrt{\lambda}}+1)^2T\alpha^2\log((d\lambda+TB_a^2)/(d\lambda))}+\eps_1 T+1,
\end{align*}
where the fourth line uses the fact that
\begin{align*}
 \frac{B_a}{\sqrt{\lambda}}\wedge\|\ba_{t,k}\|_{\bA^{-1}_t}\leq (\frac{B_a}{\sqrt{\lambda}}+1)\cdot(1\wedge\|\ba_{t,k}\|^2_{\bA^{-1}_t})
\end{align*}
and Cauchy-Schwatz inequality, the last line follows from Lemma 19.4 of~\cite{lattimore2020bandit}. Plugging in $\eps_1=1/\sqrt{T}$ and Eq.~\eqref{eq:recall_alpha_formula} gives the upper bound on expected regret 
\begin{align*}
\E_{\inst\sim\prior}\Big[\sum_{t=1}^\totlen\max_{k}\<\ba_{t,k},\bw^*\>-\totreward_{\inst,\sAlg_{\sLinUCB(\temp)}}(\totlen)\Big]&\leq\cO(d\sqrt{T}\log(T))
\end{align*}
for soft LinUCB with parameter $\temp.$

Moreover, the second part of Theorem~\ref{thm:smooth_linucb} (i.e., the upper bound on $\log\cN_{\tfparspace}$) follows directly from Lemma~\ref{lm:cover_num_bound} and Eq.~\eqref{eq:linucb_tf_param}.

%% file: Sections_arxiv/app-TS_arxiv.tex
\section{Thompson sampling for stochastic linear bandit}\label{example:ts-app}

Throughout this section, we use $c>0$ to denote universal constants whose values may vary from line to line.
Moreover, for notational simplicity, we use $\conO(\cdot)$ to hide universal constants, $\cO(\cdot)$ to hide polynomial terms in the problem parameters  $(\lambda^{\pm1},\Tpsparn^{\pm1},b_a^{-1},B_a)$, and $\tcO(\cdot)$ to hide both poly-logarithmic terms in $(\neuron,\weightn,T,A,d,1/\eps,1/\delta_0)$ and polynomial terms in $(\lambda^{\pm1},\Tpsparn^{\pm1},b_a^{-1},B_a)$. We also use the bold font letter $\ba_t\in\R^d$ to denote the selected action $\action_t$ at time $t\in[\totlen]$.

%\sm{Write a road map. Here I copied the roadmap for LinUCB proof section. This section is organized as follows. Section~\ref{sec:tf_embed_bandit} discusses the input-output format of transformers for the stochastic linear bandit environment. Section~\ref{sec:soft-LinUCB} describes the LinUCB and the soft LinUCB algorithms. Section~\ref{app:approx-ridge-estimator} introduces and proves a lemma on approximating the linear ridge regression estimator, which is important for proving Theorem~\ref{thm:approx_smooth_linucb}. We prove Theorem~\ref{thm:approx_smooth_linucb} in Section~\ref{sec:pf_thm:approx_smooth_linucb} and prove Theorem~\ref{thm:smooth_linucb} in Section~\ref{sec:pf_thm:smooth_linucb}. }
This section is organized as follows. Section~\ref{app:ts_algorithm_formula} describes the Thompson sampling algorithm for stochastic linear bandits. Section~\ref{app:thompson_def_ass} introduces some additional definitions, assumptions, and the formal version of Theorem~\ref{thm:approx_thompson_linear} as in Theorem~\ref{thm:approx_thompson_linear-formal}.
We prove Theorem~~\ref{thm:approx_thompson_linear-formal} in Section~\ref{sec:pf_thm:approx_thompson_linear-formal} and prove Theorem~\ref{thm:ts_linear_regret} in Section~\ref{sec:pf_prop:ts_linear_regret}. Lastly, the proof of Lemma~\ref{lm:lip_of_tps} used in the proof of Theorem~\ref{thm:approx_thompson_linear-formal} is provided in Section~\ref{sec:pf_lm:lip_of_tps}.

\subsection{Thompson sampling algorithm}\label{app:ts_algorithm_formula}
Consider the stochastic linear bandit setup as in Section~\ref{sec:LinUCB-statement}, but instead we assume a  Gaussian prior distribution $\bw^\star\sim \cN(0,\Tpspar\id_d)$ and Gaussian noises $\{ \eps_t \}_{t \ge 0} \sim_{iid} \cN(0,\Tpsparn)$. Furthermore, we assume there exist $(b_a, B_a)$ such that $b_a\leq\ltwo{\ba_{t,k}}\leq B_a$. At each time $t\in[\totlen]$, Thompson sampling consists of the following steps:
\begin{enumerate}
    \item Computes 
    \[
    \Tpsmean_t:= \Big(\frac{\Tpsparn}{\Tpspar  }\id_{d}+\sum_{j=1}^{t-1}\ba_j\ba_j^\top \Big)^{-1}\sum_{j=1}^{t-1}\ba_j y_j,~~~
\Tpscov_t:=\frac{\Tpsparn}{\Tpspar  }\id_{d}+\sum_{j=1}^{t-1}\ba_j\ba_j^\top.
\]
\item Selects the action $\ba_{t}=\ba_{t,k}$ with probability 
\[
\P_{\Tpssam_t\sim\cN(\Tpsmean_t,\Tpsparn\Tpscov_t^{-1})}\Big(k =\argmax_{j \in[A]}\<\ba_{t,j},\Tpssam_t\> \Big). 
\]
\end{enumerate}

Note that Thompson sampling is equivalent to the posterior sampling procedure in our stochastic linear bandit setup, i.e., we select an action with probability that equals  the posterior probability of the action being optimal.  We allow $\Tpspar$ to be either some constant independent of $\totlen,d$,  or has the form $\Tpspar=\Tpspar_0/d$ for some constant $\Tpspar_0>0$.  The latter case is considered so that the bandit parameter vector $\bw^*$ has $\ell_2$ norm of order  $\tcO(1)$ with high probability. In this case, we use $\cO(\cdot)$ to hide polynomial terms in the problem parameters $(\lambda_0^{\pm1},\Tpsparn^{\pm1},b_a^{-1},B_a)$, and $\tcO(\cdot)$ to hide both poly-logarithmic terms in $(\neuron,\weightn,T,A,d,1/\eps,1/\delta_0)$ and polynomial terms in  $(\lambda_0^{\pm1},\Tpsparn^{\pm1},b_a^{-1},B_a)$.

% \lc{In Proposition~\ref{prop:ts_linear_regret} we only need $\lambda$ less than constant, but in the construction of TF $\lambda$ can not be too smaller because of the $\log(\tilde\lambda)$ dependence and the $poly(\tilde\lambda)$ dependence in step 2c.}

\subsection{Definitions and assumptions}\label{app:thompson_def_ass}

For any actions $\ba_{t,1},\ldots,\ba_{t,A}\in\R^{d}$, we define 
\begin{align*}
f_k(\ba_{t,1},\ldots,\ba_{t,A};\Tpsmean_t,\Tpsparn\Tpscov_t^{-1}):=\log\P_{\Tpssam_t\sim\cN(\Tpsmean_t,\Tpsparn\Tpscov_t^{-1})}\Big(k =\argmax_{j\in[A]}\<\ba_{t,j},\Tpssam_t\>\Big). 
\end{align*}  
For any $k\in[A]$, $\bx_1,\ldots,\bx_A\in\R^{d}$, $y_1,\ldots,y_A\in\R$, we introduce $$
g_k(\bx_1,\ldots,\bx_A,y_1,\ldots,y_A):=\log\P_{\bz\sim\cN(0,\id_d)} \Big(\<\bx_k-\bx_j,\bz\>+y_k-y_j\geq0,\text{~for all~}j\in[A] \Big). 
$$
It can be verified that $$f_k(\ba_1,\ldots,\ba_A;\Tpsmean_t,\Tpsparn\Tpscov_t^{-1})=g_k(\sqrt{\Tpsparn}\Tpscov_t^{-1/2}\ba_{t,1},\ldots,\sqrt{\Tpsparn}\Tpscov_t^{-1/2}\ba_{t,A},\<\Tpsmean_t,\ba_{t,1}\>,\ldots,\<\Tpsmean_t,\ba_{t,A}\>).$$  

For any $\trunprob\in[0,1]$, we also define the truncated log-probabilities
\begin{align*}
f_{k,\trunprob}(\ba_{t1},\ldots,\ba_{t,A};\Tpsmean_t,\Tpsparn\Tpscov_t^{-1})&:=\log\Big[\P \Big( k =\argmax_{j \in[A]}\<\ba_{t,j},\Tpssam_t\> \Big)\vee\trunprob \Big],\\
g_{k,\trunprob}(\bx_1,\ldots,\bx_A,y_1,\ldots,y_A)&:=\log \Big[ \P \Big(\<\bx_k-\bx_j,\bz\>+y_k-y_j\geq0,\text{~for all~}j\in[A] \Big) \vee \trunprob \Big].
\end{align*}
Define in addition the region $\Trunreg_{\Trunregp}:=\{\bx_1,\ldots,\bx_A,y_1,\ldots,y_k: \|\bx_i-\bx_j\|_2\geq\Trunregp,~~\text{for all~~} i\neq j\}$. We verify that on the set $\Trunreg_\Trunregp$, the function $g_{k,\trunprob}$ is Lipschitz continuous in any of its arguments (see Lemma~\ref{lm:lip_of_tps} for more).
 
We adopt the following definition in~\cite{bai2023transformers}. 
\begin{definition}[Approximability by sum of relus]\label{def:general_mlp_approx_new}
     A function $g: \mathbb{R}^d \rightarrow \mathbb{R}$ is $(\eps, R, \neuron, \weightn)$-approximable by sum of relus, if there exists a ``$(\neuron, \weightn)$-sum of relus'' function
\begin{align*}
f_{\neuron, \weightn}(\mathbf{z})=\sum_{\ssm=1}^\neuron c_\ssm \sigma\left(\mathbf{w}_\ssm^{\top}[\mathbf{z} ; 1]\right) \quad \text { with } \quad \sum_{\ssm=1}^\neuron\left|c_\ssm\right| \leq \weightn, \max _{\ssm \in[\neuron]}\left\|\mathbf{w}_\ssm\right\|_1 \leq 1, \mathbf{w}_\ssm \in \mathbb{R}^{d+1},c_\ssm \in \mathbb{R},~
\end{align*}
such that $\sup _{\mathbf{z} \in[-R, R]^d}\left|g(\mathbf{z})-f_{\neuron,\weightn}(\mathbf{z})\right| \leq \eps$.

% Similarly,  a function $g: \sZ_{\geq0}^d \rightarrow \mathbb{R}$ is $(\eps, R, M, C)$-approximable by sum of relus, if there exists 
% \begin{align*}
% f_{M, C}(\mathbf{z})=\sum_{m=1}^M c_m \sigma\left(\mathbf{a}_m^{\top}[\mathbf{z} ; 1]\right) \quad \text { with } \quad \sum_{m=1}^M\left|c_m\right| \leq C, \max _{m \in[M]}\left\|\mathbf{w}_m\right\|_1 \leq 1, \quad \mathbf{w}_m \in \mathbb{R}^{d+1}, c_m \in \mathbb{R},
% \end{align*}
% such that $\sup _{\mathbf{z} \in[0, R]^d\cap\sZ_{\geq0}^d}\left|g(\mathbf{z})-f_{M, C}(\mathbf{z})\right| \leq \eps$.
\end{definition}

\begin{assumption}[Approximation of log-posterior probability]\label{ass:thompson_mlp_approx_linear}
 There exist  $\neuron,\weightn>0$  depending on $(1/\eps,1/\trunprob,\\1/\Trunregpa,R_\delta,A)$  such that for any $\eps>0,\trunprob\in(0,1),\Trunregpa>0,\delta\in(0,1/2)$ and $k\in[A]$, $g_{k,\trunprob}(\bx_1,\ldots,\bx_A,y_1,\ldots,y_A)$ is $(\eps,R_\delta,\neuron,\weightn)$-approximable by sum of relus on $\Trunreg_\Trunregpa$ with
$
R_\delta:=2B_a\sqrt{\lambda}(1+2\sqrt{\log(2/\delta)}+\sqrt{d})=\tcO(\sqrt{\lambda d}).
$

\end{assumption}
Assumption~\ref{ass:thompson_mlp_approx_linear} states that the (truncated) log-policy of  Thompson sampling  can be approximated via a two-layer MLP on a compact set with  $\tcO(\sqrt{d})$-radius when $\lambda=\tcO(1)$ (or with $\tcO(1)$--radius when $\lambda=\lambda_0/d=\tcO(1/d)$).  %For Theorem~\ref{thm:approx_thompson_linear-formal} to hold,  alternative assumptions\lc{add reference here?} on the uniform approximability of many layers of MLP, in which only a smaller $\neuron$ is required,  may be made in place of Assumption~\ref{ass:thompson_mlp_approx_linear}. We present Assumption~\ref{ass:thompson_mlp_approx_linear} on a single MLP layer for conceptual simplicity.

\begin{assumption}[Difference between the actions]\label{ass:thompson_mlp_diff_action_linear}
 There exists some $\Trunregp>0$  such that for all instances $\inst$ and any time $t\in[T]$, we have $\|\ba_{t,j}-\ba_{t,k}\|_2\geq\Trunregp$ for all $1\leq j< k\leq A$.
\end{assumption}

With the definitions and assumptions at hand, we now present the formal statement of Theorem~\ref{thm:approx_thompson_linear} as in Theorem~\ref{thm:approx_thompson_linear-formal}.

\begin{theorem}[Approximating the Thompson sampling, Formal statement of Theorem~\ref{thm:approx_thompson_linear}]\label{thm:approx_thompson_linear-formal}
For any $0<\delta_0<1/2$, consider the same embedding mapping $\embedmap$ and extraction mapping $\extractmap$ as for soft LinUCB in \ref{sec:tf_embed_bandit},
%\sm{Link}\lc{same as before},
and consider the standard concatenation operator $\cat$. Under Assumption~\ref{ass:thompson_mlp_approx_linear},~\ref{ass:thompson_mlp_diff_action_linear}, for $\eps<(\Trunregp\wedge1)/4$, there exists a  transformer $\TF_\btheta^{\clipval}(\cdot)$ with $\log \clipval = \tcO(1)$, 
\begin{align}
&D=\tcO(T^{1/4}Ad),~L= \tcO(\sqrt{T}),~ M =\tcO(AT^{1/4}),~\hidden=\tcO(A(T^{1/4}d+\neuron))~,\notag\\
&~~~\nrmp{\btheta}\leq \tcO(T+AT^{1/4}+\sqrt{ \neuron A}+\weightn),\label{eq:ts_tf_param-formal}
\end{align} 
such that with probability at least $1-\delta_0$ over $(\inst, \dset_{\totlen}) \sim \P_{\prior}^{\sAlg}$ for any $\sAlg$, we have
\[
\log \sAlg_{\TS}(\ba_{t,k}|\dset_{t-1},\state_t) - \log \sAlg_{\tfpar}(\ba_{t,k}|\dset_{t-1},\state_t) \leq \eps, ~~~ \text{for all } t\in[T],k\in[A].
\]
Here  $\neuron,\weightn$ are the values defined in   Assumption~\ref{ass:thompson_mlp_approx_linear} with $\trunprob=\eps/(4A),\Trunregpa=\Trunregp,\delta=\delta_0$, and $\tcO(\cdot)$ hides polynomial terms in $(\lambda^{\pm1},\Tpsparn^{\pm1},b_a^{-1},B_a)$ and poly-logarithmic terms in $(\neuron,\weightn,\totlen,A,d,1/\delta_0,1/\eps)$.
\end{theorem}

\subsection{Proof of Theorem~\ref{thm:approx_thompson_linear-formal} (and hence Theorem~\ref{thm:approx_thompson_linear}) }\label{sec:pf_thm:approx_thompson_linear-formal}

We construct a transformer implementing the following steps at each time $t\in[T]$ starting with $\bh^{\star}_{2t-1}=\bh^{\pre,\star}_{2t-1}$ for $\star\in\{\parta,\partb,\partc,\partd\}$ 
\begin{align}\label{eq:roadmap_ts}
   &~ \bh_{2t-1}=
    \begin{bmatrix}
    \bh_{2t-1}^{\pre,\parta} \\  \bh_{2t-1}^{\pre,\partb}\\  \bh_{2t-1}^{\pre,\partc}\\   \bh_{2t-1}^{\pre,\partd}
\end{bmatrix}
\xrightarrow{\text{step 1}}
   \begin{bmatrix}
    \bh_{2t-1}^{\pre,\{\parta,\partb,\partc\}} \\
        \widehat{\Tpsmean_t}\\ \star\\ \bzero \\\posv
\end{bmatrix}
\xrightarrow{\text{step 2}}
   \begin{bmatrix}
    \bh_{2t-1}^{\pre,\{\parta,\partb,\partc\}}\\\widehat{\Tpsmean_t} 
        \\
\widehat{\Tpscov_t^{1/2}\ba_{t,1}}
\\\vdots\\
         \widehat{\Tpscov_t^{1/2}\ba_{t,A}}\\
        \star\\ \bzero \\\posv
\end{bmatrix}
\xrightarrow{\text{step 3}}
   \begin{bmatrix}
    \bh_{2t-1}^{\pre,\{\parta,\partb,\partc\}}\\\widehat{\Tpsmean_t} 
        \\
 \sqrt{\Tpsparn}\widehat{\Tpscov_t^{-1/2}\ba_{t,1}}
\\\vdots\\
        \sqrt{\Tpsparn} \widehat{\Tpscov_t^{-1/2}\ba_{t,A}}\\
        \<\widehat{\Tpsmean_t},\ba_{t,1}\>\\\vdots\\
        \<\widehat{\Tpsmean_t},\ba_{t,A}\>
        \\
        \star\\ \bzero \\\posv
\end{bmatrix}
\\ &~\xrightarrow{\text{step 4}}
\begin{bmatrix}
    \bh_{2t-1}^{\pre,\{\parta,\partb\}}\\ \hat v_{t1}\\\vdots\\ \hat{v}_{tA} \vspace{0.5em}\\ \bh_{2t-1}^{\partd}
\end{bmatrix}
=:
\begin{bmatrix}
    \bh_{2t-1}^{\post,\parta} \\  \bh_{2t-1}^{\post,\partb}\\  \bh_{2t-1}^{\post,\partc}\\   \bh_{2t-1}^{\post,\partd}
\end{bmatrix},\notag
\end{align}
where $\posv:=[t,t^2,1]^\top$; $\Tpsmean_t,\Tpscov_t$ are the mean and covariance of the distribution we sample $\tilde\bw$ from; $\hat v_{tk}$ are approximations to $v^*_{tk}:=\log \P(j=\argmax_{k\in[A]}\<\ba_{t,k},\Tpssam_t\>)$. In addition, we use $\bh^\star,\star\in\{\parta,\partb,\partc,\partd\}$ to denote the corresponding parts of a token vector $\bh$. After passing through the transformer, we obtain the policy
$$
\sAlg_\tfpar(\cdot|\dset_{t-1},\state_t):=\frac{\exp(\bh^{\post,\partc}_{2t-1})}{\|\exp(\bh^{\post,\partc}_{2t-1})\|_1}=\frac{\exp(\hat \bv_t)}{\|\exp(\hat \bv_t)\|_1}\in\Delta^A.
$$

In step 1---3 of~\eqref{eq:roadmap_ts},  we use transformer to approximately generate the arguments $$(\sqrt{\Tpsparn}\Tpscov_t^{-1/2}\ba_{t,1},\ldots,\sqrt{\Tpsparn}\Tpscov_t^{-1/2}\ba_{t,A},\<\Tpsmean_t,\ba_{t,1}\>,\ldots,\<\Tpsmean_t,\ba_{t,A}\>)$$ of the function $g_k$ (or $g_{k,\Trunregpa}$), and in step 4 of~\eqref{eq:roadmap_ts}, we use transformer to approximate the truncated log-probability $g_{k,\trunprob}$ for some $\trunprob\in(0,1)$ by exploiting Assumption~\ref{ass:thompson_mlp_approx_linear},~\ref{ass:thompson_mlp_diff_action_linear}.

For any $0<\delta_0<1/2$, define $B_w:=\sqrt{\lambda}\big(\sqrt{d}+2\sqrt{\log(2/\delta_0)}\big)$ and the event $$\hpevent_{\delta_0}:=\{\max_{t\in[T]}|\eps_t|\leq\sqrt{2\Tpsparn\log(2T/\delta_0)}\}\cup \{\|\bw^*\|_2\leq  B_w\}.$$ Then by  a standard tail bound for gaussian variables $\{\eps_t\}_{t=1}^\totlen$,  a union bound over $t\in[T]$, and Eq.~(4.3) in~\cite{laurent2000adaptive}, we have $$\P(\hpevent_{\delta_0})\geq 1-{\delta_0}.$$
We claim the following results which we will prove later.

\begin{enumerate}[label=Step \arabic*,ref= \arabic*]
    \item\label{ts_step1} Under the high probability event $\hpevent_{\delta_0}$, for any $\eps>0$, 
    there exists a transformer $\TF_\btheta(\cdot)$ with 
\begin{align*}&L=\Big\lceil2\sqrt{2T}\sqrt{\frac{B_a^2+\widetilde\lambda}{\widetilde\lambda}}\log\Big(\frac{(2T(B_a^2+\widetilde\lambda)+\widetilde\lambda)TB_a(B_aB_w+\sqrt{2\Tpsparn\log(2T/{\delta_0}))}}{\widetilde\lambda^2\eps}\Big)\Big\rceil=\tcO(\sqrt{T}),\\
&~~~\max_{\ell\in[L]}M^{(l)}\leq4,~~\max_{\ell\in[L]}\hidden^\lth\leq 4d,~~~\nrmp{\btheta}\leq  10+(\tilde\lambda+2)/(B_a^2+\tilde\lambda), \end{align*}
where $\widetilde\lambda:=\Tpsparn/\Tpspar$
 that implements step 1 in~\eqref{eq:roadmap_ts} with  $\|\widehat\Tpsmean_t-\Tpsmean_t\|_2\leq\eps$.
    \item\label{ts_step2}
    For any $\eps>0$, there exists s transformer $\TF_\btheta(\cdot)$ with     $$L=\tcO(\sqrt{T}),~~~\max_{\ell\in[L]} M^{(\ell)}=\tcO(AT^{1/4}),~~~\max_{\ell\in[L]} \hidden^{(\ell)}=\tcO(T^{1/4}Ad),~~~\nrmp{\btheta}\leq \tcO(T+AT^{1/4}) $$
 that implements step 2 in~\eqref{eq:roadmap_ts} such that $\|\widehat{\Tpscov_t^{1/2}\ba_{t,k}}-{\Tpscov_t^{1/2}\ba_{t,k}}\|_2\leq \eps$ for all $k\in[A]$.
   \item\label{ts_step3} 
   Under the high probability event $\hpevent_{\delta_0}$, 
   for any $\eps>0$, assume Step~\ref{ts_step1},~\ref{ts_step2} above are implemented with the approximation error less than $\eps/B_a,\eps\tilde\lambda/\sqrt{4\Tpsparn}$ respectively, then there exists a transformer $\TF_\btheta(\cdot)$ with 
 \begin{align*}
  &L=\lceil2+2\sqrt{2 T(B_a^2+\tilde\lambda)/\tilde\lambda}\log((1+\padecond)4\sqrt{\Tpsparn}\sqrt{T(B_a^2+\tilde\lambda)}B_a/\eps)\rceil=\tcO(\sqrt{T}),~~\max_{\ell\in[L]}M^{(l)}=4A,\\
  &~~~\qquad\max_{\ell\in[L]} \hidden^\lth\leq 4Ad,~~~~ \nrmp{\btheta}\leq  \tcO(T+A) 
  \end{align*}
 that implements step 3 in~\eqref{eq:roadmap_ts} with   $\|\sqrt{\Tpsparn}\widehat{\Tpscov_t^{-1/2}\ba_{t,k}}-\sqrt{\Tpsparn}\Tpscov_t^{-1/2}\ba_{t,k}\|_2 \leq\eps,|\<\widehat\Tpsmean_t,\ba_{t,k}\>-\<\Tpsmean_t,\ba_{t,k}\>|\leq\eps$ for all $k\in[A]$. 
 
  \item\label{ts_step4} Under  Assumption~\ref{ass:thompson_mlp_approx_linear},~\ref{ass:thompson_mlp_diff_action_linear} and the high probability event $\hpevent_{\delta_0}$, suppose the approximation error $\eps_3$ in Step~\ref{ts_step3} satisfies $\eps_3\leq R_{\delta_0}/2=\tcO(\sqrt{\lambda d})$, and  suppose the vector $$\bigvec:=(\sqrt{\Tpsparn}\widehat{\Tpscov_t^{-1/2}\ba_{t,1}},\ldots
        \sqrt{\Tpsparn}\widehat{\Tpscov_t^{-1/2}\ba_{t,A}},
        \<\widehat{\Tpsmean_t},\ba_{t,1}\>\ldots
        \<\widehat{\Tpsmean_t},\ba_{t,A}\>)$$ lies in $\Trunreg_{\Trunregp/2}$, for any $\eps>0$ there exists an MLP-only transformer $\TF_\btheta(\cdot)$ with
    $$
    L=1,~~~D'= \neuron A,~~\lops{\bW_1}\leq \sqrt{\neuron A}, ~~~\lops{\bW_2}\leq \weightn$$
  that implements step 4 in~\eqref{eq:roadmap_ts} such that $|\widehat v_{tk}-g_{k,\trunprob}(\bigvec)|\leq \eps$ for all $k\in[A]$ amd $\trunprob=c\eps/A$ for some universal constant $c>0$.\footnote{Note $\neuron,\weightn$ in the formula implicitly depend on $1/\eps$.}

\end{enumerate}

To complete the proof, we in addition present the following lemma.
\begin{lemma}\label{lm:lip_of_tps}
For any $\trunprob\in(0,1),\Trunregp>0$, $g_{k,\trunprob}(\bx_1,\ldots,\bx_A,y_1,\ldots,y_A)$ is $1/2$-Holder continuous in its arguments on $\Trunreg_\Trunregp$, namely,
    \begin{align*}
        &\qquad|g_{k,\trunprob}(\bx_1,\ldots,\bx_j,\ldots,\bx_A,y_1,\ldots,y_A)-g_{k,\trunprob}(\bx_1,\ldots,\bx'_j,\ldots\bx_A,y_1,\ldots,y_A)|
        \\
        &
        \leq \frac{2A}{\trunprob}\Big(\sqrt{\frac{2\|\bx_j-\bx_j'\|_2}{\Trunregp}}+\frac{2\|\bx_j-\bx_j'\|_2}{\Trunregp}\Big),
        \\
           &\qquad |g_{k,\trunprob}(\bx_1,\ldots,\bx_A,y_1,\ldots,y_j,\ldots,y_A)-g_{k,\trunprob}(\bx_1,\ldots\bx_A,y_1,\ldots,y_j',\ldots,y_A)|
           \\
        &\leq 
       \frac{2A|y_j-y_j'|}{\Trunregp\trunprob}
    \end{align*}
    for any \begin{align*}&(\bx_1,\ldots,\bx_j,\ldots,\bx_A,y_1,\ldots,y_A),~~(\bx_1,\ldots,\bx_j',\ldots,\bx_A,y_1,\ldots,y_A)\in\Trunreg_\Trunregp ,\\&(\bx_1,\ldots,\bx_A,y_1,\ldots,y_j,\ldots,y_A),~~(\bx_1,\ldots,\bx_A,y_1,\ldots,y_j',\ldots,y_A)\in\Trunreg_\Trunregp
    \end{align*} 
  for all $k,j\in[A]$. 
\end{lemma}
See the proof in Section~\ref{sec:pf_lm:lip_of_tps}. 

Now, we complete the proof by combining Step~\ref{ts_step1}---~\ref{ts_step4} and using Lemma~\ref{lm:lip_of_tps}.

Let $\eps_1,\eps_2,\eps_3,\eps_4$ denote the approximation errors $\eps$ appearing in Step~\ref{ts_step1},~\ref{ts_step2},~\ref{ts_step3},~\ref{ts_step4}, respectively. W.l.o.g., we assume $\eps_1,\eps_2,\eps_3,\eps_4<1/4\wedge \Trunregp/4$. Define the vector $$\bigvec^*:=(\sqrt{\Tpsparn}{\Tpscov_t^{-1/2}\ba_{t,1}},\ldots
        \sqrt{\Tpsparn}{\Tpscov_t^{-1/2}\ba_{t,A}},
        \<{\Tpsmean_t},\ba_{t,1}\>\ldots
        \<{\Tpsmean_t},\ba_{t,A}\>).$$
By Assumption~\ref{ass:thompson_mlp_diff_action_linear} and a triangular inequality,  we have $\bigvec,\bigvec^*\in\Trunreg_{\Trunregp/2}$. 
 By the Lipschitz continuity of $f(x)=\exp(x)$ on $(-\infty,1.5]$, we have
\begin{align*}
|\exp(\hat v_{tk})- \sAlg_{\TS}(\ba_{t,k}|\dset_{t-1},\state_t)|&\leq|\exp(\hat v_{t,k})- \sAlg_{\TS}(\ba_{t,k}|\dset_{t-1},\state_t)\vee\trunprob|+\trunprob \\
&
\leq e^{3/2}(|\hat v_{tk}-g_{k,\trunprob}(\bigvec)|+|g_{k,\trunprob}(\bigvec)-g_{k,\trunprob}(\bigvec^*)|)+\trunprob,\\
&\leq
e^{3/2}\Big(\eps_4+\frac{2A^2}{\trunprob}\Big(\sqrt{\frac{2\eps_3}{\Trunregp}}+\frac{2\eps_3}{\Trunregp}\Big)+\frac{2A^2\eps_3}{\Trunregp\trunprob}\Big)+\trunprob=:\eps_5,
\end{align*}
where the second inequality uses 
$$g_{k,\trunprob}(\bigvec^*)=\log[\sAlg_{\TS}(\ba_{t,k}|\dset_{t-1},\state_t)\vee\trunprob],$$ and the 
third inequality uses Lemma~\ref{lm:lip_of_tps} and Step~\ref{ts_step4}. Therefore, 
\begin{align*}
 |\sum_{k=1}^A \exp(\hat v_{t,k})-1|\leq A\eps_5.   
\end{align*}
and the constructed transformer $\TF_\tfpar$ satisfies (assume $A\eps_5<1$)
\begin{align*}
&\quad~~ \log\sAlg_{\TS}(\ba_{t,k}|\dset_{t-1},\state_t)-\log\sAlg_{\tfpar}(\ba_{t,k}|\dset_{t-1},\state_t)\\
&\leq (\log[\sAlg_{\TS}(\ba_{t,k}|\dset_{t-1},\state_t)\vee\trunprob]-\hat v_{t,k}
)+
\log(\sum_{k=1}^A\exp({\hat v}_{t,k}))\\
&\leq
|\hat v_{tk}-g_{k,\trunprob}(\bigvec)|+|g_{k,\trunprob}(\bigvec)-g_{k,\trunprob}(\bigvec^*)|+A\eps_5
\\
&\leq (A+1)\eps_5,
\end{align*}
where third line uses $\log(1+x)<x$.
Finally, for the prespecified $\eps>0$,  choosing $\eps_1,\eps_2,\eps_3,\eps_4,\trunprob$ such that $\eps_5\leq\eps/(2A)$ gives 
\begin{align*}  
\log\sAlg_{\TS}(\ba_{t,k}|\dset_{t-1},\state_t)-\log\sAlg_{\tfpar}(\ba_{t,k}|\dset_{t-1},\state_t)\leq\eps.
\end{align*}

 This can be done via choosing $\eps_1= c_1\eps^4,\eps_2=c_2\eps^4,\eps_3=c_3\eps^4,\eps_4=c_4\eps,\trunprob=\eps/(4A)$, where $c_i~(i=1,2,3,4)$  hide values that could depend polynomially on $(A,1/\Trunregp)$, such that $\eps_5\leq\eps/(4A)$.

 Combining the construction in  Step~\ref{ts_step1}---~\ref{ts_step4} yields Theorem~\ref{thm:approx_thompson_linear-formal}.

 Similar to the proof of Theorem~\ref{thm:approx_smooth_linucb}, from the proof of each step, we can verify that the token dimension $D$ can be chosen to be of order $\tcO(T^{1/4}Ad)$ (see the proof of Step~\ref{ts_step2b} for details). Moreover, due to the convergence guarantee for each iteration of AGD in Proposition~\ref{prop:conv_gd_agd}, we can be verified that there exists some sufficiently large value $\clipval>0$ with $\log \clipval=\tcO(1)$ such that  we have $\| \bh_i^{\lth} \|_{2} \leq\clipval$
 for all layer $\ell\in[L]$ and all token $i\in[2\totlen]$ in our TF construction. Therefore, $\TF^\clipval_\btheta$ and $\TF^\infty_\btheta$ yield identical outputs for all token matrices considered, and hence we do not distinguish them in the proof of each step.

\paragraph{Proof of Step~\ref{ts_step1}} Note that $\Tpsmean_t$ is a ridge estimator of $\bw_*$ with parameter $\widetilde\lambda=\Tpsparn/\Tpspar$ and the noise $\sup_t|\eps_t|\leq \sqrt{2\Tpsparn\log(T/{\delta_0})}$. Step~\ref{ts_step1} follows immediately  from the second part of Lemma~\ref{lm:approx_ridge}.

\paragraph{Proof of Step~\ref{ts_step2}}
By the boundedness assumption of the actions, we have $$\tilde\lambda\leq\sigma_{\min}(\Tpscov_t)\leq\sigma_{\max}(\Tpscov_t)\leq T(B_a^2+\tilde\lambda).$$ 
Define the condition number $\padecond=T(B_a^2+\tilde\lambda)/\tilde\lambda$ and $\prodeig:=\sqrt{T\tilde\lambda(B_a^2+\tilde\lambda)}$. 
Using the Pade decomposition for the  square root function in Theorem 3.1 and the discussion afterward in~\cite{lu1998pade}, we have
\begin{align*} 
\Tpscov^{1/2}_t&=(\tilde\lambda\id_d+\sum_{j=1}^{t-1}\ba_j\ba_j^\top)^{1/2}=\sqrt{\prodeig}\Big(\id_d+\frac{(\Tpscov_t-\prodeig \id_d)}{\prodeig}\Big)^{1/2}\\
&=
\sqrt{\prodeig}\Big[\id_d+\sum_{j=1}^m \Big(\id_d+\frac{b_j^{(m)}(\Tpscov_t-\prodeig \id_d)}{\prodeig}\Big)^{-1}\frac{a_j^{(m)}(\Tpscov_t-\prodeig \id_d)}{\prodeig}\Big]+\Errmat_{m}
\end{align*}
for any $m\geq0$, 
where 
$$
a_j^{(m)}=\frac{2}{2 m+1} \sin ^2 \frac{j \pi}{2 m+1}, \quad b_j^{(m)}=\cos ^2 \frac{j \pi}{2 m+1},
$$ and the error term $\Errmat_m$ satisfies \begin{align*}
&~\lops{\Errmat_m}\\
\leq&~
\max \left\{2 \sqrt{{T(B_a^2+\tilde\lambda)}}\left[1+\left(\frac{\sqrt{T(B_a^2+\tilde\lambda)}+\sqrt{\mu}}{\sqrt{T(B_a^2+\tilde\lambda)}-\sqrt{\mu}}\right)^{2 m+1}\right]^{-1},~~~2 \sqrt{{\tilde\lambda}}\left[\left(\frac{\sqrt{\mu}+\sqrt{\tilde\lambda}}{\sqrt{\mu}-\sqrt{\tilde\lambda}}\right)^{2 m+1}-1\right]^{-1}\right\}
\\
=&~\max\Big\{2\sqrt{{T(B_a^2+\tilde\lambda)}}[1+(\frac{\padecond^{1/4}+1}{\padecond^{1/4}-1})^{2m+1}]^{-1},2\sqrt{\tilde\lambda}[(\frac{\padecond^{1/4}+1}{\padecond^{1/4}-1})^{2m+1}-1]^{-1}\Big\}\\
\leq&~2\max\Big\{\sqrt{{T(B_a^2+\tilde\lambda)}}(1+\frac{2}{\padecond^{1/4}})^{-2m-1},\sqrt{\tilde\lambda}\Big[(1+\frac{2}{\padecond^{1/4}})^{2m+1}- 1\Big]^{-1}\Big\}.
\end{align*}
Since $(1+2/\padecond^{1/4})^{\padecond^{1/4}/2+1}>e$, it follows that choosing $$
m= \Big(\frac{\padecond^{1/4}}{4}+1\Big)\max\Big\{\Big\lceil\log\Big(\frac{2\sqrt{T(B_a^2+\tilde\lambda)}}{\eps}\Big)\Big\rceil,\Big\lceil\log\Big(\frac{2\sqrt{\tilde\lambda}}{\eps}+1\Big)\Big\rceil\Big\}=\tcO(T^{1/4}).
$$ gives $\lops{\Errmat_m}\leq\eps$ for any $0<\eps<1$.

Thus, using Pade decomposition, we can write $$
\Tpscov_t^{1/2}\ba_{t,k}=\sqrt{\prodeig}\Big[\ba_{t,k}+\sum_{j=1}^m \Big({\prodeig}\id_d+{b_j^{(m)}(\Tpscov_t-\prodeig \id_d)}\Big)^{-1}{a_j^{(m)}(\Tpscov_t-\prodeig \id_d)\ba_{t,k}}\Big]+\Errmat_m^k$$ with $\ltwo{\Errmat_m^k}\leq\eps$ for all $k\in[A]$ and some $m=\tcO(T^{1/4})$. Next, we show that there exists a transformer that can implement the following intermediate steps that give Step~\ref{ts_step2}.
\begin{align}
   &~\begin{bmatrix}
    \bh_{2t-1}^{\pre,\{\parta,\partb,\partc\}} \\
        \widehat{\Tpsmean_t}\\ \star\\ \bzero \\\posv
\end{bmatrix}
\xrightarrow{\text{step 2a}}
   \begin{bmatrix}
    \bh_{2t-1}^{\pre,\{\parta,\partb,\partc\}}\\\widehat{\Tpsmean_t} 
        \\\bzero_{dA}\\
{(\Tpscov_t-\prodeig\id_d)\ba_{t,1}}
\\\vdots\\
        {(\Tpscov_t-\prodeig\id_d)\ba_{t,A}}\\
        \star\\ \bzero \\\posv
\end{bmatrix}\nonumber
\\ &~\xrightarrow{\text{step 2b}}
     \begin{bmatrix}
    \bh_{2t-1}^{\pre,\{\parta,\partb,\partc\}}\\\widehat{\Tpsmean_t} 
        \\\bzero_{dA}\\
{(\Tpscov_t-\prodeig\id_d)\ba_{t,1}}
\\\vdots\\
       {(\Tpscov_t-\prodeig\id_d)\ba_{t,A}}
        \\ 
        \Big({\prodeig}\id_d+{b_1^{(m)}(\Tpscov_t-\prodeig \id_d)}\Big)^{-1}{a_1^{(m)}(\Tpscov_t-\prodeig \id_d)\ba_{t,1}}\\
        \vdots
        \\ 
          \Big({\prodeig}\id_d+{b_m^{(m)}(\Tpscov_t-\prodeig \id_d)}\Big)^{-1}{a_m^{(m)}(\Tpscov_t-\prodeig \id_d)\ba_{t,A}}
          \\\star\\ \bzero \\\posv
\end{bmatrix}
\xrightarrow{\text{step 2c}}
   \begin{bmatrix}
    \bh_{2t-1}^{\pre,\{\parta,\partb,\partc\}}\\\widehat{\Tpsmean_t} 
        \\
\widehat{\Tpscov_t^{1/2}\ba_{t,1}}
\\\vdots\\
         \widehat{\Tpscov_t^{1/2}\ba_{t,A}}\\\vdots\\
        \star\\ \bzero \\\posv
\end{bmatrix}\label{eq:roadmap_ts_step2},
\end{align}
where $\star$ denotes additional terms in $\bh^\partd$ that are not of concern to our analysis.

\begin{enumerate}[label=Step 2\alph*,ref= 2\alph*]
        \item\label{ts_step2a}  There exists an attention-only transformer $\TF_\btheta(\cdot)$ with     $$L=2,~~~\max M^{\lth}=3A,~~~ \nrmp{\btheta}\leq T+2+{\prodeig}\leq\cO(T) $$
 that implements step 2a in~\eqref{eq:roadmap_ts_step2}.
     \item\label{ts_step2b} Denote ${(\Tpscov_t-\prodeig\id_d)\ba_{t,k}}$ by  $\intvec_{k}$ and $\Big({\prodeig}\id_d+{b_m^{(m)}(\Tpscov_t-\prodeig \id_d)}\Big)$ by $\intmat_m$. For any $\eps>0$,   there exists a  transformer $\TF_\btheta(\cdot)$ with    
     \begin{align*}
     &L=2\sqrt{2 T(B_a^2+\tilde\lambda)/\tilde\lambda}\log((1+\padecond)B_aT(B_a^2+\tilde\lambda)/(\tilde\lambda\eps))\rceil=\tcO(\sqrt{T}),\\  &
     \max_{\ell\in[L]}M^{(l)}=4Am=\tcO(T^{1/4}A),~\max_{\ell\in[L]}\hidden^\lth\leq \conO(Adm)=\tcO(T^{1/4}Ad),~ \nrmp{\btheta}\leq  \cO(Am)\leq  \tcO(T^{1/4}A)
     \end{align*}
     approximately implements step 2b in~\eqref{eq:roadmap_ts_step2} such that the output component $\widehat{\intmat_m^{-1}\intvec_k}$ satisfies $\|\widehat{a_j^{(m)}\intmat_m^{-1}\intvec_k}-{a_j^{(m)}\intmat_j^{-1}\intvec_k}\|_2\leq\eps$ for all $j\in[m]$ and $k\in[A]$.
 \item\label{ts_step2c}  There exists an MLP-only transformer $\TF_\btheta(\cdot)$ with     $$L=1,D'=2Ad(m+1)=\tcO(T^{1/4}Ad),~~~\lops{\bW_1}=\sqrt{2}, ~~~\lops{\bW_2}\leq \sqrt{\prodeig}(1+m)=\tcO(T^{3/4})$$  that implements step 2c in~\eqref{eq:roadmap_ts_step2}.
\end{enumerate}
Combining the intermediate steps with the approximation error in Step~\ref{ts_step2b} chosen as $\eps/m$  gives Step~\ref{ts_step2} as desired.

\paragraph{Proof of Step~\ref{ts_step2a}}
For all $k\in[A]$, we choose $\bQ_{k1,k2,k3}^{(1)},\bK_{k1,k2,k3}^{(1)},\bV_{k1,k2,k3}^{(1)}$ such that for even token indices $2j$ with $j\leq t-1$ and odd token indices with $j\leq t$
\begin{align*}
    &\bQ^{(1)}_{k1}\bh^{(0)}_{2t-1}=\begin{bmatrix}
        \ba_{t,k} \\\bzero
    \end{bmatrix},~~ \bK^{(1)}_{k1}\bh^{(0)}_{2j}=\begin{bmatrix}
        \ba_{j}\\\bzero
    \end{bmatrix},~~\bV^{(1)}_{k1}\bh^{(0)}_{2j}=\begin{bmatrix}
       \bzero\\ \ba_j\\ \bzero
    \end{bmatrix},~~ \bK^{(1)}_{k1}\bh^{(0)}_{2j-1}=\begin{bmatrix}
      \bzero
\end{bmatrix},~~\bV^{(1)}_{k1}\bh^{(0)}_{2j-1}=\begin{bmatrix}
       \bzero
    \end{bmatrix}\\
    &
    \bQ^{(1)}_{k2}=-\bQ^{(1)}_{k1},~~  \bK^{(1)}_{k2}=\bK^{(1)}_{k1},~~  \ \bV^{(1)}_{k2}=-\bV^{(1)}_{k1},\\
    &
    \bQ^{(1)}_{k3}\bh^{(0)}_{2t-1}=\begin{bmatrix}
      1\\ -(2t-1)\\1\\\bzero
    \end{bmatrix},~~ \bK^{(1)}_{k3}\bh^{(0)}_{2j-1}=\begin{bmatrix}
      1\\  1\\ 2j-1\\\bzero
    \end{bmatrix},~~ 
    \bK^{(1)}_{k3}\bh^{(0)}_{2j}=\begin{bmatrix}
        1\\1\\  2j\\\bzero
    \end{bmatrix},~~ \bV^{(1)}_{k3}\bh^{(0)}_{2t-1}=\begin{bmatrix}
        \bzero\\ (\tilde\lambda-\prodeig)\ba_{t,k} \\ \bzero
    \end{bmatrix},
\end{align*}
where for each $k\in[A]$, $\bV^{(1)}_{k1}\bh^{(0)}_{2j},\bV^{(1)}_{k3}\bh^{(0)}_{2t-1}$ are supported on the same $d$ entries of $\bh^\partd$. 
It is readily verified that  summing over the attention heads and $k\in[A]$ gives the updates
$$ \bzero_d\mapsto\frac{(\Tpscov_t-\prodeig\id_d)\ba_{t,k}}{2t-1}$$ for all $k\in[A]$. We assume the updated vectors are supported on some $Ad$ coordinates of $\bh_{2t-1}^d$. Moreover, one can choose the matrices such that $\lops{\bQ^{(1)}_{k1,k2,k3}}\leq1,\lops{\bK^{(1)}_{k1,k2,k3}}\leq1,\lops{\bV^{(1)}_{k1,k2,k3}}\leq \max\{1,|\tilde\lambda-\prodeig|\}\leq1+\prodeig$. Therefore the norm of the first layer of the attention-only transformer $\nrmp{\btheta^{(1)}}\leq 2+\prodeig$. 

The second layer is used to multiply the updated vectors by a factor of $2t-1$, namely, to perform the map
$$ \frac{(\Tpscov_t-\prodeig\id_d)\ba_{t,k}}{2t-1}\mapsto{(\Tpscov_t-\prodeig\id_d)\ba_{t,k}}$$ for all $k\in[A]$, where the output vectors are supported on coordinates different from the input vectors (therefore we need $2Ad$ coordinates for embedding in step 2a). This can be achieved by choosing $\lops{\bQ^{(2)}}\leq T, \lops{\bK^{(2)}}\leq T, \lops{\bV^{(2)}}\leq1$
such that \begin{align*}
    &
    \bQ^{(2)}_{1}\bh^{(1)}_{2t-1}=\begin{bmatrix}
         (2t-1)^2\\-T(2t-1)^2\\  1\\ \bzero
    \end{bmatrix},~~ \bK^{(2)}_{1}\bh^{(1)}_{2j-1}=\begin{bmatrix}
        1\\1\\ T(2j-1)^2\\\bzero
    \end{bmatrix},~~ 
    \bK^{(1)}_{k3}\bh^{(1)}_{2j}=\begin{bmatrix}
      1\\  1\\  T(2j)^2\\\bzero
    \end{bmatrix},\\&~~ \bV^{(2)}_{1}\bh^{(1)}_{2t-1}=\begin{bmatrix}
        \bzero\\ \frac{(\Tpscov_t-\prodeig\id_d)\ba_{t,1}}{(2t-1)}
\\\vdots\\
        \frac{(\Tpscov_t-\prodeig\id_d)\ba_{t,A}}{(2t-1)} \\ \bzero
    \end{bmatrix}.
\end{align*} Therefore $\nrmp{\btheta^{(2)}}\leq T+1$ and hence the two layer transformer we constructed has norm $\nrmp{\btheta}\leq T+2+\prodeig$.

\paragraph{Proof of step~\ref{ts_step2b}}
The construction is similar to the construction in Step~\ref{slinucb_step2} of the proof of Theorem~\ref{thm:approx_smooth_linucb}. Hence we only provide a sketch of proof here. Note that 
\begin{align*}
    a_j^{(m)}\intmat_j^{-1}\intvec_k=\argmin_{\bx\in\R^d}\frac{1}{2(2t-1)}\bx^\top\intmat_j\bx-\frac{1}{2t-1}\<\bx, a_j^{(m)}\intvec_{k}\>=:\argmin_{\bx\in\R^d} L_{k,j}(\bx)
\end{align*}
is the global minimizer of a $\tilde\lambda/(2t-1)$-convex and $(B_a^2+\tilde\lambda)$-smooth function with the conditional number $\padecond\leq 2T(B_a^2+\tilde\lambda)/\tilde\lambda$. Since $$\|a_j^{(m)}\intmat_j^{-1}\intvec_k\|_2\leq|a_j^{(m)}|\lops{\intmat_j^{-1}}\|\intvec_k\|_2\leq B_aT(B_a^2+\tilde\lambda)/\tilde\lambda.$$
Therefore by Proposition~\ref{prop:conv_gd_agd} we have
$L=\lceil2\sqrt{2 T(B_a^2+\tilde\lambda)/\tilde\lambda}\log((1+\padecond)B_aT(B_a^2+\tilde\lambda)/(\tilde\lambda\eps))\rceil$ steps of accelerated gradient descent with step size $\eta=1/(B_a^2+\tilde\lambda)$ gives  $\|\widehat{a_j^{(m)}\intmat_j^{-1}\intvec_k}-a_j^{(m)}\intmat_j^{-1}\intvec_k\|_2\leq\eps$. Now, it remains to construct a transformer that can implement the accelerated gradient descent steps. Here we only provide the construction of the gradient $\nabla L_{k,j}(\bx)$ at the $l$-th iteration $\bx=\bx_{k,j}^{\ell-1}\in\R^d$ which belongs to the output after $\ell-1$ transformer layers.  The full construction of AGD steps follows from similar techniques as in Step~\ref{slinucb_step2} of  the proof of  
Theorem~\ref{thm:approx_smooth_linucb}. Concretely, for each layer $\ell\in[L]$ and $k\in[A],j\in[m]$, we choose 
$\bQ_{kj1,kj2,kj3}^{(\ell)},\bK_{kj1,kj2,kj3}^{(\ell)},\bV_{kj1,kj2,kj3}^{(\ell)}$ such that for even token indices $2j$ with $s\leq t-1$ and odd token indices with $s\leq t$
\begin{align*}
    &\bQ_{kj1}^{(\ell)}\bh^{(\ell-1)}_{2t-1}=\begin{bmatrix}
        \bx_{k,j}^{\ell-1}\\\bzero
    \end{bmatrix},~~ \bK_{kj1}^{(\ell)}\bh^{(\ell-1)}_{2s}=\begin{bmatrix}
        \ba_{s}\\\bzero
\end{bmatrix},~~\bK_{kj1}^{(\ell)}\bh^{(\ell-1)}_{2s-1}=\bzero,\\
&~~~~ ~~~\bV_{kj1}^{(\ell)}\bh^{(\ell-1)}_{2s}=-\eta\begin{bmatrix}
        \bzero\\ b_j^{(m)}\ba_s \\ \bzero
    \end{bmatrix},~~\bV_{kj1}^{(\ell)}\bh^{(\ell-1)}_{2s-1}=\bzero,\\
    &
    \bQ_{kj2}^{(\ell)}=-\bQ_{kj1}^{(\ell)},~~ \bK_{kj2}^{(\ell)}=\bK_{kj1}^{(\ell)},~~  \bV_{kj2}^{(\ell)}=-\bV_{kj1}^{(\ell)},\\
    %  &
    %  \bQ_{kj3}^{(\ell)}\bh^{(\ell-1)}_{2t-1}=\begin{bmatrix}
    %      1\\1-2t\\ 1\\\bzero
    % \end{bmatrix},~~ \bK_{kj3}^{(\ell)}\bh^{(\ell-1)}_{2j}=\begin{bmatrix}
    %     1\\ 1 \\2j\\\bzero
    % \end{bmatrix},~~ 
    % \bK_{kj3}^{(\ell)}\bh^{(\ell-1)}_{2j-1}=\begin{bmatrix}
    %     1\\ 1 \\2j-1\\\bzero
    % \end{bmatrix},~~ \bV_{kj3}^{(\ell)}\bh^{(\ell-1)}_{2t-1}=\eta\begin{bmatrix}
    %     \bzero\\ \ba_{j,k}-\lambda\hat\bv_k^{\ell-1}\\ \bzero
    % \end{bmatrix},\\
     &
     \bQ_{kj3}^{(\ell)}\bh^{(\ell-1)}_{2t-1}=\begin{bmatrix}
         1\\-(2t-1)\\ 1\\\bzero
    \end{bmatrix}, 
    % \bK_{kj4}^{(\ell)}\bh^{(\ell)}_{2j}=\begin{bmatrix}
    %     1\\ 1 \\(2j)^2\\\bzero
    % \end{bmatrix}, 
\bK_{kj3}^{(\ell)}\bh^{(\ell)}_{2s-1}=\begin{bmatrix}
        1\\ 1 \\(2s-1)\\\bzero
    \end{bmatrix}, \\&\qquad~\bV_{kj3}^{(\ell)}\bh^{(\ell-1)}_{2t-1}=\eta\begin{bmatrix}
        \bzero\\ a_j^{(m)}\intvec_k-[(1-b_j^{(m)})\prodeig+b_j^{(m)}\tilde\lambda]\bx_{k,j}^{\ell-1}\\ \bzero
    \end{bmatrix}.
\end{align*}
Similarly, it can be verified that the constructed attention layer generates 
\begin{align*}
  -\eta\cdot\nabla L_{k,j}(\bx_{k,j}^{\ell-1})=-\frac{\eta}{2t-1}\Big[ [(1-b_j^{(m)})\prodeig\id_d+b_j^{(m)}\tilde\lambda]+b_j^{(m)}\sum_{s=1}^{t-1}\ba_s\ba_s^\top \Big]\bx_{k,j}^{\ell-1}+\frac{\eta a_j^{(m)}\intvec_k}{2t-1}.
\end{align*} 
Therefore,  a  construction similar to  Step~\ref{slinucb_step2} of the proof of Theorem~\ref{thm:approx_smooth_linucb} yields Step~\ref{ts_step2b}. Moreover, note that for the construction to exist  we need the embedding dimension  $D=\conO(Adm)=\tcO(T^{1/4}Ad)$ and the number of hidden neurons $\hidden=\conO(Adm)=\tcO(T^{1/4}Ad)$.

\paragraph{Proof of Step~\ref{ts_step2c}}
Note that step 2c is a linear transformation from $\ba_{t,k},\widehat{a_j^{(m)}\intmat_j^{-1}\intvec_k}$, $k\in[A],j\in[m]$ to  $\sqrt{\prodeig} [\ba_{t,k}+\sum_{j=1}^m ({\prodeig}\id_d+{b_j^{(m)}(\Tpscov_t-\prodeig \id_d)})^{-1}{a_j^{(m)}(\Tpscov_t-\prodeig \id_d)\ba_{t,k}} ]$ and we have the fact $x=\sigma(x)-\sigma(-x)$. One can thus choose $\bW_1=\begin{bmatrix}
    \id_{A(m+1)d} & -\id_{A(m+1)d}&\bzero
\end{bmatrix}$ with $\hidden=2A(m+1)d$ and $\bW_2$ with $\lops{\bW_2}\leq \sqrt{\prodeig}(1+m)$ that implements the linear map.

\paragraph{Proof of Step~\ref{ts_step3}}
Similar to the proof of Step~\ref{ts_step2b}, given $\widehat{\Tpscov_t^{1/2}\ba_{t,k}}$ we can apprxoimate $\Tpscov_t^{-1}\widehat{\Tpscov_t^{1/2}\ba_{t,k}}\approx \Tpscov_t^{-1/2}\ba_{t,k}$ using accelerated gradient descent. Concretely, note that 
\begin{align*}
   \sqrt{\Tpsparn}\Tpscov_t^{-1}\widehat{\Tpscov_t^{1/2}\ba_{t,k}} =\argmin_{\bx\in\R^d}\frac{1}{2(2t-1)}\bx^\top\Tpscov_t\bx-\frac{1}{2t-1}\<\bx,  \sqrt{\Tpsparn}\widehat{\Tpscov_t^{1/2}\ba_{t,k}}\>=:\argmin_{\bx\in\R^d} L_{k}(\bx)
\end{align*}
is the global minimizer of a $\tilde\lambda/(2t-1)$-convex and $(B_a^2+\tilde\lambda)$-smooth function with the conditional number $\padecond\leq 2T(B_a^2+\tilde\lambda)/\tilde\lambda$. Since 
$$
\|\sqrt{\Tpsparn}\Tpscov_t^{-1}\widehat{\Tpscov_t^{1/2}\ba_{t,k}} \|_2\leq\sqrt{\Tpsparn}\lops{\intmat_j^{-1}}\|\widehat{\Tpscov_t^{1/2}\ba_{t,k}} \|_2\leq 2\sqrt{\Tpsparn}\sqrt{T(B_a^2+\tilde\lambda)}B_a,$$
where the last inequality uses the assumption in Step~\ref{ts_step3}.
Therefore for any $\eps_0>0$, it follows from Proposition~\ref{prop:conv_gd_agd} that
$L=\lceil2\sqrt{2 T(B_a^2+\tilde\lambda)/\tilde\lambda}\log((1+\padecond)2\sqrt{\Tpsparn}\sqrt{T(B_a^2+\tilde\lambda)}B_a/\eps_0)\rceil$ steps of accelerated gradient descent with step size $\eta=1/(B_a^2+\tilde\lambda)$ gives  $$\|\sqrt{\Tpsparn}\widehat{\Tpscov_t^{-1/2}\ba_{t,k}} -\sqrt{\Tpsparn}\Tpscov_t^{-1}\widehat{\Tpscov_t^{1/2}\ba_{t,k}} \|_2\leq\eps_0.$$ 

Following the construction in Step~\ref{slinucb_step2} of the proof of Theorem~\ref{thm:approx_smooth_linucb} it can be verified that there exists a transformer $\TF_\btheta(\cdot)$ with 
 \begin{align*}
  &L=\lceil2\sqrt{2 T(B_a^2+\tilde\lambda)/\tilde\lambda}\log((1+\padecond)2\sqrt{\Tpsparn}\sqrt{T(B_a^2+\tilde\lambda)}B_a/\eps_0)\rceil=\tcO(\sqrt{T}),~~\max_{\ell\in[L]}M^{(l)}=4A,~~~ \\
  &\qquad\qquad
  \max_{\ell\in[L]}\hidden^{(l)}=4Ad,~~~
  \nrmp{\btheta}\leq  \cO(A)
  \end{align*} that implements the AGD steps. Therefore, the approximation error 
  \begin{align*}
     &\quad\|\sqrt{\Tpsparn}\widehat{\Tpscov_t^{-1/2}\ba_{t,k}} 
      -
      \sqrt{\Tpsparn}{\Tpscov_t^{-1/2}\ba_{t,k}}
      \|_2\\
      &\leq
  \|\sqrt{\Tpsparn}\widehat{\Tpscov_t^{-1/2}\ba_{t,k}} -\sqrt{\Tpsparn}\Tpscov_t^{-1}\widehat{\Tpscov_t^{1/2}\ba_{t,k}} \|_2+ \|\sqrt{\Tpsparn}\Tpscov_t^{-1}\widehat{\Tpscov_t^{1/2}\ba_{t,k}}-\sqrt{\Tpsparn}\Tpscov_t^{-1}{\Tpscov_t^{1/2}\ba_{t,k}} \|_2\\
      &\leq \eps_0+\sqrt{\Tpsparn}\lops{\Tpscov_t^{-1}}\|\widehat{\Tpscov_t^{1/2}\ba_{t,k}}-\Tpscov_t^{1/2}\ba_{t,k} \|_2\\
      &\leq 
\eps_0+\frac{\sqrt{\Tpsparn}}{\tilde\lambda}\|\widehat{\Tpscov_t^{1/2}\ba_{t,k}}-\Tpscov_t^{1/2}\ba_{t,k} \|_2\leq \eps_0+\eps/2,
  \end{align*} where the last inequality uses the assumption on the approximation error in Step~\ref{ts_step2}. 
  Letting  $\eps_0=\eps/2$ yields $ \|\sqrt{\Tpsparn}\widehat{\Tpscov_t^{-1/2}\ba_{t,k}} 
      -
      \sqrt{\Tpsparn}{\Tpscov_t^{-1/2}\ba_{t,k}}
      \|_2\leq \eps$.

In addition to the calculation of  $\sqrt{\Tpsparn}\widehat{\Tpscov_t^{-1/2}\ba_{t,k}}$, we construct a two-layer attention-only transformer that computes $\<\widehat\Tpsmean_t,\ba_{t,k}\>$. Namely, we choose $\bQ^{(1)}_{k1,k2},\bK^{(1)}_{k1,k2},\bV^{(1)}_{k1,k2}$ such that
\begin{align*}
    &\bQ_{k1}^{(1)}\bh^{(0)}_{2t-1}=\begin{bmatrix}
        \widehat\Tpsmean_t\\-(2t-1)\\ \tfthres \\\bzero
    \end{bmatrix},~~ \bK_{k1}^{(1)}\bh^{(0)}_{2j}=\begin{bmatrix}
     \bzero\\  \tfthres\\ 2j\\\bzero
\end{bmatrix},~~\bK_{k1}^{(1)}\bh^{(0)}_{2j-1}=\begin{bmatrix}
    \ba_{j,k}\\  \tfthres\\ 2j-1\\\bzero
\end{bmatrix},~~ \bV_{k1}^{(1)}\bh^{(0)}_{2t-1}=\begin{bmatrix}
        \bzero\\ \be_k \\ \bzero
    \end{bmatrix}
    \\
      &\bQ_{k2}^{(1)}\bh^{(0)}_{2t-1}=\begin{bmatrix}
        -\widehat\Tpsmean_t\\-(2t-1)\\ \tfthres \\\bzero
    \end{bmatrix},~~ \bK_{k2}^{(1)}\bh^{(0)}_{2j}=\begin{bmatrix}
     \bzero\\  \tfthres\\ 2j\\\bzero
\end{bmatrix},~~\bK_{k2}^{(1)}\bh^{(0)}_{2j-1}=\begin{bmatrix}
    \ba_{j,k}\\  \tfthres\\ 2j-1\\\bzero
\end{bmatrix},~~ \bV_{k2}^{(1)}\bh^{(0)}_{2t-1}=-\begin{bmatrix}
        \bzero\\ \be_k \\ \bzero
    \end{bmatrix},
\end{align*}
where $\tfthres=2TB_a^2(B_aB_w+\sqrt{2\Tpsparn\log(T/{\delta_0}))})/\tilde\lambda=\tcO(T)$ is an upper bound of $\<\widehat\Tpsmean_t,\ba_{t,k}\>$ for all $k\in[A]$ under the event $\hpevent_{{\delta_0}}$, and $\be_k=(0,0,\ldots,1,0,\ldots,0)\in\R^{A}$ is the one-hot vector supported on the $k$-th entry. Summing up the attention heads gives the update 
$$
\bzero\mapsto\frac{\<\widehat\Tpsmean_t,\ba_{t,k}\>}{2t-1}.
$$ Note that one can choose the matrices such that 
$$\lops{\bQ^{(1)}_{k1,k2}}\leq \tfthres,~~~\lops{\bK^{(1)}_{k1,k2}}\leq \tfthres,~~~\lops{\bV^{(1)}_{k1,k2}}\leq1.$$ Thus the norm of the attention layer  $\nrmp{\btheta^{(1)}}\leq \tfthres+2$.

Finally, as in the proof of Step~\ref{ts_step2a} we can construct a single-layer  single-head  attention-only transformer with $\nrmp{\btheta^{(2)}}\leq T+1$ that performs the multiplication
$$
\frac{\<\widehat\Tpsmean_t,\ba_{t,k}\>}{2t-1}\mapsto{\<\widehat\Tpsmean_t,\ba_{t,k}\>}.
$$
 To estimate the approximation error, note that $\|\ba_{t,k}\|_2\leq B_a$ and $\|\hat\Tpsmean_t-\Tpsmean_t\|_2\leq\eps/B_a$ by our assumption in Step~\ref{ts_step3}, it follows immediately that $|\<\widehat\Tpsmean_t,\ba_{t,k}\>-\<\Tpsmean_t,\ba_{t,k}\>|\leq\eps$ for all $k\in[A]$. Combining the construction of the transformer layers above gives Step~\ref{ts_step3}.

\paragraph{Proof of Step~\ref{ts_step4}}
By Assumption~\ref{ass:thompson_mlp_approx_linear},  $g_{k,\trunprob}(\bigvec)$ are $(\eps,R_{\delta_0}, \neuron,\weightn)$-approximable by sum of relus for some $ \neuron,\weightn$ depend polynomially on $(1/\eps,1/\trunprob,1/\Trunregp,1/\delta_0,A)$. Since
\begin{align*}
\|\sqrt{\Tpsparn}\Tpscov_t^{-1/2}\ba_{t,k}\|_2\leq \sqrt{\Tpsparn}\|\Tpscov_t^{-1/2}\|_2\|\ba_{t,k}\|_2\leq\sqrt{\lambda}B_a,~~|\<\Tpsmean_t,\ba_{t,k}\>|\leq \|\Tpsmean_t\|_2\|\ba_{t,k}\|_2=B_wB_a
\end{align*}
and $R_{\delta_0}=2(B_wB_a+\sqrt{\lambda}B_a)$, it follows from 
the assumption  $\eps_3\leq R_{\delta_0}/2$ and  a triangular inequality that $\linf{\bv}\leq R_\delta$. Therefore, using Assumption~\ref{ass:thompson_mlp_approx_linear} and stacking up the approximation functions for each coordinate $k\in[A]$ we construct a two-layer MLP with $\lops{\bW_1}\leq \sqrt{A\neuron},\lops{\bW_1}\leq \weightn$, $D'= \neuron A$ such that 
\begin{align*}\bW_2\sigma(\bW_1\bh^{(1)}_{2t-1})=\begin{bmatrix}
    \bzero\\\hat \bv_{t1}\\\vdots\\\hat\bv_{tA}\\\bzero,
\end{bmatrix}
\end{align*}
where $(\hat \bv_{t1},\ldots,\hat\bv_{tA})$ is supported on $\bh^c_{2t-1}$ and $|\hat v_{tk}-g_{k,\trunprob}(\bigvec)|\leq \eps$ for all $k\in[A]$.

\subsection{Proof of Theorem~\ref{thm:ts_linear_regret}}\label{sec:pf_prop:ts_linear_regret}
Denote the transformer constructed in Theorem~\ref{thm:approx_thompson_linear-formal} by $\TF_\tfpar$. 
From the proof of Theorem~\ref{thm:approx_thompson_linear-formal}, we have 
$$\log\frac{\sAlg_{\TS}(\eaction_t|\dset_{t-1},\state_t)}{\sAlg_{\tfpar}(\eaction_t|\dset_{t-1},\state_t)}\leq \eps$$ under the event $$\hpevent_{\delta_0}:=\{\max_{t\in[T]}|\eps_t|\leq\sqrt{2\Tpsparn\log(2T/\delta_0)}\}\cup \{\|\bw^*\|_2\leq  B_w\}~~\text{ for all } t\in[\totlen]$$ with probability at least $1-\delta_0$, where  $B_w:=\sqrt{\lambda}\big(\sqrt{d}+2\sqrt{\log(2/\delta_0)}\big)$. Note that due to the unboundedness of the noise $\eps_t$ and parameter vector $\bw^*$, Assumption~\ref{asp:realizability} may not be satisfied. However, setting $\delta_0=\delta/(2n)$ and applying a union bound gives 
\begin{align}
    \log\frac{\sAlg_{\TS}(\eaction^i_t|\dset^i_{t-1},\state^i_t)}{\sAlg_{\tfpar}(\eaction^i_t|\dset^i_{t-1},\state^i_t)}\leq \eps,~~~\text{ for } t\in[\totlen],i\in[n].\label{eq:unif_traj_realize}
\end{align} 
with probability at least $1-\delta/2$. 
From the proof of Theorem~\ref{thm:diff_reward} we see that Assumption~\ref{asp:realizability} is only used in Eq.~\eqref{eq:pf_hellinger_control_general2} in the proof of Lemma~\ref{lm:general_imit}. Moreover, it can be verified that the same result holds with  Assumption~\ref{asp:realizability}  replaced by the condition in Eq.~\eqref{eq:unif_traj_realize}. Therefore, we have \begin{align*}
\Big|\totreward_{\prior,\sAlg_\EstPar}(\totlen)-\totreward_{\prior,\sAlg_\TS}(\totlen)\Big|
&\leq 
c \totlen^2 \sqrt{\distratio} \bigg(\sqrt{\frac{\log \brac{ \cN_{\Parspace} \cdot 
 \totlen/\delta } }{n}} +  \sqrt{\geneps}\bigg)\\
 &\leq 
 c \totlen^2 \sqrt{\distratio} \bigg(\sqrt{\frac{\log \brac{ \cN_{\Parspace} \cdot 
 \totlen/\delta } }{n}}  \bigg)+\sqrt{\totlen}
\end{align*}
with probability at least $1-\delta$ as in Theorem~\ref{thm:diff_reward}, where the second inequality follows as in our setting $\geneps=\eps=1/(\totlen^3\distratio)$. 
Now, it suffices to show Thompson sampling has the expected regret with 
\begin{align*}
    \E_{\inst\sim\prior}\Big[\sum_{t=1}^\totlen\max_{k}\<\ba_{t,k},\bw^*\>-\totreward_{\inst,\sAlg_\TS}(\totlen)\Big]=\cO(d\sqrt{T}\log(Td)). 
\end{align*}

The proof follows similar arguments as in  Theorem 36.4 in~\cite{lattimore2020bandit}. Define $$\tilde\lambda:=\Tpsparn/\lambda,~~~\beta:=\sqrt{\Tpsparn}\Big(\sqrt{2\Tpsparn d\log(4d/{\delta_{\TS}})}+\sqrt{2\log(2/{\delta_{\TS}})+d\log(1+TB_a^2/\tilde\lambda d)}\Big),
$$
where ${\delta_{\TS}}$ will be specified later,  and recall $\Tpscov_t=\tilde\lambda \id_d+\sum_{j=1}^{t-1}\ba_j\ba_j^\top$. Since $\|\bw^\star\|_2\leq\sqrt{2\lambda d\log(4d/{\delta_{\TS}})}$ with probability at least $1-{\delta_{\TS}}/2$ by a union bound, it  follows from Theorem 20.5 in~\cite{lattimore2020bandit} that  $$\P(\|\bw^\star\|_2\leq\sqrt{2\lambda d\log(4d/{\delta_{\TS}})},~~~\text{and~~~}\|\Tpsmean_t-\bw^*\|_{\Tpscov_t}\geq\beta,\text{for some } i\in[T])\leq{\delta_{\TS}},$$ where the probability is taken over the both randomness of the noise and of the bandit instance $\inst$.  

Let $\hpevent$ be the event where $\|\Tpsmean_t-\bw^*\|_{\Tpscov_t}\leq\beta$ for all $i\in[T]$, and let $\hpevent_0$ be the event where $\{\|\bw^\star\|_2\leq\sqrt{2\lambda d\log(4d/{\delta_{\TS}})}\}$. Then $\P(\hpevent\cap\hpevent_0)\geq 1-{\delta_{\TS}}$ and the expected regret
\begin{align*}
&\qquad\E_{\inst\sim\prior}\Big[\sum_{t=1}^\totlen\max_{k}\<\ba_{t,k},\bw^*\>-\totreward_{\inst,\sAlg_\TS}(\totlen)\Big]\\
&=\E[\sum_{t=1}^T \max_{j\in[A]}\<\bw^\star,\ba_{t,k}\>-\<\bw^\star,\ba_{t}\>]\\
  &=
  \E\Big[\sum_{t=1}^T (\max_{j\in[A]}\<\bw^\star,\ba_{t,k}\>-\<\bw^\star,\ba_{t}\>) \bone_{\hpevent\cap\hpevent_0}\Big]
  +
   \E\Big[\sum_{t=1}^T \max_{j\in[A]}(\<\bw^\star,\ba_{t,k}\>-\<\bw^\star,\ba_{t}\>)\bone_{(\hpevent\cap\hpevent_0)^c}\Big]\\
   &\leq 
   \E\Big[\sum_{t=1}^T (\max_{j\in[A]}\<\bw^\star,\ba_{t,k}\>-\<\bw^\star,\ba_{t}\>) \bone_{\hpevent\cap\hpevent_0}\Big]
+
\E[2B_aT\|\bw^*\|_2\bone_{(\hpevent\cap\hpevent_0)^c}]
  \\
  &\leq 
  \E\Big[\sum_{t=1}^T (\max_{j\in[A]}\<\bw^\star,\ba_{t,k}\>-\<\bw^\star,\ba_{t}\>) \bone_{\hpevent\cap\hpevent_0}\Big]+2B_aT\E[\|\bw^*\|_2\bone_{(\hpevent\cap\hpevent_0)^c}].
\end{align*}
Since 
\begin{align*}
 \E[\|\bw^*\|_2\bone_{(\hpevent\cap\hpevent_0)^c}]
    &\leq
\P({(\hpevent\cap\hpevent_0)^c})\sqrt{2\lambda d\log(4d/{\delta_{\TS}})}
+
\int^\infty_{\sqrt{2\lambda d\log(4d/{\delta_{\TS}})}}\P(\|\bw^*\|_2\geq t)dt\\
&\leq
\sqrt{2\lambda d\log(4d/{\delta_{\TS}})}{\delta_{\TS}}+d^{3/2}\int^\infty_{\sqrt{2\lambda \log(4d/{\delta_{\TS}})}}\P(|w_1^\star|\geq t)dt
\\
&\leq
\sqrt{2\lambda d\log(4d/{\delta_{\TS}})}{\delta_{\TS}}+2\sqrt{2}d^{3/2}\lambda^{1/2}\int^\infty_{\sqrt{ \log(4d/{\delta_{\TS}})}}\exp(-{t^2})dt\\
&\leq
\sqrt{2\lambda d\log(4d/{\delta_{\TS}})}{\delta_{\TS}}+\sqrt{2}d^{3/2}\lambda^{1/2}\int^\infty_{{ \log(4d/{\delta_{\TS}})}}\frac{1}{t^{1/2}}\exp(-{t})dt\\
&\leq 
2\sqrt{2\lambda d\log(4d/{\delta_{\TS}})}{\delta_{\TS}},
\end{align*}
where the second line follows from a union bound over $[d]$, and the third line uses properties of subgaussian variables. Therefore, choosing ${\delta_{\TS}}=1/[T\sqrt{ d}]$  gives 
\begin{align}
&\qquad\E_{\inst\sim\prior}\Big[\sum_{t=1}^\totlen\max_{k}\<\ba_{t,k},\bw^*\>-\totreward_{\inst,\sAlg_\TS}(\totlen)\Big]\notag\\
    &\leq
     \E\Big[\sum_{t=1}^T (\max_{j\in[A]}\<\bw^\star,\ba_{t,k}\>-\<\bw^\star,\ba_{t}\>) \bone_{\hpevent\cap\hpevent_0}\Big]+6B_a\sqrt{\lambda\log(4d^2T)}.\label{eq:ts_linear_regret_decomp}
\end{align}
Now define the event $\hpevent_t:=\{\|\Tpsmean_t-\bw^*\|_{\Tpscov_t}\leq\beta\}$,  then  we have $\hpevent_t\in\cF_{t-1}$  and $\cap_{t=1}^T \hpevent_t=\hpevent$. Also, we define the upper confidence bound $U_t(\ba):=\<\Tpsmean_t,\ba\>+\beta\|\ba\|_{\Tpscov_t^{-1}}$, which does not depend on the true parameter $\bw^*$. Let $(\cF_{t})_{t\geq 0}$ denote the filtration generated by the data collected up to time $t$ and the random parameter vector $\bw^*$. 

Let $\ba_t^*$ denote the optimal action at time $t$. Due to the construction of Thompson sampling, we have the distribution of $\ba_t^*$ and $\ba_t$ are the same conditioned on $\cF_{t-1}$. Therefore, $\E[U_t(\ba_t^*)|\cF_{t-1}]=\E[U_t(\ba_t)|\cF_{t-1}]$ and
\begin{align*}
    \E\Big[ (\<\bw^\star,\ba_{t}^*\>-\<\bw^\star,\ba_{t}\>) \bone_{\hpevent\cap\hpevent_0}\mid\cF_{t-1}\Big]
   &\leq
\E\Big[ \big(\<\bw^\star,\ba_{t}^*\>-U_t(\ba_t^*)+U_t(\ba_t)-\<\bw^\star,\ba_{t}\>\big) \bone_{\hpevent_t}\mid\cF_{t-1}\Big]\\
    &\leq
  \E\Big[ \big(U_t(\ba_t)-\<\bw^\star,\ba_{t}\>\big) \bone_{\hpevent_t}\mid\cF_{t-1}\Big]\\
   &\leq
  \E\Big[ \big(\|\ba_t\|_{\Tpscov_t^{-1}}\|\Tpsmean_t-\bw^*\|_{\Tpscov_t} +\beta \|\ba_t\|_{\Tpscov_t^{-1}} \big)\bone_{\hpevent_t}\mid\cF_{t-1}\Big]\\
    &\leq
   2 \beta\E[\|\ba_t\|_{\Tpscov_t^{-1}}|\cF_{t-1}].
\end{align*}
Moreover, we have $$\|\ba_t\|_{\Tpscov_t^{-1}}\leq B_a/\sqrt{\tilde\lambda}.$$ 
Combining the last two displays, we obtain
\begin{align*}
     &\qquad\E\Big[\sum_{t=1}^T (\max_{j\in[A]}\<\bw^\star,\ba_{t,k}\>-\<\bw^\star,\ba_{t}\>) \bone_{\hpevent\cap\hpevent_0}\Big]\\
    &\leq\sum_{t=1}^T \E\Big[2 \beta \|\ba_t\|_{\Tpscov_t^{-1}}\wedge({B_a}/{\sqrt{\tilde\lambda}})\Big]\\
    &\leq 
    2\big(\beta\vee ({B_a}/{\sqrt{\tilde\lambda}})\big)\E\Big[\sum_{t=1}^T (\|\ba_t\|_{\Tpscov_t^{-1}}\wedge1)\Big]\\
    &\leq 
    2\big(\beta\vee ({B_a}/{\sqrt{\tilde\lambda}})\big)\sqrt{T}\sqrt{\E\Big[\sum_{t=1}^T (\|\ba_t\|^2_{\Tpscov_t^{-1}}\wedge1)\Big]}\\
    &\leq
    2\big(\beta\vee ({B_a}/{\sqrt{\tilde\lambda}})\big)\sqrt{T}\sqrt{2d\log(1+TB_a^2/(\tilde\lambda d))}\\
    &= \cO( d\sqrt{T}\log(Td)),
\end{align*}
where the  fourth line uses  Cauchy-Schwartz inequality and the fifth line follows from Lemma 19.4 in~\cite{lattimore2020bandit}. Combining the last display with~Eq.~\eqref{eq:ts_linear_regret_decomp} completes the proof of first part of Theorem~\ref{thm:ts_linear_regret}. Moreover, the second part of Theorem~\ref{thm:ts_linear_regret}  (i.e., the upper bound on $\log\cN_{\tfparspace}$) follows directly from Lemma~\ref{lm:cover_num_bound} and Eq.~\eqref{eq:ts_tf_param-main}.

\subsection{Proof of Lemma~\ref{lm:lip_of_tps}}\label{sec:pf_lm:lip_of_tps}

For any $j\neq k$, by definition of $g_{k,\trunprob}$
\begin{align*}
     &\quad|g_{k,\trunprob}(\bx_1,\ldots,\bx_j,\ldots,\bx_A,y_1,\ldots,y_A)-g_{k,\trunprob}(\bx_1,\ldots,\bx'_j,\ldots\bx_A,y_1,\ldots,y_A)|
        \\
     &\leq
     \frac{1}{\trunprob}\Big|\P(\<\bx_k-\bx_i,\bz\>+y_k-y_i\geq0,\text{~for all~}i\in[A])\\
     &\qquad~~~ -\P(\<\bx_k-\bx_i,\bz\>+y_k-y_i\geq0,\text{~for all~}i\neq j, \<\bx_k-\bx_j',\bz\>+y_k-y_j\geq0)\Big|
     \\
      &\leq
     \frac{1}{\trunprob}\Big(\P( \<\bx_k-\bx_j',\bz\>+y_k-y_j\geq0\geq\<\bx_k-\bx_j,\bz\>+y_k-y_j )\\
     &\qquad~~~
     +\P( \<\bx_k-\bx_j',\bz\>+y_k-y_j\leq0\leq\<\bx_k-\bx_j,\bz\>+y_k-y_j)\Big)\\
     &\leq 
      \frac{1}{\trunprob}\Big(\P(\<\bx_j-\bx_j',\bz\>\geq\<\bx_k-\bx_j',\bz\>+y_k-y_j\geq0 )\\&\qquad~~~+\P( \<\bx_k-\bx_j',\bz\>+y_k-y_j\leq0\leq\<\bx_k-\bx_j,\bz\>+y_k-y_j)\Big)\\
       &\leq 
      \frac{1}{\trunprob}\Big(\P(\<\bx_j-\bx_j',\bz\>\geq\<\bx_k-\bx_j',\bz\>+y_k-y_j\geq0 )+\P(\<\bx_j-\bx_j',\bz\>\leq \<\bx_k-\bx_j',\bz\>+y_k-y_j\leq0\Big).   
\end{align*}
Note that conditioned on $\bx_j,\bx_j'$ we have
\begin{align*}
    \P(|\<\bx_j-\bx_j',\bz\>|\leq \|\bx_j-\bx_j'\|_2\sqrt{2\log(2/{\delta_1})})\geq 1-{\delta_1}
\end{align*} for any ${\delta_1}>0$. Therefore we further have
\begin{align*}
    &\quad|g_{k,\trunprob}(\bx_1,\ldots,\bx_j,\ldots,\bx_A,y_1,\ldots,y_A)-g_{k,\trunprob}(\bx_1,\ldots,\bx'_j,\ldots\bx_A,y_1,\ldots,y_A)|
    \\
    &\leq
     \frac{1}{\trunprob}\Big[\P\Big(\<\bx_k-\bx_j',\bz\>+y_k-y_j\in[-\|\bx_j-\bx_j'\|_2\sqrt{2\log(2/{\delta_1})},\|\bx_j-\bx_j'\|_2\sqrt{2\log(2/{\delta_1})}] \Big)+{\delta_1}\Big]\\
     &\leq 
    \frac{1}{\trunprob}[\sup_{A\in\cF,\mu(A)=2\|\bx_j-\bx_j'\|_2\sqrt{2\log(2/{\delta_1})}}\P(\<\bx_k-\bx_j',\bz\>\in A)+{\delta_1}]\\
    &\leq 
   \frac{1}{\trunprob}(\frac{2\|\bx_j-\bx_j'\|_2\sqrt{2\log(2/{\delta_1})}}{\sqrt{2\pi}\Trunregp}+{\delta_1})\leq \frac{1}{\trunprob}(\frac{2\|\bx_j-\bx_j'\|_2}{\Trunregp{\delta_1}}+{\delta_1})
\end{align*}
for any ${\delta_1}>0$, 
where the last inequality follows from the fact that standard Gaussian has probability density less than $1/\sqrt{2\pi}$ everywhere,  $\|\bx_k-\bx_j'\|_2\leq \Trunregp$ and $\log(2/{\delta_1})\leq 4/{\delta_1}^2$. Choosing ${\delta_1}=1\wedge\sqrt{\frac{2\|\bx_j-\bx_j'\|_2}{\Trunregp}}$ gives 
\begin{align*}
 &\quad |g_{k,\trunprob}(\bx_1,\ldots,\bx_j,\ldots,\bx_A,y_1,\ldots,y_A)-g_{k,\trunprob}(\bx_1,\ldots,\bx'_j,\ldots\bx_A,y_1,\ldots,y_A)|\\
 &\leq  
  \frac{2}{\trunprob}\Big(\sqrt{\frac{2\|\bx_j-\bx_j'\|_2}{\Trunregp}}+\frac{2\|\bx_j-\bx_j'\|_2}{\Trunregp}\Big)
\end{align*}

Similarly, for $\bx_k\neq\bx_k'$, we have 
\begin{align*}
     &\quad|g_{k,\trunprob}(\bx_1,\ldots,\bx_k,\ldots,\bx_A,y_1,\ldots,y_A)-g_{k,\trunprob}(\bx_1,\ldots,\bx'_k,\ldots\bx_A,y_1,\ldots,y_A)|
        \\
     &\leq
     \frac{1}{\trunprob}\Big(\P(\<\bx_k-\bx_i,\bz\>+y_k-y_i\geq0\geq \<\bx'_k-\bx_i,\bz\>+y_k-y_i,\text{~for some~}i\in[A])\\&\qquad~~~+\P(\<\bx_k-\bx_i,\bz\>+y_k-y_i\leq0\leq \<\bx_k'-\bx_i,\bz\>+y_k-y_i,\text{~for some~}i\in[A])\Big)
     \\
       &\leq
    \sum_{i\neq k} \frac{1}{\trunprob}\Big(\P(\<\bx_k-\bx_i,\bz\>+y_k-y_i\geq0\geq \<\bx_k'-\bx_i,\bz\>+y_k-y_i)\\&\qquad~~~+\P(\<\bx_k-\bx_i,\bz\>+y_k-y_i\leq0\leq \<\bx_k'-\bx_i,\bz\>+y_k-y_i)\Big)\\
    &\leq \frac{A}{\trunprob}\max_{i
    \neq k}\Big(\P(\<\bx_k-\bx_k',\bz\>\geq\<\bx_i-\bx_k',\bz\>+y_i-y_k\geq0 )+\P(\<\bx_k-\bx_k',\bz\>\leq \<\bx_i-\bx_k',\bz\>+y_i-y_k\leq0\Big).
\end{align*}
Following the same argument, we have 
\begin{align*}
&\quad~|g_{k,\trunprob}(\bx_1,\ldots,\bx_k,\ldots,\bx_A,y_1,\ldots,y_A)-g_{k,\trunprob}(\bx_1,\ldots,\bx'_k,\ldots\bx_A,y_1,\ldots,y_A)|\\&\leq   \frac{2A}{\trunprob}\Big(\sqrt{\frac{2\|\bx_k-\bx_k'\|_2}{\Trunregp}}+\frac{2\|\bx_k-\bx_k'\|_2}{\Trunregp}\Big).
\end{align*}

Likewise, for any $j\neq k$ we have 
\begin{align*}
    &\quad|g_{k,\trunprob}(\bx_1,\ldots,\bx_A,y_1,\ldots,y_j,\ldots,y_A)-g_{k,\trunprob}(\bx_1,\ldots,\bx_A,y_1,\ldots,y_j',\ldots,y_A)|\\
    &\leq
    \frac{1}{\Trunregp}
    \Big(\P( \<\bx_k-\bx_j,\bz\>\in[\min\{y_k-y_j,y_k-y_j'\},\max\{y_k-y_j,y_k-y_j'\}]\Big)\\
      &\leq 
      \frac{1}{\trunprob}\sup_{A\in\cF,\mu(A)=2|y_j-y_j'|}\P(\<\bx_k-\bx_j,\bz\>\in A)\\
      &\leq
   \frac{1}{\trunprob}\frac{2|y_j-y_j'|}{\sqrt{2\pi}\Trunregp}\leq \frac{2|y_j-y_j'|}{\Trunregp\trunprob} 
\end{align*} and
\begin{align*}
    &\quad|g_{k,\trunprob}(\bx_1,\ldots,\bx_A,y_1,\ldots,y_k,\ldots,y_A)-g_{k,\trunprob}(\bx_1,\ldots,\bx_A,y_1,\ldots,y_k',\ldots,y_A)|\\
    &\leq
    \sum_{j\neq k}\frac{1}{\Trunregp}
    \Big(\P( \<\bx_k-\bx_j,\bz\>\in[\min\{y_k'-y_j,y_k-y_j\},\max\{y_k'-y_j,y_k-y_j\}]\Big)\\
      &\leq 
      \frac{A}{\trunprob}\sup_{A\in\cF,\mu(A)=2|y_j-y_j'|}\P(\<\bx_k-\bx_j,\bz\>\in A)\\
      &\leq
   \frac{A}{\trunprob}\frac{2|y_j-y_j'|}{\sqrt{2\pi}\Trunregp}\leq \frac{2A|y_j-y_j'|}{\Trunregp\trunprob}.
\end{align*}

%% file: Sections_arxiv/app-mdp_arxiv.tex
\section{Learning in-context RL in markov decision processes}

Throughout this section, we use $c>0$ to denote universal constants whose values may vary from line to line.
Moreover, for notational simplicity, we use $\conO(\cdot)$ to hide universal constants,  $\tcO(\cdot)$ to hide poly-logarithmic terms in $(\horizon,\Numepi,\Numst,\Numact,1/\temp)$.

This section is organized as follows. Section~\ref{sec:tf_embed_mdp} discusses the embedding and extraction formats of transformers for Markov decision processes. Section~\ref{sec:example_ucbvi} describes the UCB-VI and the soft UCB-VI algorithms.  We prove Theorem~\ref{thm:approx_ucbvi} in Section~\ref{sec:pf_approx_ucbvi} and prove Theorem~\ref{thm:ucbvi_icrl-main} in Section~\ref{sec:pf_thm:ucbvi_icrl-main}.

\subsection{Embedding and extraction mappings}\label{sec:tf_embed_mdp}

To embed MDP problems into transformers, we consider an embedding similar to that for linear bandits. For each episode $k\in[\Numepi]$, we construct $2\horizon+1$ tokens. Concretely, for each $t\in[\totlen]$ in the $k$-th episode, we write  $t=\horizon(k-1)+h$  and construct two tokens
\[
\begin{aligned}
\bh_{2(t-1)+k}=
\left[
\begin{array}{cc}
     \bzero_{\Numact+1} \\
     \hdashline 
     \state_{k,h}\\  
      \hdashline 
    \bzero_{\Numact}\\  
    \hdashline  
    \bzero\\
    \posv_{2(t-1)+k}
\end{array}\right]
=:
\begin{bmatrix}
     \bh_{2(t-1)+k}^{\parta} \\  \bh_{2(t-1)+k}^{\partb}\\  \bh_{2(t-1)+k}^{\partc}\\   \bh_{2(t-1)+k}^{\partd}\\
\end{bmatrix},~~
\bh_{2t-1+k}=
\left[
\begin{array}{cc}
     \action_{k,h} \\
      \reward_{k,h}\\  
      \hdashline 
      \bzero_{\Numst}\\ 
      \hdashline 
      \bzero_{\Numact}\\ 
     \hdashline  
      \bzero\\ 
      \posv_{2t-1+k}
\end{array}\right]=:
\begin{bmatrix}
    \bh_{2t-1+k}^{\parta} \\  \bh_{2t-1+k}^{\partb}\\   \bh_{2t-1+k}^{\partc}\\   \bh_{2t-1+k}^{\partd}
\end{bmatrix},
\end{aligned}
\]
where  $\state_{k,h},\action_{k,h}$ are represented using one-hot embedding (we let $\state_{k,\horizon+1}=\bzero_\Numst$), $\bh^\partc_{2(t-1)+k}$ is used to store the (unnormalized)  policy at time  $t$ given current state $\state_{k,h}$, $\bzero$ in $\bh^\partd$ denotes an additional zero vector. At the end of each episode $k$, we add an empty token 
$$
\bh_{(2\horizon+1)k}=\bh^{\emp}_{k}:=\begin{bmatrix}
    \bzero &\posv_{(2\horizon+1)k}
\end{bmatrix}^\top
$$ to store intermediate calculations. We also include in the tokens  the positional embedding $\posv_i:=(k,h,v_i,i,i^2,1)^\top$ for $i\in[2\totlen+\Numepi]$, where $\oddeven_i:=\bone_{\{\bh_i^\parta=\bzero\}}$  denote the tokens that do not embed actions and rewards.    In addition, we define  the token matrix $\bH_t:=\begin{bmatrix}
    \bh_1,\ldots,\bh_{2t-1+k}
\end{bmatrix}\in\R^{D\times (2t-1+k)}$ for all $t\in[\totlen]$.

\paragraph{Offline pretraining} 
Similar to the bandit setting, during
pretraining the transformer $\TF_\tfpar$ takes in   $\bH_\totlen^\pre:=\bH_\totlen$ as the input token matrix, and generates $\bH_\totlen^\post:=\TF_\tfpar(\bH_\totlen^\pre)$ as the output. For each time $t\in[\totlen]$, we define the  induced policy  $\sAlg_\tfpar(\cdot|\dset_{t-1},\state_t):=\frac{\exp(\bh^{\post,\partc}_{2(t-1)+k})}{\|\exp(\bh^{\post,\partc}_{2(t-1)+k})\|_1}\in\Delta^\Numact$, whose $i$-th entry is the probability of selecting the $i$-th action (denoted by the  one-hot vector $\be_i$) given $(\dset_{t-1},\state_t)$. We then find the transformer $\esttfpar\in\tfparspace$ by solving Eq.~\eqref{eq:general_mle}. 

\paragraph{Rollout}
At each time $t\in[\totlen]$, given the  current state $\state_t$ and  previous data $\dset_{t-1}$, we first construct the token matrix $\bH^{\pre}_{\roll,t}\in\R^{D\times 2(t-1)+k}$ that consists of tokens up to the first token for time $t$.   The transformer then takes $\bH^{\pre}_{\roll,t}$ as the input  and generates $\bH^{\post}_{\roll,t}=\TF_\tfpar(\bH^{\pre}_{\roll,t})$. Next,  the agent selects an action $\action_t\in\actionsp$ following  the induced  policy $\sAlg_\tfpar(\cdot|\dset_{t-1},\state_t):=\frac{\exp(\bh^{\post,\partc}_{2(t-1)+k})}{\|\exp(\bh^{\post,\partc}_{2(t-1)+k})\|_1}\in\Delta^\Numact$ and observes the reward $\reward_t$ and next state $\state_{t+1}$ ($\state_{t+1}\sim\init$ if $t$ is the last time step in an episode).

\paragraph{Embedding and extraction mappings}
To integrate the above construction into our general framework in Section~\ref{sec:framework}, for $t=(k-1)\horizon+h$,  we have the embedding vectors $$\embedmap(\state_t):=\bh_{2(t-1)+k},~~~\embedmap(\action_t,\reward_t):=\bh_{2t-1+k}.$$  For $N\geq 1$, write $$\lceil(N+1)/2\rceil=(k_N-1)\horizon+h_N$$ for some $h_N\in[\horizon]$, and define the  concatenation operator 
\begin{align*}
\cat(\bh_1, \ldots, \bh_N): = [\bh_1, \ldots,\bh_{2\horizon},\bh_{1}^{\emp}, \bh_{2\horizon+1},\ldots,\bh_{4\horizon},\bh_2^\emp,\bh_{4\horizon+1},\ldots, \bh_N]\in\R^{N+k_N-1},
\end{align*}
where we insert an empty token $\bh_k^\emp$ (i.e., a token with $\bh^{\{\parta,\partb,\partc\}}=\bzero$) at the end of  each episode $k$. 

In this case, we have the input token matrix  $$\bH=\bH^\pre_{\roll,t}: = \cat(\embedmap(\state_1), \embedmap(\action_1, \reward_1), \ldots, \embedmap(\action_{t-1}, \reward_{t-1}), \embedmap(\state_t)) \in \R^{D \times [2(t-1)+k]},$$ the output token matrix $\bar{\bH}=\bH^\post_{\roll,t}$, and the linear extraction map $\extractmap$  satisfies $$\extractmap\cdot\bar{\bh}_{-1}=\extractmap\cdot\bar{\bh}^\post_{2(t-1)+k}=\bh^{\post,\partc}_{2(t-1)+k}.$$

\subsection{UCB-VI and soft UCB-VI}\label{sec:example_ucbvi}
We show that transformers with the embedding in Section~\ref{sec:tf_embed_mdp}  can approximately implement the  UCB-VI algorithm in~\cite{azar2017minimax}. Namely, UCB-VI implements the following steps:

for each episode $k\in[\Numepi]$ and each step $h=\horizon,\ldots,1$
\begin{enumerate}
    \item Compute the estimated transition matrix $\tresttransit_h(\state'|\state,\action):=   \frac{\Numvi_h(\state,\action,\state')}{\Numvi_h(\state,\action)\vee 1}$,  where $\Numvi_h(\state,\action,\state')$ denotes the number of times the state-action-next-state tuple $(\state,\action,\state')$ has been visited in the first $k-1$ episodes, and $\Numvi_h(\state,\action)=\sum_{\state'}\Numvi_h(\state,\action,\state')$ (we assume $\Numvi_\horizon(\state,\action,\state')=0$ and let $\Numvi_\horizon(\state,\action)$  be the number of times $(\state,\action)$ is visited at timestep $\horizon$).
    \item Calculate the estimated Q-function \begin{align*}
\trestQfun_h(\state,\action)=\min\{\horizon,\reward_h(\state,\action)+\bonus_h(\state,\action)+\sum_{\state'\in\statesp}\tresttransit_h(\state'\mid\state,\action)\trestVfun_{h+1}(\state')\},\end{align*}
where the bonus $\bonus_h(\state,\action)=2\horizon\sqrt{\frac{\log(\Numst\Numact\totlen/\delta)}{\Numvi_h(\state,\action)\vee1}}$,  $\trestVfun_{\horizon+1}(\state):=0$ for all $\state\in\statesp$ and $\trestVfun_h(\state):=\max_{\action\in\actionsp}\trestQfun_h(\state,\action)$. 
\end{enumerate} Throughout this section, we choose the small probability $\delta=1/(\Numepi\horizon)$.

During policy execution, at each step $h\in[\horizon]$, UCB-VI takes the greedy action $\action_h:=\argmax_{\action}\trestQfun(\state_h,\action)$ and observes the reward and next state $(\reward_h,\state_{h+1})$. To facilitate pretraining, in this work we consider a soft version of UCB-VI, which takes action $\action_h$  following the softmax policy 
\begin{align*}
\plc_h(\action|\state_h)=\frac{\exp(\trestQfun_h(\state_h,\action)/\temp)}{\lone{\exp(\trestQfun_h(\state_h,\action)/\temp)}}
\end{align*}
using the estimated $Q$-function for some sufficiently small $\temp>0$. Note that soft UCB-VI recovers UCB-VI as $\temp\to 0$.

\subsection{Proof of Theorem~\ref{thm:approx_ucbvi}}\label{sec:pf_approx_ucbvi}
Throughout the proof, we abuse the notations $\bh_i^{\star}$ for $\star\in\{\parta,\partb,\partc,\partd\}$ to denote the corresponding positions in the token vector $\bh_i$. For any $t'\in[\totlen]$, we let $k(t'),h(t')$ be the non-negative integers such that $t'=\horizon(k(t')-1)+h(t')$ and $h(t')\in[\horizon]$. For the current time $t$, we use the shorthands $k=k(t),h=h(t)$.  For a token index $i\in[(2\horizon+1)\Numepi]$, let $\bar{k}(i),\bar{h}(i)$ be the episode and time step the $i$-th token corresponds to (for the empty tokens we set $h=\horizon+1$). 
Given the input token matrix $\bH^\pre_{\roll,t}$, we construct a transformer that implements the following steps on the last token. 
$\bh^{\star}_{2(t-1)+k}=\bh^{\pre,\star}_{2(t-1)+k}$ for $\star\in\{\parta,\partb,\partc,\partd\}$ 
\begin{align}
    \begin{bmatrix}
    \bh_{2(t-1)+k}^{\pre,\parta} \\  \bh_{2(t-1)+k}^{\pre,\partb}\\  \bh_{2(t-1)+k}^{\pre,\partc}\\   \bh_{2(t-1)+k}^{\pre,\partd}
\end{bmatrix}&
\xrightarrow{\text{step 1}}
   \begin{bmatrix}
    \bh_{2(t-1)+k}^{\pre,\{\parta,\partb,\partc\}} \\
        \Numvi_{1}(\state,\action,\state') \\ \vdots \\
         \Numvi_{\horizon}(\state,\action,\state') 
         \\ 
        \Numvi_{1}(\state,\action) \\\vdots\\
        \Numvi_{\horizon}(\state,\action,\state') 
         \\  
         \Numvi_{1}(\state,\action)\reward_{1}(\state,\action) \\\vdots\\
        \Numvi_{\horizon}(\state,\action,\state') \reward_{\horizon}(\state,\action)
         \\ 
         \star\\ \bzero \\\posv_{2(t-1)+k}
\end{bmatrix}
\xrightarrow{\text{step 2}}
\begin{bmatrix}
    \bh_{2(t-1)+k}^{\pre,\{\parta,\partb,\partc\}} \\
         \esttransit_{1}(\state,\action,\state') \\ \vdots \\
         \esttransit_{\horizon}(\state,\action,\state')  \\ \star
        \\ \bzero \\\posv_{2(t-1)+k}
\end{bmatrix}
\xrightarrow{\text{step 3}}
\begin{bmatrix}
    \bh_{2(t-1)+k}^{\pre,\{\parta,\partb,\partc\}} \\
         \estQfun_{1}(\state,\action,\state') \\ \vdots \\
         \estQfun_{\horizon}(\state,\action,\state')  \\ 
           \estVfun_{1}(\state) \\ \vdots \\
         \estVfun_{\horizon}(\state)  \\ \star
        \\ \bzero \\\posv_{2(t-1)+k}
\end{bmatrix}\notag\\
&
\xrightarrow{\text{step 4}}
\begin{bmatrix}
    \bh_{2(t-1)+k}^{\pre,\{\parta,\partb\}}\\ \frac{\estQfun_h(\state_t,\action_1)}{\temp}\\\vdots\\
\frac{\estQfun_h(\state_t,\action_1)}{\temp}
    \\ \bh_{2(t-1)+k}^{\partd}
\end{bmatrix}
=:
\begin{bmatrix}
    \bh_{2(t-1)+k}^{\post,\parta} \\  \bh_{2(t-1)+k}^{\post,\partb}\\  \bh_{2(t-1)+k}^{\post,\partc}\\   \bh_{2(t-1)+k}^{\post,\partd}
\end{bmatrix},\label{eq:ucbvi_roadmap}
\end{align}
where $\Numvi_{\sh}(\state,\action,\state'),\esttransit_{\sh}(\state,\action,\state'),\estQfun_{\sh}(\state,\action,\state')\in\R^{\Numst^2\times\Numact},~\Numvi_{\sh}(\state,\action)\in\R^{\Numst\times\Numact},\estVfun_{\sh}(\state)\in\R^\Numst$ for all $\sh\in[\horizon]$, and $\star$ denote additional quantities in $\bh^\partd_{2(t-1)+k}$. Given the current state $\state_t$, the transformer $\TF_\tfpar(\cdot)$ generates the policy $$\sAlg_\tfpar(\cdot|\dset_{t-1},\state_t):=\frac{\exp(\bh^{\post,\partc}_{2(t-1)+k})}{\|\exp(\bh^{\post,\partc}_{2(t-1)+k})\|_1}\in\Delta^\Numact.$$
We claim the following results which we will prove later.

\begin{enumerate}[label=Step \arabic*,ref= \arabic*]
    \item\label{mdp_step1}There exists an attention-only transformer $\TF_\btheta(\cdot)$ with     $$L=4,~~\max_{\ell\in[L]}M^{(l)}\leq \conO(\horizon\Numst^2\Numact),~~~ \nrmp{\btheta}\leq \conO(\horizon\Numepi+\horizon\Numst^2\Numact) $$
 that implements step 1 in~\eqref{eq:roadmap_ts}.
   \item \label{mdp_step2}
 There exists a one-layer  transformer $\TF_\btheta(\cdot)$ with 
$$
L=1,~~\head\leq \conO(\horizon\Numst^2\Numact),~~~ \hidden\leq \conO(\Numepi^2\horizon\Numst^2\Numact),~~~\nrmp{\btheta}\leq \tcO(\horizon\Numst^2\Numact+\Numepi^3+\Numepi\horizon)
$$
 that implements step 2 in~\eqref{eq:roadmap_ts}. 
  \item\label{mdp_step3}  There exists a transformer  $\TF_\tfpar(\cdot)$ with 
    $$L=2\horizon,~~\max_{\ell\in[L]}M^{(l)}\leq 2\Numst\Numact,~~\max_{\ell\in[L]}\hidden^{(l)}\leq 3\Numst\Numact,~~~ \nrmp{\btheta}\leq  \conO(\horizon+\Numst\Numact) $$
 that implements step 3 (i.e., value iteration) in~\eqref{eq:roadmap_ts}.
  \item \label{mdp_step4}
 There exists an attention-only transformer $\TF_\btheta(\cdot)$ with 
$$\layer=3,~~\max_{\ell\in[\layer]}\head^\lth=\conO(\horizon\Numact),~~\nrmp{\btheta}\leq \conO(\horizon(\Numepi+\Numact)+{1}/{\temp}) $$
 that implements step 4 in~\eqref{eq:roadmap_ts}.
\end{enumerate}
From the construction of Step~\ref{mdp_step1}---\ref{mdp_step4}, we verify that one can choose the constructed transformer to have the embedding dimension $D=\conO(\horizon\Numst^2\Numact)$. Moreover, due to the boundedness of the reward function, $Q$-function  and the fact that the bonus $\bonus(\state,\action)\leq\tcO(\horizon)$, we verify that there exists some $\clipval>0$ with $\log \clipval=\tcO(1)$ such that $\|\bh_i^{\lth}\|_2\leq \clipval$ for all layer $\ell\in[\layer]$ and all token $i\in[\Numepi(2\horizon+1)]$. Therefore, similar to what we do in the proof of Theorem~\ref{thm:approx_smooth_linucb},~\ref{thm:approx_thompson_linear}, we may w.l.o.g. consider transformers without truncation (i.e.,  $\clipval=\infty$) in our construction of step 1---4 in~\eqref{eq:ucbvi_roadmap}.

% Denote the errors $\eps$ appear in step 2--4   by $\eps_2,\eps_3,\eps_4$, respectively. By the Lipschitz continuity of the log-softmax function shown in Lemma~\ref{lm:log_softmax}, we have
% \begin{align*}
% \Big|\log\frac{\sAlg_\tfpar(\cdot|\dset_{t-1},\state_t)}{\sAlg_\tfpar(\cdot|\dset_{t-1},\state_t)}\Big|\leq 2\linf{\estQfun_h(\state_t,\cdot)/\temp-\trestQfun_\sh(\state_t,\cdot)/\temp}=2\eps_4.
% \end{align*}
% Thus, choosing (w.l.o.g. assume $\eps<1$) $$\eps_4=\frac\eps2,~~~\eps_3=\frac{\eps_4\temp}{2},~~~\eps_2=\frac{\eps_3}{\horizon(\Numst\horizon+2)}=\frac{\temp\eps}{4\horizon(\Numst\horizon+2)}$$ completes the proof.

\paragraph{Proof of Step~\ref{mdp_step1}} We prove this step by constructing a transformer that implements the following two steps:
\begin{enumerate}[label= Step 1\alph*, ref= 1\alph*]
    \item\label{mdp_step1a} For each $t'< t$ with $t'=(k'-1)\horizon+h'$ for some $h'\in[\horizon]$, we add  $\state_{k',h'},(\action_{k',h'},\reward_{k',h'})$ from  $\bh^\partb_{2(t'-1)+k'}$ and $\bh^\parta_{2t'-1+k'}$  to  $\bh^\partd_{2t'+k'}$.
    \item\label{mdp_step1b} Compute $\Numvi_{\sh}(\state,\action,\state'),\Numvi_{\sh}(\state,\action)$ for $\sh\in[\horizon]$ and assign them to the current token $\bh^\partd_{2(t-1)+k}$.
\end{enumerate}
For step~\ref{mdp_step1a}, we can construct a two-layer attention-only transformer with $\bQ^{(1)}_{1,2,3},\bK^{(1)}_{1,2,3},\bV^{(1)}_{1,2,3}$ such that for all $i\leq 2(t-1)+k$
\begin{align*}
&\bQ^{(1)}_{1}\bh^{(0)}_{i}=
\begin{bmatrix}
        \bar{k}(i)+1-\oddeven_{i}\\
        \tfthres\\
        1\\
       i
    \end{bmatrix},~~ \bK^{(1)}_{1}\bh^{(0)}_{i}=\begin{bmatrix}
        -\tfthres\\\bar{k}(i)\\ i+3\\-1
\end{bmatrix},~~ \bV^{(1)}_{1}\bh^{(0)}_{2(t'-1)+k'}=
\begin{bmatrix}
\bzero\\\bzero_{\Numact+1}\\\state_{k',h'}\\\bzero\end{bmatrix},\\
&~~~
\bV^{(1)}_{1}\bh^{(0)}_{2t'-1+k'}=
\begin{bmatrix}
\bzero\\\action_{k',h'}\\
\reward_{k',h'}
\\\bzero_\Numst\\\bzero,\end{bmatrix}
\end{align*}
where we choose $\tfthres=4$ and $\bV\bh^{(0)}$ are supported on some entries in $\bh^{(0),\partd}$. Moreover, we choose $\bQ^{(1)}_{3}=\bQ^{(1)}_{2}=\bQ^{(1)}_1$, $\bV^{(1)}_2=\bV^{(1)}_3=-\bV^{(1)}_1$ and $\bK^{(1)}_2,\bK^{(1)}_3$ such that
\begin{align*}
    \bK^{(1)}_{2}\bh^{(0)}_{i}=\begin{bmatrix}
        -\tfthres\\\bar{k}(i)\\ i+2\\-1
\end{bmatrix},~~\bK^{(1)}_{3}\bh^{(0)}_{i}=\begin{bmatrix}
        -\tfthres\\\bar{k}(i)\\ i+1\\-1.
\end{bmatrix}
\end{align*}
We verify that $\lops{\bQ^{(1)}_{\star}},\lops{\bK^{(1)}_{\star}}= 4,\lops{\bV^{(1)}_{\star}}=1$ for $\star\in[3]$. 
Summing up the  heads, we obtain the following update on a subset of coordinates in $\bh^{(0),\partd}_{2t'+k'}$:
\begin{align*}
\bzero_{\Numst+\Numact+1}
    &\rightarrow 
\bzero_{\Numst+\Numact+1}+\sum_{j=1}^3\sum_{i=1}^{2t'+k'}\sigma(\<\bQ^{(1)}_{j}\bh^{(0)}_{2t'+k'},\bK^{(1)}_{j}\bh^{(0)}_{i}\>)\bV_j\bh^{(0)}_i\\
    &=\frac{1}{2t'+k'} [(\bV^{(1)}_1\bh^{(0)}_{2t'+k'-2}+2\bV^{(1)}_1\bh^{(0)}_{2t'+k'-1}+3\bV^{(1)}_1\bh^{(0)}_{2t'+k'})\\
    &\qquad-(\bV^{(1)}_1\bh^{(0)}_{2t'+k'-1}+2\bV^{(1)}_1\bh^{(0)}_{2t'+k'})-\bV^{(1)}_1\bh^{(0)}_{2t'+k'})]\\
    &=\frac{1}{2t'+k'}(\bV^{(1)}_1\bh^{(0)}_{2t'+k'-2}+\bV^{(1)}_1\bh^{(0)}_{2t'+k'-1})\\
    &=\frac{1}{2t'+k'}\begin{bmatrix}
        \action_{k',h'} \\\reward_{k',h'}\\\state_{k',h'}
    \end{bmatrix}.
\end{align*}
Note that $\<\bQ^{(1)}\bh^{(0)}_i,\bK^{(1)}\bh^{(0)}_j\>\leq0$ for $i=2t'-1+k'$ (i.e., all tokens that embed the action and reward) since $\oddeven_i=0$, it follows that no update happens on the tokens in which we embed the action and reward (i.e., the corresponding part of $\bh^\partd$ remains zero). Moreover, it should be noted that no  update happens on tokens with $h=1$. 

We then use another  attention layer to  multiply the updated vectors by a factor of $2t'+k'$, namely, to perform the map
$$\frac{1}{2t'+k'}\begin{bmatrix}
        \action_{k',h'} \\\reward_{k',h'}\\\state_{k',h'}
    \end{bmatrix}\mapsto\begin{bmatrix}
        \action_{k',h'} \\\reward_{k',h'}\\\state_{k',h'}
    \end{bmatrix},$$ where the output vector is supported on coordinates different from the input vectors. This can be achieved by choosing $\lops{\bQ_1^{(2)}}\leq (2\horizon+1)\Numepi, \lops{\bK_1^{(2)}}\leq (2\horizon+1)\Numepi, \lops{\bV_1^{(2)}}\leq1$
such that \begin{align}
    &
\bQ^{(2)}_{1}\bh^{(1)}_{i}=\begin{bmatrix}
         i^2\\-(2\horizon+1)\Numepi i^2\\  1\\ \bzero
    \end{bmatrix},~~ \bK^{(2)}_{1}\bh^{(1)}_{j}=\begin{bmatrix}
        1\\1\\ (2\horizon+1)\Numepi j^2\\\bzero
    \end{bmatrix},~~ 
\bV^{(2)}_{1}\bh^{(1)}_{2t'+k'}=\frac{1}{2t'+k'}\begin{bmatrix}
        \bzero\\ 
        \action_{k',h'} \\\reward_{k',h'}\\\state_{k',h'}
    \\ \bzero
    \end{bmatrix},\label{eq:tf_constrcut_ucbvi_multi}
\end{align}
and noting that $\<\bQ^{(2)}_1\bh_i^{(1)},\bQ^{(2)}_1\bh_j^{(1)}\>=i$ when $j=i$ and otherwise $0$.

For step~\ref{mdp_step1b},  we show that it can be implemented using a two-layer attention-only transformer. 

To compute $\Numvi_\sh(\state,\action,\state')$, in the first layer we construct $\head=10\horizon\Numst^2\Numact$ heads with the query, key, value matrices $\{\bQ^{(1)}_{\si\sj\sk\sh,s}\}_{s=1}^{10},\{\bK^{(1)}_{\si\sj\sk\sh,s}\}_{s=1}^{10},\{\bV^{(1)}_{\si\sj\sk\sh,s}\}_{s=1}^{10}$  such that for all $i\leq 2(t-1)+k$ and $\si,\sk\in[\Numst],\sj\in[\Numact],\sh\in[\horizon]$
\begin{align*}
&\bQ^{(1)}_{\si\sj\sk\sh,1}\bh^{(0)}_{i}=
\begin{bmatrix}
\tfthres(\oddeven_i-1)\\
        \tfthres\be_\si\\
         \tfthres\be_\sj\\
          \tfthres\be_\sk\\
          1\\
          1\\
         \sh
    \end{bmatrix},~~ \bK^{(1)}_{\si\sj\sk\sh,1}\bh^{(0)}_{i}=\begin{bmatrix}1\\
     \state_{\bar{k}(i),\bar{h}(i)-1}\\
        \action_{\bar{k}(i),\bar{h}(i)-1}\\
         \state_{\bar{k}(i),\bar{h}(i)}\\
        -3\tfthres\\
        1- \bar{h}(i)\\
        1
\end{bmatrix},~~ \bV^{(1)}_{\si\sj\sk\sh,1}\bh^{(0)}_{i}=-
\begin{bmatrix}
\bzero\\\be^{\Numvi_\sh}_{\si\sj\sk}
\\\bzero\end{bmatrix},
\end{align*}
where we choose $\tfthres=2\horizon$ and $\be_{\si\sj\sk}^\sh$ denotes the one-hot vector supported on the $(\si,\sj,\sk)$-entry in $\Numvi_\sh(\state,\action,\state')$.
We similarly construct 
\begin{align*}
&\bQ^{(1)}_{\si\sj\sk\sh,2}\bh^{(0)}_{i}=
\begin{bmatrix}
\tfthres(\oddeven_i-1)\\
        \tfthres\be_\si\\
         \tfthres\be_\sj\\
          \tfthres\be_\sk\\
          1\\
          1\\
         \sh
    \end{bmatrix},~~ \bK^{(1)}_{\si\sj\sk\sh,2}\bh^{(0)}_{i}=\begin{bmatrix}1\\
     \state_{\bar{k}(i),\bar{h}(i)-1}\\
        \action_{\bar{k}(i),\bar{h}(i)-1}\\
         \state_{\bar{k}(i),\bar{h}(i)}\\
        -3\tfthres\\
     - \bar{h}(i)\\
        1
\end{bmatrix},~~ \bV^{(1)}_{\si\sj\sk\sh,2}\bh^{(0)}_{i}=
\begin{bmatrix}
\bzero\\\be^{\Numvi_\sh}_{\si\sj\sk}
\\\bzero\end{bmatrix},
\\
&\bQ^{(1)}_{\si\sj\sk\sh,3}\bh^{(0)}_{i}=
\begin{bmatrix}
\tfthres(\oddeven_i-1)\\
        \tfthres\be_\si\\
         \tfthres\be_\sj\\
          \tfthres\be_\sk\\
          1\\
          1\\
         -\sh
    \end{bmatrix},~~ \bK^{(1)}_{\si\sj\sk\sh,3}\bh^{(0)}_{i}=\begin{bmatrix}1\\
     \state_{\bar{k}(i),\bar{h}(i)-1}\\
        \action_{\bar{k}(i),\bar{h}(i)-1}\\
         \state_{\bar{k}(i),\bar{h}(i)}\\
        -3\tfthres\\
      \bar{h}(i)-1\\
        1
\end{bmatrix},~~ \bV^{(1)}_{\si\sj\sk\sh,3}\bh^{(0)}_{i}=-
\begin{bmatrix}
\bzero\\\be^{\Numvi_\sh}_{\si\sj\sk}
\\\bzero\end{bmatrix},\\
&\bQ^{(1)}_{\si\sj\sk\sh,4}\bh^{(0)}_{i}=
\begin{bmatrix}
\tfthres(\oddeven_i-1)\\
        \tfthres\be_\si\\
         \tfthres\be_\sj\\
          \tfthres\be_\sk\\
          1\\
          1\\
         -\sh
    \end{bmatrix},~~ \bK^{(1)}_{\si\sj\sk\sh,4}\bh^{(0)}_{i}=\begin{bmatrix}1\\
     \state_{\bar{k}(i),\bar{h}(i)-1}\\
        \action_{\bar{k}(i),\bar{h}(i)-1}\\
         \state_{\bar{k}(i),\bar{h}(i)}\\
        -3\tfthres\\
      \bar{h}(i)-2\\
        1
\end{bmatrix},~~ \bV^{(1)}_{\si\sj\sk\sh,4}\bh^{(0)}_{i}=
\begin{bmatrix}
\bzero\\\be^{\Numvi_\sh}_{\si\sj\sk}
\\\bzero\end{bmatrix},\\
&\bQ^{(1)}_{\si\sj\sk\sh,5}\bh^{(0)}_{i}=
\begin{bmatrix}
\tfthres(\oddeven_i-1)\\
        \tfthres\be_\si\\
         \tfthres\be_\sj\\
          \tfthres\be_\sk\\
          1\\
    \end{bmatrix},~~ \bK^{(1)}_{\si\sj\sk\sh,5}\bh^{(0)}_{i}=\begin{bmatrix}1\\
     \state_{\bar{k}(i),\bar{h}(i)-1}\\
        \action_{\bar{k}(i),\bar{h}(i)-1}\\
         \state_{\bar{k}(i),\bar{h}(i)}\\
        -3\tfthres\\
\end{bmatrix},~~ \bV^{(1)}_{\si\sj\sk\sh,5}\bh^{(0)}_{i}=
\begin{bmatrix}
\bzero\\\be^{\Numvi_\sh}_{\si\sj\sk}
\\\bzero\end{bmatrix}.\\
\end{align*}
Summing up the first five heads, we verify that such attention updates the token with $\bh_i^\parta=\bzero$ and has the form
\begin{align*}
\bzero\rightarrow\bzero+\frac1i\widetilde{\Numvi}_\sh(\si,\sj,\sk)\be_{\si\sj\sk}^{\Numvi_h}
\end{align*} on  $\bh^\partd_i$, where $\widetilde{\Numvi}_\sh(\si,\sj,\sk)$ denote the number of visits to the state-action-next-state tuple $(\si,\sj,\sk)$ at time step $\sh$ before token $i$. For $\star\in[5]$, we choose $\bV^{(1)}_{\si\sj\sk\sh,\star+5}=-\bV^{(1)}_{\si\sj\sk\sh,\star+5}$ and $\bQ^{(1)}_{\si\sj\sk\sh,\star+5},\bK^{(1)}_{\si\sj\sk\sh,\star+5}$ be such that 
\begin{align*}
&\bQ^{(1)}_{\si\sj\sk\sh,\star+5}\bh^{(0)}_{i}=
\begin{bmatrix}
\bQ^{(1)}_{\si\sj\sk\sh,\star}\bh^{(0)}_{i}\\
\tfthres\\
-\bar{k}(i)
    \end{bmatrix},~~ \bK^{(1)}_{\si\sj\sk\sh,\star+5}\bh^{(0)}_{i}=\begin{bmatrix}
\bK^{(1)}_{\si\sj\sk\sh,\star}\bh^{(0)}_{i}\\ \bar{k}(i)\\\tfthres
\end{bmatrix}
\end{align*} which adds positional embedding about the current episode $\bar{k}(i)$. We verify that summing up the sixth to the tenth heads gives the update  \begin{align*}
\bzero\rightarrow\bzero+\frac{1}{i}({\Numvi}_\sh(\si,\sj,\sk)-\widetilde{\Numvi}_\sh(\si,\sj,\sk))\be_{\si\sj\sk}^{\Numvi_h}
\end{align*} on  $\bh^\partd_i$ for $i\leq 2(t-1)+k$ with $\bh_i^\parta=\bzero$. Therefore, combining all the heads together we have the update
\begin{align*}
\bzero\rightarrow\bzero+\frac{1}{i}{\Numvi}_\sh(\si,\sj,\sk)\be_{\si\sj\sk}^{\Numvi_h}\text{~~for all~}\si,\sk\in[\Numst],\sj\in[\Numact],\sh\in[\horizon]
\end{align*} on $\bh^\partd_i$ for $i\leq 2(t-1)+k$ with $\bh_i^\parta=\bzero$, in particular when $i=2(t-1)+k$. Moreover,  notice that the matrices $\{\bQ^{(1)}_{\si\sj\sk\sh,s}\}_{s=1}^{10},\{\bK^{(1)}_{\si\sj\sk\sh,s}\}_{s=1}^{10}$ can be constructed with the operator norm less than $10\tfthres=10\horizon$, and $\{\bV^{(1)}_{\si\sj\sk\sh,s}\}_{s=1}^{10}$ with the operator norm equals $1$.

Following a similar construction, we can also compute $\Numvi_\sh(\state,\action),\Numvi_\sh(\state,\action)\reward_\sh(\state,\action)$ for all $\sh,\state,\action,\state'$ on different supports of coordinates in $\bh_i^\partd$ via adding additional $\head=\conO(\horizon\Numst\Numact)$ heads to the attention-only layer.

Next, we construct the second attention layer to multiply the token vector by the index number $i$ as in the proof of Step~\ref{mdp_step1a}. The construction is similar to that in Eq.~\eqref{eq:tf_constrcut_ucbvi_multi} and we omit it here. Moreover, note that Step~\ref{mdp_step1b} can be implemented with the embedding dimension $D\leq\conO(\horizon\Numst^2\Numact)$ as we need $\conO(1)$ dimensions for each quadruple $(\si,\sj,\sk,\sh)$.  Combining Step~\ref{mdp_step1a},~\ref{mdp_step1b} concludes the proof of Step~\ref{mdp_step1}.

\paragraph{Proof of Step~\ref{mdp_step2}}
After Step~\ref{mdp_step1}, for the current token $i=2(t-1)+k$, we have $\Numvi_\sh(\state,\action,\state'),\reward_\sh(\state,\action),\Numvi_\sh(\state,\action)$, $\Numvi_\sh(\state,\action)\reward_\sh(\state,\action)$  lie in $\bh_i^\partd$ for all $\sh\in[\horizon]$. Given these vectors that store the number of visits and rewards, note that
\begin{align*}
\reward_\sh(\state,\action)&=\frac{\Numvi_\sh(\state,\action)\reward_\sh(\state,\action)}{\Numvi_\sh(\state,\action)\vee1},~~\text{ when  } \Numvi_\sh(\state,\action)\geq1,~~~\\
\bonus_\sh(\state,\action)
&=2\horizon\sqrt{\frac{\log(\Numst\Numact\totlen/\delta)}{\Numvi_\sh(\state,\action)\vee1}},
\\
\tresttransit_\sh(\state,\action,\state')
&=\frac{\Numvi_\sh(\state,\action,\state')}{\Numvi_\sh(\state,\action)\vee 1}.
\end{align*}

Therefore, we may compute $\esttransit_\sh,\estbonus_\sh$ via using a transformer layer  to implement the functions $f_1(x,y)=\frac{x}{y\vee1},f_2(y)=2\horizon\sqrt{\frac{\log(\Numst\Numact\totlen/\delta)}{y\vee 1}},f_3(x,y)=\frac{x}{y\vee1}+\horizon\bone_{y=0}$ for $x,y\in\{0\}\cup[\Numepi]$. 
We demonstrate the computation of $\esttransit_\sh(\state,\action,\state')$  (i.e., the computation of $f_1(x,y)$) here. We start with constructing an attention layer with $\head=\conO(\horizon\Numst^2\Numact)$ heads such that it implements  $x\mapsto x^2$ for $x=\Numvi_\sh(\state,\action,\state'),\Numvi_\sh(\state,\action)$. For $\Numvi_\sh(\state,\action,\state')$, this can be done by  choosing  $\lops{\bQ_{\si\sj\sk\sh}^{(1)}}\leq\Numepi,\lops{\bK_{\si\sj\sk\sh}^{(1)}}\leq\Numepi,\lops{\bV_{\si\sj\sk\sh}^{(1)}}=1$ such that
\begin{align*}
    \bQ_{\si\sj\sk\sh}^{(1)}\bh_i^{(0)}=\begin{bmatrix}
        \Numepi\\
        -i\\
        \Numvi_\sh(\be_\si,\be_\sj,\be_\sk)
    \end{bmatrix},~~\bK_{\si\sj\sk\sh}^{(1)}\bh_j^{(0)}=\begin{bmatrix}
        j\\ \Numepi\\
        \Numvi_\sh(\be_\si,\be_\sj,\be_\sk)
    \end{bmatrix},~~
    \bV_{\si\sj\sk\sh}^{(1)}\bh_j^{(0)}=\begin{bmatrix}
        \bzero\\j\\\bzero
    \end{bmatrix}, 
\end{align*} where $\be_\si,\be_\sk$ denote the $i,j$-th states and $\be_{sj}$ denotes the $k$-th action.
Similarly, we can construct $\horizon\Numst\Numact$ additional heads
to compute $\Numvi_\sh(\state,\action)^2$ for all possible $\state,\action.$

Next, we compute the exact values of $\esttransit(\state,\action,\state')$ using an MLP layer. Namely, 
we  construct $\bW^{(1)}_1=\bW^{(1)}_{12}\bW^{(1)}_{11},\bW^{(1)}_2=\bW^{(1)}_{23}\bW^{(1)}_{22}\bW^{(1)}_{21}$  such that for all $\sh,\state,\action,\state'$, on the corresponding vector component we have
\begin{align*}
   & \bW^{(1)}_{11}\bh_i^{(0)}=     \begin{bmatrix}
    1\\
       \Numvi_\sh(\state,\action,\state')^2 \\\vdots
       \\
        (\Numvi_\sh(\state,\action,\state')-\Numepi)^2\\
         \Numvi_\sh(\state,\action)^2 \\\vdots
       \\
        (\Numvi_\sh(\state,\action)-\Numepi)^2
    \end{bmatrix}= \begin{bmatrix}
    1\\
       \Numvi_\sh(\state,\action,\state')^2 \\\vdots
       \\
        \Numvi_\sh(\state,\action,\state')^2+\Numepi^2-2\Numepi\Numvi_\sh(\state,\action,\state')\\
          \Numvi_\sh(\state,\action)^2 \\\vdots
       \\
        \Numvi_\sh(\state,\action)^2+\Numepi^2-2\Numepi\Numvi_\sh(\state,\action)
    \end{bmatrix},
    \\
&\bW^{(1)}_{12}\bW^{(1)}_{11}\bh_i^{(0)}=     \begin{bmatrix}
    1-
       \Numvi_\sh(\state,\action,\state')^2 - \Numvi_\sh(\state,\action)^2 \\\vdots
       \\ 1-
       (\Numvi_\sh(\state,\action,\state')-x)^2 - (\Numvi_\sh(\state,\action)-y)^2 \\\vdots\\
       1- (\Numvi_\sh(\state,\action,\state')-\Numepi)^2-(\Numvi_\sh(\state,\action)-\Numepi)^2
    \end{bmatrix} ,
\end{align*} where $x,y\in\{0\}\cup[\Numepi]$. Moreover,  we construct $\bW^{(1)}_{21}$ so that on the entries corresponding to $\sh,\state,\action,\state'$ it implements  
\begin{align*}
\bW_{2}^{(1)}\sigma(\bW_{1}^{(1)}\bh_i^{(0)})
    =\begin{bmatrix}
        \sum_{x,y=0}^{\Numepi}\sigma(1- (\Numvi_\sh(\state,\action,\state')-x)^2 - (\Numvi_\sh(\state,\action)-y)^2)\cdot\frac{x}{y\vee1}
    \end{bmatrix}=\begin{bmatrix}
        \frac{\Numvi_\sh(\state,\action,\state')}{\Numvi_\sh(\state,\action)\vee1}.
    \end{bmatrix}
\end{align*} 

It can be verified that we can find such $\bW_1^{(1)},\bW_2^{(1)}$ with $$\lops{\bW_1^{(1)}}\leq \lops{\bW_{11}^{(1)}}\lops{\bW_{12}^{(1)}}\leq \conO(\Numepi^2)\cdot \conO(\Numepi)=\conO(\Numepi^3),$$ $\lops{\bW_2^{(1)}}\leq \conO(\Numepi)
$, and the number of hidden neurons $\hidden=\conO(\Numepi^2\horizon\Numst^2\Numact)$. Simlarly, we can compute $f_2(\cdot)$ (or $f_3(\cdot)$) exactly following the same construction but with a different $\bW_2^{(1)}$ that records all possible values of  $f_2(\cdot)$ (or $f_3(\cdot)$). Combining the upper bounds on the operator norm of the  weight matrices, we further have $\nrmp{\tfpar}\leq\tcO(\horizon\Numst^2\Numact+\Numepi^3+\Numepi\horizon)$.

% Therefore, we may construct approximations  $\tresttransit_\sh,\estbonus_\sh$ using an two-layer MLP  to approximate the functions $f_1(x,y)=\frac{x}{y\vee1},f_2(y)=2\horizon\sqrt{\frac{\log(\Numst\Numact\totlen/\delta)}{y\vee 1}}$ for $x,y\in[\Numepi]$. Moreover, we can construct an approximation $\estreward_\sh(\state,\action)$ such that $\estreward_\sh(\state,\action)\approx \reward_\sh(\state,\action)$ when $\Numvi_\sh(\state,\action)>0$ and otherwise $\estreward_\sh(\state,\action)\geq 3\horizon$. This can be achieved by constructing an two-layer MLP approximating $f_3(x,y)=\frac{x}{y\vee 1}+4\horizon\sigma(1-y)$, which can be approximated as accurate as $f_1$ by adding an additional neuron. Using Proposition A.1 in~\cite{bai2023transformers}, it can be verified that for all $\eps_0>0$ a smooth extension of $f_1(x,y)$ is $(\eps_0,\Numepi,\tilde c\Numepi^6/\eps_0^2,c\Numepi^3)$-approximable by sum of relus, where $c>0$ is some numerical constant and $\tilde c$ hides logarithmic dependence on $1/\eps_0$ and $\Numepi$. Similarly, we verify that $f_2,f_3$ are $(\eps_0,\Numepi,\tilde c\Numepi^4/\eps_0^2,c\Numepi^2)$-approximable by sum of relus for all $\eps_0>0$. Therefore, choosing $\eps_0=\eps$, we can construct a two-layer MLP with $\lops{\bW_1}\leq \tilde c\Numepi^3/\eps,\lops{\bW_2}\leq c\Numepi^3$ and $\hidden=\tilde O(\Numepi^6\horizon\Numst^2\Numact/\eps^2)$  that approximately computes $\transit_\sh(\state,\action,\state'),\bonus_\sh(\state,\action),\reward_\sh(\state,\action)$ such that satisfies the condition in step 2.

\paragraph{Proof of Step~\ref{mdp_step3}} Given $\trestVfun_{\horizon+1}=\estVfun_{\horizon+1}=\bzero$, we show the there exists an transformer with
\begin{align*}
\layer=2,~~~\max_{\ell\in[\layer]}\head^\lth\leq2\Numst\Numact,~~~\max_{\ell\in[\layer]}\hidden^{\lth}\leq3\Numst\Numact,~~~\nrmp{\tfpar}\leq \conO(\horizon+\Numst\Numact)
\end{align*}
that implements one step of value iteration
\begin{align*}
    \estQfun_\sh(\state,\action)&=\max\{\min\{\horizon,\estreward_\sh(\state,\action)+\estbonus_\sh(\state,\action)+\sum_{\state'\in\statesp}\esttransit_\sh(\state'\mid\state,\action)\estVfun_{\sh+1}(\state')\},0\},\\
\estVfun_\sh(\state)&=\max_{\action\in\actionsp}\estQfun_\sh(\state,\action)
\end{align*} for some $\sh\in[\horizon]$. 
Namely, we start with constructing an-attention layer with $\head=2\Numst\Numact$ and $\{\bQ_{\si\sj\sh,s}^{(1)}\}^2_{s=1}$, $\{\bK_{\si\sj\sh,s}^{(1)}\}^2_{s=1}$, $\{\bV_{\si\sj\sh,s}^{(1)}\}^2_{s=1}$ such that for all $i\leq 2(t-1)+k$
\begin{align*}
   \bQ_{\si\sj,1}^{(1)}\bh_i^{(0)} =\begin{bmatrix}
       \tfthres\\-i\\
       \estVfun_{\sh+1}(\cdot)\\
   \end{bmatrix},~~ \bK_{\si\sj,1}^{(1)}\bh_i^{(0)} =\begin{bmatrix}
       i\\ \tfthres\\
       \esttransit_{\sh+1}(\cdot|\state,\action)\\
   \end{bmatrix},~~~
   \bV_{\si\sj,1}^{(1)}\bh_i^{(0)} =\begin{bmatrix}
       \bzero \\ i\be^{\Qfun_\sh}_{\si\sj}\\
       \bzero,
   \end{bmatrix}
   \\
    \bQ_{\si\sj,2}^{(1)}\bh_i^{(0)} =\begin{bmatrix}
       \tfthres\\-i\\
       -\estVfun_{\sh+1}(\cdot)\\
   \end{bmatrix},~~ \bK_{\si\sj,2}^{(1)}=\bK_{\si\sj,1}^{(1)},~~~
\bV_{\si\sj,2}^{(1)}=-\bV_{\si\sj,2}^{(1)}
\end{align*}
where $\tfthres=3\horizon$ and $\be^{\Qfun_\sh}_{\si\sj}\in\R^{\Numst\Numact}$ is a vector supported on some coordinates in $\bh^\partd_i$ reserved for $\estQfun_\sh$. Moreover, we have $\lops{\bQ^{(1)}_{\si\sj\sh,s}},\lops{\bK^{(1)}_{\si\sj\sh,s}}\leq \tfthres,~\lops{\bV^{(1)}_{\si\sj\sh,s}}=1$. 
Since $$\Big|\<\estVfun_{\sh+1}(\cdot),\esttransit_{\sh+1}(\cdot|\state,\action)\>\Big|\leq \linf{\estVfun_{\sh+1}(\cdot)}\cdot\lone{\esttransit_{\sh+1}(\cdot|\state,\action)}\leq\horizon$$ as $\estVfun_{\sh+1}(\state)\in[0,\horizon]$ and $\lone{\esttransit_{\sh+1}(\cdot|\state,\action)}=1$, it follows that summing up two heads gives the update for $i\leq 2(t-1)+k$
\begin{align*}
\bzero
&\mapsto\bzero+\Big[\sigma(\< \bQ_{\si\sj,1}^{(1)}\bh_i^{(0)} , \bK_{\si\sj,1}^{(1)}\bh_j^{(0)} \>)-\sigma(\< \bQ_{\si\sj,1}^{(1)}\bh_i^{(0)} , \bK_{\si\sj,1}^{(1)}\bh_j^{(0)} \>)\Big]\be_{\si\sj}^{\Qfun_\sh}\\
&=\< \bQ_{\si\sj,1}^{(1)}\bh_i^{(0)} , \bK_{\si\sj,1}^{(1)}\bh_j^{(0)} \>\be_{\si\sj}^{\Qfun_\sh}.
\end{align*}
Denote the resulting token vector by $\bh_i^{(1)}$. 
Moreover, we can construct a two-layer MLP with $$\lops{\bW^{(1)}_1}=\conO(\horizon),~~ \lops{\bW^{(1)}_2}\leq 3,~~\hidden=3\Numst\Numact$$
such that for any state-action pair $(\state,\action)\in\statesp\times\actionsp$ on the corresponding coordinates
\begin{align*}
\bW^{(1)}_1\bh_i^{(1)}=
\begin{bmatrix}
\vdots\\-[\estreward_\sh(\state,\action)+\estbonus_\sh(\state,\action)+\sum_{\state'\in\statesp}\esttransit_\sh(\state'\mid\state,\action)\estVfun_{\sh+1}(\state')]
\\
\estreward_\sh(\state,\action)+\estbonus_\sh(\state,\action)+\sum_{\state'\in\statesp}\esttransit_\sh(\state'\mid\state,\action)\estVfun_{\sh+1}(\state')-\horizon\\\estreward_\sh(\state,\action)+\estbonus_\sh(\state,\action)+\sum_{\state'\in\statesp}\esttransit_\sh(\state'\mid\state,\action)\estVfun_{\sh+1}(\state')\\\vdots
\end{bmatrix}
\end{align*}
 and \begin{align*}
  \bW^{(1)}_2\sigma(\bW^{(1)}_1\bh_i^{(1)})
  &=   \sigma(-[\estreward_\sh(\state,\action)+\estbonus_\sh(\state,\action)+\sum_{\state'\in\statesp}\esttransit_\sh(\state'\mid\state,\action)\estVfun_{\sh+1}(\state')])\\&\quad-\sigma(\estreward_\sh(\state,\action)+\estbonus_\sh(\state,\action)+\sum_{\state'\in\statesp}\esttransit_\sh(\state'\mid\state,\action)\estVfun_{\sh+1}(\state')-\horizon)\\
&\qquad+\sigma(\estreward_\sh(\state,\action)+\estbonus_\sh(\state,\action)+\sum_{\state'\in\statesp}\esttransit_\sh(\state'\mid\state,\action)\estVfun_{\sh+1}(\state'))\\
&=\max\{\min\{\horizon,\estreward_\sh(\state,\action)+\estbonus_\sh(\state,\action)+\sum_{\state'\in\statesp}\esttransit_\sh(\state'\mid\state,\action)\estVfun_{\sh+1}(\state')\},0\}=\estQfun_h(\state,\action).
 \end{align*}
 Denote the resulting token vector by $\bh_i^{(2)}$.
 Next, we construct a second MLP layer with $$\lops{\bW_1^{(2)}}\leq2,~~\lops{\bW_2^{(2)}}\leq\sqrt{\Numact},~~\hidden=\Numact\Numst
 $$ such that for any $\state\in\statesp$ on the corresponding coordinates we have
 \begin{align*}
     \bW_1^{(2)}\bh_i^{(2)}=\begin{bmatrix}
        \vdots\\ \estQfun_\sh(\state,\action_1)\\
         \estQfun_\sh(\state,\action_2)-\estQfun_\sh(\state,\action_1)\\
         \vdots\\
          \estQfun_\sh(\state,\action_\Numact)-\estQfun_\sh(\state,\action_{\Numact-1})
    \\ \vdots\end{bmatrix},
 \end{align*}
 where $\action_j$ denotes the $j-$th action, 
 and 
 \begin{align*}
  \bW^{(2)}_2\sigma(\bW^{(2)}_1\bh_i^{(2)})
 &=
 \sigma(\estQfun_\sh(\state,\action_1))+\sum_{j=2}^\Numact\sigma( \estQfun_\sh(\state,\action_j)-\estQfun_\sh(\state,\action_{j-1}))\\
 &=\max_{\action\in\actionsp}\estQfun_\sh(\state,\action)=\estVfun_\sh(\state).
 \end{align*}
 Using the upper bounds on the operator norm of the  weight matrices, we further have $\nrmp{\tfpar}\leq\conO(\Numst\Numact+\horizon)$. 
Combining the steps concludes the construction in Step~\ref{mdp_step3}.

 \paragraph{Proof of Step~\ref{mdp_step4}}
 we start with constructing an-attention layer with $\head=2\horizon\Numact$ and $\{\bQ_{\sj\sh,s}^{(1)}\}^2_{s=1}$, $\{\bK_{\sj\sh,s}^{(1)}\}^2_{s=1}$, $\{\bV_{\sj\sh,s}^{(1)}\}^2_{s=1}$ such that for all the current token $i= 2(t-1)+k$ and $j\leq i$
\begin{align*}
  & \bQ_{\sj\sh,1}^{(1)}\bh_i^{(0)} =\begin{bmatrix}
       \state_{\bar{k}(i),\bar{h}(i)}\\-i\\
       \tfthres
   \end{bmatrix},~~ \bK_{\sj\sh,1}^{(1)}\bh_j^{(0)} =\begin{bmatrix}
    \estQfun_\sh(\cdot,\action_\sj)\\  \tfthres\\ j
   \end{bmatrix},~~~
   \bV_{\sj\sh,1}^{(1)}\bh_i^{(0)} =\begin{bmatrix}
       \bzero \\ i\be_{\sj\sh}\\
       \bzero,
   \end{bmatrix}
   \\
    &\bQ_{\sj\sh,2}^{(1)}\bh_i^{(0)} =\begin{bmatrix}
       -\state_{\bar{k}(i),\bar{h}(i)}\\-i\\
       \tfthres
   \end{bmatrix},~~ \bK_{\sj\sh,2}^{(1)}=\bK_{\sj\sh,1}^{(1)},~~~
\bV_{\sj\sh,2}^{(1)}=-\bV_{\sj\sh,1}^{(1)},
\end{align*}
where we choose $\tfthres=2\horizon$ and $\bV_{\sj\sh,1}^{(1)}\bh_i^{(0)}$ is a one-hot vector supported on some entry of $\bh_i^\partd$. We verify  that summing up the heads gives the update 
\begin{align*}
    \bzero\mapsto \estQfun_{\sh}(\state_{k,h},\action_\sj)\be_{\sj\sh}
\end{align*}
 for all $\sh\in[\horizon],\sj\in[\Numact]$. Moreover, we have $\lops{\bQ_{\sj\sh,s}^{(1)}}\leq2\horizon,\lops{\bK_{\sj\sh,s}^{(1)}}\leq2\horizon,\lops{\bV_{\sj\sh,s}^{(1)}}\leq1$ for $s=1,2$. Through this attention-only layer, we extract the values $\estQfun_\sh(\state_{k,h},\action_j)$ for all $\sh\in[\horizon]$ from the Q-function.
 
Similar to the proof of Step~\ref{mdp_step1b}, we construct a second attention-only layer with attention heads   $\{\bQ_{\sj\sh,s}^{(2)}\}^2_{s=1}$, $\{\bK_{\sj\sh,s}^{(2)}\}^2_{s=1}$, $\{\bV_{\sj\sh,s}^{(2)}\}^2_{s=1}$ that
\begin{align*}
  & \bQ_{\sj\sh,1}^{(2)}\bh_i^{(1)} =\begin{bmatrix}
       1\\-\bar{h}(i)\\-i\\
       \tfthres
   \end{bmatrix},~~ \bK_{\sj\sh,1}^{(2)}\bh_j^{(1)} =\begin{bmatrix}
   \sh\\ 1 \\\tfthres\\ j
   \end{bmatrix},~~~
   \bV_{\sj\sh,1}^{(2)}\bh_i^{(1)} =-\begin{bmatrix}
       \bzero \\  \estQfun_\sh(\state_{k,h},\action_\sj)\be_{\sj}\\
       \bzero
   \end{bmatrix},
   \\
   & \bQ_{\sj\sh,2}^{(2)}=\bQ_{\sj\sh,1}^{(2)},~~ \bK_{\sj\sh,2}^{(2)}\bh_j^{(1)} =\begin{bmatrix}
   \sh-1\\ 1 \\\tfthres\\ j
   \end{bmatrix},~~~
   \bV_{\sj\sh,2}^{(2)} =-  \bV_{\sj\sh,2}^{(2)},
\end{align*} where $ \bV_{\sj\sh,1}^{(2)}\bh_i^{(1)}$ are supported on some entry of $\bh_i^\partd$ for $s=1,2$. 
Summing up the heads gives the update
\begin{align*}
    \bzero\mapsto -\sum_{s=\bar{h}(i)+1}^{\horizon}\frac{1}{i}\estQfun_{s}(\state_{k,h},\action_\sj).
\end{align*}
Similarly, we can construct attention heads   $\{\bQ_{\sj\sh,s}^{(2)}\}^4_{s=3},\{\bK_{\sj\sh,s}^{(2)}\}^4_{s=3},\{\bV_{\sj\sh,s}^{(2)}\}^4_{s=3}$ that implements
\begin{align*}
    \bzero\mapsto -\frac{1}{i}\sum_{s=1}^{\bar{h}(i)-1}{\estQfun_{s}(\state_{k,h},\action_\sj)}.
\end{align*}
Moreover, we construct 
$\bQ_{\sj\sh,5}^{(2)},\bK_{\sj\sh,5}^{(2)},\bV_{\sj\sh,5}^{(2)}$ with
\begin{align*}
  & \bQ_{\sj\sh,5}^{(2)}\bh_i^{(1)} =\begin{bmatrix}
       1\\-i\\
       \tfthres
   \end{bmatrix},~~ \bK_{\sj\sh,5}^{(2)}\bh_j^{(1)} =\begin{bmatrix}
    1 \\\tfthres\\ j
   \end{bmatrix},~~~
   \bV_{\sj\sh,1}^{(2)}\bh_i^{(1)} =\begin{bmatrix}
       \bzero \\  \estQfun_\sh(\state_{k,h},\action_\sj)\be_{\sj}\\
       \bzero,
   \end{bmatrix}
\end{align*} that implements
\begin{align*}
    \bzero\mapsto \frac{1}{i}\sum_{s=1}^{\horizon}{\estQfun_{s}(\state_{k,h},\action_\sj)}.
\end{align*}
Therefore, summing up the $\head=5\horizon\Numact$ heads we obtain the update
\begin{align*}
  \bzero_{\Numact}\mapsto \frac{1}{i}\estQfun_{h}(\state_{k,h},\cdot).
\end{align*}
Note that $\lops{\bQ_{\sj\sh,s}^{(1)}}\leq4\horizon,\lops{\bK_{\sj\sh,s}^{(1)}}\leq4\horizon,\lops{\bV_{\sj\sh,s}^{(1)}}\leq1$ for $s\in[5]$.

Finally, we apply an attention-only layer to implement the multiplication by a factor of $i/\temp$ using a similar construction as in Eq.~\eqref{eq:tf_constrcut_ucbvi_multi} with $\lops{\bQ_1^{(3)}}=\conO(\horizon\Numepi),\lops{\bK_1^{(3)}}=\conO(\horizon\Numepi),\lops{\bV_1^{(3)}}=\conO(1/\temp)$, and assign the resulting vector $\estQfun(\state_{k,h},\cdot)/\temp$ to $\bh_i^\partc$. 
Combining the three attention-only layers completes Step~\ref{mdp_step4}.

\subsection{Proof of Theorem~\ref{thm:ucbvi_icrl-main}}\label{sec:pf_thm:ucbvi_icrl-main}
By Theorem~\ref{thm:diff_reward}~and~\ref{thm:approx_ucbvi}, it suffices to show the regret of soft UCB-VI satisfies
\begin{align*}
\E[\Numepi\Vfun_\inst(\plc^*)-\totreward_{\inst,\sAlg_\sUCBVI(\temp)}(\totlen)]\leq \tcO (\horizon^2\sqrt{\Numst\Numact\Numepi}+\horizon^3\Numst^2\Numact)
\end{align*} for all MDP instances $\inst$, where $\temp=1/\Numepi$ and $\tcO(\cdot)$ hides logarithmic dependencies on $(\horizon,\Numepi,\Numst,\Numact)$.

Throughout the proof, we may drop the dependence on $\inst$ for notational simplicity when there is no confusion. For each episode $k\in[\Numepi]$, let $\Numvi^k_\sh,\tresttransit^k_\sh,\trestQfun^k_\sh,\trestVfun^k_\sh,\bonus^k_\sh$ denote the corresponding quantities $\Numvi_\sh,\tresttransit_\sh,\trestQfun_\sh,\trestVfun_\sh,\bonus_\sh$  introduced in UCB-VI (see Section~\ref{sec:example_ucbvi}). 

For a policy $\plc$ and time step $\sh\in[\horizon]$, we define the Q-function $\Qfun_\sh^\plc$ and the value function $\Vfun_\sh^\plc$
\begin{align*}
\Qfun_\sh^\plc(\state,\action)&:=\E[\sum_{t=\sh}^\horizon \reward(\state_t,\action_t)\mid \state_\sh=\state,\action_\sh=\action,\plc],\\  
\Vfun_\sh^\plc(\state)&:=\E[\sum_{t=\sh}^\horizon \reward(\state_t,\action)\mid \state_\sh=\state,\plc].
\end{align*}
 We use $\plc^{k}=(\plc^k_1,\ldots,\plc^k_\horizon),\plc^k_{\s}=(\plc^k_{\s,1},\ldots,\plc^k_{\s,\sh},\ldots\plc^k_{\s,\horizon})$ to denote the policies given by UCB-VI and soft UCB-VI in the $k$-th episode, respectively. Note that we have $\Vfun_\inst(\plc)=\E_{\state\sim\init}[\Vfun_\sh^\plc(\state)]$ and cumulative the regret 
 $$
 \E[\Numepi\Vfun_\inst(\plc^*)-\totreward_{\inst,\sAlg_\sUCBVI(\temp)}(\totlen)]]=\E\Big[\sum_{k=1}^{\Numepi}[\Vfun_1^{\optplc}(\state_{k,1})-\Vfun_1^{\plc^k_{\s}}(\state_{k,1})
 ]\Big]
 $$ 
 where the expectation is taken over the collected data $$\dset_\totlen=\{(\state_{k,\sh},\action_{k,\sh},\reward_{k,\sh})\}_{k\in[\Numepi],\sh\in[\horizon]}\sim\P_\inst^{\sUCBVI(\temp)}.$$ 
 For any function $f=f(\state,\action)$, we abuse the notation $f(\state,\plc(
\cdot)):=\E_{\action\sim\plc}[f(\state,\action)]$. Lastly, we define
\begin{align*}
\epstemp=\max_{k\in[\Numepi],\sh\in[\horizon],\state\in\statesp}[\trestQfun^k_\sh(\state,\plc^k_{\sh}(\cdot))-\trestQfun^k_\sh(\state,\plc^k_{\s,\sh}(\cdot)) ].
\end{align*}
We claim the following which we will prove later 
\begin{align}\label{eq:epstemp_control}
  \epstemp\leq \Numact\temp.  
\end{align}

The proof follows from similar arguments as in the proof of Theorem 1 in~\cite{azar2017minimax} (see also Theorem 7.6 in~\cite{agarwal2019reinforcement}). Hence we only provide a sketch of proof here.
First, from the proof of Theorem 7.6 in~\cite{agarwal2019reinforcement} , it can be shown that 
\begin{align*}
 \trestVfun_\sh^k(\state)\geq\Vfun^{\optplc}_\sh(\state)
\end{align*}
for any $k,\sh,\state$ with probability at least $1-\delta$.
Thus with probability at least $1-\delta$ for all $\sh\in[\horizon],k\in[\Numepi]$
\begin{align*}
&\qquad\Vfun_\sh^{\optplc}(\state_{k,\sh})- \Vfun^{\plc^k_\s}_\sh(\state_{k,\sh})  \\
&\leq
\trestVfun^k_\sh(\state_{k,\sh})- \Vfun^{\plc^k_\s}_\sh(\state_{k,\sh})  \\
&=\trestQfun^k_\sh(\state_{k,\sh},\plc^k_{\sh}(\cdot))-\trestQfun^k_\sh(\state_{k,\sh},\plc^k_{\s,\sh}(\cdot))+\trestQfun^k_\sh(\state_{k,\sh},\plc^k_{\s,\sh}(\cdot))-
\Qfun^{\plc^k_\s}_\sh(\state_{k,\sh},\plc^k_{\s,\sh}(\cdot))\\
&\leq 
\trestQfun^k_\sh(\state_{k,\sh},\plc^k_{\s,\sh}(\cdot))-
\Qfun^{\plc^k_\s}_\sh(\state_{k,\sh},\plc^k_{\s,\sh}(\cdot))+\epstemp\\
&=
\trestQfun^k_\sh(\state_{k,\sh},\action_{k,\sh})-
\Qfun^{\plc^k_\s}_\sh(\state_{k,\sh},\action_{k,\sh})+\MD^{(1)}_{k,\sh}+\epstemp,
\end{align*}
where the first equality uses $\trestVfun^k_\sh(\state_{k,\sh})=\argmax_{\action}\trestQfun^k_\sh(\state_{k,\sh},\action_{k,\sh})=\trestQfun^k_\sh(\state_{k,\sh},\plc_\sh^k(\cdot))$, and in the last line $$
\MD^{(1)}_{k,\sh}:=[\trestQfun^k_\sh(\state_{k,\sh},\plc^k_{\s,\sh}(\cdot))-
\Qfun^{\plc^k_\s}_\sh(\state_{k,\sh},\plc^k_{\s,\sh}(\cdot))]-[\trestQfun^k_\sh(\state_{k,\sh},\action_{k,\sh})-
\Qfun^{\plc^k_\s}_\sh(\state_{k,\sh},\action_{k,\sh})].
$$  Note that for any fixed $\sh\in[\horizon]$, $\{\MD^{(1)}_{k,\sh}\}_{k=1}^{\Numepi}$ is a bounded martingale difference sequence. Following  the proof of Theorem 7.6 in~\cite{agarwal2019reinforcement}, we further have 
\begin{align*}
&\quad\Vfun_\sh^{\optplc}(\state_{k,\sh})- \Vfun^{\plc^k_\s}_\sh(\state_{k,\sh})\\  
&\leq \trestQfun^k_\sh(\state_{k,\sh},\action_{k,\sh})-
\Qfun^{\plc^k_\s}_\sh(\state_{k,\sh},\action_{k,\sh})+
\MD^{(1)}_{k,\sh}+\epstemp\\
&\leq
\Big(1+\frac{1}\horizon\Big)\Big[\trestVfun_{\sh+1}^k(\state_{k,\sh+1})- \Vfun^{\plc^k_\s}_{\sh+1}(\state_{k,\sh+1})  \Big]+2\bonus_\sh^k(\state_{k,\sh},\action_{k,\sh})
\\
&\qquad+
\frac{c_0L_0\horizon^2\Numst}{\Numvi^k_\sh(\state_{k,\sh},\action_{k,\sh})}+\MD^{(2)}_{k,\sh}+\MD^{(1)}_{k,\sh}+\epstemp,
\end{align*}with probability at least $1-c\delta$ for some universal constant $c>0$, 
where $L_0=\log(\Numst\Numact\Numepi\horizon/\delta)$,  $c_0>0$ is some universal constant and 
\begin{align*}
 \MD^{(2)}_{k,\sh}:=\transit_\sh(\cdot\mid\state_{k,\sh},\action_{k,\sh})\cdot(\Vfun_{\sh+1}^\optplc-\Vfun_{\sh+1}^{\plc^k_\s})-  (\Vfun_{\sh+1}^\optplc(\state_{k,\sh+  1})-\Vfun_{\sh+1}^{\plc^k_\s}(\state_{k,\sh+  1})) 
\end{align*} is a bounded martingale difference sequence for any fixed $\sh\in[\horizon]$. Using the recursive formula and the fact that $(1+1/\horizon)^\horizon<e$, we obtain
\begin{align*}
 &\quad \E\Big[ \sum_{k=1}^\Numepi [\Vfun_1^{\optplc}(\state_{k,1})-\Vfun_1^{\plc^k_{\s}}(\state_{k,1})
 ]  \Big]\\
 &\leq c\E\Bigg[\sum_{k=1}^\Numepi\sum_{\sh=1}^\horizon\Big[2\bonus_\sh^k(\state_{k,\sh},\action_{k,\sh})
+
\frac{c_0L_0\horizon^2\Numst}{\Numvi^k_\sh(\state_{k,\sh},\action_{k,\sh})}+\MD^{(2)}_{k,\sh}+\MD^{(1)}_{k,\sh}\Big]\Bigg]+\E[\Numepi\sum_{\sh=0}^{\horizon-1}(1+\frac{1}\horizon)^\sh\epstemp]\\
&\leq c\E\Bigg[\sum_{k=1}^\Numepi\sum_{\sh=1}^\horizon\Big[2\bonus_\sh^k(\state_{k,\sh},\action_{k,\sh})
+
\frac{c_0L_0\horizon^2\Numst}{\Numvi^k_\sh(\state_{k,\sh},\action_{k,\sh})}+\MD^{(2)}_{k,\sh}+\MD^{(1)}_{k,\sh}\Big]\Bigg]+c\Numepi\horizon\Numact\temp\\
&\leq 
\tcO (\horizon^2\sqrt{\Numst\Numact\Numepi}+\horizon^3\Numst^2\Numact)+c\Numepi\horizon\Numact\temp\\
&\leq \tcO(\horizon^2\sqrt{\Numst\Numact\Numepi}+\horizon^3\Numst^2\Numact),
\end{align*} where $c>0$ is some universal constant,  $\tcO(\cdot)$ hides logarithmic dependencies on $(\horizon,\Numepi,\Numst,\Numact)$, and the last line follows again from the proof of Theorem~7.6 in~\cite{agarwal2019reinforcement}, and the assumption that $\temp=1/\Numepi$.
 We omit the detailed derivations here as they are similar to those in~\cite{azar2017minimax,agarwal2019reinforcement}. Therefore, we conclude the proof of the first part of Theorem~\ref{thm:ucbvi_icrl-main}. Moreover,  the second part of Theorem~\ref{thm:ucbvi_icrl-main} (i.e., the upper bound on $\log\cN_{\tfparspace}$) follows immediately from Lemma~\ref{lm:cover_num_bound} and Eq.~\eqref{eq:ucbvi_tf_param-main}.

 \paragraph{Proof of Eq.~\eqref{eq:epstemp_control}}
 By definition of $\trestQfun_\sh^k$ and $\plc_\sh^k,\plc^k_{\s,\sh}$, we have
 \begin{align*}
     \trestQfun^k_\sh(\state,\plc^k_{\sh}(\cdot))-\trestQfun^k_\sh(\state,\plc^k_{\s,\sh}(\cdot)) &=\max_{\action} \trestQfun^k_\sh(\state,\action)-
\sum_{\action}\frac{\exp(\trestQfun^k_\sh(\state,\action)/\temp)}{\sum_\action \exp(\trestQfun^k_\sh(\state,\action)/\temp)}\cdot\trestQfun^k_\sh(\state,\action)\\
&=
\sum_{\action}\frac{\exp(\trestQfun^k_\sh(\state,\action)/\temp)}{\sum_\action \exp(\trestQfun^k_\sh(\state,\action)/\temp)}\cdot[\max_{\action} \trestQfun^k_\sh(\state,\action)-\trestQfun^k_\sh(\state,\action)]\\
&\leq
\sum_{\action}\frac{\exp(\trestQfun^k_\sh(\state,\action)/\temp)}{ \exp(\max_\action\trestQfun^k_\sh(\state,\action)/\temp)}\cdot[\max_{\action} \trestQfun^k_\sh(\state,\action)-\trestQfun^k_\sh(\state,\action)]\\
&\leq \Numact \cdot [\sup_{t\geq 0}t\exp(-t/\temp)]\leq\Numact\temp.
 \end{align*}